\documentclass{article} 
\usepackage[margin=1.2in]{geometry}

\usepackage{hyperref}
\usepackage{url}

\usepackage{xcolor}
\usepackage{amsmath}
\usepackage{amsthm}
\usepackage[ruled, noend, noline]{algorithm2e}
\usepackage{mathtools}
\usepackage{amsfonts}
\usepackage{graphicx,epstopdf}
\usepackage{tensor}

\DeclareMathOperator{\tr}{tr}
\DeclareMathOperator{\rank}{rank}
\DeclareMathOperator{\conv}{\textbf{conv}}

\DeclareMathOperator{\find}{find}
\DeclareMathOperator{\vect}{vec}
\DeclareMathOperator{\ones}{\bar{1}}

\newtheorem{theorem}{Theorem}[section]
\newtheorem{remark}{Remark}[section]
\newtheorem{definition}{Definition}
\newtheorem{lemma}[theorem]{Lemma}
\newtheorem{corollary}[theorem]{Corollary}

\author{%
   \textbf{Burak Bartan} \\
  Department of Electrical Engineering\\
  Stanford University\\
  \texttt{bbartan@stanford.edu} \\
   \and
  \textbf{Mert Pilanci} \\
  Department of Electrical Engineering\\
  Stanford University\\
  \texttt{pilanci@stanford.edu} \\
  }

\title{Neural Spectrahedra and Semidefinite Lifts: Global Convex Optimization of Polynomial Activation Neural Networks in Fully Polynomial-Time}

\begin{document}
\maketitle

\begin{abstract}
The training of two-layer neural networks with nonlinear activation functions is an important non-convex optimization problem with numerous applications and promising performance in layerwise deep learning. In this paper, we develop exact convex optimization formulations for two-layer neural networks with second degree polynomial activations based on semidefinite programming. Remarkably, we show that semidefinite lifting is always exact and therefore computational complexity for global optimization is polynomial in the input dimension and sample size for all input data. The developed convex formulations are proven to achieve the same global optimal solution set as their non-convex counterparts. More specifically, the globally optimal two-layer neural network with polynomial activations can be found by solving a semidefinite program (SDP) and decomposing the solution using a procedure we call Neural Decomposition. Moreover, the choice of regularizers plays a crucial role in the computational tractability of neural network training. We show that the standard weight decay regularization formulation is NP-hard, whereas other simple convex penalties render the problem tractable in polynomial time via convex programming. We extend the results beyond the fully connected architecture to different neural network architectures including networks with vector outputs and convolutional architectures with pooling. We provide extensive numerical simulations showing that the standard backpropagation approach often fails to achieve the global optimum of the training loss. The proposed approach is significantly faster to obtain better test accuracy compared to the standard backpropagation procedure.
\end{abstract}
\section{Introduction} \label{sec:introduction}

We study neural networks from the optimization perspective by deriving equivalent convex optimization formulations with identical global optimal solution sets. The derived convex problems have important theoretical and practical implications concerning the computational complexity of optimal training of neural network models. Moreover, the convex optimization perspective provides a more concise parameterization of neural network models that enables further analysis of their interesting properties. 

In non-convex optimization, the choice of optimization method and its internal hyperparameters, such as initialization, mini-batching and step sizes, have a considerable effect on the quality of the learned model.
This is in sharp contrast to convex optimization problems, where locally optimal solutions are globally optimal and optimizer parameters have no influence on the solution and therefore the model. Moreover, the solutions of convex optimization problems can be obtained in a very robust, efficient and reproducible manner thanks to the elegant and extensively studied structure of convex programs. Therefore, our convex optimization based globally optimal training procedure enables the study of the neural network model and the optimization procedure in a \textit{decoupled} way. For instance, step sizes employed in the optimization can be considered hyperparameters of non-convex models, which affect the model quality and may require extensive tuning. For a classification task, in our convex optimization formulation, step sizes as well as the choice of the optimizers are no longer hyperparameters to obtain better classification accuracy. Any convex optimization solver can be applied to the convex problem to obtain a globally optimal model.

Various types of activation functions were proposed in the literature as nonlinearities in neural network layers. Among the most widely adopted ones is the ReLU (rectified linear unit) activation given by $\sigma(u)=\max(0,u)$. A recently proposed alternative is the swish activation $\sigma(u)=u (1+e^{-u})^{-1}$, which performs comparably well \cite{Ramachandran18swish}. Another important class is the polynomial activation where the activation function is a scalar polynomial of a fixed degree. We focus on second degree polynomial activation functions, i.e., $\sigma(u) = au^2+bu+c$. Although polynomial coefficients $a,b,c$ can be regarded as hyperparameters, it is often sufficient to choose the coefficients in order to approximate a target nonlinear activation function such as the ReLU or swish activation. ReLU and swish activations are plotted in Figure \ref{fig:plot_activation_functions} along with their second degree polynomial approximations. 

\begin{figure} 
\begin{minipage}[b]{0.48\linewidth}
  \centering
  \centerline{\includegraphics[width=0.8\columnwidth]{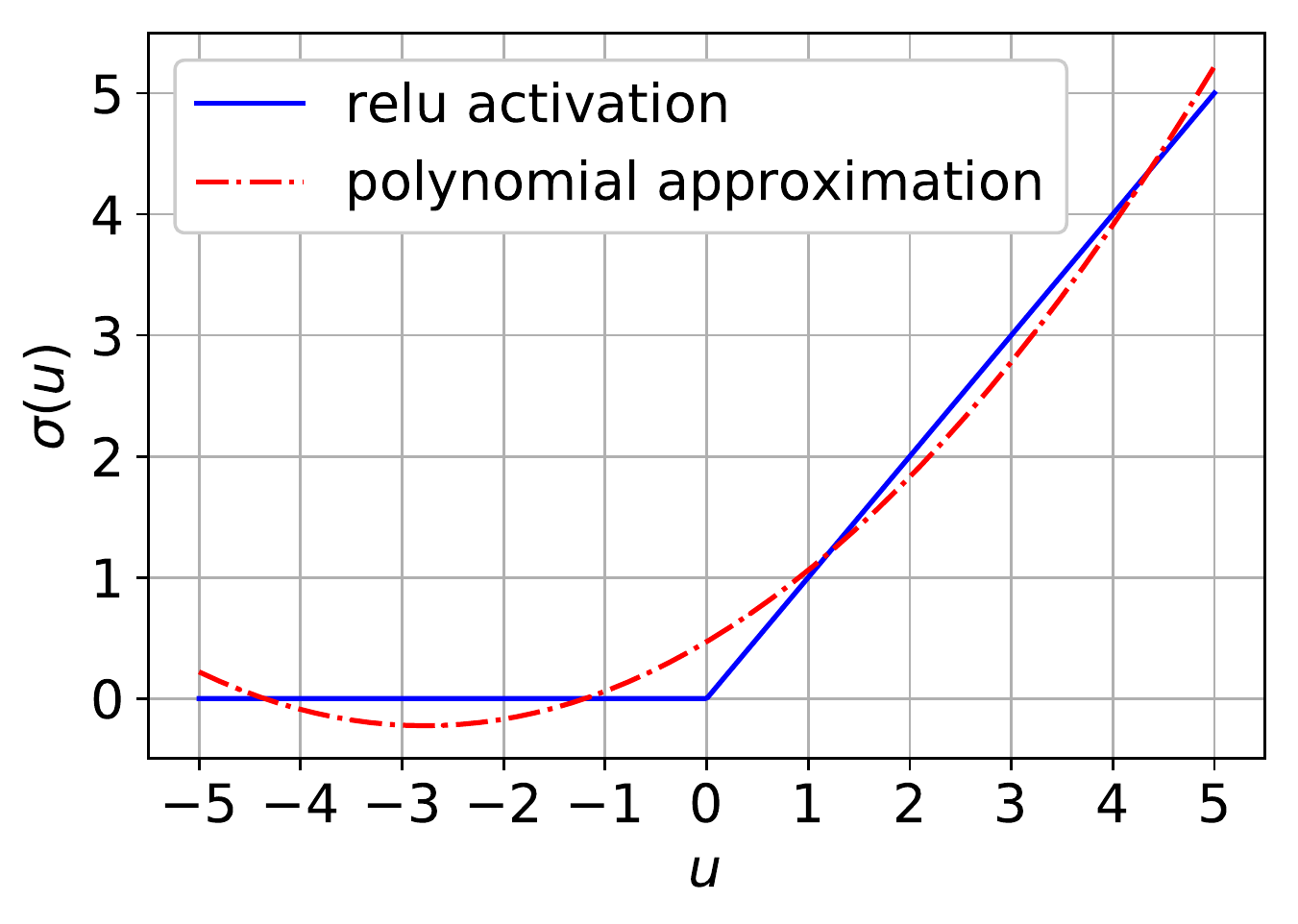}}
\end{minipage}
\hfill
\begin{minipage}[b]{0.48\linewidth}
  \centering
  \centerline{\includegraphics[width=0.8\columnwidth]{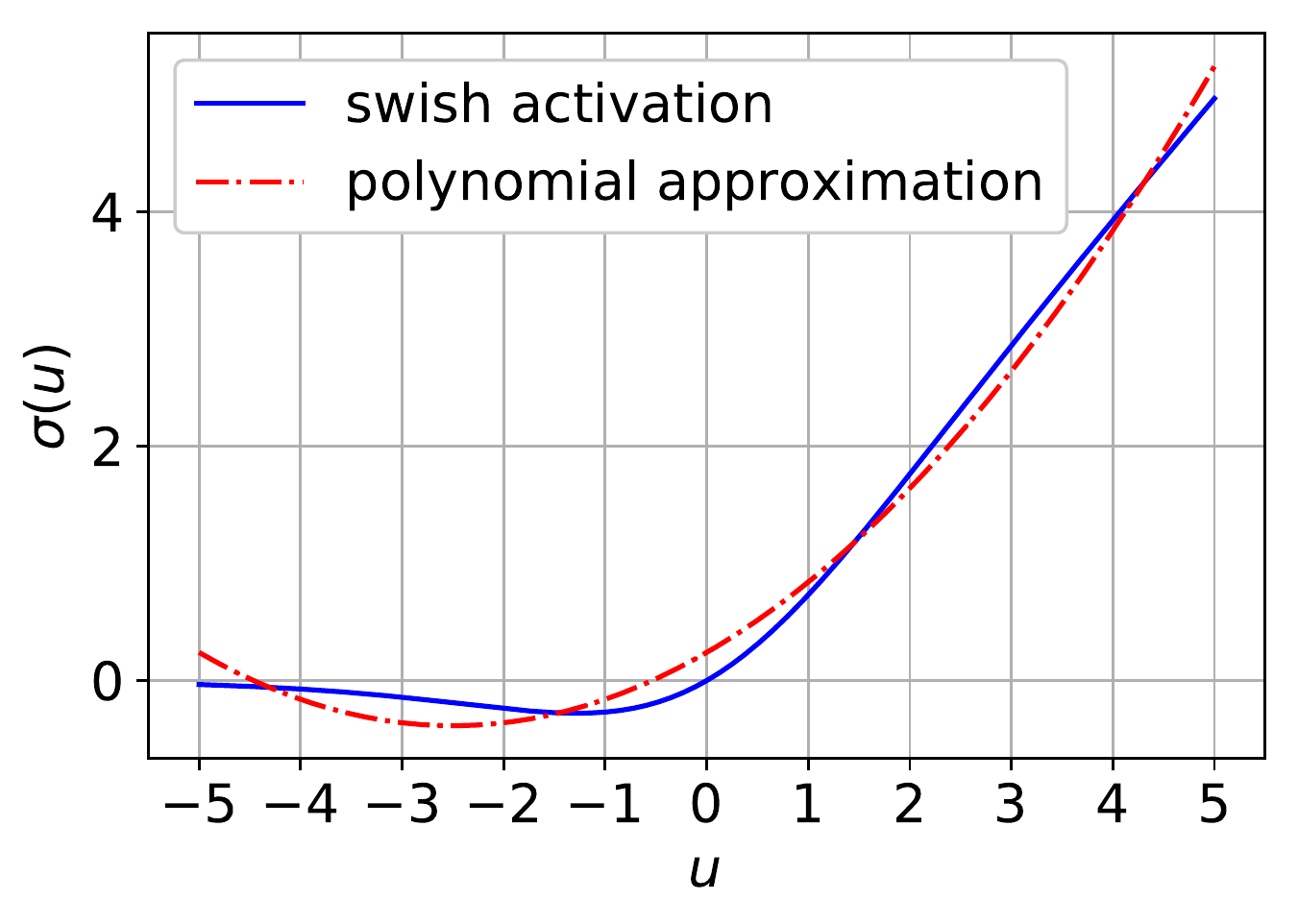}}
\end{minipage}
\caption{ReLU (left) and swish (right) activation functions and their second degree polynomial approximations. ReLU activation: $\sigma(u)=\max(0,u)$ and its polynomial approximation: $\sigma(u) = 0.09u^2+0.5u+0.47$. Swish activation: $\sigma(u)=u (1+e^{-u})^{-1}$ and its polynomial approximation: $\sigma(u)=0.1u^2+0.5u+0.24$.}
\label{fig:plot_activation_functions}
\end{figure}

Our derivation of the convex program for polynomial activations leverages convex duality and the S-procedure, and can be stated as a simple semidefinite program (SDP). We refer the reader to \cite{polik07slemma} for a survey of the S-procedure and applications in SDPs.
In addition, another commonly used activation function in the literature, quadratic activation, is a special case of polynomial activations ($b=c=0$) and we devote a separate section to this case (Section \ref{sec:quadratic_act}). The corresponding convex program is an SDP and takes a simpler form.



Main aspects of our work that differ from others in the literature that study the \textit{optimization landscape} of two-layer neural networks (e.g. see section \ref{sec:prior_work}) are the following: Our results (1) provide global optimal solutions in fully polynomial time (polynomial in all problem parameters), (2) uncover an important role of the regularizer in computational tractability, (3) hold for arbitrary loss function and other network architectures such as vector output, convolutional and pooling, (4) are independent of the choice of the numerical optimizer and its parameters.


We summarize the types of neural network architectures considered in this work and the corresponding convex problems that we have derived to train them to global optimality in Table \ref{table:summary_results}. The fourth column of Table \ref{table:summary_results} shows the upper bounds for critical width $m^*$, i.e., the optimal number of neurons that one needs for global optimization of any problems with number of neurons $m \geq m^*$. The fifth column, named "construction algorithm", refers to the method for obtaining the optimal neural network weights from the solution of the associated convex program. The last column contains the references to the theorems for each result.

\subsection{Overview of Our Contributions}
\begin{itemize}
    \item We show that the standard optimization formulation for training neural networks $f_\theta (x) = \sum_{j=1}^m \sigma(x^T u_j) \alpha_j $ with trainable parameters $\theta=(u_1,\dots,u_m,\alpha_1,\dots,\alpha_m)$ and degree two polynomial activations $\sigma(u) = au^2+bu+c$, training data $X=[x_1,\dots,x_n]^T\in\mathbb{R}^{n\times d}$, $y\in \mathbb{R}^n$, and $\ell_2^2$ regularization on the parameters given by
    \begin{align}
        \min_{\theta} \, \ell (f_\theta (X),y) + \beta \sum_{j=1}^m (\|u_j\|_2^2 + \|\alpha_j\|_2^2)
    \end{align}
    is computationally intractable via a reduction to the NP-hard subset sum problem. 
    \item Surprisingly, for quadratic activation networks, $\sigma(u)=u^2$, we show that modifying the quadratic weight decay regularization to \emph{cubic regularization}
    \begin{align}
        \min_{\theta } \, \ell (f_\theta (X),y) + \beta \sum_{j=1}^m (\|u_j\|_2^3 +  \|\alpha_j\|_2^3)
        \label{eq:contributions_cubic}
    \end{align}
    enables global optimization in fully polynomial time via convex semidefinite programming. The computational complexity is polynomial in all problem parameters ($n,d,m)$.
    \item Furthermore, for any degree two polynomial activation $\sigma$, the non-convex neural network training problem
        \begin{align}
        \min_{\theta\mbox{ s.t. } \|u_j\|_2=1 \,, \forall j\in[m]} \ell (f_\theta (X),y) + \beta \|\alpha\|_1
    \end{align}
    can be equivalently stated as a convex semidefinite problem and globally solved in fully polynomial time. In fact, the cubic regularization strategy in \eqref{eq:contributions_cubic} is a special case of this convex program. The result holds universally for all input data without any conditions and also holds when $\beta\rightarrow 0$.
    
    \item In deriving the convex formulations, we identify a concise re-parameterization of the neural network parameters that enables exact convexification by removing the redundancy in the classical overparameterized formulation. This is similar in spirit to the semidefinite lifting procedure in relaxations of combinatorial optimization problems. In contrast to these relaxations, we show that our lifting is always exact as soon as the network width exceeds a critical threshold which can be efficiently determined.
    
    \item We develop a matrix decomposition procedure called Neural Decomposition to extract the optimal network parameters from the solution of convex optimization, which is guaranteed to produce an optimal neural network. Neural Decomposition transforms the convex re-parameterization to the overparameterized, i.e., redundant, formulation in a similar spirit to (a non-orthogonal version of) Eigenvalue Decomposition.
    
    \item In addition to the fully connected neural network architecture, we derive the equivalent convex programs for various other architectures such as convolutional, pooling and vector output architectures.
    
    \item We provide extensive numerical simulations showing that the standard backpropagation approach with or without regularization fails to achieve the global optimum of the training loss. Moreover, the test accuracy of the proposed convex optimization is considerably higher in standard datasets as well as random planted models. Our convex optimization solver is significantly faster in total computation time to achieve similar or better test accuracy.
\end{itemize}

\begin{table}
\scalebox{0.67}{
\begin{tabular}{|l|l|l|l|l|l|}
\hline
& Non-convex objective  & Tractable convex & Upper bound on & Construction & Thms \\ 
&   & formulation & critical width $m^*$ & algorithm & \\ \hline
Poly (scalar) & $\ell \left( \sum_{j=1}^m \sigma(Xu_j) \alpha_j \,,\, y \right) + \beta \sum_{j=1}^m \vert\alpha_j\vert$ s.t. $\|u_j\|=1$ &  Eq \eqref{eq:polyact_convex_program_final}  & $2(d+1)$ & Neural decomp & Thm \ref{thm:poly_act_thm} \\ \hline
Poly (vector) & $\ell \left( \sum_{j=1}^m \sigma(Xu_j) \alpha_j^T \,,\, Y \right) + \beta \sum_{j=1}^m \|\alpha_j\|_1$ s.t. $\|u_j\|=1$ & Eq \eqref{eq:vector_output_convex_program} & $2(d+1)C$ & Neural decomp & Thm \ref{thm:vector_output_thm} \\ \hline 
Convolutional & $\ell \left( \sum_{j=1}^m \sum_{k=1}^K \sigma(X_k u_j) \alpha_{jk},\, y \right) + \beta \sum_{j=1}^m \|\alpha_j\|_1$ s.t. $\|u_j\|=1$ & Eq \eqref{eq:convolutional_convex_program} & $2(f+1)K^2 $ & Neural decomp & Thm \ref{thm:convolutional_thm} \\ \hline
Pooling & $\ell \left(\sum_{j=1}^m \sum_{k=1}^{K/P} \frac{1}{P} \sum_{l=1}^P \sigma(X_{(k-1)P+l} u_j) \alpha_{jk},\, y \right) + \beta \sum_{j=1}^m \|\alpha_j\|_1$ & Eq \eqref{eq:avgpooling_convex_program} & $2(f+1)\frac{K^2}{P^2}$ & Neural decomp & Thm \ref{thm:pooling_thm} \\
& \hspace*{\fill} s.t. $\|u_j\|=1$  & & & & \\ \hline
Quad (scalar, & $\ell\left(\sum_{j=1}^m\sigma(Xu_j)\alpha_j,\, y\right) + \beta \sum_{j=1}^m |\alpha_j|$ s.t. $\|u_j\|=1$, or & Eq \eqref{eq:convex_problem_quad_act} & $d$ & Eigen- & Thm \ref{thm:quad_act_thm} \\
cubic reg) & $\ell\left(\sum_{j=1}^m \sigma(Xu_j)\alpha_j,\, y\right) + \frac{\beta}{c} \sum_{j=1}^m (|\alpha_j|^3 + \|u_j\|_2^3)$ & & & decomposition & 
\\ \hline
Quad (scalar, & $\ell\left(\sum_{j=1}^m\sigma(Xu_j)\alpha_j,\, y\right) + \beta \sum_{j=1}^m |\alpha_j|^{2/3}$ s.t. $\|u_j\|=1$, or 
& NP-hard & - & - & Thm \ref{thm:nphard_general} \\
quad reg) & $\ell\left(\sum_{j=1}^m \sigma(Xu_j)\alpha_j,\, y\right) + \frac{\beta}{c} \sum_{j=1}^m (|\alpha_j|^2 + \|u_j\|_2^2)$ & (intractable) & & & 
\\ \hline
\end{tabular}
}
\caption{List of the neural network architectures that we have studied in this work and the corresponding convex programs. Abbreviations are as follows. Poly (scalar): Polynomial activation scalar output, Poly (vector): Polynomial activation vector output, Convolutional: CNN with polynomial activation, Pooling: CNN with polynomial activation and average pooling, Quad (scalar, cubic reg): Quadratic activation scalar output with cubic regularization, Quad (scalar, quad reg): Quadratic activation scalar output with quadratic regularization. $K$ is the number of patches and $f$ is the filter size for the convolutional architecture. $C$ is the output dimension for the vector output case. $P$ is the pool size for average pooling. $\sigma(u)$ is defined as $u^2$ for quadratic activation, and $au^2+bu+c$ for polynomial activation.
}
\label{table:summary_results}
\end{table}

\subsection{Prior Work} \label{sec:prior_work}
A considerable fraction of recent works on the analysis of optimization landscape of neural networks focuses on explaining why gradient descent performs well.
The works \cite{du2018quadact,soltan2019shallow_nn} consider the optimization landscape of a restricted class of neural networks with quadratic activation and quadratic regularization where the second layer weights are \textit{fixed}.
They show that when the neural network is overparameterized, i.e., $m \geq d$, the non-convex loss function has benign properties: all local minima are global and all saddle points have a direction of negative curvature. However, in this paper we show that training both the first and second layer weights with quadratic regularization in fact makes global optimization NP-hard. In contrast, we provide a different formulation to obtain the global optimal solution via convex optimization in the more general case when the second layer weights are also optimized, the activation function is any arbitrary degree two polynomial, and global optimum is achieved for all values of $m$.
%
The work in \cite{mannelli2020optimization} similarly studies two-layer neural networks with quadratic activation function and squared loss and states results on both optimization and generalization properties.
The authors in \cite{gamarnik2020stationary} focus on quadratic activation networks from the perspectives of optimization landscape and generalization performance, where the setting is based on a planted model with a full rank weight matrix. 
In \cite{lacotte2020all,lederer2020spurious} it was shown that sufficiently wide ReLU networks have a benign landscape when each layer is sufficiently wide, satisfying $m\ge n+1$.

%

Another recent work analyzing the training of neural networks with quadratic-like activations for deeper architectures is \cite{allenzhu2020backward}. Authors in \cite{allenzhu2020backward} consider polynomial activation functions and investigate layerwise training and compare with end-to-end training of layers. It is demonstrated in \cite{allenzhu2020backward} that the degree two polynomial activation function performs comparably to ReLU activation in deep networks. More specifically, it is reported in \cite{allenzhu2020backward} that for deep neural networks, ReLU activation achieves a classification accuracy of $0.96$ and a degree two polynomial activation yields an accuracy of $0.95$ on the Cifar-10 dataset. Similarly for the Cifar-100 dataset, they obtain an accuracy of $0.81$ for ReLU activation and $0.76$ for the degree two activation. These numerical results are obtained for the activation $\sigma(u)=u+0.1u^2$, which the authors prefer over the standard quadratic activation $\sigma(u)=u^2$ to make the neural network training stable. Moreover, the performance of layerwise learning with such activation functions is considerably high, although there is a gap between end-to-end trained models. These results verify that degree two polynomial activations are promising and worth studying from both theoretical and practical perspectives.
%

In a recent series of papers, the authors derived convex formulations for training ReLU neural networks to global optimality \cite{pilanci2020neural,ergen2020cnn,ergen2020convexdeep,ergen2020aistats,sahiner2020vector,sahiner2020convex}. Our work takes a similar convex duality approach in deriving the convex equivalents of non-convex neural network training problems. In particular, the previous work in this area deals with ReLU activations while in this work we focus on polynomial activations. Hence, the mathematical techniques involved in deriving the convex programs and the resulting convex programs are substantially different. The convex program derived for ReLU activation in \cite{pilanci2020neural} has polynomial time trainability for fixed rank data matrices, whereas the convex programs developed in this work are all polynomial-time trainable with respect to all problem dimensions. More specifically, their convex program is given by
\begin{align}
    \min_{\{v_i,w_i\}_{i=1}^P} &\frac{1}{2} \left\Vert \sum_{i=1}^PD_iX(v_i-w_i) - y \right\Vert_2^2 + \beta \sum_{i=1}^P (\|v_i\|_2 + \|w_i\|_2) \nonumber \\
    \mbox{s.t.} \quad &(2D_i-I_n)Xv_i \geq 0, \,\, (2D_i-I_n)Xw_i \geq 0,  \forall i \in [P]\,,
\end{align}
where the neural network weights are constructed from $v_i \in \mathbb{R}^d$ and $w_i \in \mathbb{R}^d$, $i=1,\dots,P$. The matrices $D_i$ are diagonal matrices whose diagonal entries consist of $1_{x_1^Tu \geq 0}, 1_{x_2^Tu \geq 0}, \dots, 1_{x_n^Tu \geq 0}$ for all possible $u \in \mathbb{R}^d$. The number of distinct $D_i$ matrices, denoted by $P$ is the number of hyperplane arrangements corresponding to the data matrix $X$. It is known that $P$ is bounded by $2 r \left( \frac{e(n-1)}{r} \right)^r$ where $r=\rank(X)$ (see \cite{pilanci2020neural} for the details). In particular, convolutional neural networks have a fixed value of $r$, for instance $m$ filters of size $3\times 3$ yield $r=9$. This is an exponential improvement over previously known methods that train optimal ReLU networks which are exponential in the number of neurons $m$ and/or the number of samples $n$ \cite{arora2018understanding,goel17a,bienstock2018principled}.

The work in \cite{blondel2015convexFM} presents formulations for convex factorization machines with nuclear norm regularization, which is known to obtain low rank solutions.
Vector output extension for factorization machines and \textit{polynomial networks}, which are different from \textit{polynomial activation networks}, is developed in \cite{blondel2017multi}. Polynomial networks are equivalent to quadratic activation networks with an addition of a linear neuron. In \cite{blondel2017multi}, the authors consider learning an infinitely wide quadratic activation layer by a greedy algorithm. However, this algorithm does not provide optimal finite width networks even in the quadratic activation case. Furthermore, \cite{roi2014training} presents a greedy algorithm for training polynomial networks. The algorithm provided in \cite{roi2014training} is based on gradually adding neurons to the neural network to reduce the loss. More recently, \cite{soltani2019learning} considers applying lifting for quadratic activation neural networks and presents non-convex algorithms for low rank matrix estimation for two-layer neural network training.



\subsection{Notation}
Throughout the text, $\sigma : \mathbb{R} \rightarrow \mathbb{R}$ denotes the activation function of the hidden layer. We refer to the function $\sigma(u) = u^2$ as quadratic activation and $\sigma(u) = au^2+bu+c$ where $a,b,c \in \mathbb{R}$ as polynomial activation.
We use $X \in \mathbb{R}^{n \times d}$ to denote the data matrix, where its rows $x_i \in \mathbb{R}^d$ correspond to data samples and columns are the features.
In the text, whenever we have a function mapping from $\mathbb{R}$ to $\mathbb{R}$ with a vector argument (e.g., $\sigma(v)$ or $v^2$ where $v$ is a vector), this means the elementwise application of that function to all the components of the vector $v$.
We denote a column vector of ones by $\ones$ and its dimension can be understood from the context.
$\vect(\cdot)$ denotes the vectorized version of its argument. In writing optimization problems, we use $\min$ and $\max$ to refer to "minimize" and "maximize". 
We use the notations $[m]$ and $1,\dots,m$ interchangeably. 

We use $\ell(\hat{y}, y)$ for convex loss functions throughout the text for both scalar and vector outputs. $\ell^*(v) = \sup_{z}(v^Tz - \ell(z, y))$ denotes the Fenchel conjugate of the function $\ell(\cdot, y)$. Furthermore, we assume $\ell^{**}=\ell$ which holds when $\ell$ is a convex and closed function \cite{boyd2004convex}. 

We use $Z \succeq 0$ for positive semidefinite matrices (PSD). $\mathbb{S}$ refers to the set of symmetric matrices. $\tr$ refers to matrix trace. $\otimes$ is used for outer product. The operator $\conv$ stands for the convex hull of a set.

\subsection{Preliminaries on Semidefinite Lifting}

We defer the discussion of semidefinite lifting for two-layer neural networks with polynomial activations to Section \ref{sec:semidefinite_representations}. We now briefly discuss a class of problems where SDP relaxations lead to exact solutions of the original problem and also instances where they fail to be exact. Let us consider the following quadratic objective problem with a single quadratic constraint:
\begin{align} \label{eq:quadratic_noncvx}
    \min_{u} \quad & u^TQ_1u + b_1^Tu + c_1 \nonumber \\
    \mbox{s.t.} \quad & u^TQ_2u + b_2^Tu + c_2 \leq 0
\end{align}
where $Q_1, Q_2$ are indefinite, i.e., not assumed to be positive semidefinite. Due to the indefinite quadratics, this is a non-convex optimization problem. By introducing a matrix variable $U=uu^T$, one can equivalently state this problem as
\begin{align} 
    \min_{U, u} \quad & \tr(Q_1 U) + b_1^Tu + c_1 \nonumber \\
    \mbox{s.t.} \quad & \tr(Q_2 U) + b_2^Tu + c_2 \leq 0 \nonumber \\
    & U = uu^T \,.
\end{align}
This problem can be relaxed by replacing the equality by the matrix inequality $U \succeq uu^T$. Re-writing the expression $U \succeq uu^T$ as a linear matrix inequality via the Schur complement formula yields the following SDP 
\begin{align} \label{eq:sdp_schur}
    \min_{U, u} \quad & \tr(Q_1 U) + b_1^Tu + c_1 \nonumber \\
    \mbox{s.t.} \quad & \tr(Q_2 U) + b_2^Tu + c_2 \leq 0 \nonumber \\
    & \begin{bmatrix} U & u \\ u^T & 1 \end{bmatrix} \succeq 0 \,.
\end{align}
Remarkably, it can be shown that the original non-convex problem in \eqref{eq:quadratic_noncvx} can be solved exactly by solving the convex SDP in \eqref{eq:sdp_schur} via duality, under the mild assumption that the original problem is strictly feasible (see \cite{boyd2004convex}). This shows that the SDP relaxation is exact in this problem, returning a globally optimal solution when one exists. We note that there are alternative numerical procedures to compute the global optimum of quadratic programs with one quadratic constraint \cite{boyd2004convex}.

We also note that the lifting approach $U=uu^T$ and the subsequent relaxation $U\succeq uu^T$ for quadratic programs with more than two quadratic constraints is not tight in general \cite{nesterov2000semidefinite,burer2012copositive}. A notable case with multiple constraints is the NP-hard Max-Cut problem and its SDP relaxation \cite{Goemans95}
\begin{align}
\max_{u_i^2=1,\forall i} u^TQu = \max_{u_i^2=1,\forall i} \tr(Quu^T) \le  \max_{U\succeq 0,\, U_{ii}=1,\forall i} \tr(QU).
\end{align}

The SDP relaxation of Max-Cut is not tight since its feasible set contains the cut polytope 
$$\textbf{conv}\left\{uu^T\,:u_i\in\{-1,+1\}\,\forall i\right\}$$
and other non-integral extreme points \cite{laurent1995positive}. Nevertheless, an approximation ratio of $0.878$ can be obtained via the Goemans-Williamson randomized rounding procedure \cite{Goemans95}. It is conjectured that this is the best approximation ratio for Max-Cut \cite{khot2007optimal}, whereas it can be formally proven to be NP-hard to approximate within a factor of $\frac{16}{17}$ \cite{haastad2001some,trevisan2000gadgets}. Hence, in general we cannot expect to obtain exact solutions to problems of combinatorial nature, such as Max-Cut and variants using SDP relaxations.

It is instructive to note that a naive application of the SDP lifting strategy is not immediately tractable for two-layer neural networks. For simplicity, consider a scalar output polynomial activation network $f(x)=\sum_{j=1}^m \sigma(x^Tu_j)\alpha_j$ where $\sigma(u)=u^2+u$, and $\{u_j,\alpha_j\}_{j=1}^m$ are trainable parameters. The corresponding training problem for a given loss function $\ell(\cdot, y)$ and its SDP relaxation are as follows
\begin{align}
    \min_{\{u_j,\alpha_j\}_{j=1}^m} \sum_{x \in \mathcal{X}} \ell\big( \sum_{j=1}^m ((x^Tu_j)^2+x^Tu_j)\alpha_j,\, y \big) \ge \min_{\{U_j\succeq u_ju_j^T,\alpha_j\}_{j=1}^m} \sum_{x \in \mathcal{X}} \ell\big( \sum_{j=1}^m x^TU_j x \alpha_j + x^Tu_j\alpha_j,\, y \big).
\end{align}
The above problem is \emph{non-convex} due to the bilinear terms $\{U_j\alpha_j\}_{j=1}^m$. Moreover, a variable change $\hat U_j=U_j\alpha_j$ does not respect semidefinite constraints $U_j\succeq u_ju_j^T$ when $\alpha_j \in \mathbb{R}$. Another limitation is the prohibitively high number of variables in the lifted space, which is $d^2 m+dm+m$ as opposed to $dm+m$ in the original problem. Therefore, a different convex analytic formulation is required to address all these concerns.

Although SDP relaxations are extensively studied for various non-convex problems (see e.g. \cite{wolkowicz2012handbook} for a survey of applications), instances with exact SDP relaxations are exceptionally rare. As will be discussed in the sequel, our main result for two-layer neural networks is another instance of an SDP relaxation leading to \emph{exact} formulations where the semidefinite lifting and relaxation is tight. 

In convex geometry, a \emph{spectrahedron} is a convex body that can be represented as a linear matrix inequality which are the feasible sets of semidefinite programs. An example is the \emph{elliptope} defined as the feasible set of the Max-Cut relaxation given by $U\succeq 0, U_{ii}=1\,\forall i$, which is a subset of $n\times n$ symmetric positive-definite matrices. Due to the existence of efficient projection operators and barrier functions of linear matrix inequalities, optimizing convex objectives over spectrahedra can be efficiently implemented, which renders SDPs tractable. We will show that polynomial activation neural networks can be represented via a class of simple linear matrix inequalities, dubbed \emph{neural spectrahedra} (see Figure \ref{fig:neuralcone} for an example), and enables global optimization in fully polynomial time and elucidates their parameterization in convex analytic terms.

\subsection{Paper Organization}

Section \ref{sec:semidefinite_representations} gives an overview of the theory developed in this work. Section \ref{sec:convex_duality} describes the convex optimization formulation via duality and S-procedure for polynomial activation neural networks. Section \ref{sec:neural_decomp} establishes via the neural decomposition method that the convex problem developed in Section \ref{sec:convex_duality} can be used to train two-layer polynomial activation networks to global optimality. Quadratic activation neural networks and the hardness result are studied in Section \ref{sec:quadratic_act} and \ref{sec:np_hardness}. Vector output and convolutional neural network architectures are studied in Section \ref{sec:vector_outputs} and \ref{sec:convolutional}, respectively and convolutional networks with average pooling is in Section \ref{sec:avg_pooling}. We discuss the implementation details for solving the convex programs and give experimental results in Section \ref{sec:numerical_results}.
\section{Lifted Representations of Networks with Polynomial Activations} \label{sec:semidefinite_representations}
\newcommand{\append}[1]{\left[ \begin{array}{c} #1 \\ 1 \end{array}\right]}
Consider the network $f(x) = \sum_{j=1}^m \sigma(x^Tu_j) \alpha_j$ where the activation function $\sigma$ is the degree two polynomial $\sigma(u) = au^2 + bu + c$. First, we note that the neural network output can be written as
\begin{align}
    f(x) = \sum_{j=1}^m \left(a(x^T u_j)^2 + b x^T u_j + c\right)\alpha_j &= \sum_{j=1}^m \left(\langle axx^T, u_j u_j^T \rangle + \langle b x, u_j \rangle + c \right)\alpha_j\nonumber\\
    & =  \left \langle \left[\begin{array}{c} a xx^T\\ bx \\ c \end{array}\right], \left[\begin{array}{c}\sum_{j=1}^m u_ju_j^T \alpha_j\\ \sum_{j=1}^m u_j \alpha_j \\ \sum_{j=1}^m \alpha_j  \end{array}\right] \right \rangle\nonumber\\
    &= \langle \phi(x), \psi(\{u_j,\alpha_j\}_{j=1}^m) \rangle\,,
    \label{eq:liftednetwork}
\end{align}
where $\phi:\mathbb{R}^d\rightarrow \mathbb{R}^{d^2+d+1}$ and  $\psi:\mathbb{R}^{m(d+1)}\rightarrow \mathbb{R}^{d^2+d+1}$ are formally defined in the sequel.
The above identity shows that the nonlinear neural network output is linear over the \emph{lifted features} $$\phi(x):=\big(axx^T, bx, c\big) \in \mathbb{R}^{d^2 + d + 1}.$$ In turn, the nonlinear model $f(x)$ is completely characterized by the \emph{lifted parameters} which we define as the following matrix-vector-scalar triplet $$\psi(\{u_j,\alpha_j\}_{j=1}^m):=\Big(\sum_{j=1}^m u_ju_j^T \alpha_j, \sum_{j=1}^m u_j \alpha_j, \sum_{j=1}^m \alpha_j  \Big) \in \mathbb{R}^{d^2 + d + 1}.$$
Optimizing over the lifted parameter space initially appears as hard as the original non-convex neural network training problem. This is due to the cubic and quadratic terms involving the weights of the hidden and output layer in the lifted parameters. Nevertheless, one of our main results shows that the lifted parameters can be exactly described using linear matrix inequalities.

We begin by characterizing the lifted parameter space as a non-convex cone.
\begin{figure}[!t]
    \centering
    \begin{minipage}{.50\textwidth}
        \centering
        \includegraphics[trim={0cm 2cm 0cm 2cm}, clip, width=0.7\linewidth]{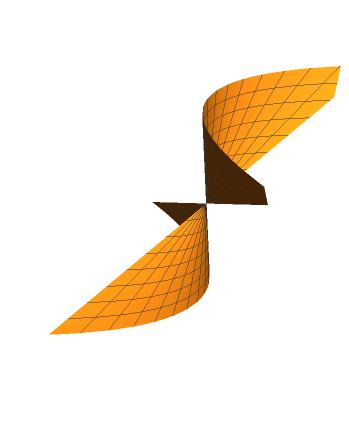}
    \end{minipage}%
    \hfill
    \begin{minipage}{0.50\textwidth}
        \centering
        \includegraphics[trim={2cm 2.6cm 0.5cm 5cm},clip, width=0.8\linewidth]{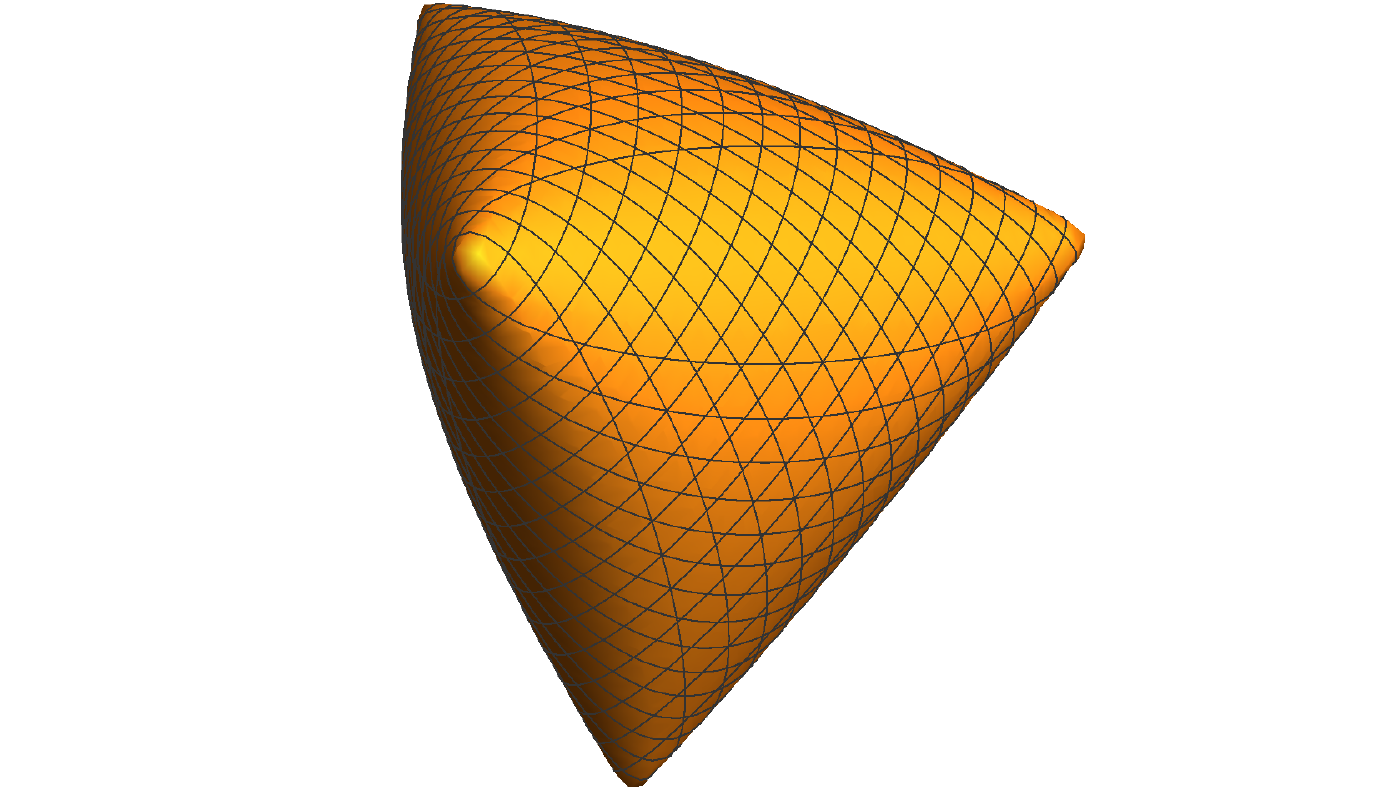}
    \end{minipage}
            \caption[]{(Left) The Neural Cone $\mathcal{C}_2^1$  described by $(u^2\alpha,u\alpha,\alpha)\in \mathbb{R}^3$ where $u,\alpha \in \mathbb{R}, |u|\le 1$. (Right) Neural Spectrahedron $\mathcal{M}(1)$ described by $(Z_{11},Z_{12},Z_{22})\in \mathbb{R}^3$ where $Z=\left[ \begin{array}{ccc} Z_{11} & Z_{12} & Z_{13} \\ Z_{12} & Z_{22} & Z_{23} \\ Z_{13} & Z_{23} & Z_{33}\end{array} \right]\succeq 0,\, Z_{11}+Z_{22}=Z_{33}\le 1$ (constrained to the slice $Z_{22}=Z_{11}$ and $Z^\prime=0$ in \eqref{eq:neuralset}).}
    \label{fig:neuralcone}
\end{figure}
%
\begin{definition}[\textbf{Neural Cone of degree two}]
We define the non-convex cone $\mathcal{C}^m_{2} \subseteq \mathbb{R}^{d^2+d+1}$ as
\begin{align}
    \mathcal{C}^m_{2} :=  \left\{ \Big(\sum_{j=1}^m u_ju_j^T \alpha_j, \sum_{j=1}^m u_j \alpha_j, \sum_{j=1}^m \alpha_j  \Big) \, : u_j\in\mathbb{R}^d, \|u_j\|_2 = 1,\, \alpha_j \in \mathbb{R}\,\forall j\in[m] \right\}.
\end{align}
See Figure \ref{fig:neuralcone} (left) for a depiction of $\mathcal{C}_2^1\subseteq\mathbb{R}^3$ corresponding to the case $m=1,d=1$.
\end{definition}
%
Surprisingly, we will show that the original non-convex neural network problem is solved exactly to \emph{global optimality} when the optimization is performed over a convex set which we define as the \emph{Neural Spectrahedron}, given by the convex hull of the cone $\mathcal{C}_2$. In other words, every element of the convex hull can be associated with a neural network of the form $f(x)=\sum_{j=1}^m \sigma(x^Tu_j)\alpha_j$ through a special matrix decomposition procedure which we introduce in Section \ref{sec:neural_decomp}. Moreover, a Neural Spectrahedron can be described by a simple linear matrix inequality. Consequently, these two results enable global optimization of neural networks with polynomial activations of degree two in fully polynomial time with respect to all problem parameters: dimension $d$, number of samples $n$ and number of neurons $m$. To the best of our knowledge, this is the first instance of a method that globally optimizes a standard neural network architecture with computational complexity polynomial in all problem dimensions. We refer the reader to the recent work \cite{pilanci2020neural} for a convex optimization formulation of networks with ReLU activation, where the worst case computational complexity is $\mathcal{O}((\frac{n}{r})^r)$ with $r=\rank(X)$.

It is equally important that our results characterize neural networks as constrained linear learning methods $\langle \phi(x), \psi \rangle$ in the lifted feature space $\phi(x)$, where the constraints on the lifted parameters $\psi$ are precisely described by a Neural Spectrahedron via linear matrix inequalities. These constraints can be easily tackled with convex semidefinite programming or closed-form projections onto these sets in iterative first-order algorithms. We also investigate interesting regularization properties of this convex set, and draw similarities to $\ell_1$ norm and nuclear norm. In contrast, Reproducing Kernel Hilbert Space methods and Neural Tangent Kernel approximations \cite{jacot2018neural,chizat2019lazy} are linear learning methods over lifted feature maps where the corresponding parameter constraints are ellipsoids. These approximations fall short of explaining the extraordinary power of finite width neural networks employed in practical applications.

We extend the definition of the Neural Cone to degree $k$ activations as follows.
\begin{definition}[\textbf{Neural Cone of degree $k$}]
We define the non-convex cone $\mathcal{C}^m_{k} \subseteq \mathbb{R}^{\sum_{i=0}^k d^i}$ as follows
\begin{align}
    \mathcal{C}^m_{k} :=  \left\{ \Big(\sum_{j=1}^m u_j^{\otimes k} \alpha_j, \cdots, \sum_{j=1}^m u_j \otimes u_j \alpha_j, \sum_{j=1}^m u_j \alpha_j, \sum_{j=1}^m \alpha_j  \Big) \, :  u_j\in\mathbb{R}^d,\|u_j\|_2 = 1,\, \alpha_j \in \mathbb{R}\,\forall j\in[m] \right\}
\end{align}
where we use the notation $ u^{\otimes k} := \underbrace{u\otimes\cdots\otimes u}_{\mbox{k times}}$.
\end{definition}
It is easy to see that two-layer neural networks with degree $k$ polynomial activations can be represented linearly using the lifted parameter space $\mathcal{C}_k$ and corresponding lifted features. Taking the closure of the union $\{\mathcal{C}\}_{k=0}^{\infty}$, any analytic activation function can be represented in this fashion. In this paper we limit the analysis to the degree $2$ case.

Next, we describe a compact set that we call \emph{neural spectrahedron} which describes the lifted parameter space of networks with a constraint on the $\ell_1$ norm of output layer weights.
\begin{definition}
A neural spectrahedron $\mathcal{S}^m_2(t) \subseteq \mathbb{R}^{d^2+d+1}$ is defined as the compact convex set
\begin{align}
    \mathcal{S}^m_2(t) := \conv \left\{ \Big(\sum_{j=1}^m u_ju_j^T \alpha_j, \sum_{j=1}^m u_j \alpha_j, \sum_{j=1}^m \alpha_j  \Big) : \|u_j\|_2 = 1, \alpha_j \in \mathbb{R},\forall j=1,\dots,m,  \sum_{j=1}^m |\alpha_j| \le t\right\}
\end{align}
\end{definition}
\newcommand{\M}{\mathcal{M}}
We will show that a neural spectrahedron can be equivalently described as a linear matrix inequality via defining $S_2^m(t)=\big(\M_{11}(t), \M_{12}(t), \M_{22}(t)\big)$ for all $m\ge m^*$ where
\begin{align}
    \M(t) = \left\{ Z - Z^\prime : Z=\left[ \begin{array}{c c} Z_1 & Z_2 \\ Z_2^T & Z_4 \end{array}\right]\succeq 0,\, Z^
    \prime=\left[ \begin{array}{c c} Z_1^\prime & {Z_2^\prime}  \\ {Z_2^\prime}^T & Z_4^\prime \end{array}\right]\succeq 0, \tr(Z_1) = Z_4,\tr(Z_1^\prime) = Z_4^\prime,\, Z_4 + Z_4^\prime \le t  \right\},
    \label{eq:neuralset}
\end{align}
 $Z,Z^\prime \in \mathbb{S}^{(d+1)\times (d+1)}$, $Z_1,Z_1^\prime \in \mathbb{S}^{d\times d}$, $Z_2,Z_2^\prime \in \mathbb{R}^{d\times 1}$ and $Z_4,Z_4^\prime \in \mathbb{R}_+$, and $m^*=m^*(t)$ is a critical number of neurons that satisfies $m^*(0)=0$ and $m^*(t)\le2(d+1)\,\forall t$, which will be explicitly defined in the sequel.
Therefore, an efficient description of the set $\mathcal{M}(t)$ in terms of linear matrix inequalities enables efficient convex optimization methods in polynomial time. Moreover, it should be noted that in non-convex optimization, the choice of the optimization algorithm and its internal hyperparameters, such as initialization, mini-batching and step sizes have a substantial contribution to the quality of the learned neural network model. This is in stark contrast to convex optimization problems, where optimizer hyperparameters have no effect, and solutions can be obtained in a very robust, efficient and reproducible manner.
\subsection{A geometric description of the Neural Spectrahedron for the special case of nonnegative output layer weights}
Here we describe a simpler case with the restriction $\alpha_j\ge 0\, \forall j\in[m]$ in the Neural Cone $\mathcal{C}_2^m$ and we will suppose that $m\ge d+1$. In this special case, let us define the one-sided positive Neural Spectrahedron as
\begin{align}
     ~^+\!{\mathcal{S}}^m_2(t) := \conv \left\{ \Big(\sum_{j=1}^m u_ju_j^T \alpha_j, \sum_{j=1}^m u_j \alpha_j, \sum_{j=1}^m \alpha_j  \Big) : \|u_j\|_2 = 1, \alpha_j \in \mathbb{R}_+,\forall j=1,\dots,m,  \sum_{j=1}^m \alpha_j \le t\right\}.
\end{align}
We observe that $~^+\!{\mathcal{S}}^m_2(t)$ is identical to the set $\big(~^+\!\M_{11},~^+\!\M_{12},~^+\M_{22}\big)\subseteq\mathbb{R}^{d^2+d+1}$ where
\begin{align} \label{eq:neural_spectrahedron_def}
    ~^+\!\M(t) : & = t \conv \left\{ \sum_{j=1}^m  \append{u_j}\append{u_j}^T\alpha_j\,: u_j\in\mathbb{R}^d, \|u_j\|_2 = 1, \alpha_j\in\mathbb{R}_+,\,\sum_{j=1}^m \alpha_j \le 1 \right\},
\end{align}
which is partitioned as $ ~^+\!\M(t)=\left[\begin{array}{cc} ~^+\!\M_{11} & ~^+\!\M_{12}\\ ~^+\!\M_{12}^T & ~^+\!\M_{22} \end{array}\right]$ where $~^+\!\M_{11}\subseteq\mathbb{S}^{d\times d}, ~^+\!\M_{12}\subseteq \mathbb{R}^{d \times 1}$ and $~^+\!\M_{22}\subseteq\mathbb{R}_+$.

Next, we note that as soon as the network width\footnote{This assumption is not required in our later analysis.} satisfies $m\ge d+1$, we have
%
\begin{align} \label{eq:neural_spectrahedron_1}
    ~^+\!\M(t) : & = t\conv \left\{\left\{   \append{u}\append{u}^T:\,\|u\|_2 = 1 \right\} \cup \mathbf{0} \right\},
\end{align}
where $\mathbf{0}$ is the zero matrix, since $\sum_{j=1}^m \append{u_j}\append{u_j}^T\alpha_j\in\mathbb{S}^{(d+1)\times (d+1)}$ is a positive semidefinite matrix, and hence can be factorized\footnote{We describe the details of this factorization in Section \ref{sec:neural_decomp}.} as a convex combination of at most $d+1$ rank-one matrices of the form $\append{u}\append{u}^T$. Note that the zero matrix is included to account for the inequality $\sum_{j=1}^m \alpha_j \leq 1$ in \eqref{eq:neural_spectrahedron_def}. This important observation enables us to represent the convex hull of the non-convex Neural Cone (an example is shown in Figure \ref{fig:neuralcone}), via the simple convex body $~^+\!\M(t)$ given in \eqref{eq:neural_spectrahedron_1}.

Most importantly, the positive Neural Spectrahedron set $~^+\!\M(t)$ provides a representation of the non-convex Neural Cone $\mathcal{C}_2^m$ via its extreme points. Furthermore, $~^+\!\M(t)$  has a simple description as a linear matrix inequality provided in the following lemma (the proof can be found in the appendix).
\begin{lemma} \label{lem:spectra_LMI}
For $m\ge d+1$, it holds that
\begin{align} \label{eq:neural_spectrahedron_2}
    ~^+\!\M(t) =\left\{Z:\, Z=\left[ \begin{array}{c c} Z_1 & Z_2 \\ Z_2^T & Z_4 \end{array}\right]\succeq 0,\,\tr(Z_1)=Z_4\leq t \right\}.
\end{align}
Therefore the positive Neural Spectrahedron can be represented as the intersection of the positive semidefinite cone and linear inequalities. Moreover, every element of $~^+\!\M(t)$ can be factorized as $\sum_{j=1}^m \Big[ \begin{array}{cc} u_ju_j^T \alpha_j & u_j \alpha_j \\ u_j^T\alpha_j &  \alpha_j\end{array}\Big]$ for some $\|u_j\|_2=1,\alpha_j\ge 0,\,\forall j\in[m],\,\sum_{j=1}^m \alpha_j\leq t$, which can be identified as an element of the non-convex Neural Cone $\mathcal{C}_2^m$ and a neural network in the lifted parameter space as shown in \eqref{eq:liftednetwork}.
\end{lemma}

%
The assumption $m\ge d+1$ is not required and only used here to illustrate this simpler special case.
In the more general case of arbitrary output layer weights $\alpha_j\in\mathbb{R},\,\forall j\in[m]$, we have the more general linear matrix inequality representation in \eqref{eq:neuralset}, which is in terms of two positive semidefinite cones and three linear inequalities.  In general, such a restriction on the number of neurons $m$ in terms of the dimension $d$ is not necessary. In the next sections, we only require $m\ge m^*$, where $m^*$ can be determined via a convex program. Furthermore, the regularization parameter directly controls the number of neurons $m^*$. We illustrate the effect of the regularization parameter on $m^*$ in the numerical experiments section, and show that $m^*$ can be made arbitrarily small.

\section{Convex Duality for Polynomial Activation Networks} \label{sec:convex_duality}

We consider the non-convex training of a two-layer fully connected neural network with polynomial activation and derive a convex dual optimization problem.  The input-output relation for this architecture is
\begin{align}
    f(x) = \sum_{j=1}^m \sigma(x^Tu_j)\alpha_j \,,
\end{align}
where $\sigma$ is the degree two polynomial $\sigma(u)=au^2+bu+c$. This neural network has $m$ neurons with the first layer weights $u_j \in \mathbb{R}^d$ and second layer weights $\alpha_j \in \mathbb{R}$. We refer to this case where $f: \mathbb{R}^d \rightarrow \mathbb{R}$ as the scalar output case. Section \ref{sec:vector_outputs} extends the results to the vector output case.

It is relatively easy to obtain a weak dual that provides a lower-bound via Lagrangian duality. However, in non-convex problems, a duality gap may exist since strong duality does not hold in general.  Remarkably, we show that strong duality holds as soon as the network width exceeds a critical threshold which can be easily determined.

We will assume $\ell_1$ norm regularization on the second layer weights as regularization and include constraints that the first layer weights are unit norm. We note that $\ell_1$ norm regularization on the second layer weights results in a special dual problem and hence is crucial in the derivations. We show in Section \ref{sec:quadratic_act} that this formulation is equivalent to cubic regularization when the activation is quadratic. For the standard $\ell_2^2$, i.e., weight decay regularization, we will in fact show that the problem is NP-hard (see Section \ref{sec:np_hardness}). The training of a network under this setting requires solving the non-convex optimization problem given by
\begin{align} \label{eq:poly_act_non-convex}
    p^* = &\min_{\{\alpha_j,\, u_j\}_{j=1}^m, \, \mbox{s.t.}\, \|u_j\|_2=1,\,\forall j} \ell \left( \sum_{j=1}^m \sigma(Xu_j) \alpha_j \,,\, y \right) + \beta \sum_{j=1}^m \vert\alpha_j\vert \,.
\end{align}

Theorem \ref{thm:poly_act_thm} states the main result for polynomial activation neural networks that the non-convex optimization problem in \eqref{eq:poly_act_non-convex} can be solved globally optimally via a convex problem. Before we state Theorem \ref{thm:poly_act_thm}, we briefly describe the numerical examples shown in Figure \ref{fig:converge_to_localminima} and \ref{fig:uci_multiclass_classification_1} which compare the solution of the non-convex problem via backpropagation and the solution of the corresponding convex program via a convex solver (see Section \ref{sec:numerical_results} for details on the solver). Figure \ref{fig:converge_to_localminima} shows the training and test costs on a regression task with randomly generated data for the two-layer quadratic activation neural network. We observe that convex SDP takes a much shorter time to optimize and obtains a globally optimal solution while the SGD algorithm converges to local minima in some of the trials where the initialization is different. Furthermore, Figure \ref{fig:uci_multiclass_classification_1} compares the classification accuracies for the two-layer vector output polynomial activation network on a multiclass classification problem with real data. The exact statement of the vector output extension of the main result is provided in Section \ref{sec:vector_outputs}. In Section \ref{sec:numerical_results}, we present additional numerical results verifying all of the theoretical results on various datasets.

\begin{figure} 
\begin{minipage}[b]{0.48\linewidth}
  \centering
  \centerline{\includegraphics[width=\columnwidth]{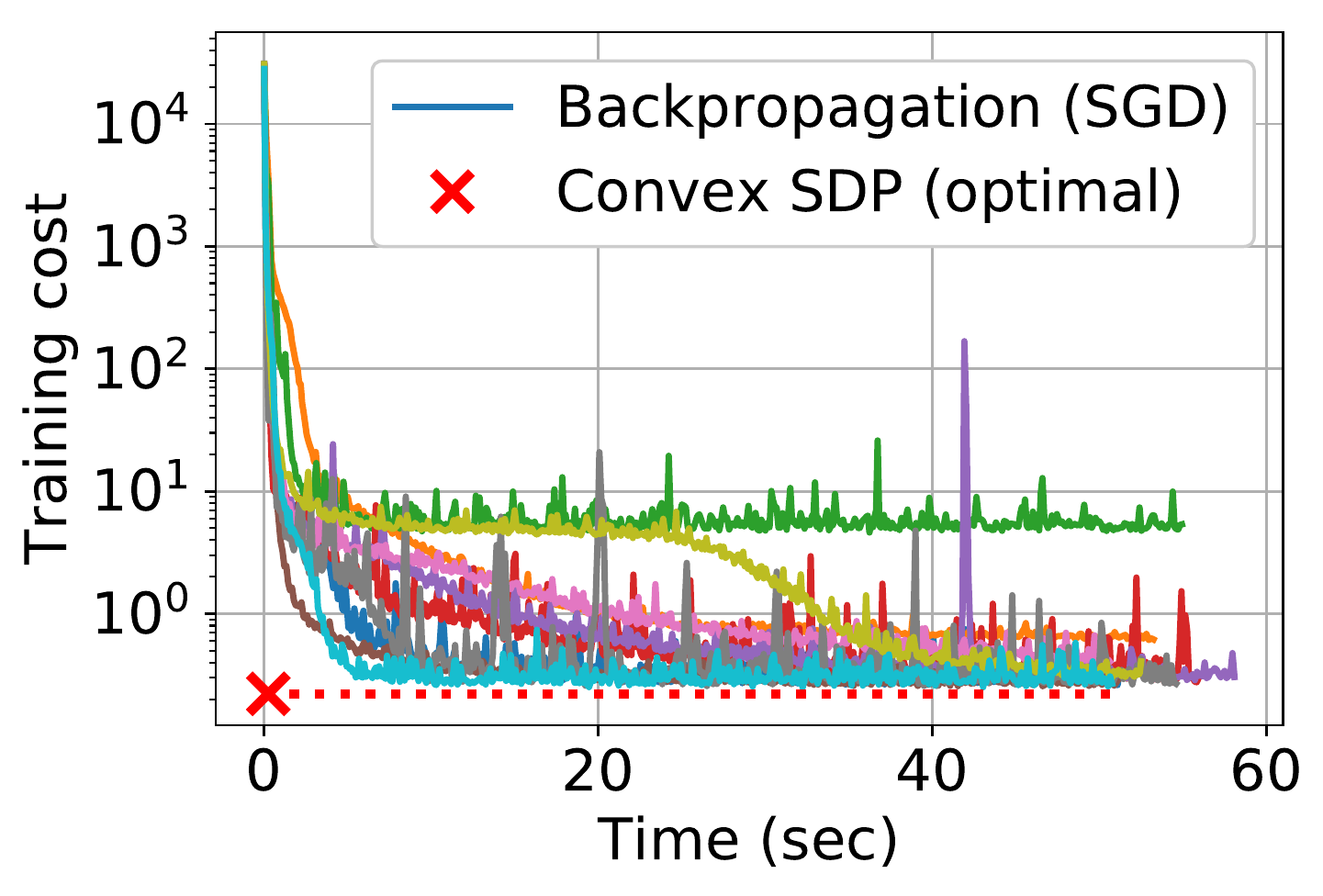}}
\end{minipage}
\hfill
\begin{minipage}[b]{0.48\linewidth}
  \centering
  \centerline{\includegraphics[width=\columnwidth]{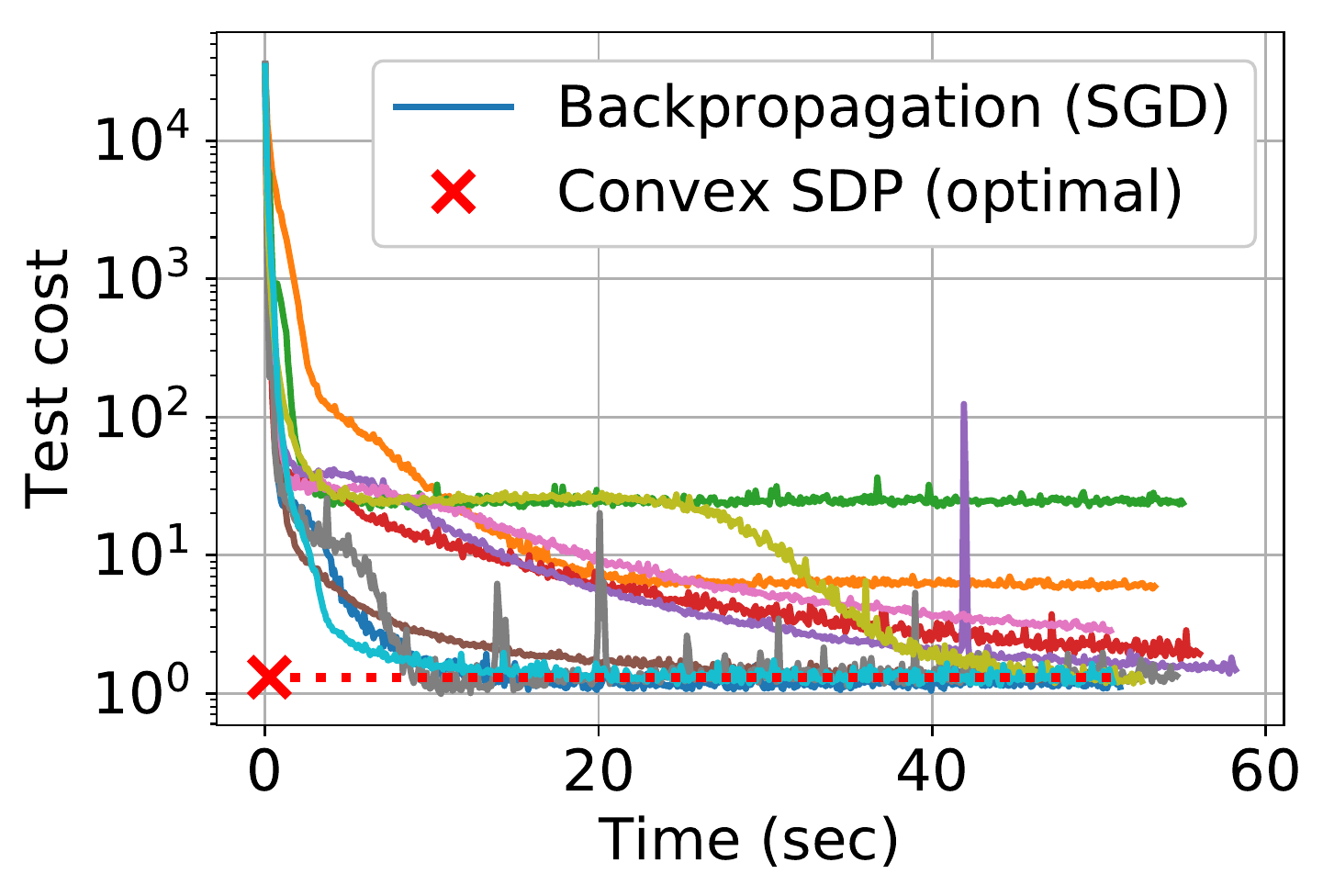}}
\end{minipage}
\caption{Cost against wall-clock time on the training (left) and test (right) sets for stochastic gradient descent (SGD) and the convex SDP for quadratic activation networks. The solid lines show the training curve of the non-convex model with SGD (with learning rate tuned offline) and each line corresponds to an independent trial. The dotted horizontal line shows the cost for the convex SDP and the cross indicates the time that it takes to solve the convex SDP. The dataset $X$ is synthetically generated by sampling from the i.i.d. Gaussian distribution and has dimensions $n=100,d=10$. Labels $y$ are generated by a teacher network with $10$ planted neurons. The regularization coefficient is $\beta=10^{-6}$ and the batch size for SGD is $10$.}
\label{fig:converge_to_localminima}
\end{figure}

\begin{figure}
\begin{minipage}[b]{0.48\linewidth}
  \centering
  \centerline{\includegraphics[width=\columnwidth]{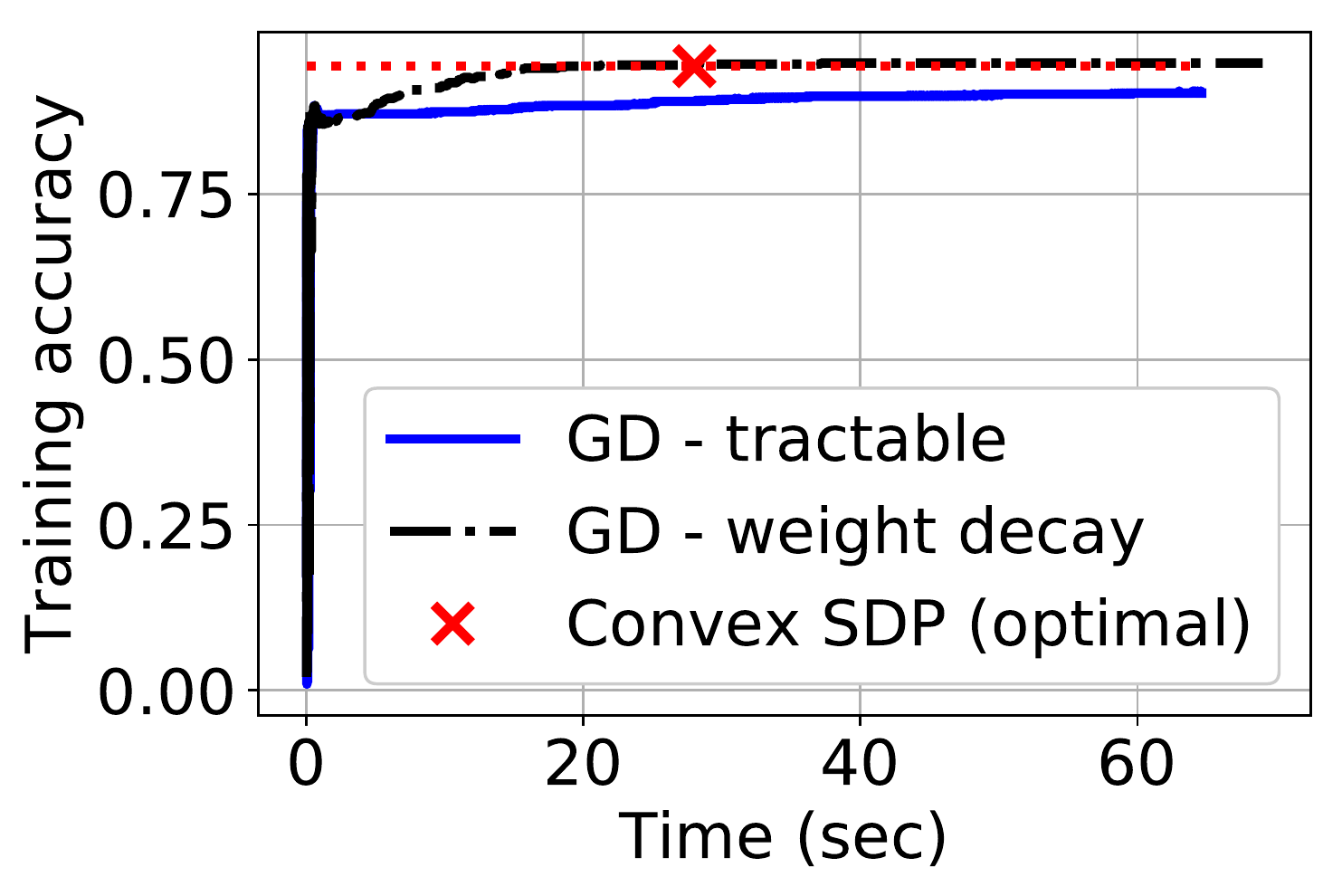}}
\end{minipage}
\hfill
\begin{minipage}[b]{0.48\linewidth}
  \centering
  \centerline{\includegraphics[width=\columnwidth]{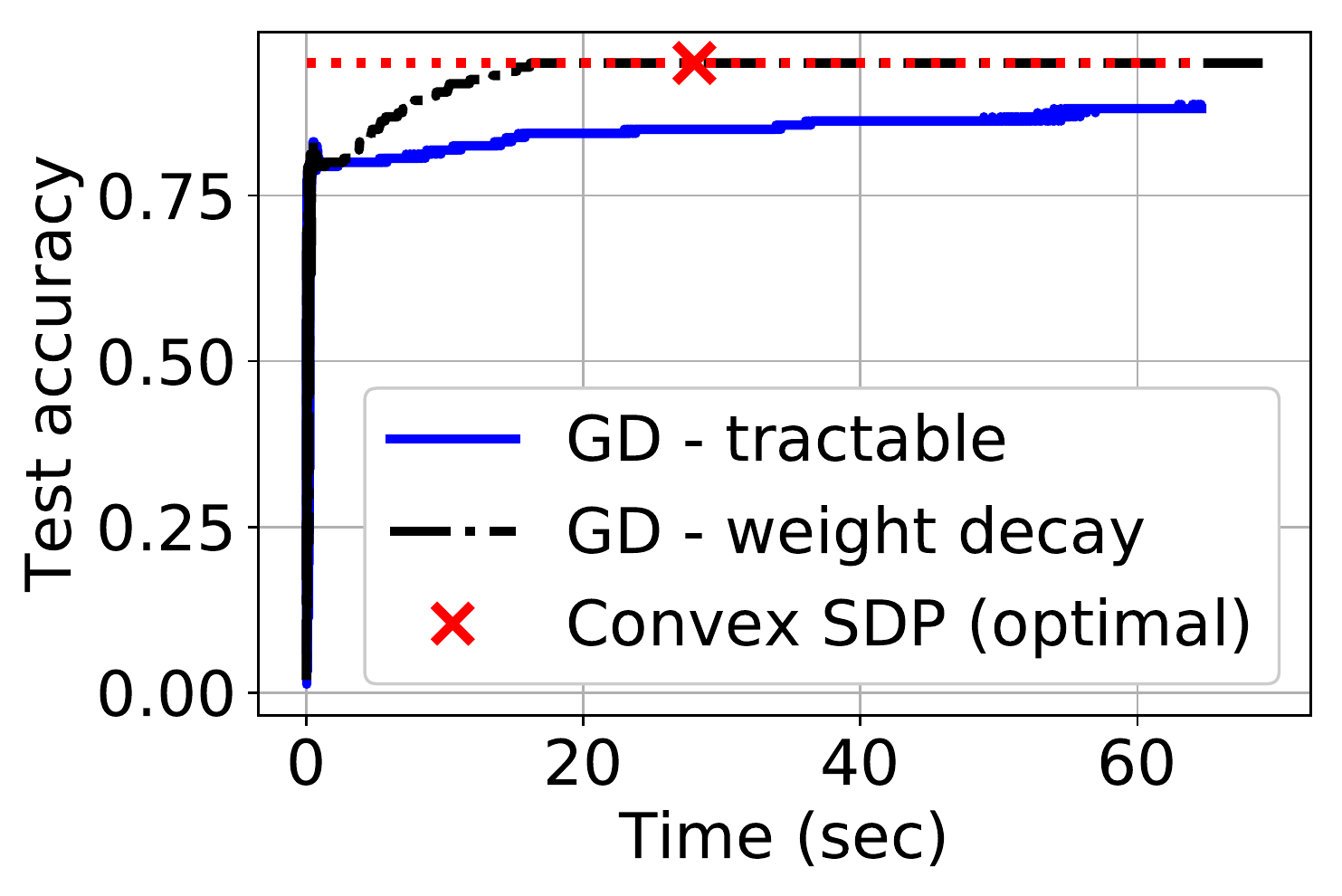}}
\end{minipage}
\caption{Classification accuracy results on the UCI dataset "annealing" ($n=638,d=31$) for polynomial activation networks. This is a multiclass classification dataset with $C=5$ classes. Both training (left) and test (right) set accuracies are shown for the gradient descent (GD) and the convex SDP methods. Legend labels are as follows. \textit{GD - tractable}: The non-convex problem in \eqref{eq:convolutional_nonconvex_avgpool} solved via gradient descent, \textit{GD - weight decay}: Non-convex problem with quadratic regularization on all weights solved via gradient descent, \textit{Convex SDP (optimal)}: The convex problem in \eqref{eq:vector_output_convex_program}. Degree two polynomial activation with coefficients $a=0.09$, $b=0.5$, $c=0.47$ is used. The regularization coefficient is $\beta=1$. The learning rate for GD is optimized offline and only the best performing learning rate is shown. The resulting number of neurons from the convex program is $172$. 
}
\label{fig:uci_multiclass_classification_1}
\end{figure}

Figure \ref{fig:optimizer_and_lr} compares the accuracy of the non-convex polynomial activation model when it is trained with different optimizers (SGD and Adam) for a range of step sizes. Figure \ref{fig:optimizer_and_lr} shows that the convex formulations outperform the non-convex solution via SGD and Adam. The extension of the main result to convolutional neural networks is discussed in Section \ref{sec:convolutional} and \ref{sec:avg_pooling}.

\begin{figure} 
\begin{minipage}[b]{0.42\linewidth}
  \centering
  \centerline{\includegraphics[width=\columnwidth]{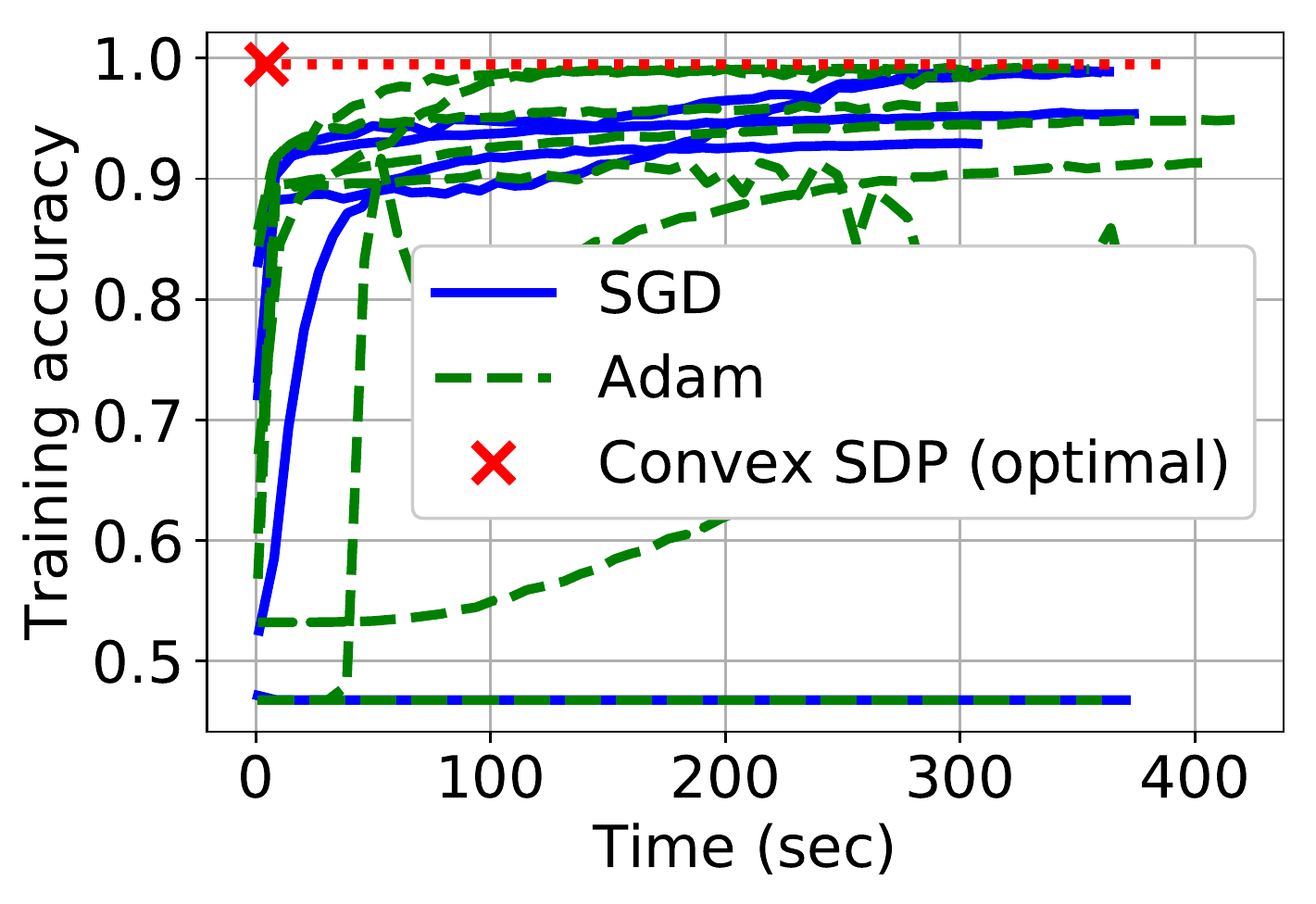}}
  \centerline{(a) CNN, MNIST, training accuracy}\medskip
\end{minipage}
\hfill
\begin{minipage}[b]{0.42\linewidth}
  \centering
  \centerline{\includegraphics[width=\columnwidth]{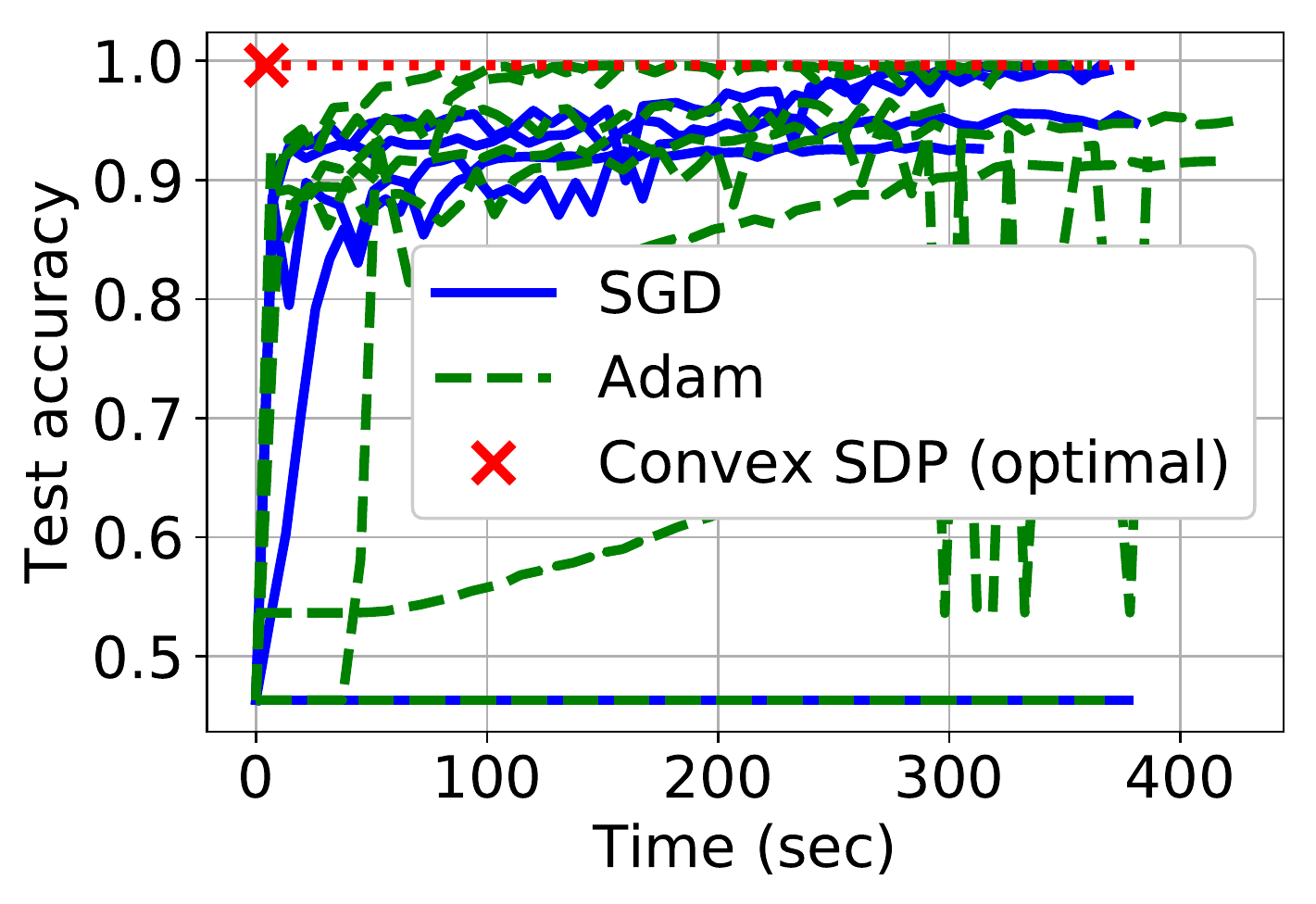}}
  \centerline{(b) CNN, MNIST, test accuracy}\medskip
\end{minipage}

\begin{minipage}[b]{0.42\linewidth}
  \centering
  \centerline{\includegraphics[width=\columnwidth]{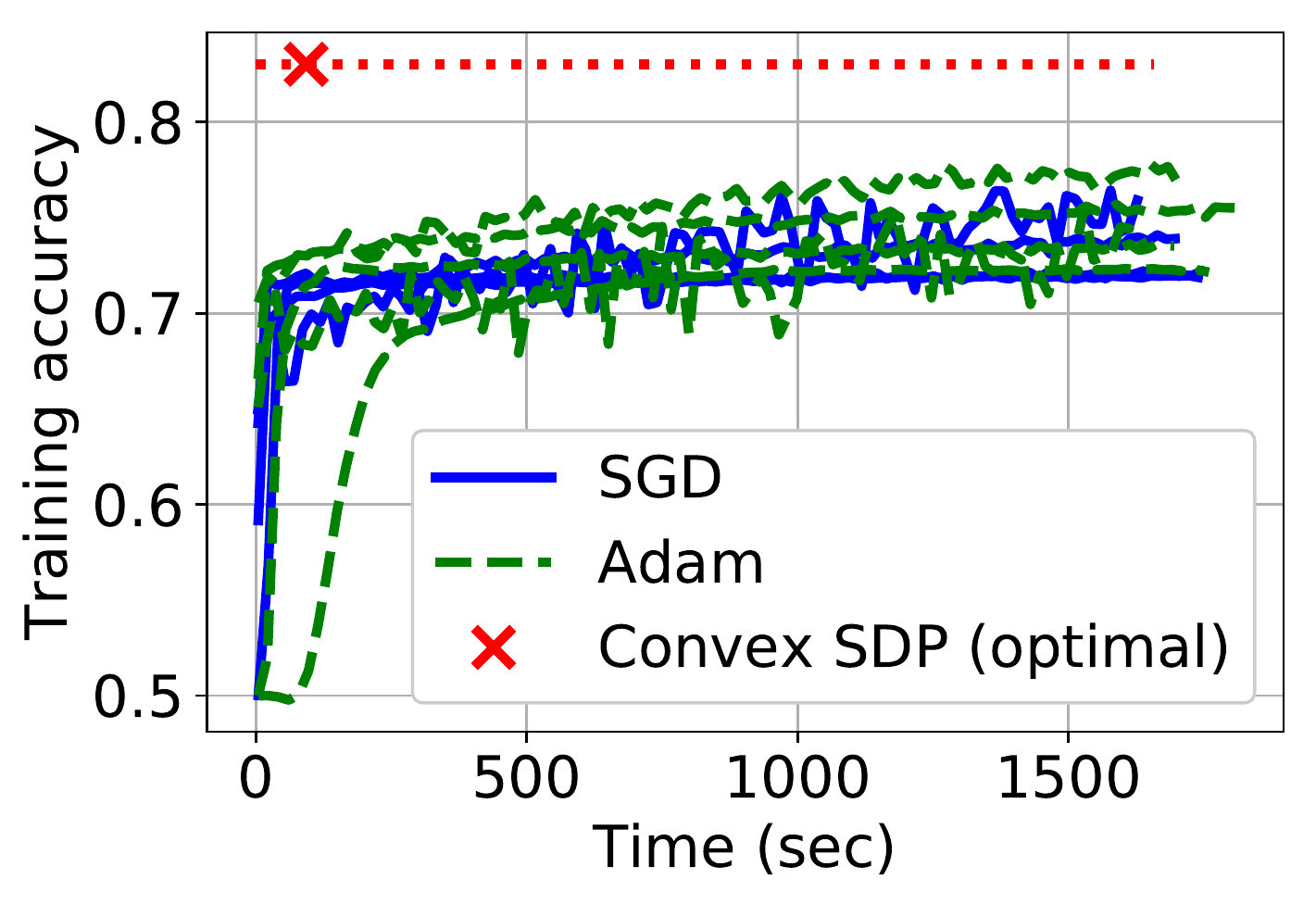}}
  \centerline{(c) CNN, CIFAR, training accuracy}\medskip
\end{minipage}
\hfill
\begin{minipage}[b]{0.42\linewidth}
  \centering
  \centerline{\includegraphics[width=\columnwidth]{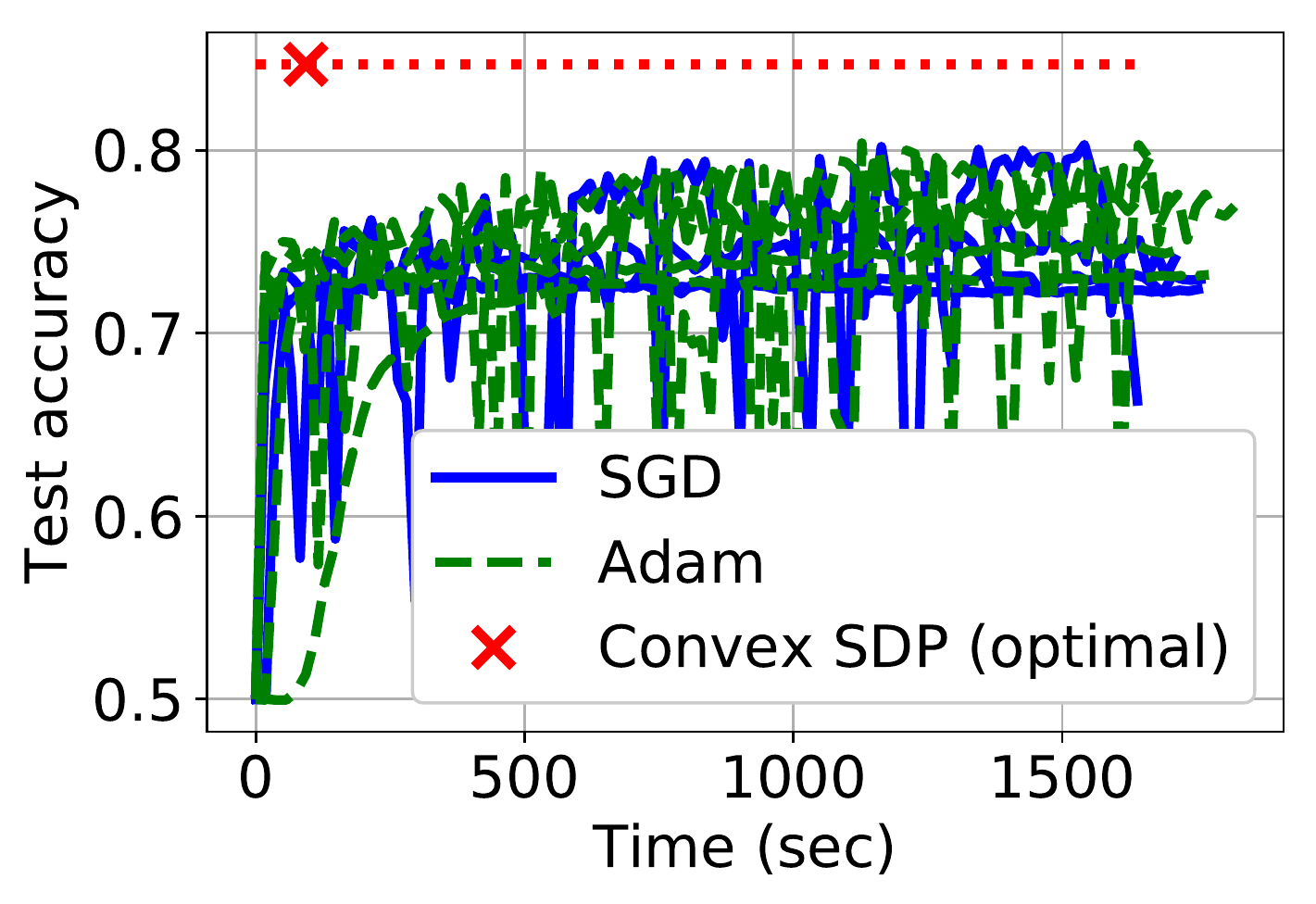}}
  \centerline{(d) CNN, CIFAR, test accuracy}\medskip
\end{minipage}

\begin{minipage}[b]{0.42\linewidth}
  \centering
  \centerline{\includegraphics[width=\columnwidth]{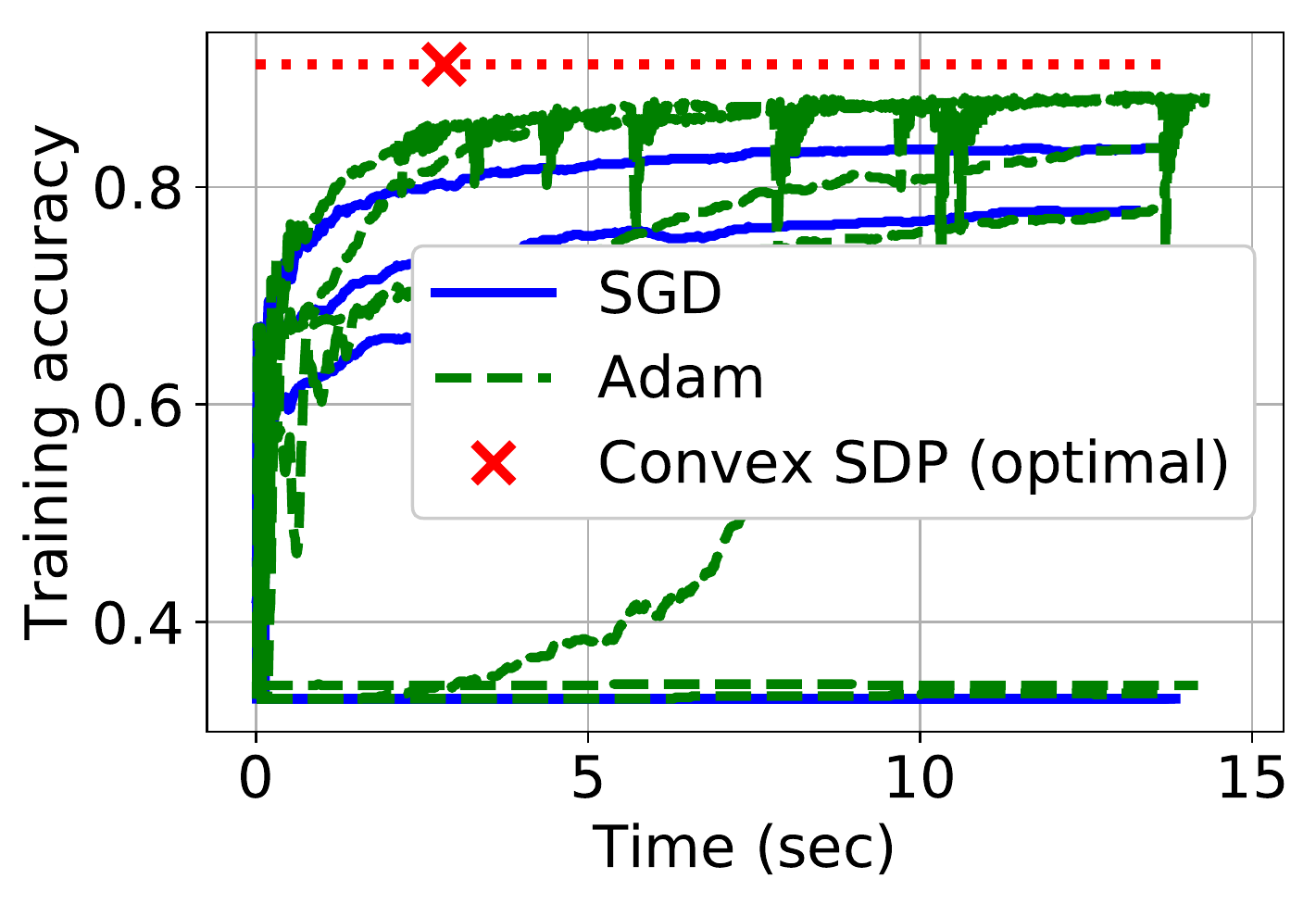}}
  \centerline{(e) Fully connected, oocytes, training accuracy}\medskip
\end{minipage}
\hfill
\begin{minipage}[b]{0.42\linewidth}
  \centering
  \centerline{\includegraphics[width=\columnwidth]{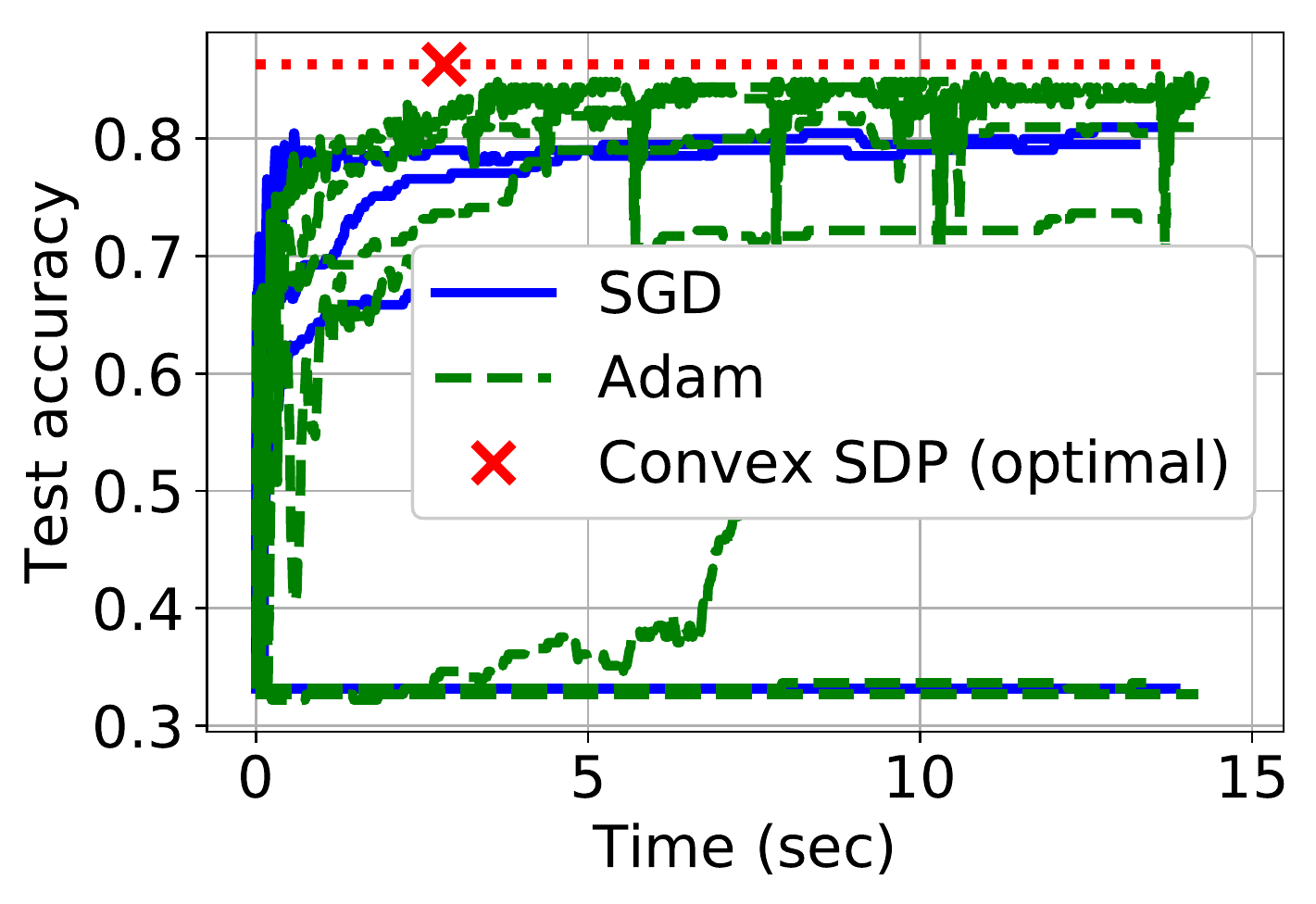}}
  \centerline{(f) Fully connected, oocytes, test accuracy}\medskip
\end{minipage}
\caption{Classification accuracy for various learning rates and optimizers are plotted on the same figure. SGD and Adam are used in solving the non-convex optimization problem. The solid blue lines each correspond to a different learning rate for SGD and each dashed green line corresponds to a different learning rate for the Adam algorithm. Plots a, b: CNN with degree two polynomial activations and global average pooling for binary classification on the first two classes of the MNIST dataset. Plots c, d: The same architecture as plots a, b and the dataset is the first two classes of the CIFAR-10 dataset. Plots e, f: Fully connected architecture for binary classification on the dataset oocytes-merluccius-nucleus-4d.}
\label{fig:optimizer_and_lr}
\end{figure}

\begin{theorem} [Globally optimal convex program for polynomial activation networks] \label{thm:poly_act_thm}
The solution of the convex problem
\begin{align}\label{eq:polyact_convex_program_final}
    \min_{Z=Z^T, Z^\prime={Z^\prime}^T} \, &\ell(\hat{y}, \, y) + \beta (Z_4+Z_4^\prime) \nonumber \\
    \mbox{s.t.} \quad &\hat{y}_i = a x_i^T(Z_1-Z_1^\prime)x_i + b x_i^T(Z_2-Z_2^\prime) + c(Z_4-Z_4^\prime), \quad i \in [n] \nonumber \\ 
    &\tr(Z_1) = Z_4, \, \tr(Z_1^\prime) = Z_4^\prime \nonumber \\
    &Z \succeq 0, \, Z^\prime \succeq 0
\end{align}
provides a global optimal solution for the non-convex problem in \eqref{eq:poly_act_non-convex} when the number of neurons satisfies $m \geq m^*$ where 
\begin{align}
    m^* = \rank(Z^*) + \rank({Z^\prime}^*).
\end{align}
Here $Z^*$ and ${Z^\prime}^*$ denote the solution of \eqref{eq:polyact_convex_program_final}. The variables $Z \in \mathbb{S}^{(d+1)\times (d+1)}$ and $Z^\prime \in \mathbb{S}^{(d+1)\times (d+1)}$ are defined in \eqref{eq:defn_of_Z}.
It follows that the optimal number of neurons is upper bounded by $m^* \leq 2(d+1)$.
\end{theorem}

The proof of Theorem \ref{thm:poly_act_thm} is established in this section and the next. In this section we show that the solution of the convex program \eqref{eq:polyact_convex_program_final} provides a lower bound for the solution of the non-convex problem \eqref{eq:poly_act_non-convex}. In the next section, we prove, via the method of neural decomposition, that the solution of the convex problem provides also an upper bound, which concludes the proof of Theorem \ref{thm:poly_act_thm}.

In proving the lower bound, we leverage duality. Minimizing over first $\alpha_j$'s and then $u_j$'s, we can restate the problem in \eqref{eq:poly_act_non-convex} as
\begin{align}
    p^* = &\min_{\{u_j \}_{j=1}^m \, \mbox{s.t.}\, \|u_j\|_2=1,\,\forall j} \, \min_{\{\alpha_j \}_{j=1}^m, \hat{y}} \ell \left( \hat{y},\, y \right) + \beta \sum_{j=1}^m \vert\alpha_j\vert \quad \mbox{s.t.} \quad \hat{y} = \sum_{j=1}^m \sigma(Xu_j) \alpha_j \,.
\end{align}
The dual problem for the inner minimization problem is given by
\begin{align}
    \max_{v} -\ell^*(-v) \quad \mbox{s.t.} \quad |v^T \sigma(Xu_j) | \leq \beta, \, \forall j \,.
\end{align}
%
Next, let us call the optimal solution of the following problem $d^*$
\begin{align}
    d^* = \min_{\{u_j \}_{j=1}^m \, \mbox{s.t.}\, \|u_j\|_2=1,\,\forall j} \, \max_{|v^T \sigma(Xu_j)| \leq \beta \,, \forall j} -\ell^*(-v).
\end{align}
By changing the order of the minimization and maximization operations, we obtain the following bound
\begin{align} \label{eq:max_problem_initial}
    d^* \geq \max_{|v^T \sigma(Xu_j)| \leq \beta \,, \|u_j\|_2=1,\,\forall j} -\ell^*(-v).
\end{align}
We note that the constraints $|v^T \sigma(Xu_j)| \leq \beta$ can equivalently be written as two quadratic (in $u_j$) inequalities for each $j=1,\dots,m$,
\begin{align}
    u_j^T \left(a \sum_{i=1}^n x_ix_i^Tv_i \right) u_j  + bv^T Xu_j + cv^T \ones \leq \beta, \, \, -u_j^T \left(a \sum_{i=1}^n x_ix_i^Tv_i \right) u_j - bv^T Xu_j - cv^T \ones \leq \beta.
\end{align}
Next, we use the S-procedure given in Corollary \ref{l:S_lemma_w_equality} to reformulate the quadratic inequality constraints as linear matrix inequality constraints. Corollary \ref{l:S_lemma_w_equality} is based on Lemma \ref{l:S_lemma_w_equality_orig} which characterizes the solvability of a quadratic system. The proof of Corollary \ref{l:S_lemma_w_equality} is given in the appendix.

\begin{lemma} [Proposition 3.1 from \cite{polik07slemma}] \label{l:S_lemma_w_equality_orig}
Let $f_1$ and $f_2$ be quadratic functions where $f_2$ is strictly concave (or strictly convex) and assume that $f_2$ takes both positive and negative values. Then, the following two statements are equivalent:
\begin{enumerate}
    \item $f_1(u) < 0, f_2(u)=0$ is not solvable.
    \item There exists $\lambda \in \mathbb{R}$ such that $f_1(u) + \lambda f_2(u) \geq 0$, $\forall u$.
\end{enumerate}
\end{lemma}

\begin{corollary} [S-procedure with equality] \label{l:S_lemma_w_equality}
$\max_{\|u\|_2 = 1} u^TQu + b^Tu \leq \beta$ if and only if there exists $\lambda \in \mathbb{R}$ such that
\begin{align*}
    \begin{bmatrix}
\lambda I - Q & -\frac{1}{2} b \\ -\frac{1}{2} b^T & \beta - \lambda
\end{bmatrix} \succeq 0.
\end{align*}
\end{corollary}

Corollary \ref{l:S_lemma_w_equality} allows us to write the maximization problem in \eqref{eq:max_problem_initial} as the equivalent problem given by
\begin{align} \label{eq:max_problem_lmi}
    \max & -\ell^*(-v) \nonumber \\
    \text{s.t. } 
    & \begin{bmatrix} \rho_1 I - a \sum_{i=1}^n x_ix_i^Tv_i & -\frac{1}{2}bX^Tv \\ -\frac{1}{2}bv^TX & \beta - c\ones^Tv - \rho_1 \end{bmatrix} \succeq 0 \nonumber \\
    &\begin{bmatrix} \rho_2 I + a \sum_{i=1}^n x_ix_i^Tv_i & \frac{1}{2}bX^Tv \\ \frac{1}{2}bv^TX & \beta + c\ones^Tv - \rho_2 \end{bmatrix} \succeq 0 \,,
\end{align}
where we note the two additional variables $\rho_1, \rho_2 \in \mathbb{R}$ are introduced.
Next, we will find the dual of the problem in \eqref{eq:max_problem_lmi}. Let us first define the following Lagrange multipliers
\begin{align} \label{eq:defn_of_Z}
    Z = \begin{bmatrix} Z_1 & Z_2 \\ Z_3 & Z_4  \end{bmatrix}, \quad Z^\prime = \begin{bmatrix} Z_1^\prime & Z_2^\prime \\ Z_3^\prime & Z_4^\prime \end{bmatrix},
\end{align}
where $Z, Z^\prime \in \mathbb{S}^{(d+1)\times (d+1)}$ are symmetric matrices, and the dimensions for each block matrix are $Z_1,Z_1^\prime \in \mathbb{S}^{d\times d}$, $Z_2,Z_2^\prime \in \mathbb{R}^{d\times 1}$, $Z_3,Z_3^\prime \in \mathbb{R}^{1\times d}$, $Z_4,Z_4^\prime \in \mathbb{R}^{1\times 1}$. We note that because of the symmetry of $Z$ and $Z^\prime$, we have $Z_2^T=Z_3$ and $Z_2{^\prime}^T = Z_3^\prime$. 
The Lagrangian for the problem in \eqref{eq:max_problem_lmi} is
\begin{align}
    L(v, \rho_1, \rho_2, Z, Z^\prime) &= -\ell^*(-v) + \rho_1 \tr(Z_1) +\rho_2 \tr(Z_1^\prime) - a\sum_{i=1}^n v_ix_i^T(Z_1-Z_1^\prime)x_i - bv^TX(Z_2-Z_2^\prime) + \nonumber \\
    &+ (\beta-\rho_1) Z_4 + (\beta-\rho_2) Z_4^\prime - c \sum_{i=1}^n v_i (Z_4-Z_4^\prime).
\end{align}
Maximizing the Lagrangian with respect to $v$, $\rho_1$, $\rho_2$, we obtain the problem in \eqref{eq:polyact_convex_program_final}, which concludes the lower bound part of the proof.
In the next section, we introduce a method for decomposing the solution of this convex program (i.e. $Z^*$ and ${Z^\prime}^*$) into feasible neural network weights to prove the upper bound.

\section{Neural Decomposition} \label{sec:neural_decomp}

We have shown that a lower bound on the optimal value of the non-convex problem in \eqref{eq:poly_act_non-convex} is obtained via the solution of the convex program in \eqref{eq:polyact_convex_program_final} that we have derived using Lagrangian duality. Now we show that this lower bound is in fact identical to the optimal value of the non-convex problem, thus proving strong duality. Our approach is based on proving an upper bound by constructing neural network weights from the solution of the convex problem such that the convex objective achieves the same objective as the non-convex objective. Suppose that $(Z^*, {Z^\prime}^*)$ is a solution to \eqref{eq:polyact_convex_program_final}. Let us denote the rank of $Z^*$ by $r$ and the rank of ${Z^\prime}^*$ by $r^\prime$. We will discuss the decomposition for $Z^*$ and then complete the picture by considering the same decomposition for ${Z^\prime}^*$. We begin by noting that $Z^*$ satisfies the constraints of \eqref{eq:polyact_convex_program_final}, i.e.,
\begin{align}
     \quad Z^* \succeq 0 \,\, \mbox{and} \,\, \tr(Z_1^*) = Z_4^*, \,\, \mbox{or equivalently} \,\, \tr \Bigl(Z^* \underbrace{\begin{bmatrix} I_d & 0 \\ 0 & -1 \end{bmatrix}}_{G} \Bigr) = 0.
\end{align}
%
Suppose that we have a decomposition of $Z^*$ as a sum of rank-1 matrices such that $Z^*=\sum_{j=1}^r p_j p_j^T$ where $p_j \in \mathbb{R}^{d+1}$ and $\tr(p_jp_j^TG) = p_j^TGp_j = 0$ for $j=1,\dots,r$. We show how this can always be done in subsection \ref{s:decomp_proc} by introducing a new matrix decomposition method, dubbed the \emph{neural decomposition} procedure.
%

Letting $p_j := \begin{bmatrix} c_j^T & d_j \end{bmatrix}^T$ with $c_j \in \mathbb{R}^d$ and $d_j \in \mathbb{R}$, we note that $p_j^T G p_j = 0$ implies $\|c_j\|_2^2 = d_j^2$. We may assume $p_j\neq 0,\,\forall j$ in the decomposition (otherwise we can simply remove zero components), implying $\|c_j\|_2^2>0,\,\forall j$. Furthermore, this expression for $p_j$'s allows us to establish that
\begin{align}
    \sum_{j=1}^r p_j p_j^T &= \sum_{j=1}^r \begin{bmatrix} c_j \\ d_j \end{bmatrix} \begin{bmatrix} c_j^T & d_j \end{bmatrix} = \sum_{j=1}^r \begin{bmatrix} c_jc_j^T & c_jd_j \\ d_j c_j^T & d_j^2 \end{bmatrix} = \begin{bmatrix} Z_1^* & Z_2^* \\ Z_3^* & Z_4^* \end{bmatrix}.
\end{align}
As a result, we have the following decompositions:
\begin{align}
    Z_1^* &= \sum_{j=1}^r c_jc_j^T = \sum_{j=1}^r u_j u_j^T \|c_j\|_2^2 = \sum_{j=1}^r u_j u_j^T d_j^2 \\
    Z_2^* &= \sum_{j=1}^r c_jd_j = \sum_{j=1}^r u_j d_j \|c_j\|_2 = \sum_{j=1}^r u_j d_j |d_j| \\
    Z_4^* &= \sum_{j=1}^r d_j^2 \,,
\end{align}
where we have introduced the normalized weights $u_j = \frac{c_j}{\|c_j\|_2}$, $j=1,\dots,r$.
If $d_j \leq 0$ for some $j$, we redefine the corresponding $p_j$ as $p_j\leftarrow -p_j$, which does not modify the decomposition $\sum_j p_j p_j^T$ and the equality $p_j^TGp_j=0$. Hence, without loss of generality, we can assume that $d_j \geq 0$ for all $j=1,\dots,r$, which leads to
\begin{align} \label{eq:Z_star_decomp}
    Z_1^* = \sum_{j=1}^r u_j u_j^T d_j^2, \quad Z_2^* = \sum_{j=1}^r u_j d_j^2, \quad Z_4^* &= \sum_{j=1}^r d_j^2.
\end{align}
Similarly for ${Z^\prime}^*$, we will form the following decompositions:
\begin{align} \label{eq:Z_prime_star_decomp}
    {Z_1^\prime}^* = \sum_{j=1}^{r^\prime} u_j^\prime  {u_j^\prime}^T {d_j^\prime}^2, \quad {Z_2^\prime}^* = \sum_{j=1}^{r^\prime} u_j^\prime {d_j^\prime}^2, \quad {Z_4^\prime}^* &= \sum_{j=1}^{r^\prime} {d_j^\prime}^2.
\end{align}
Considering the decompositions for both $Z^*$ and ${Z^\prime}^*$, finally we obtain a neural network with first layer weights as $\{u_1,\dots,u_r,u_1^\prime,\dots,u_{r^\prime}^\prime\}$, and second layer weights as $\{d_1^2,\dots,d_r^2,-{d_1^\prime}^2,\dots, -{d_{r^\prime}^\prime}^2\}$. We note that this corresponds to a neural network with $r+r^\prime$ neurons. If both $Z^*$ and ${Z^\prime}^*$ are full rank, then we will have $2(d+1)$ neurons, which is the maximum.

To see why we can use the decompositions of $Z^*$ and ${Z^\prime}^*$ to construct neural network weights, we plug-in the expressions \eqref{eq:Z_star_decomp} and \eqref{eq:Z_prime_star_decomp} in the objective of the convex program in \eqref{eq:polyact_convex_program_final}:
\begin{align}
    \ell(\hat{y}, y) + \beta \biggl(\sum_{j=1}^r |d_j^2| + \sum_{j=1}^{r^\prime} |-{d_j^\prime}^2| \biggr), \quad \mbox{where} \quad 
    \hat{y}_i = ax_i^T \biggl(\sum_{j=1}^r u_ju_j^Td_j^2 + \sum_{j=1}^{r^\prime} u_j^\prime {u_j^\prime}^T(-{d_j^\prime}^2) \biggr) x_i + \nonumber \\
    + bx_i^T\biggl(\sum_{j=1}^r u_j d_j^2 + \sum_{j=1}^{r^\prime} u_j^\prime (-{d_j^\prime}^2) \biggr) + c\biggl(\sum_{j=1}^rd_j^2 + \sum_{j=1}^{r^\prime} (-{d_j^\prime}^2) \biggr), \quad i=1,\dots,n \,.
\end{align}
%
We note that this expression exactly matches the optimal value of the non-convex objective in \eqref{eq:poly_act_non-convex} for a neural network with $r+r^\prime$ neurons. Also, the unit norm constraints on the first layer weights are satisfied (hence feasible) since $u_j$'s and $u_j^\prime$'s are normalized. This establishes that the neural network weights obtained from the solution of the convex program provide an upper bound for the minimum value of the original non-convex problem. Consequently, we have shown that the optimal solution of the convex problem \eqref{eq:polyact_convex_program_final} provides a global optimal solution to the non-convex problem \eqref{eq:poly_act_non-convex} and this concludes the proof of Theorem \ref{thm:poly_act_thm}.

\subsection{Neural Decomposition Procedure} \label{s:decomp_proc}
Here we describe the procedure for computing the decomposition $Z^*=\sum_{j=1}^r p_j p_j^T \succeq 0$ such that $p_j^TGp_j = 0$, $j=1,\dots,r$.
This algorithm is inspired by the constructive proof of the S-procedure given in Lemma 2.4 of \cite{polik07slemma} with modifications to account for the equalities $p_j^TGp_j = 0$.

\fbox{\parbox{0.95\textwidth}{
\begin{enumerate}
    \setcounter{enumi}{-1}
    \item[] \textbf{Neural Decomposition for Symmetric Matrices:}
    \item \textbf{Compute a rank-1 decomposition $Z^*=\sum_{j=1}^r p_jp_j^T$.} 
    
    This can be done with the eigenvalue decomposition $Z^*=\sum_{j=1}^r q_jq_j^T\lambda_j$. Since $Z^* \succeq 0$, we have $\lambda_j > 0$, for $j=1,\dots,r$. Then we can obtain the desired rank-1 decomposition $Z^*=\sum_{j=1}^r p_jp_j^T$ by defining $p_j=\sqrt{\lambda_j} q_j$, $j=1,\dots,r$.

    \item \textbf{If $p_1^T G p_1 = 0$, return $y=p_1$. If not, find a $j \in \{2,\dots,r\}$ such that $(p_1^T G p_1) (p_j^T G p_j) < 0$.} 
    
    We know such $j$ exists since $\tr(Z^*G)=\sum_{j=1}^r p_j^TGp_j=0$ (this is true since it is one of the constraints of the convex program), and $p_1^T G p_1 \neq 0$. Hence, for at least one $j \in \{2,\dots,r\}$, $p_j^T G p_j$ must have the opposite sign as $p_1^T G p_1$.

    \item \textbf{Return $y=\frac{p_1+\alpha p_j}{\sqrt{1+\alpha^2}}$ where $\alpha \in \mathbb{R}$ satisfies $(p_1+\alpha p_j)^TG(p_1+\alpha p_j)=0$.}
    
    We know that such $\alpha$ exists since the quadratic equation
    \begin{align}
        (p_1+\alpha p_j)^TG(p_1+\alpha p_j) = \alpha^2 p_j^TGp_j + 2\alpha p_1^T p_j + p_1^TGp_1 = 0
    \end{align}
    has real solutions since the discriminant $4(p_1^T p_j)^2 - 4 (p_1^TGp_1)(p_j^TGp_j)$ is positive due to step 1 where we picked $j$ such that $(p_1^TGp_1)(p_j^TGp_j) < 0$. To find $\alpha$, we simply solve the quadratic equation for $\alpha$.

    \item \textbf{Update $r \leftarrow r-1$, and then the vectors $p_1,\dots,p_r$ as follows:}
    
    Remove $p_1$ and $p_j$ and insert $u=\frac{p_j - \alpha p_1}{\sqrt{1+\alpha^2}}$. Consequently, we will be dealing with the updated matrix $Z^* \leftarrow Z^*-yy^T$ in the next iteration, which is of rank $r-1$:
    \begin{align}
        Z^* - yy^T = uu^T + \sum_{i = 2, i\neq j}^r p_ip_i^T.
    \end{align}
\end{enumerate}
}}

\,\\
\noindent Note that Step 0 is carried out only once and then steps 1 through 3 are repeated $r-1$ times. At the end of $r-1$ iterations, we are left with the rank-1 matrix $p_1p_1^T$ which satisfies $p_1^TGp_1=0$ since initial $Z^*$ satisfies $\tr(Z^*G)=0$ and the following $r-1$ updates are of the form $yy^T$ which satisfies $y^TGy=0$.
If we denote the returned $y$ vectors as $y_i$ for the iteration $i$ and $y_r$ is the last one we are left with, then $y_i$'s satisfy the desired decomposition that $Z^* = \sum_{i=1}^r y_iy_i^T$ and $y_i^TGy_i=0$, $i=1,\dots,r$.

\begin{figure} 
\centering
\begin{minipage}[b]{0.8\linewidth}
  \centering
  \centerline{\includegraphics[width=0.8\columnwidth]{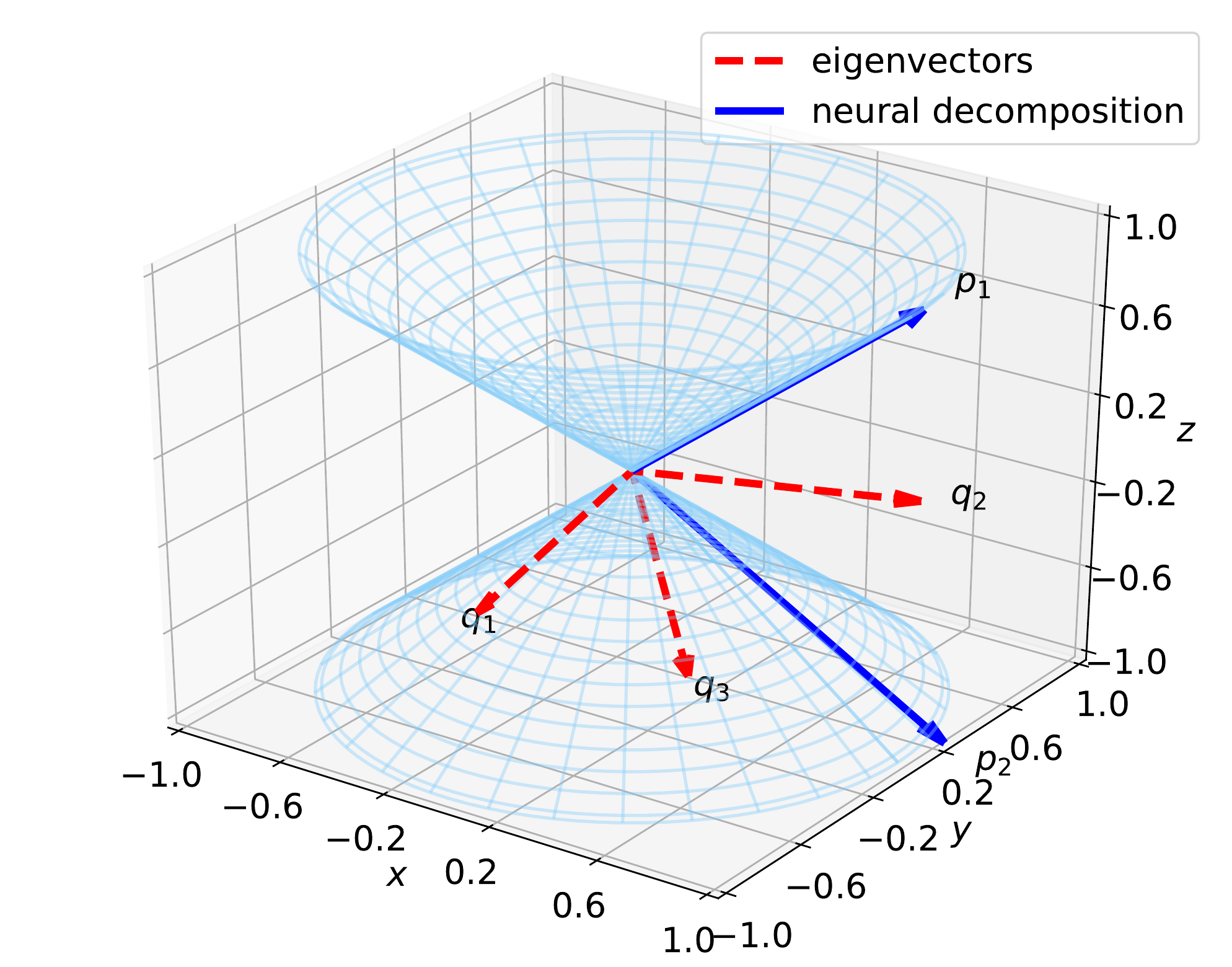}}
\end{minipage}
\caption{Illustration of the neural decomposition procedure for $d=2$ (i.e. $Z^* \in \mathbb{R}^{3\times 3})$. The dashed red arrows correspond to the eigenvectors of $Z^*$ ($q_1,q_2,q_3$) and the solid blue arrows show the decomposed vectors $p_1$ and $p_2$. In this example, the rank of $Z^*$ is $2$ where $q_1$ and $q_2$ are its two principal eigenvectors. The eigenvalue corresponding to the eigenvector $q_1$ is zero. The light blue colored surface shows the Lorentz cones $z=\sqrt{x^2+y^2}$ and $z=-\sqrt{x^2+y^2}$. We observe that the decomposed vectors $p_1$ and $p_2$ lie on the boundary of Lorentz cones. 
}
\label{fig:neural_decomp_procedure}
\end{figure}


Figure \ref{fig:neural_decomp_procedure} is an illustration of the neural decomposition procedure for a toy example with $d=2$ where the eigenvectors of $Z^*$ and the vectors $p_j$ are plotted together. Due to the constraints $p_j^TGp_j=0$, $j=1,2$, the vectors $p_j$ have to lie on the boundary of Lorentz cones\footnote{In special relativity, Lorentz cones describe the path that a flash of light, emanating from a single event traveling in all directions takes through spacetime (see Figure 1.3.1 in \cite{naber2012geometry}).} $z=\sqrt{x^2+y^2}$ and $z=-\sqrt{x^2+y^2}$. Decomposing the solution of the convex problem $Z^*$ and ${Z^\prime}^*$ onto these cones, i.e., neural decomposition, enables the construction of neural network weights from $Z^*$ and ${Z^\prime}^*$.


\section{Quadratic Activation Networks} \label{sec:quadratic_act}

In this section, we derive the corresponding convex program when the activation function is quadratic, i.e., $\sigma(u)=u^2$. 
The resulting convex problem takes a simpler form than the polynomial activation case. We start by noting that the bound in \eqref{eq:max_problem_initial} holds for any activation function. The inequalities $|v^T\sigma(Xu_j)| \leq \beta$ however lead to different constraints than the polynomial activation case. Note that $|v^T(Xu_j)^2| \leq \beta$ is equivalent to the inequalities
%
%
%
\begin{align}
      u_j^T \left(\sum_{i=1}^n x_i x_i^Tv_i \right) u_j \le \beta \quad\mbox{and} \quad  u_j^T \left(-\sum_{i=1}^n x_i x_i^Tv_i \right) u_j \le \beta\,.
\end{align}
The constraint $\max_{u_j:\|u_j\|_2=1}|v^T(Xu_j)^2| \le \beta$ can be expressed as largest eigenvalue inequalities
\begin{align}
      \lambda_{\max} \left(\sum_{i=1}^n x_i x_i^Tv_i \right) \le \beta \quad\mbox{and}\quad   \lambda_{\max}  \left(-\sum_{i=1}^n x_i x_i^Tv_i \right) \le \beta\,,
\end{align}
where $\lambda_{\max}$ denotes the maximum eigenvalue. Next, representing the largest eigenvalue constraints as linear matrix inequality constraints, we arrive at the following maximization problem
\begin{align} \label{eq:quadact_dual_LMI}
    \max_v \quad & -\ell^*(-v) \nonumber \\ 
    \text{s.t.} \quad &\sum_{i=1}^n x_ix_i^T v_i - \beta I_d \preceq 0, \quad  -\sum_{i=1}^n x_ix_i^T v_i - \beta I_d \preceq 0.
\end{align}

Writing the Lagrangian for \eqref{eq:quadact_dual_LMI} as
$
    L(v, Z_1, Z_2) = -\ell^*(-v) - \sum_{i=1}^n v_i x_i^T (Z_1 - Z_2)x_i + \beta \tr(Z_1 + Z_2)
$
with $Z_1,Z_2 \in \mathbb{S}^{d\times d}$
and maximizing with respect to $v$, we obtain the following convex problem
\begin{align}
    \min_{Z_1, Z_2 \succeq 0} \ell \left(\begin{bmatrix} x_1^T (Z_1 - Z_2)x_1 & \hdots & x_n^T (Z_1 - Z_2)x_n \end{bmatrix}^T, \, y \right) + \beta \tr(Z_1 + Z_2) \,.
\end{align}
Replacing $Z=Z_1-Z_2$, where $Z_1\succeq 0, Z_2\succeq 0$, we recall that any matrix $Z$ can be uniquely decomposed in this form thanks to the Moreau decomposition onto the cone of positive definite matrices and its polar dual, which is the set of negative semidefinite matrices. In particular, suppose that the eigenvalue decomposition of $Z$ is $Z=\sum_j \lambda_j z_j z_j^T$. Then, $Z_1$ and $Z_2$ are uniquely determined by $Z_1=\sum_{j:\lambda_j>0} \lambda_j z_jz_j^T$ and $Z_2=-\sum_{j:\lambda_j<0} \lambda_j z_jz_j^T$. Note that $\tr (Z_1 + Z_2) = \sum_{j:\lambda_j>0}\lambda_j+\sum_{j:\lambda_j<0}(-\lambda_j)=\sum_j |\lambda_j|=\|Z\|_*$ is the sum of the absolute values of the eigenvalues of $Z$, which is equivalent to the nuclear norm for symmetric matrices.
Consequently, this leads to the following simplified problem with nuclear norm regularization:
\begin{align} \label{eq:convex_problem_quad_act}
    \min_{Z=Z^T} \quad &\ell(\hat{y}, y) + \beta \|Z\|_* \nonumber \\
    \mbox{s.t.} \quad & \hat{y}_i = x_i^TZx_i, \quad i=1,\dots,n \,.
\end{align}

Theorem \ref{thm:quad_act_thm} states the main result for the global optimization of quadratic activation neural networks. The rest of this section is devoted to the proof and interpretation of Theorem \ref{thm:quad_act_thm}.

\begin{theorem} [Globally optimal convex program for quadratic activation cubic regularization networks] \label{thm:quad_act_thm}
The solution of the convex problem in \eqref{eq:convex_problem_quad_act} provides a global optimal solution to the non-convex problem for quadratic activation and cubic regularization given in \eqref{eq:cubic_reg_quad_act_problem} when the number of neurons satisfies $m \geq m^*$ where 
\begin{align}
    m^* = \rank(Z^*).
\end{align}
The optimal neural network weights are determined from the solution of the convex problem via eigenvalue decomposition of $Z^*$ and the rescaling given in \eqref{eq:rescaling_quad_act}. The optimal number of neurons is upper bounded by $m^* \leq d$ since $\rank(Z^*) \leq d$.
\end{theorem}

\subsection{Strong Duality for Quadratic Activation}
We have shown that a lower bound on the non-convex problem for quadratic activation is given by the nuclear norm regularized convex objective. Now we show that this lower bound is in fact identical to the non-convex problem. Suppose that $Z^*$ is a solution to \eqref{eq:convex_problem_quad_act}. Let us decompose $Z^*$ via eigenvalue decomposition as $Z^* = \sum_j \lambda_j z_j z_j^T$. We can generate an upper bound on the non-convex problem by constructing neural network parameters as $\alpha_j=\lambda_j$, and $u_j=z_j$ with objective value $\ell\left(\sum_j(Xz_j)^2\lambda_j, y\right) + \beta \sum_j |\lambda_j|$. Noting that this value exactly matches the optimal value of the convex objective in \eqref{eq:convex_problem_quad_act}, we conclude that the optimal solution of \eqref{eq:convex_problem_quad_act} provides a global optimal solution to the non-convex problem.

\subsection{Equivalent Non-convex Problem: Quadratic Activation with Cubic Regularization}
We now show that the non-convex problem with unit norm first layer weights and the $\ell_1$ norm regularized second layer weights is in fact equivalent to the non-convex problem with cubic regularization on all the weights. Let us consider the \emph{unconstrained} problem with cubic regularization:
\begin{align} \label{eq:cubic_reg_quad_act_problem}
    p^* := &\min_{\{\alpha_j, u_j\}_{j=1}^m} \ell\left(\sum_{j=1}^m(Xu_j)^2\alpha_j,\, y\right) + \frac{\beta}{c} \sum_{j=1}^m (|\alpha_j|^3 + \|u_j\|_2^3)\,,
\end{align}
where $c = 2^\frac{1}{3} + 2^{-\frac{2}{3}} \approx 1.88988$. 
Rescaling the variables $u_j\leftarrow u_j t_j^{1/2}$ and $\alpha_j\leftarrow \alpha_j/t_j\,,\forall j$ for $t_j>0$, $j=1,\dots,m$ yields
\begin{align}
    p^* = &\min_{\{\alpha_j, u_j\}_{j=1}^m} \ell\left(\sum_{j=1}^m(Xu_j)^2\alpha_j,\, y\right) + \frac{\beta}{c} \sum_{j=1}^m (|\alpha_j|^3/{t_j^{3}} + \|u_j\|_2^3 t_j^{3/2}) \,.
\end{align}
Noting the regularization term is convex in $t_j$ for $t_j > 0$ and optimizing it with respect to $t_j$, we obtain $t_j = 2^{2/9} \left( \frac{|\alpha_j|}{\|u_j\|_2} \right)^{2/3} $. Plugging the expression for $t_j$ in yields 
\begin{align}
    p^* = &\min_{\{\alpha_j, u_j\}_{j=1}^m} \ell\left(\sum_{j=1}^m(Xu_j)^2\alpha_j,\, y\right) + \beta \sum_{j=1}^m \vert\alpha_j\vert \|u_j\|_2^2\,.
\end{align}
Now we define the scaled second layer weights $\alpha_j^\prime = \alpha_j \|u_j\|_2^2$. Noting that $(Xu_j)^2\alpha_j=(X\frac{u_j}{\|u_j\|_2})^2\alpha_j^\prime$ and defining $u_j^\prime = u_j/\|u_j\|_2$, we obtain the equivalent problem with the $\ell_1$ norm of the second layer weights as the regularization term
\begin{align} \label{eq:l1_reg_form_quad_act}
    p^* = &\min_{\{\alpha_j^\prime,\, u_j^\prime\}_{j=1}^m,\, \mbox{s.t.} \|u_j^\prime\|_2=1,\,\forall j} \ell\left(\sum_{j=1}^m(Xu_j^\prime)^2\alpha_j^\prime,\, y\right) + \beta \sum_{j=1}^m \vert\alpha_j^\prime\vert \,.
\end{align}

\subsubsection{Rescaling}
We note that the weights $\alpha_j$ and $u_j$ that the eigenvalue decomposition of the solution of \eqref{eq:convex_problem_quad_act} gives are scaled versions of the weights of the problem with cubic regularization in \eqref{eq:cubic_reg_quad_act_problem}. The solution to the problem in \eqref{eq:cubic_reg_quad_act_problem} can be constructed by rescaling the weights as
\begin{align} \label{eq:rescaling_quad_act}
    u_j \leftarrow u_j \sqrt{t_j^\prime}, \quad \alpha_j \leftarrow \frac{\alpha_j}{t_j^\prime}  \quad, \text{where } \quad t_j^\prime = 2^{2/9} |\alpha_j |^{2/3} \quad j=1,\dots,m.
\end{align}
This concludes the proof of Theorem \ref{thm:quad_act_thm}.

\subsection{Comparison with Polynomial Activation Networks}
In this subsection, we list the important differences between the results for quadratic activation and polynomial activation neural networks. The convex program for the quadratic activation network does not have the equality constraints that appear in the convex program for the polynomial activation. In addition, for the quadratic activation, the upper bound on the critical width $m^*$ is $d$ while it is $2(d+1)$ for the polynomial activation case.

We note that in the case of quadratic activation, the optimal neural network weights are determined from eigenvalue decomposition of $Z^*$. This results in the first layer weights to be orthonormal because they can be chosen as the eigenvectors of the real and symmetric matrix $Z^*$. In contrast, we do not have this property for polynomial activations as the associated optimal weights are determined via neural decomposition. In this case, the resulting hidden neurons are not necessarily orthogonal, which shows that the Neural Decomposition is a type of non-orthogonal matrix decomposition. This can also be seen in Figure \ref{fig:neural_decomp_procedure}.

\subsection{Constructing Multiple Globally Optimal Solutions in the Neural Network Parameter Space} \label{sec:construct_opt}
%
%
Once we find an optimal $Z^*$ using the SDP in \eqref{eq:convex_problem_quad_act}, we can transform it to the neural network parameter space with at most $d$ neurons using the eigenvalue decomposition of $Z^*$ as $Z^*=\sum_{j=1}^d u_ju_j^T \alpha_j$. However, we can also generate a neural network with an arbitrary number of neurons, which is also optimal. We now describe this construction below for an arbitrary number of neurons $m\ge 2d$.
Let us pick an arbitrary $m/2\times d$ matrix $H$ with orthonormal columns, i.e.,
\begin{align}
    I_d=H^TH = \sum_{j=1}^{m/2} h_j h_j^T \,,
\end{align}
where $h_1,\dots,h_{m/2}$ are the rows of $H$ and we assume $m/2\ge d$. One can generate such matrices using randomized Haar ensemble, or partial Hadamard matrices. 
Then, we can represent $Z^*$ using
\begin{align*}
    Z^* &= Z^* H^T H\\
    &= \sum_{j=1}^{m/2} Z^* h_j h_j^T\,.
\end{align*}%
Since $Z^*$ is a symmetric matrix, $\sum_{j=1}^{m/2} Z^* h_j h_j^T$ is also symmetric, and we can write
\begin{align*}
    Z^* = \frac{1}{2} \sum_{j=1}^{m/2} (Z^* h_j h_j^T + h_j h_j^TZ^*) \,.
\end{align*}
Finally, for each term in the above summation, we employ the symmetrization identity
\begin{align*}
    xy^T + yx^T = \frac{1}{2}\left( (x+y)(x+y)^T - (x-y)(x-y)^T \right)\,,
\end{align*}
valid for any $x,y\in\mathbb{R}^d$. We arrive at the representation
\begin{align}
    Z^* &= \frac{1}{4} \sum_{j=1}^{m/2} ((Z^* h_j+ h_j)(Z^* h_j+ h_j)^T - (Z^* h_j- h_j)(Z^* h_j- h_j)^T) \\
    &= \sum_{j=1}^m u_ju_j^T \alpha_j\,,
\end{align}
where $u_j=Z^* h_j +h_j$, $\alpha_j=1/4$ for $j=1,\dots,m/2$ and $u_j=Z^* h_j -h_j$, $\alpha_j=-1/4$ for $j=m/2+1,\dots,m$. 

Since the matrix $H$ is arbitrary, one can map an optimal $Z^*$ matrix from the convex semidefinite program to infinitely many optimal solutions in the neural network parameterization space.
\section{Standard Weight Decay Formulation is NP-Hard} \label{sec:np_hardness}

In Section \ref{sec:quadratic_act}, we have studied two-layer neural networks with quadratic activation and cubic regularization and derive a convex program whose solution globally optimizes the non-convex problem. In this section, we show that if, instead of cubic regularization, we have quadratic regularization (i.e. weight decay), the resulting optimization problem is an NP-hard problem. 

\begin{theorem} \label{thm:nphard_general}
The two-layer neural network optimization problem with quadratic activation and standard $\ell_2$-squared regularization, i.e., weight decay, in \eqref{eq:noncvx_quadact_quadreg} is NP-hard for $\beta \rightarrow 0$. 
\end{theorem}

The remainder of this section breaks down the proof of Theorem \ref{thm:nphard_general}. At the core of the proof is the polynomial-time reduction of the problem to the NP-hard problem of phase retrieval.

\subsection{Reduction to an Equivalent Problem}
The optimization problem for training a two-layer fully connected neural network with quadratic activation and quadratic regularization can be stated as
\begin{align} \label{eq:noncvx_quadact_quadreg}
    p^* := &\min_{\{\alpha_j, u_j\}_{j=1}^m} \ell \left( \sum_{j=1}^m (Xu_j)^2 \alpha_j,\, y \right) + \frac{\beta}{c} \sum_{j=1}^m (|\alpha_j|^2 + \|u_j\|_2^2)\,,
\end{align}
where the scaling factor $c$ is the same as before (i.e. $c=2^{\frac{1}{3}}+2^{-\frac{2}{3}}\approx1.88988$).
Rescaling $u_j \leftarrow u_j t_j^{1/2}$ and $\alpha_j \leftarrow \alpha_j / t_j$ for $t_j > 0$, $j=1,\dots,m$, we obtain the following equivalent optimization problem
\begin{align}
    p^* = &\min_{\{\alpha_j, u_j\}_{j=1}^m} \ell \left( \sum_{j=1}^m (Xu_j)^2 \alpha_j,\, y \right) + \frac{\beta}{c} \sum_{j=1}^m (|\alpha_j|^2 / t_j^2 + \|u_j\|_2^2 t_j )\,.
\end{align}
Note that the regularization term is convex in $t_j$ for $t_j > 0$. Optimizing the regularization term with respect to $t_j$ leads to $t_j=2^{1/3} \left(\frac{|\alpha_j|}{\|u_j\|_2}\right)^{2/3}$ and plugging this in yields
\begin{align}
    p^* = &\min_{\{\alpha_j, u_j\}_{j=1}^m} \ell \left( \sum_{j=1}^m (Xu_j)^2 \alpha_j,\, y \right) + \beta \sum_{j=1}^m |\alpha_j|^{2/3} \|u_j\|_2^{4/3} \,.
\end{align}
Defining scaled weights $\alpha_j^\prime = \alpha_j \|u_j\|_2^2$ and $u_j^\prime = u_j / \|u_j\|_2$, we obtain the equivalent problem 
\begin{align} \label{eq:quadratic_regul_equiv_problem}
    p^* = &\min_{\{\alpha_j^\prime, u_j^\prime\}_{j=1}^m\, \mbox{s.t.} \|u_j^\prime\|_2=1,\,\forall j} \ell \left( \sum_{j=1}^m (Xu_j^\prime)^2 \alpha_j^\prime,\, y \right) + \beta \sum_{j=1}^m |\alpha_j^\prime|^{2/3} \,.
\end{align}

This shows that solving the standard weight decay formulation is equivalent to solving a $2/3$-norm penalized problem with unit norm first layer weights.
%

\subsection{Hardness Result}



We design a data matrix such that the solution coincides with solving the phase retrieval problem which is NP-hard (see \cite{fickus2013phase}).
We consider the equality constrained version of \eqref{eq:quadratic_regul_equiv_problem}, i.e., $\beta \rightarrow 0$, which is given by
\begin{align}
    \min_{\{\alpha_j,u_j\}_{j=1}^m \mbox{ s.t. } \|u_j\|_2=1, \forall j} &\sum_{j=1}^m |\alpha_j|^{2/3} \nonumber \\
    \mbox{s.t.}\quad  &\sum_{j=1}^m (Xu_j)^2 \alpha_j = y  \,.
\end{align}

\subsubsection{Addition of a Simplex Constraint}
Let the first $d$ rows of the data matrix $X$ be $e_1^T,\dots,e_d^T$ and let the first $d$ entries of $y$ be $1/d$. Then, the constraint $\sum_{j=1}^m (Xu_j)^2 = y$ implies
\begin{align}
    \sum_{j=1}^m u_{jk}^2\alpha_j=1/d\,\mbox{ for } k=1,\dots,d\,.
\end{align}
Summing the above for all $k=1,\dots,d$, and noting that $\sum_{k=1}^d u_{jk}^2=1$ lead to the constraint $\sum_{j=1}^m \alpha_j=1$.

\subsubsection{Reduction to the NP-Hard Phase Retrieval and Subset Sum Problem}

We let $X=[I; \,\tilde X]$ and $y=[\frac{1}{d}\ones; \,\tilde y]$ to obtain the simplex constraint $\sum_{j=1}^m \alpha_j = 1$ as shown in the previous subsection. In this case, the optimization problem reduces to
\begin{align} \label{eq:problem_before_phase_rt}
    \min_{\{\alpha_j,u_j\}_{j=1}^m \mbox{ s.t. } \|u_j\|_2=1, \forall j} &\sum_{j=1}^m |\alpha_j|^{2/3} \nonumber \\
    \mbox{s.t.} \quad &\sum_{j=1}^m (\tilde{X}u_j)^2 \alpha_j = \tilde{y} \nonumber \\
    &\sum_{j=1}^m u_{jk}^2\alpha_j=1/d, \quad k=1,\dots,d \nonumber \\
    &\sum_{j=1}^m \alpha_j = 1\,.
\end{align}
Suppose that there exists a feasible solution $\{ \alpha_j^*, u_j^* \}_{j=1}^m$, which satisfies $\|\alpha^*\|_0=1$, where $\alpha^*_1=1$ and ${u_1^*}^Tu_1^*=1$ with only one nonzero neuron. Then, it follows from Lemma \ref{lem:lpmin} that this solution is strictly optimal. Consequently, the problem in \eqref{eq:problem_before_phase_rt} is equivalent to
\begin{align} \label{eq:phase_retrieval_pr}
    \find \quad &u_1 \nonumber \\
    \mbox{s.t.} \quad &(\tilde{x}_i^T u_1)^2 = \tilde{y}_i, \quad i=1,\dots,(n-d) \nonumber \\
    &u_{1k}^2 = 1/d, \quad k=1,\dots,d\,.
\end{align}

\begin{lemma}[$\ell_p$ minimization recovers 1-sparse solutions when $0<p<1$]
\label{lem:lpmin}
\,\\Consider the optimization problem
\begin{align}
    \min_{\alpha_1,\dots,\alpha_m} &\sum_{i=1}^m |\alpha_i|^p \nonumber \\
    \mbox{s.t.} \quad &\sum_{i=1}^m \alpha_i = 1,\, \alpha \in \mathcal{C}\,,
\end{align}
where $\mathcal{C}$ is a convex set and $p\in (0,1)$. Suppose that there exists a feasible solution $\alpha^*\in\mathcal{C}$ and $\sum_i \alpha_i^*=1$ such that $\|\alpha^*\|_0=1$. Then, $\alpha^*$ is strictly optimal with objective value $1$. More precisely, any solution with cardinality strictly greater than 1 has objective value strictly larger than $1$.
\end{lemma}
%

%
%

\subsubsection{NP-hardness Proof}
Subset sum problem given in Definition \ref{def:subsetsum} is a decision problem known to be NP-complete (e.g. \cite{fickus2013phase}). The decision version of the problem in \eqref{eq:phase_retrieval_pr} can be stated as follows: Does there exist a feasible $u_1$? We show that this decision problem is NP-hard via a polynomial-time reduction to the subset sum problem.

\begin{definition}[Subset sum problem] \label{def:subsetsum} Given a set of integers $\mathcal{A}$, does there exist a subset $\mathcal{A}_S$ whose elements sum to $z$?
\end{definition}

Lemma \ref{lem:decision_pr_nphard} establishes the reduction of the decision version of \eqref{eq:phase_retrieval_pr} to the subset sum problem. The proof is provided in the appendix and follows the same approach used in the proof for the NP-hardness of phase retrieval in \cite{fickus2013phase}, with the main difference being the additional constraints $u_{1k}^2=1/d$, $k=1,\dots,d$ in \eqref{eq:phase_retrieval_pr}. Finally, Lemma \ref{lem:decision_pr_nphard} concludes the proof of Theorem \ref{thm:nphard_general}.


\begin{lemma} \label{lem:decision_pr_nphard}
Consider the problem in \eqref{eq:phase_retrieval_pr}. Let the first $d$ samples of $\tilde{X} \in \mathbb{R}^{(d+1) \times d}$, denoted $\tilde{X}_D \in \mathbb{R}^{d \times d}$, be any diagonal matrix with $-1$'s and $+1$'s on its diagonal, and let the $(d+1)$'st sample be $\tilde{x}_{d+1}=\sqrt{d} \begin{bmatrix} a_1 & \dots & a_d \end{bmatrix}^T$.
Then, the decision version of the resulting problem returns 'yes' if and only if the answer for the subset sum problem with $\mathcal{A}=\{a_1,\dots,a_d\}$ is 'yes'.
\end{lemma}

\begin{remark}
It follows from Theorem \ref{thm:nphard_general} that the two-layer neural network training problem with polynomial activation and unit norm first layer weights and $\sum_j |\alpha_j|^p$ as the regularization term with $p<1$ is also NP-hard for $\beta \rightarrow 0$ since it reduces to the quadratic activation case for the polynomial coefficients $a=1,b=0,c=0$.
\end{remark}
\section{Vector Output Networks} \label{sec:vector_outputs}

The derivations until this point have been for neural network architectures with scalar outputs, i.e., $y_i \in \mathbb{R}$. In this section, we turn to the vector output case $y_i \in \mathbb{R}^C$ where $C$ is the output dimension, and derive a convex problem that has the same optimal value as the non-convex neural network optimization problem. We exploit the same techniques described in the scalar output case except for the part for constructing the \textit{vector} second layer weights from the solution of the convex program. In the scalar output case, the convex problem is over the symmetric matrices $Z, Z^\prime$ and in the vector output case, the optimization is over $C$ such matrix pairs $Z_k, Z_k^\prime$, $k=1,\dots,C$. 

We begin our treatment of the vector output case by considering the neural network defined by
\begin{align} \label{eq:vector_output_forward}
    f(x) = \sum_{j=1}^m \sigma(x^Tu_j) \alpha_j^T \,,
\end{align}
where $\alpha_j \in \mathbb{R}^C$, $j=1,\dots,m$ are the \textit{vector} second layer weights. Note that in the scalar output case, the second layer weights $\alpha_j$ were scalars.
Taking the regularization to be the $\ell_1$ norm of the second layer weights, the neural network training requires solving the following non-convex optimization problem
\begin{align} \label{eq:vector_output_non-convex}
    p^* = &\min_{\{u_j,\, \alpha_j\}_{j=1}^m, \, \mbox{s.t.}\, \|u_j\|_2=1,\,\forall j} \ell \left( \sum_{j=1}^m \sigma(Xu_j) \alpha_j^T \,,\, Y \right) + \beta \sum_{j=1}^m \|\alpha_j\|_1 \,,
\end{align}
where $Y \in \mathbb{R}^{n\times C}$ is the output matrix.
Equivalently,
\begin{align} \label{eq:vector_output_non-convex_2}
    p^* = &\min_{\{u_j \}_{j=1}^m \, \mbox{s.t.}\, \|u_j\|_2=1,\,\forall j} \, \min_{\{\alpha_j \}_{j=1}^m, \hat{Y}} \ell \left( \hat{Y},\, Y \right) + \beta \sum_{j=1}^m \|\alpha_j\|_1 \quad \mbox{s.t.} \quad \hat{Y} = \sum_{j=1}^m \sigma(Xu_j) \alpha_j^T \,.
\end{align}
The dual problem for the inner minimization problem is given by
\begin{align} \label{eq:vector_output_dual_pr}
    \max_{v} -\ell^*(-v) \quad \mbox{s.t.} \quad |v_k^T \sigma(Xu_j)| \leq \beta \,, \forall j,k \,,
\end{align}
where $v \in \mathbb{R}^{n\times C}$ is the dual variable and $v_k \in \mathbb{R}^{n}$ is the $k$'th column of $v$. 

Theorem \ref{thm:vector_output_thm} gives the main result of this section.

\begin{theorem} [Globally optimal convex program for polynomial activation vector output networks] \label{thm:vector_output_thm}
The solution of the convex problem in \eqref{eq:vector_output_convex_program} provides a global optimal solution for the vector output non-convex problem in \eqref{eq:vector_output_non-convex} when the number of neurons satisfies $m \geq m^*$ where 
\begin{align}
    m^* = \sum_{k=1}^C (\rank(Z_k^*) + \rank({Z_k^\prime}^*)).
\end{align}
The optimal neural network weights are determined from the solution of the convex problem via the neural decomposition procedure for each $Z_k^*$ and ${Z_k^\prime}^*$ and the construction given in \eqref{eq:vector_output_weight_construction}. The optimal number of neurons is upper bounded by $m^* \leq 2(d+1)C$.
\end{theorem}

\begin{proof}[Proof of Theorem \ref{thm:vector_output_thm}]
Applying the S-procedure for the constraints in the dual problem \eqref{eq:vector_output_dual_pr}, we obtain the following maximization problem
\begin{align}
    \max & -\ell^*(-v) \nonumber \\
    \text{s.t. } 
    & \begin{bmatrix} \rho_{k,1} I - a \sum_{i=1}^n x_ix_i^T v_{i,k} & -\frac{1}{2}bX^Tv_k \\ -\frac{1}{2}bv_k^TX & \beta - c\ones^Tv_k - \rho_{k,1} \end{bmatrix} \succeq 0,  \quad k=1,\dots,C \nonumber \\
    &\begin{bmatrix} \rho_{k,2} I + a \sum_{i=1}^n x_ix_i^Tv_{i,k} & \frac{1}{2}bX^Tv_k \\ \frac{1}{2}bv_k^TX & \beta + c\ones^Tv_k - \rho_{k,2} \end{bmatrix} \succeq 0,  \quad k=1,\dots,C \,.
\end{align}
Next, let us introduce the following Lagrange multipliers
\begin{align}
    Z_k = \begin{bmatrix} Z_{k,1} & Z_{k,2} \\ Z_{k,3} & Z_{k,4}  \end{bmatrix} \in \mathbb{S}^{(d+1)\times (d+1)}, \quad Z_k^\prime = \begin{bmatrix} Z_{k,1}^\prime & Z_{k,2}^\prime \\ Z_{k,3}^\prime & Z_{k,4}^\prime \end{bmatrix} \in \mathbb{S}^{(d+1)\times (d+1)}, \quad k=1,\dots,C.
\end{align}
Then, the Lagrangian is
\begin{align}
    &L\left(v, \{\rho_{k,1}, \rho_{k,2}, Z_k, Z_k^\prime\}_{k=1}^C\right)= \nonumber \\
    &= -\ell^*(-v) + \sum_{k=1}^C \left(\rho_{k,1} \tr(Z_{k,1}) +\rho_{k,2} \tr(Z_{k,1}^\prime)\right) - a\sum_{k=1}^C  \sum_{i=1}^n v_{i,k}x_i^T(Z_{k,1}-Z_{k,1}^\prime)x_i - b\sum_{k=1}^Cv_k^TX(Z_{k,2}-Z_{k,2}^\prime) + \nonumber \\
    &+\sum_{k=1}^C \left((\beta-\rho_{k,1}) Z_{k,4} + (\beta-\rho_{k,2}) Z_{k,4}^\prime\right) - c\sum_{k=1}^C  \sum_{i=1}^n v_{k,i} (Z_{k,4}-Z_{k,4}^\prime) \,.
\end{align}
Finally maximizing the Lagrangian leads to the following convex SDP:
\begin{align} \label{eq:vector_output_convex_program}
    \min_{\{Z_k=Z_k^T, Z_k^\prime={Z_k^\prime}^T\}_{k=1}^C} & \ell (\hat{Y}, Y) + \beta \sum_{k=1}^C(Z_{k,4}+Z_{k,4}^\prime) \nonumber \\
    \mbox{s.t.} \quad &\hat{Y}_{ik} = a x_i^T(Z_{k,1}-Z^\prime_{k,1})x_i + b x_i^T(Z_{k,2}-Z^\prime_{k,2}) + c(Z_{k,4}-Z^\prime_{k,4}), \quad i\in [n],\,k\in [C] \nonumber
    \\
    &\tr(Z_{k,1}) = Z_{k,4}, \,\, \tr(Z_{k,1}^\prime) = Z_{k,4}^\prime, \quad k=1,\dots,C \nonumber \\
    & Z_k \succeq 0, \,\, Z_k^\prime \succeq 0,  \quad k=1,\dots,C \,.
\end{align}

We construct the neural network weights from the optimal solution of the convex program as follows.
We follow the neural decomposition procedure from Section \ref{sec:neural_decomp} for extracting neurons from each of the matrices $Z_k^*$ and ${Z_k^\prime}^*$, $k=1,\dots,C$. The decompositions for $Z_k^*$ will be of the form
\begin{align} \label{eq:vector_output_decomposition}
    Z_{k,1}^* = \sum_{j=1}^{r_k} u_{k,j} u_{k,j}^T d_{k,j}^2, \quad Z_{k,2}^* = \sum_{j=1}^{r_k} u_{k,j} d_{k,j}^2, \quad Z_{k,4}^* &= \sum_{j=1}^{r_k} d_{k,j}^2.
\end{align}
Then, the weights due to $Z_k^*$, $k=1,\dots,C$ are determined as follows:
\begin{align} \label{eq:vector_output_weights}
    \mbox{First layer weights: } \quad &\{u_{1,1}, u_{1,2},\dots, u_{1,r_1}\}, \dots, \{u_{C,1}, u_{C,2}, \dots, u_{C,r_C}\} \nonumber \\
    \mbox{Second layer weights: } \quad & \{d_{1,1}^2e_1^T, d_{1,2}^2e_1^T, \dots, d_{1,r_1}^2e_1^T\}, \dots, \{d_{C,1}^2e_C^T, d_{C,2}^2e_C^T, \dots, d_{C,r_C}^2e_C^T\} \,,
\end{align}
where $e_k$ denotes the $k$'th $C$-dimensional unit vector, and $r_k$ is the rank of the matrix $Z_k^*$. In short, the matrix $Z_k^*$ with rank $r_k$ leads to the first layer weights $\{u_{k,1}, u_{k,2}, \dots, u_{k,r_k} \}$ and the second layer weights $\{d_{k,1}^2e_k^T, d_{k,2}^2e_k^T, \dots, d_{k,r_k}^2e_k^T\}$.
The weights due to ${Z_k^\prime}^*$, $k=1,\dots,C$ are determined the same way.
Then, we reach the following neural network construction:
\begin{align} \label{eq:vector_output_weight_construction}
    f(X) = \sum_{k=1}^C \sum_{j=1}^{r_k} \sigma(Xu_{k,c}) d_{k,j}^2e_k^T + \sum_{k=1}^C \sum_{j=1}^{r_k^\prime} \sigma(Xu_{k,c}^\prime) {d_{k,j}^\prime}^2e_k^T.
\end{align}

Finally, the total number of neurons that the convex problem finds is $\sum_{k=1}^C (r_k + r_k^\prime)$. The maximum number of neurons occurs if all $Z_k^*$ and ${Z_k^\prime}^*$ are full rank, and this corresponds to a maximum total of $2(d+1)C$ neurons. 

We plug the decomposition expressions given in \eqref{eq:vector_output_decomposition} in the convex program in \eqref{eq:vector_output_convex_program} to conclude that the optimal value of the convex program is an upper bound for the non-convex optimization problem \eqref{eq:vector_output_non-convex}. The $k$'th entry of the estimate for the $i$'th training sample is
\begin{align}
    \hat{Y}_{ik} &= ax_i^T \left(\sum_{j=1}^{r_k} u_{k,j} u_{k,j}^T d_{k,j}^2 + \sum_{j=1}^{r_k^\prime} u_{k,j}^\prime {u_{k,j}^\prime}^T (-{d_{k,j}^\prime}^2) \right)x_i + bx_i^T\left(\sum_{j=1}^{r_k} u_{k,j} d_{k,j}^2 + \sum_{j=1}^{r_k^\prime} u_{k,j}^\prime (-{d_{k,j}^\prime}^2) \right) + \nonumber \\ 
    &+ c\left(\sum_{j=1}^{r_k} d_{k,j}^2 + \sum_{j=1}^{r_k^\prime} (-{d_{k,j}^\prime}^2) \right) \nonumber \\
    &= \sum_{j=1}^{r_k} \sigma(x_i^Tu_{k,j}) d_{k,j}^2 + \sum_{j=1}^{r_k^\prime} \sigma(x_i^Tu_{k,j}^\prime) (-{d_{k,j}^\prime}^2) \,.
\end{align}
It follows that the output vector for the $i$'th sample is
\begin{align}
    \hat{y}_i = \sum_{k=1}^C \sum_{j=1}^{r_k} \sigma(x_i^Tu_{k,j}) d_{k,j}^2 e_k^T + \sum_{k=1}^C \sum_{j=1}^{r_k^\prime} \sigma(x_i^Tu_{k,j}^\prime) (-{d_{k,j}^\prime}^2) e_k^T \,.
\end{align}
We note that this output is of the same form as the non-convex case \eqref{eq:vector_output_non-convex_2}. We also need to check that the regularization term is equivalent to the sum of $\ell_1$ norms of the second layer weights:
\begin{align}
    \beta \sum_{k=1}^C (Z_{k,4} + Z_{k,4}^\prime) &= \beta \sum_{k=1}^C \sum_{j=1}^{r_k} d_{k,j}^2 + \beta \sum_{k=1}^C \sum_{j=1}^{r_k^\prime} {d_{k,j}^\prime}^2  \nonumber \\
    &= \beta \sum_{k=1}^C \sum_{j=1}^{r_k} \| d_{k,j}^2 e_k^T \|_1 + \beta \sum_{k=1}^C \sum_{j=1}^{r_k^\prime} \| -{d_{k,j}^\prime}^2 e_k^T \|_1 \,,
\end{align}
which is of the form $\beta \sum_{j=1}^m\|\alpha_j\|_1$. Hence, the neural network weights that we obtain via the neural decomposition procedure lead to an upper bound for the original non-convex optimization problem.
This concludes the proof that the optimal solution of the convex problem \eqref{eq:vector_output_convex_program} provides a global optimal solution to the non-convex problem \eqref{eq:vector_output_non-convex}.
\end{proof}

\section{Convolutional Neural Networks} \label{sec:convolutional}

In this section, we consider two-layer convolutional networks with a convolutional first layer and a fully connected second layer. We will denote the filter size by $f$. Let us denote the patches of a data sample $x$ by $x_1,\dots,x_K$ where the patches have the same dimension as the filters, i.e., $x_k \in \mathbb{R}^f$. The stride and padding do not affect the below derivations as they can be readily handled when forming the patches. The output of this network is expressed as:
\begin{align} \label{eq:convolutional_noncvx}
    f(x) = \sum_{j=1}^m \sum_{k=1}^K \sigma(x_k^T u_j) \alpha_{jk} \,,
\end{align}
where $u_j \in \mathbb{R}^f$ denotes the  $j$'th filter. We will take the regularization to be the $\ell_1$ norm of the second layer weights $\alpha_j=\begin{bmatrix} \alpha_{j1} & \dots & \alpha_{jK} \end{bmatrix}^T \in \mathbb{R}^K$, $j=1,\dots,m$:
\begin{align} \label{eq:convolutional_non-convex}
    p^* = &\min_{\{u_j \}_{j=1}^m \, \mbox{s.t.}\, \|u_j\|_2=1,\,\forall j} \, \min_{\{\alpha_j \}_{j=1}^m, \hat{y}} \ell ( \hat{y},\, y ) + \beta \sum_{j=1}^m \|\alpha_j\|_1 \quad \mbox{s.t.} \quad \hat{y} = \sum_{j=1}^m \sum_{k=1}^K \sigma(X_k^T u_j) \alpha_{jk} \,
\end{align}
where we use $X_k \in \mathbb{R}^{n \times f}$ to denote the matrix with the $k$'th patch of all the data samples. 
The dual for the inner minimization problem is given by
\begin{align} \label{eq:convolutional_dual}
    \max_{v} -\ell^*(-v) \quad \mbox{s.t.} \quad |v^T \sigma(X_ku_j)| \leq \beta \,, \forall j,k \,.
\end{align}

We state the main result of this section in Theorem \ref{thm:convolutional_thm}.

\begin{theorem} [Globally optimal convex program for polynomial activation convolutional neural networks] \label{thm:convolutional_thm}
The solution of the convex problem in \eqref{eq:convolutional_convex_program} provides a global optimal solution for the non-convex convolutional neural network problem in \eqref{eq:convolutional_noncvx} when the number of filters is at least $(\rank(Z_k^*) + \rank({Z_k^\prime}^*))$ and equivalently, the number of neurons satisfies $m \geq m^*$ where 
\begin{align}
    m^* = K\sum_{k=1}^K (\rank(Z_k^*) + \rank({Z_k^\prime}^*)).
\end{align}
The optimal neural network weights are determined from the solution of the convex problem via the neural decomposition procedure for each $Z_k^*$ and ${Z_k^\prime}^*$. The optimal number of filters is upper bounded by $2(f+1)K$ and the optimal number of neurons is upper bounded by $m^* \leq 2(f+1)K^2$.
\end{theorem}

\begin{proof}[Proof of Theorem \ref{thm:convolutional_thm}]
We apply the S-procedure to replace the constraints of \eqref{eq:convolutional_dual} with equivalent LMI constraints and this yields
\begin{align}
    \max & -\ell^*(-v) \nonumber \\
    \text{s.t. } 
    & \begin{bmatrix} \rho_{k,1} I - a \sum_{i=1}^n x_{i,k}x_{i,k}^T v_i & -\frac{1}{2}bX_k^Tv \\ -\frac{1}{2}bv^TX_k & \beta - c\ones^Tv - \rho_{k,1} \end{bmatrix} \succeq 0,  \quad k=1,\dots,K \nonumber \\
    &\begin{bmatrix} \rho_{k,2} I + a \sum_{i=1}^n x_{i,k}x_{i,k}^Tv_i & \frac{1}{2}bX_k^Tv \\ \frac{1}{2}bv^TX_k & \beta + c\ones^Tv - \rho_{k,2} \end{bmatrix} \succeq 0,  \quad k=1,\dots,K \,,
\end{align}
where $x_{i,k} \in \mathbb{R}^f$ denotes the $k$'th patch of the $i$'th data sample.
The Lagrangian is as follows
\begin{align}
    &L\left(v, \{\rho_{k,1}, \rho_{k,2}, Z_k, Z_k^\prime\}_{k=1}^K\right)= \nonumber \\
    &= -\ell^*(-v) + \sum_{k=1}^K \left(\rho_{k,1} \tr(Z_{k,1}) +\rho_{k,2} \tr(Z_{k,1}^\prime)\right) - a\sum_{k=1}^K \sum_{i=1}^n v_i x_{i,k}^T(Z_{k,1}-Z_{k,1}^\prime)x_{i,k} - b\sum_{k=1}^Kv^TX_k(Z_{k,2}-Z_{k,2}^\prime) + \nonumber \\
    &+\sum_{k=1}^K \left((\beta-\rho_{k,1}) Z_{k,4} + (\beta-\rho_{k,2}) Z_{k,4}^\prime\right) - c\sum_{k=1}^K  \sum_{i=1}^n v_i (Z_{k,4}-Z_{k,4}^\prime) \,,
\end{align}
where $Z_k,Z_k^\prime$ are $(f+1)\times (f+1)$ dimensional symmetric matrices.
Maximizing the Lagrangian with respect to $v$, $\rho_{k,1}$, $\rho_{k,2}$, $k=1,\dots,K$ yields the convex SDP
\begin{align} \label{eq:convolutional_convex_program}
    \min_{\{Z_k=Z_k^T, Z_k^\prime={Z_k^\prime}^T\}_{k=1}^K} & \ell (\hat{y}, y) + \beta \sum_{k=1}^K(Z_{k,4}+Z_{k,4}^\prime) \nonumber \\
    \mbox{s.t.} \quad &\hat{y}_{i} = a \sum_{k=1}^K x_{i,k}^T(Z_{k,1}-Z^\prime_{k,1})x_{i,k} + b \sum_{k=1}^K x_{i,k}^T(Z_{k,2}-Z^\prime_{k,2}) + c \sum_{k=1}^K (Z_{k,4}-Z^\prime_{k,4}), \quad i\in [n] \nonumber
    \\
    &\tr(Z_{k,1}) = Z_{k,4}, \,\, \tr(Z_{k,1}^\prime) = Z_{k,4}^\prime, \quad k=1,\dots,K \nonumber \\
    & Z_k \succeq 0, \,\, Z_k^\prime \succeq 0,  \quad k=1,\dots,K \,.
\end{align}

We now show that the convex program in \eqref{eq:convolutional_convex_program} provides an upper bound for the non-convex problem via the same strategy that we have used for the vector output case in Section \ref{sec:vector_outputs}.
We construct the neural network weights from each of the matrices $Z_k^*$ and ${Z_k^\prime}^*$, $k=1,\dots,K$ via neural decomposition:
\begin{align} \label{eq:convolutional_decomposition}
    Z_{k,1}^* = \sum_{j=1}^{r_k} u_{k,j} u_{k,j}^T d_{k,j}^2, \quad Z_{k,2}^* = \sum_{j=1}^{r_k} u_{k,j} d_{k,j}^2, \quad Z_{k,4}^* &= \sum_{j=1}^{r_k} d_{k,j}^2 \,,
\end{align}
and the weights due to each $Z_k^*$ are
\begin{align} \label{eq:convolutional_weights_construction}
    \mbox{First layer filters: } \quad &u_{k,1}, u_{k,2}, \dots, u_{k,r_k} \nonumber \\
    \mbox{Second layer weights: } \quad & \{d_{k,1}^2, 0, 0, \dots, 0\}, \{0, d_{k,2}^2, 0, \dots, 0\}, \dots, \{0,0,0,\dots,d_{k,r_k}^2\} \,. 
\end{align}
To clarify, for each filter $u_{k,j}$, we have $K$ (scalar) weights in the second layer because we apply the same filter to $K$ different patches and the resulting $K$ numbers (after being input to the activation function) each are multiplied by a different second layer weight. The second layer weights associated with the filter $u_{k,j}$ will be these $K$ numbers: $\{0,\dots,0,d_{k,j}^2,0\dots,0\}$, where the only nonzero entry is the $j$'th one. Consequently, each $Z_k^*$ matrix produces $\rank(Z_k^*)$ filters and $K\rank(Z_k^*)$ neurons.
Including the weights due to ${Z_k^\prime}^*$ as well, we will have $\sum_{k=1}^K (r_k+r_k^\prime)$ filters and $K\sum_{k=1}^K (r_k+r_k^\prime)$ neurons in total. The optimal number of filters is upper bounded by $2(f+1)K$ and the optimal number of neurons is upper bounded by $2(f+1)K^2$.

We omit the details of plugging the weights into the convex objective to show that it becomes equivalent to the non-convex objective. The details are similar to the vector output case. 
\end{proof}

\section{Average Pooling} \label{sec:avg_pooling}

In this section we will consider convolutional neural networks with average pooling. We will denote the pool size by $P$. Let us consider a two-layer neural network where the first layer is a convolutional layer with filter size $f$. The convolutional layer is followed by the polynomial activation, average pooling, and a fully connected layer. We will denote the number of patches per sample by $K$. The output of this architecture can be expressed as
\begin{align}
    f(x) = \sum_{j=1}^m \sum_{k=1}^{K/P} \left(\frac{1}{P} \sum_{l=1}^P \sigma(x_{(k-1)P+l}^T u_j) \right) \alpha_{jk} \,.
\end{align}
We note that the number of parameters in the second layer (i.e. $\alpha_{jk}$'s) is equal to $m\frac{K}{P}$. The optimization problem for this architecture can be written as
\begin{align} \label{eq:convolutional_nonconvex_avgpool}
    p^* = &\min_{\{u_j \}_{j=1}^m \, \mbox{s.t.}\, \|u_j\|_2=1,\,\forall j} \, \min_{\{\alpha_j \}_{j=1}^m, \hat{y}} \ell ( \hat{y},\, y ) + \beta \sum_{j=1}^m \|\alpha_j\|_1 \quad \mbox{s.t.} \quad \hat{y} = \sum_{j=1}^m \sum_{k=1}^{K/P} \left(\frac{1}{P} \sum_{l=1}^P \sigma(X_{(k-1)P+l} u_j) \right) \alpha_{jk} \,,
\end{align}
where $\alpha_j=\begin{bmatrix} \alpha_{j1} & \dots & \alpha_{j,K/P} \end{bmatrix}^T$, $j=1,\dots,m$. 
The dual of the inner minimization problem is given by
\begin{align} \label{eq:pooling_dual}
    \max_{v} -\ell^*(-v) \quad \mbox{s.t.} \quad \left|v^T \left(\frac{1}{P} \sum_{l=1}^P \sigma(X_{(k-1)P+l} u_j) \right)\right| \leq \beta \,, \forall j,k \,.
\end{align}

Theorem \ref{thm:pooling_thm} states our result for CNN with average pooling.

\begin{theorem} [Globally optimal convex program for polynomial activation convolutional neural networks with average pooling] \label{thm:pooling_thm}
The solution of the convex problem in \eqref{eq:avgpooling_convex_program} provides a global optimal solution for the non-convex problem for the convolutional neural network with average pooling in \eqref{eq:convolutional_nonconvex_avgpool} when the number of neurons satisfies $m \geq m^*$ where 
\begin{align}
    m^* = \frac{K}{P}\sum_{k=1}^{K/P} (\rank(Z_k^*) + \rank({Z_k^\prime}^*)).
\end{align}
The optimal neural network weights are determined from the solution of the convex problem via the neural decomposition procedure for each $Z_k^*$ and ${Z_k^\prime}^*$. The optimal number of neurons is upper bounded by $m^* \leq 2(f+1)\frac{K^2}{P^2}$.
\end{theorem}

\begin{proof}[Proof of Theorem \ref{thm:pooling_thm}]
We rewrite the constraints of the dual problem \eqref{eq:pooling_dual} as follows:
\begin{align}
    -\beta \leq \frac{1}{P}\sum_{l=1}^P \left(u_j^T \left(a \sum_{i=1}^n x_{i,(k-1)P+l}x_{i,(k-1)P+l}^Tv_i \right) u_j  + bv^T X_{(k-1)P+l}u_j + cv^T \ones \right) \leq \beta, \quad \forall j,k \,.
\end{align}
S-procedure allows us to write this problem equivalently as
\begin{align}
    \max & -\ell^*(-v) \nonumber \\
    \text{s.t. } 
    & \begin{bmatrix} \rho_{k,1} I - a\frac{1}{P}\sum_{l=1}^P \sum_{i=1}^n x_{i,(k-1)P+l}x_{i,(k-1)P+l}^T v_i & -\frac{1}{2P}b\sum_{l=1}^P X_{(k-1)P+l}^Tv \\ -\frac{1}{2P}b\sum_{l=1}^P v^TX_{(k-1)P+l} & \beta - c\ones^Tv - \rho_{k,1} \end{bmatrix} \succeq 0,  \quad k=1,\dots,K/P \nonumber \\
    &\begin{bmatrix} \rho_{k,2} I + a\frac{1}{P} \sum_{l=1}^P \sum_{i=1}^n x_{i,(k-1)P+l}x_{i,(k-1)P+l}^Tv_i & \frac{1}{2P}b\sum_{l=1}^PX_{(k-1)P+l}^Tv \\ \frac{1}{2P}b\sum_{l=1}^Pv^TX_{(k-1)P+l} & \beta + c\ones^Tv - \rho_{k,2} \end{bmatrix} \succeq 0,  \quad k=1,\dots,K/P \,.
\end{align}
The Lagrangian is as follows
\begin{align}
    &L\left(v, \{\rho_{k,1}, \rho_{k,2}, Z_k, Z_k^\prime\}_{k=1}^{K/P}\right)= \nonumber \\
    &= -\ell^*(-v) + \sum_{k=1}^{K/P} \left(\rho_{k,1} \tr(Z_{k,1}) +\rho_{k,2} \tr(Z_{k,1}^\prime)\right) - a\frac{1}{P}\sum_{k=1}^{K/P} \sum_{l=1}^P \sum_{i=1}^n v_i x_{i,(k-1)P+l}^T(Z_{k,1}-Z_{k,1}^\prime)x_{i,(k-1)P+l}  \nonumber \\  
    &- b\frac{1}{P}\sum_{k=1}^{K/P}\sum_{l=1}^Pv^TX_{(k-1)P+l}(Z_{k,2}-Z_{k,2}^\prime) + \sum_{k=1}^{K/P} \left((\beta-\rho_{k,1}) Z_{k,4} + (\beta-\rho_{k,2}) Z_{k,4}^\prime\right) - c\sum_{k=1}^{K/P}  \sum_{i=1}^n v_i (Z_{k,4}-Z_{k,4}^\prime) \,,
\end{align}
where $Z_k,Z_k^\prime$ are $(f+1)\times (f+1)$ dimensional symmetric matrices.
Maximizing the Lagrangian with respect to $v$, $\rho_{k,1}$, $\rho_{k,2}$, $k=1,\dots,K/P$ yields the following convex SDP:
\begin{align} \label{eq:avgpooling_convex_program}
    \min_{\{Z_k=Z_k^T, Z_k^\prime={Z_k^\prime}^T\}_{k=1}^{K/P}} & \ell (\hat{y}, y) + \beta \sum_{k=1}^{K/P}(Z_{k,4}+Z_{k,4}^\prime) \nonumber \\
    \mbox{s.t.} \quad &\hat{y}_{i} = a \frac{1}{P} \sum_{k=1}^{K/P} \sum_{l=1}^P x_{i,(k-1)P+l}^T(Z_{k,1}-Z^\prime_{k,1})x_{i,(k-1)P+l} + b \frac{1}{P}\sum_{k=1}^{K/P} \sum_{l=1}^P x_{i,(k-1)P+l}^T(Z_{k,2}-Z^\prime_{k,2}) + \nonumber \\ 
    &+ c \sum_{k=1}^{K/P} (Z_{k,4}-Z^\prime_{k,4}), \quad i\in [n] \nonumber
    \\
    &\tr(Z_{k,1}) = Z_{k,4}, \,\, \tr(Z_{k,1}^\prime) = Z_{k,4}^\prime, \quad k=1,\dots,K/P \nonumber \\
    & Z_k \succeq 0, \,\, Z_k^\prime \succeq 0,  \quad k=1,\dots,K/P \,.
\end{align}

We omit the details of constructing the neural network weights from the solution of the convex SDP $Z_k^*, {Z_k^\prime}^*$, $k=1,\dots,K/P$ which follows in a similar fashion as the proof of Theorem \ref{thm:convolutional_thm}.
\end{proof}

We note that when we pick the pool size as $P=1$, this is the same as not having average pooling, and the corresponding convex program is the same as \eqref{eq:convolutional_convex_program}, derived in Section \ref{sec:convolutional}. The other extreme for the pool size is when $P=K$ and this corresponds to what is known as \textit{global average pooling} in which case the convex SDP simplifies to
\begin{align} \label{eq:global_avgpooling_convex_program}
    \min_{Z=Z^T, Z^\prime={Z^\prime}^T} & \ell (\hat{y}, y) + \beta (Z_4+Z_4^\prime) \nonumber \\
    \mbox{s.t.} \quad &\hat{y}_{i} = a \frac{1}{K} \sum_{l=1}^K x_{i,l}^T(Z_1-Z^\prime_1)x_{i,l} + b \frac{1}{K}\sum_{l=1}^K x_{i,l}^T(Z_2-Z^\prime_2) + c(Z_4-Z^\prime_4), \quad i\in [n] \nonumber
    \\
    &\tr(Z_1) = Z_4, \,\, \tr(Z_1^\prime) = Z_4^\prime \nonumber \\
    & Z \succeq 0, \,\, Z^\prime \succeq 0.
\end{align}
We note that the problem \eqref{eq:global_avgpooling_convex_program} has only two variables $Z$ and $Z^\prime$. This should be contrasted with the convolutional architecture with no pooling \eqref{eq:convolutional_convex_program} which has $2K$ variables.

\section{Numerical Results} \label{sec:numerical_results}

In this section, we present numerical results that verify the presented theory of the convex formulations along with experiments comparing the test set performance of the derived formulations. 
All experiments have been run on a MacBook Pro with 16GB RAM.

\textbf{Solvers:} We have used CVXPY \cite{diamond2016cvxpy,agrawal2018rewriting} for solving the convex SDPs. In particular, we have used the open source solver SCS (splitting conic solver) \cite{scs2016paper,scs2016code} in CVXPY, which is a scalable first order solver for convex cone problems. 

Furthermore, we have solved the non-convex problems via backpropagation for which we have used PyTorch \cite{pytorch}. We have used the SGD algorithm for the non-convex models. For all the experiments involving SGD in this section, we show only the results corresponding to the best learning rate that we select via an offline hyperparameter search. The momentum parameter is $0.9$. In the plots, the non-convex models are either labeled as 'Backpropagation (GD)' or 'Backpropagation (SGD)'. The first one, short for gradient descent, means that the batch size is equal to the number of samples $n$, and the second one, short for stochastic gradient descent, means that the batch size is not $n$ and the exact batch size is explicitly stated in the figure captions.

\textbf{Polynomial approximation of activation functions:} To obtain the degree-2 polynomial approximation of a given activation function $\sigma(u)$ such as the ReLU activation, one way is to select the polynomial coefficients $a,b,c$ that minimize the $\ell_2$ norm objective $\|T\begin{bmatrix} a & b & c \end{bmatrix}^T - s\|_2$ with
\begin{align}
    T = \begin{bmatrix} t_1^2 & t_1 & 1 \\ & \vdots & \\ t_N^2 & t_N & 1 \end{bmatrix}, \quad \quad s = \begin{bmatrix} \sigma(t_1) \\ \vdots \\ \sigma(t_N) \end{bmatrix},
\end{align}
where $t_i$'s are linearly spaced in $[L,U]$. The lower and upper limits $L$ and $U$ specify the range in which we would like to approximate the given activation function. For instance, when $L=-5$, $U=5$, $N=1000$ and $\sigma(u)$ is the ReLU activation, the optimal polynomial coefficients are $a=0.09,b=0.5,c=0.47$. When we change the approximation range to a slightly narrower one with $L=-4,U=4$, the coefficients then become $a=0.12,b=0.5,c=0.38$. Note that the training data can be normalized appropriately to confine the range of the input to the neurons and control the approximation error.

\subsection{Results for Verifying the Theoretical Formulations}
The first set of numerical results in Figure \ref{fig:verifying_theoretical_claims} is for verifying that the derived convex problems have the same optimal value as their non-convex counterparts. The plots in Figure \ref{fig:verifying_theoretical_claims} show the non-convex cost against time when 1) the non-convex problem is solved in PyTorch and 2) the corresponding convex problem (see Table \ref{table:summary_results}) is solved using CVXPY. The number of neurons for the non-convex models in all of the plots in Figure \ref{fig:verifying_theoretical_claims} is set to the optimal number of neurons $m^*$ found by the convex problem.

Figure \ref{fig:verifying_theoretical_claims} demonstrates that solving the convex SDP takes less time than solving the associated non-convex problem using backpropagation for all of the neural network architectures. Figure \ref{fig:verifying_theoretical_claims} also shows that the training of the non-convex models via the backpropagation algorithm does not always yield the global optimal but instead may converge to local minima. In addition, we note that the plots do not reflect the time it takes to tune the learning rate for the non-convex models, which was performed offline.

\begin{figure} 
\begin{minipage}[b]{0.32\linewidth}
  \centering
  \centerline{\includegraphics[width=\columnwidth]{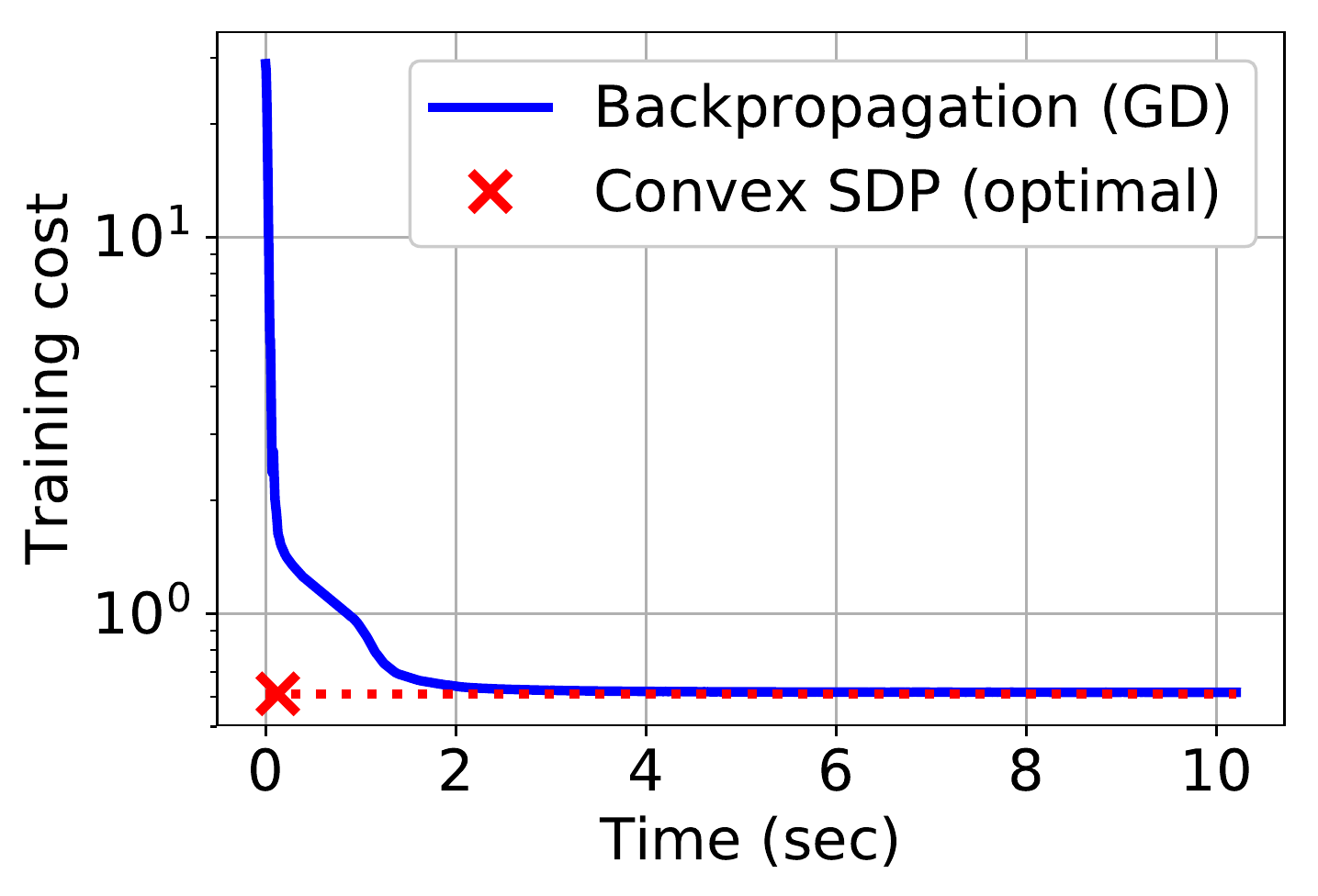}}
  \centerline{(a) Quad act ($10,20,9$)}\medskip
\end{minipage}
\hfill
\begin{minipage}[b]{0.32\linewidth}
  \centering
  \centerline{\includegraphics[width=\columnwidth]{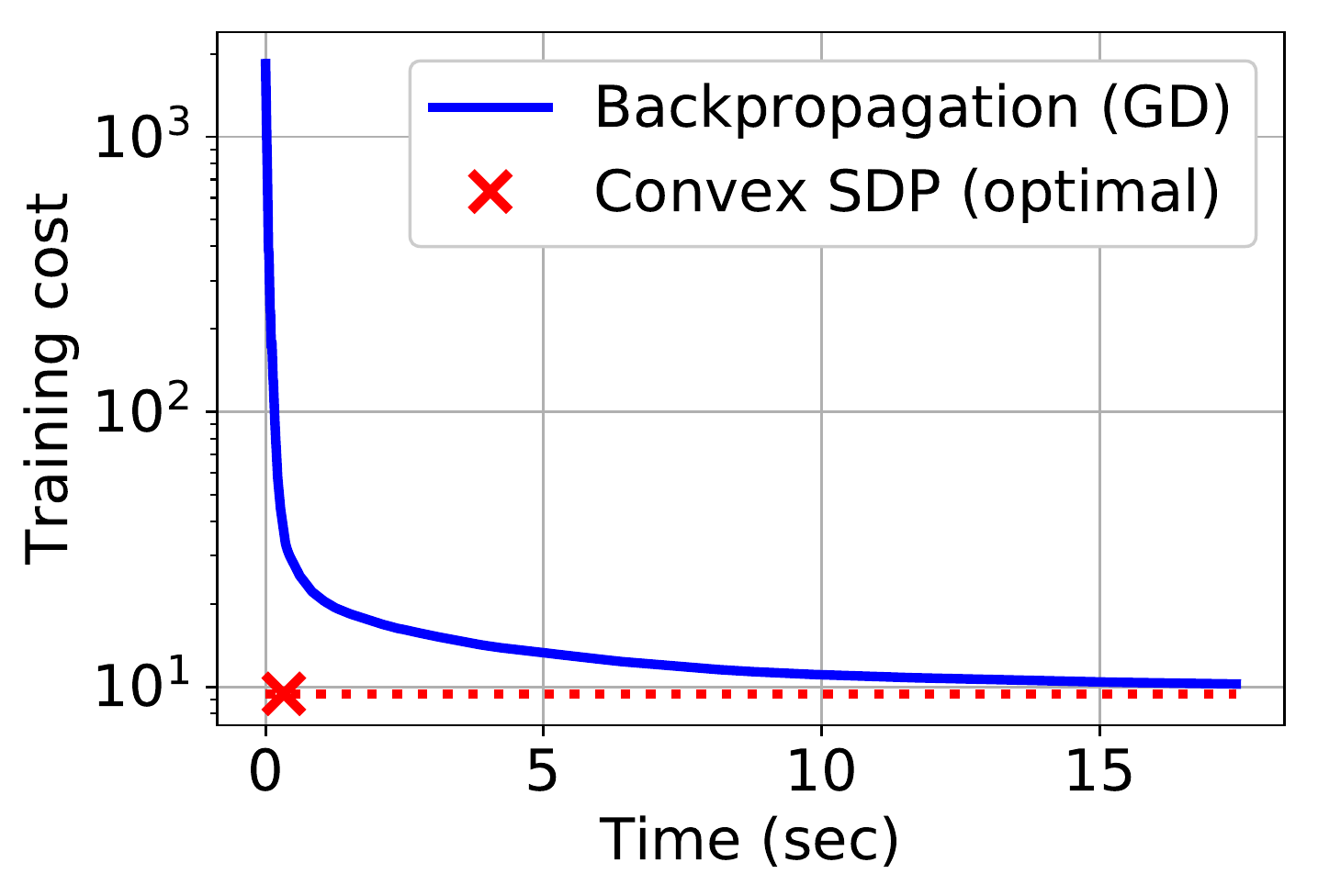}}
  \centerline{(b) Quad act ($100,20,20$)}\medskip
\end{minipage}
\hfill
\begin{minipage}[b]{0.32\linewidth}
  \centering
  \centerline{\includegraphics[width=\columnwidth]{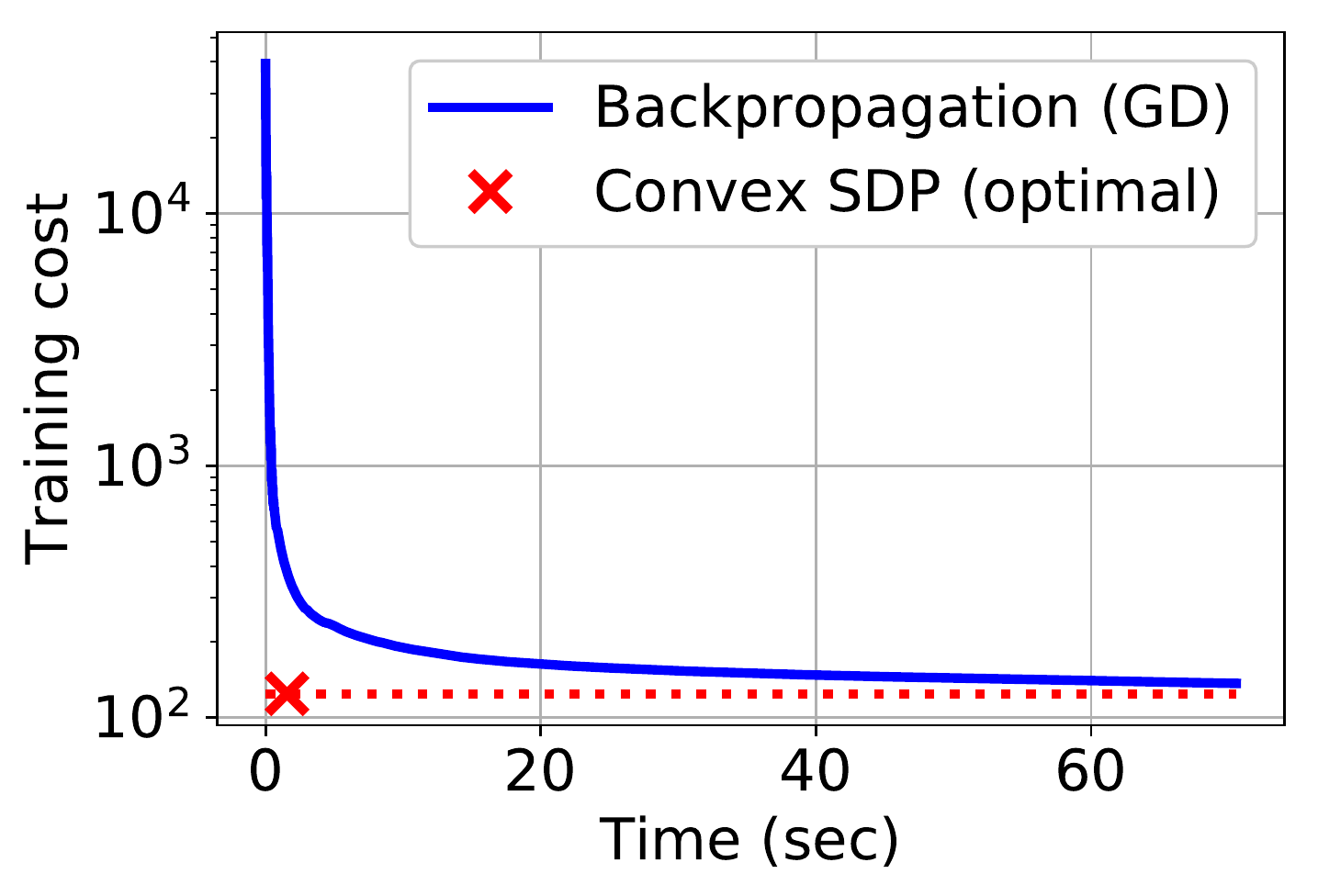}}
  \centerline{(c) Quad act ($500,20,20$)}\medskip
\end{minipage}

\begin{minipage}[b]{0.32\linewidth}
  \centering
  \centerline{\includegraphics[width=\columnwidth]{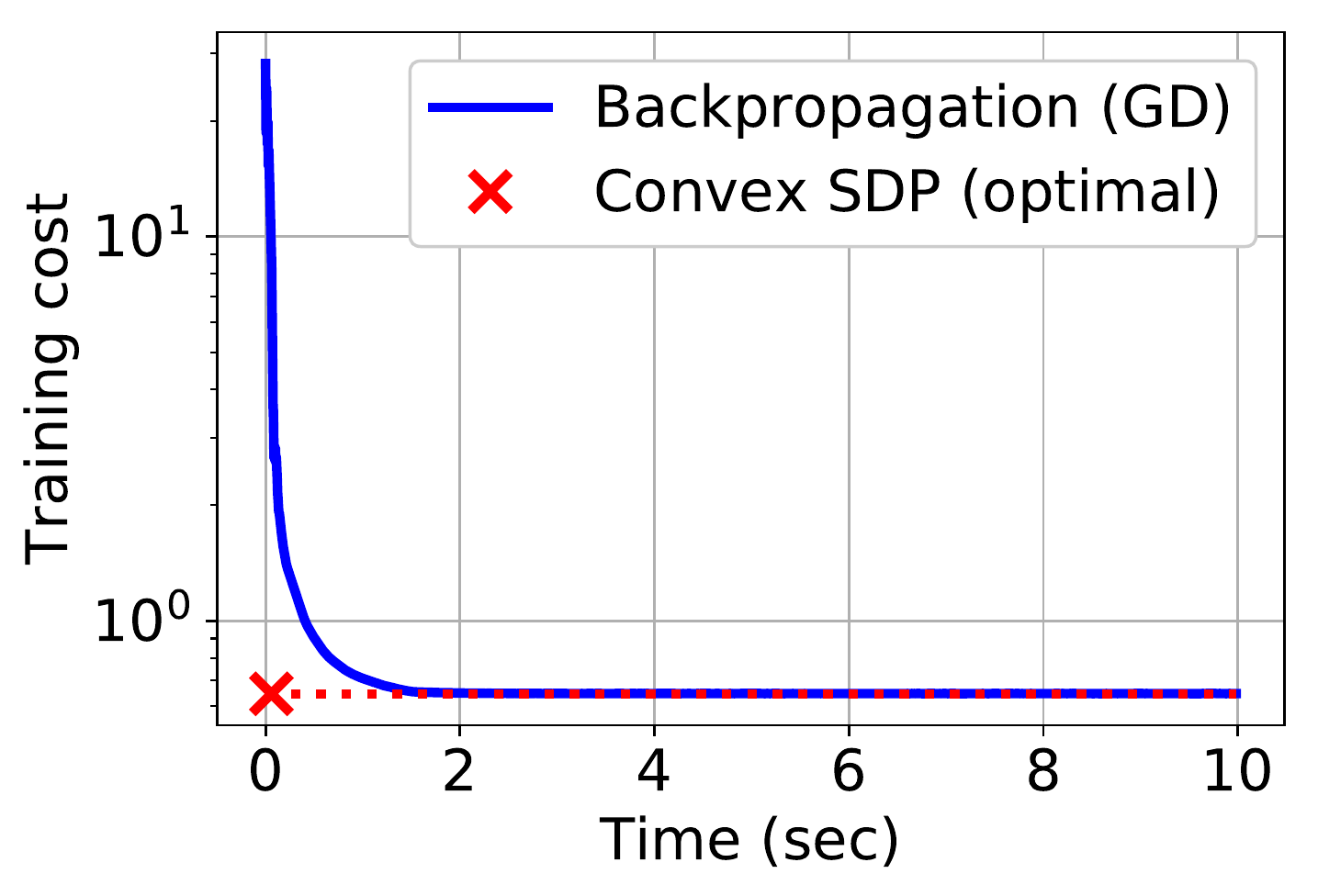}}
  \centerline{(d) Poly act ($10,20,10$)}\medskip
\end{minipage}
\hfill
\begin{minipage}[b]{0.32\linewidth}
  \centering
  \centerline{\includegraphics[width=\columnwidth]{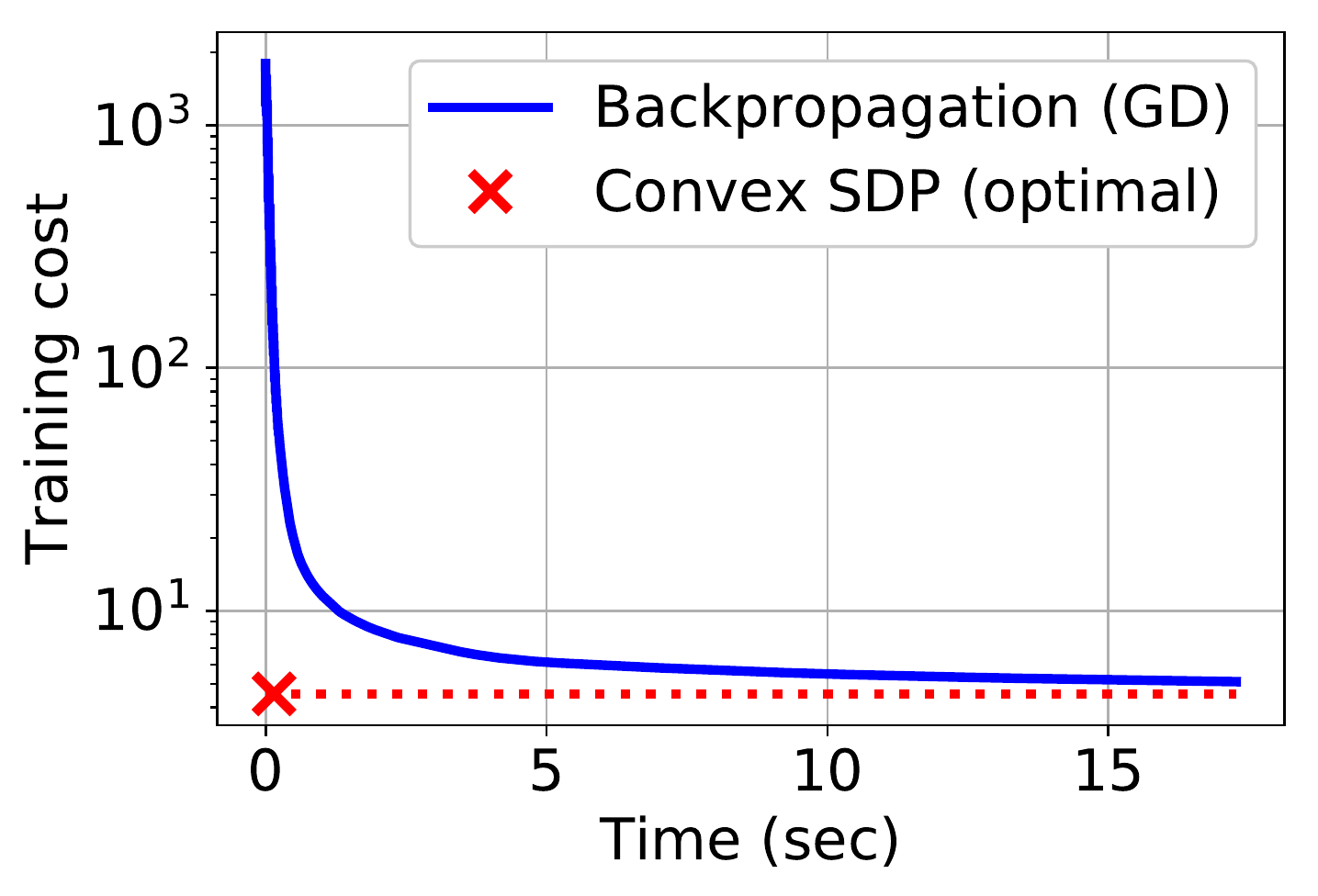}}
  \centerline{(e) Poly act ($100,20,23$)}\medskip
\end{minipage}
\hfill
\begin{minipage}[b]{0.32\linewidth}
  \centering
  \centerline{\includegraphics[width=\columnwidth]{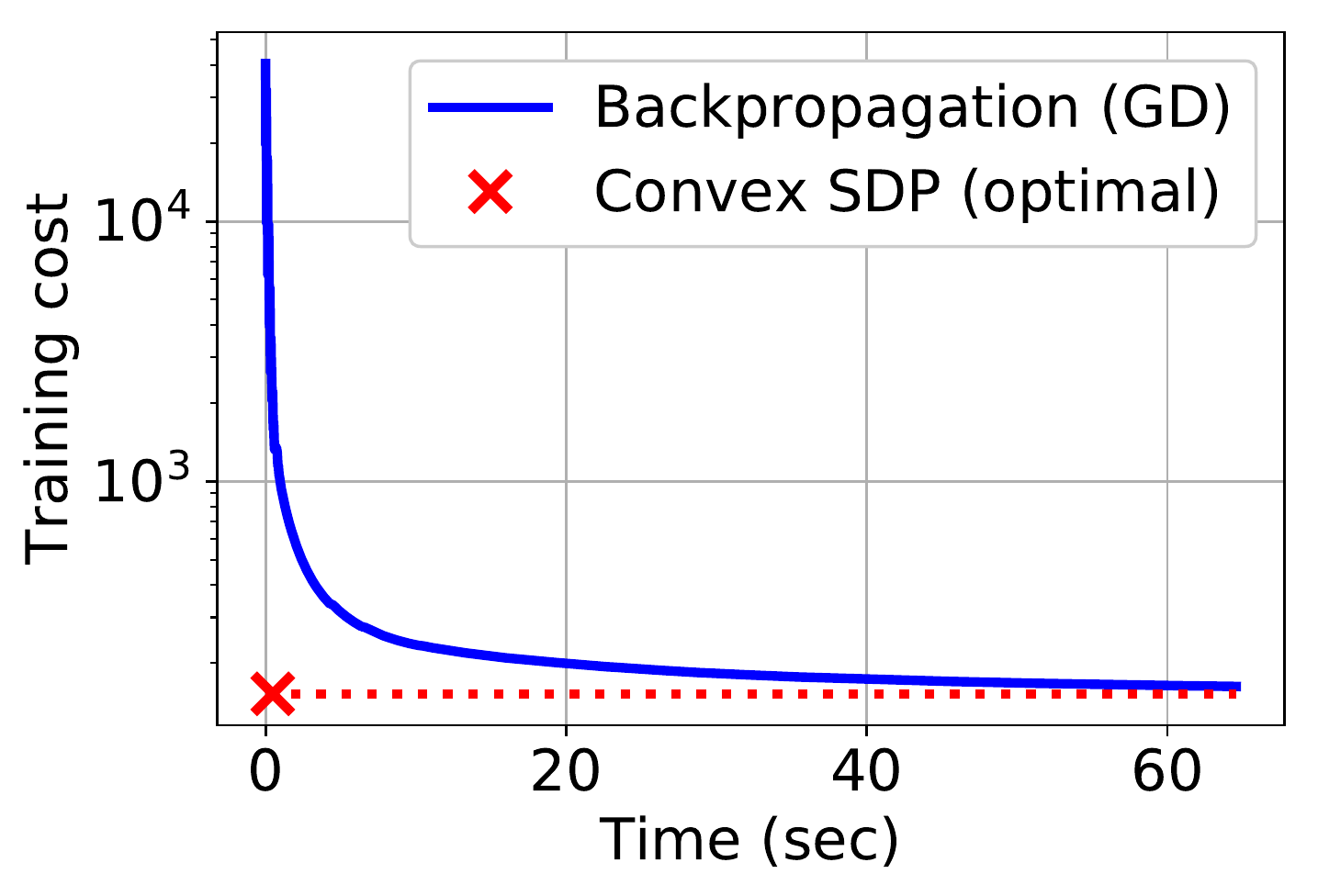}}
  \centerline{(f) Poly act ($500,20,35$)}\medskip
\end{minipage}

\begin{minipage}[b]{0.32\linewidth}
  \centering
  \centerline{\includegraphics[width=\columnwidth]{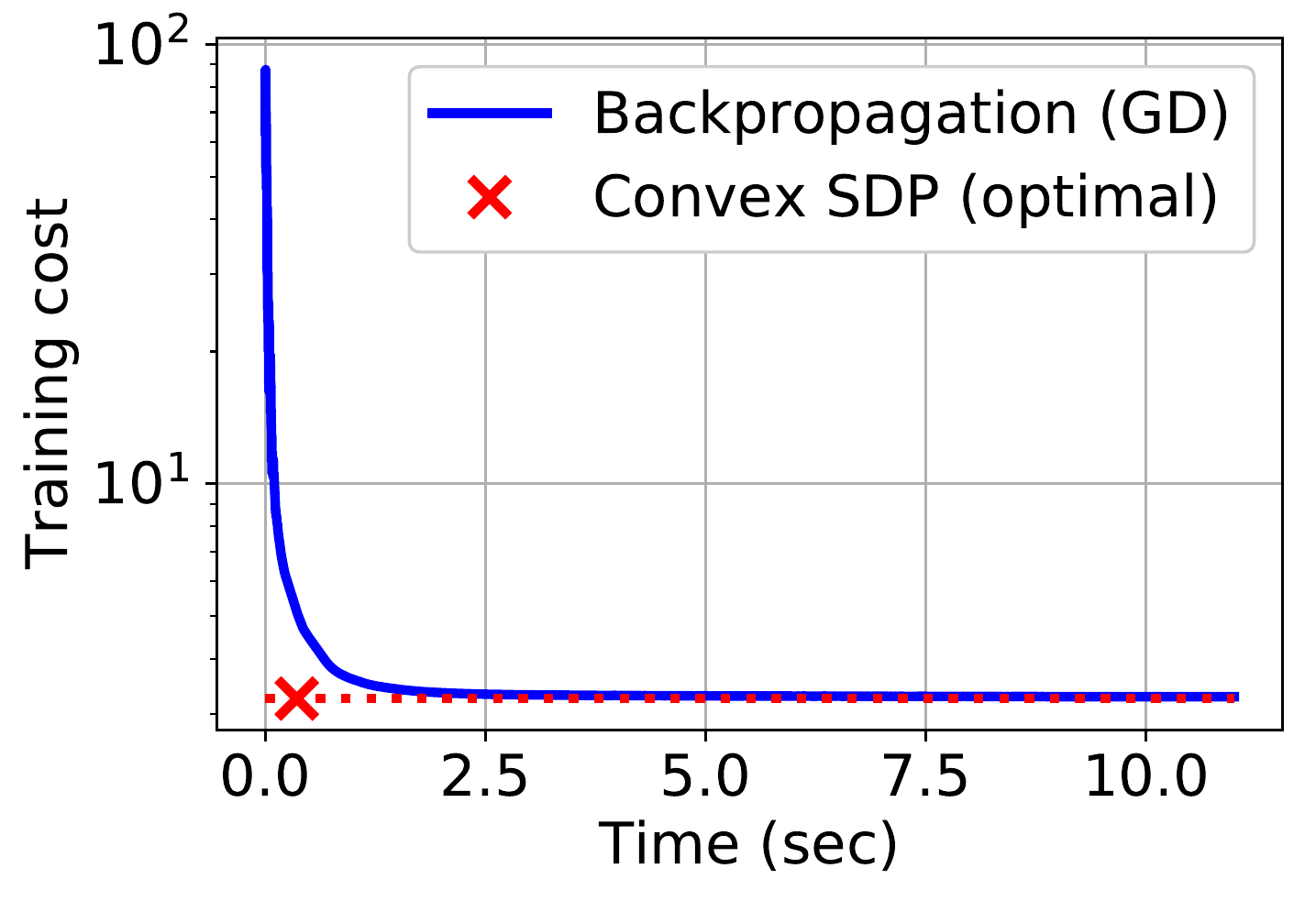}}
  \centerline{(g) Vect out $C=3$, ($10,20,61$)}\medskip
\end{minipage}
\hfill
\begin{minipage}[b]{0.32\linewidth}
  \centering
  \centerline{\includegraphics[width=\columnwidth]{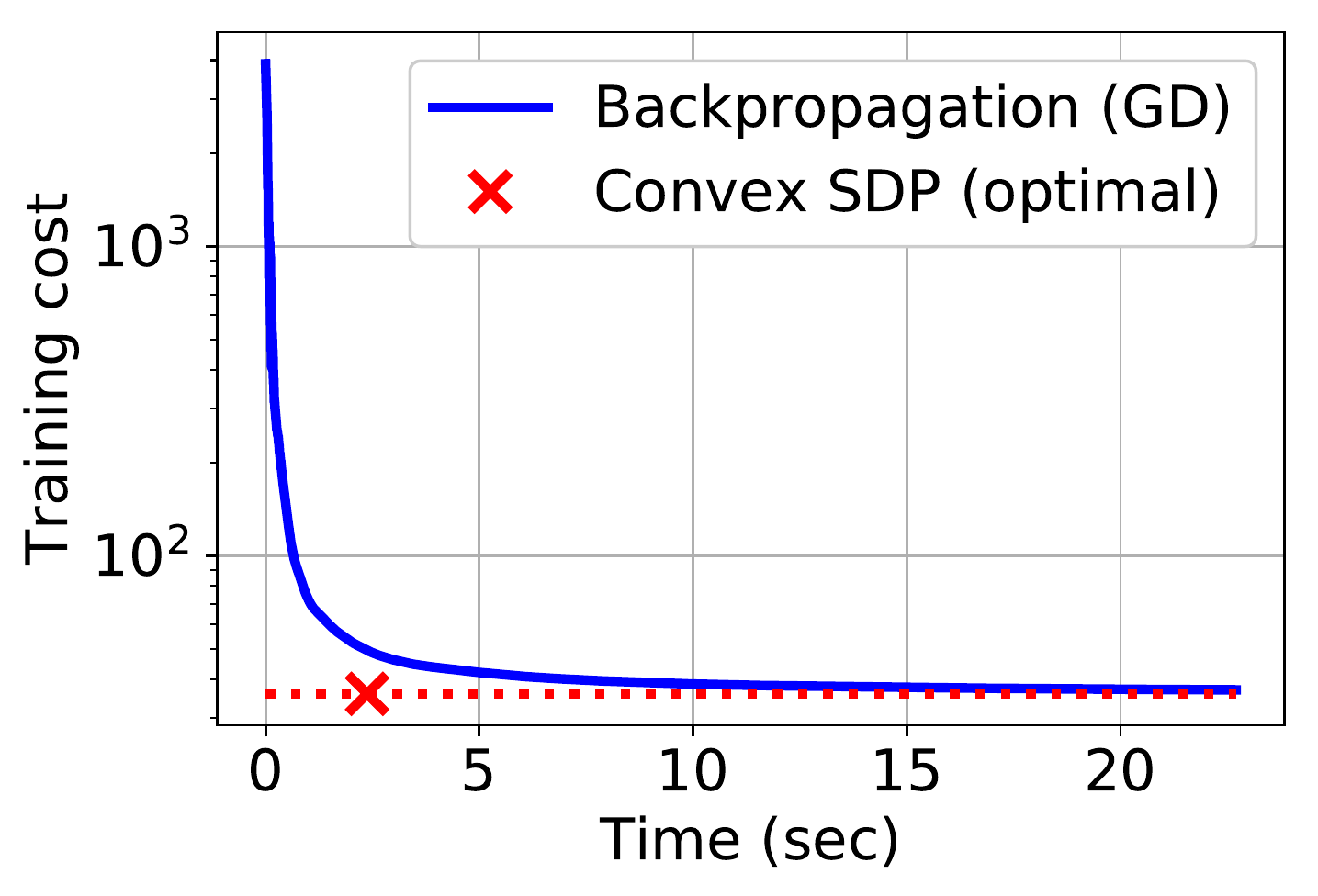}}
  \centerline{(h) Vect out $C=3$, ($100,20,86$)}\medskip
\end{minipage}
\hfill
\begin{minipage}[b]{0.32\linewidth}
  \centering
  \centerline{\includegraphics[width=\columnwidth]{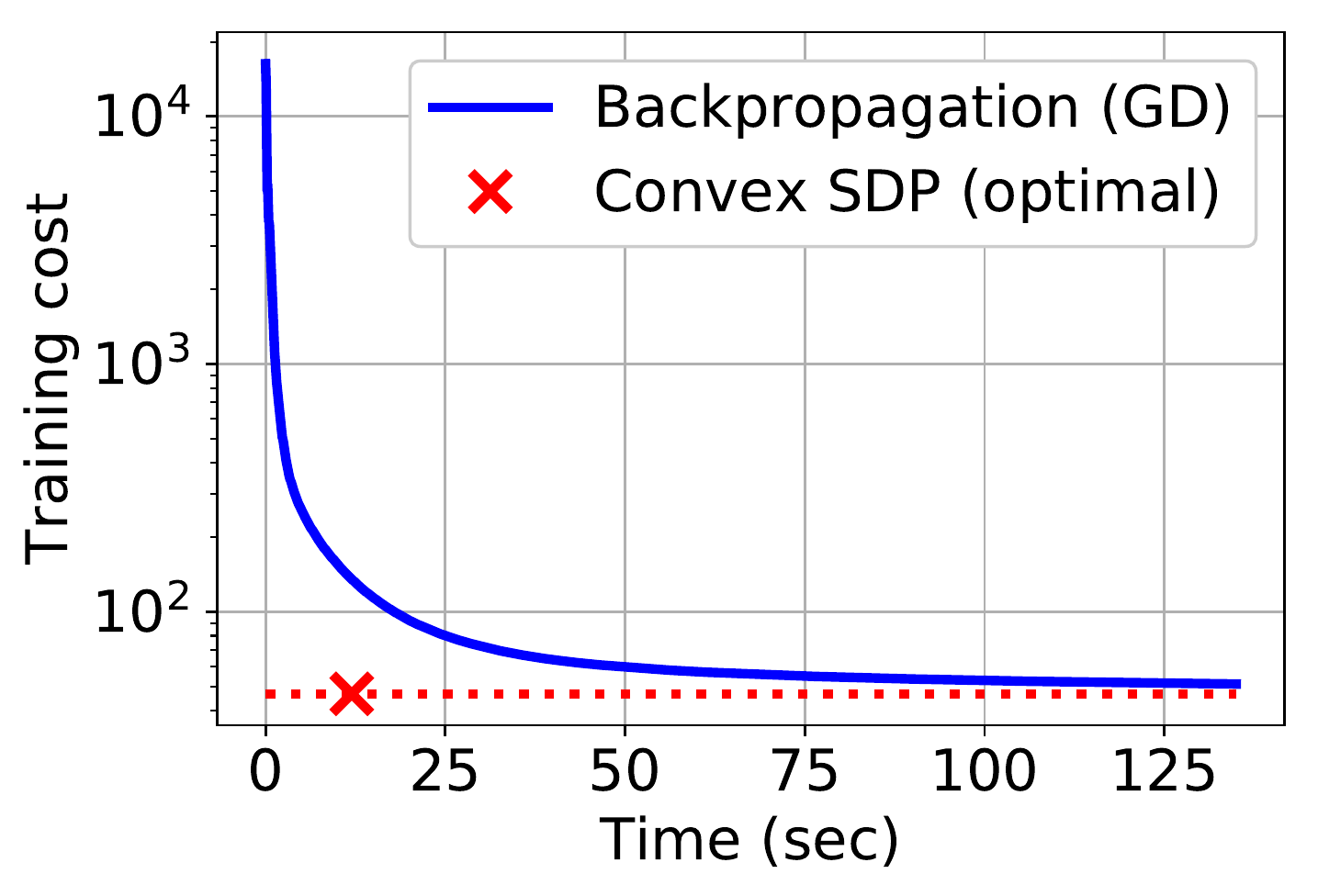}}
  \centerline{(i) Vect out $C=3$, ($500,20,75$)}\medskip
\end{minipage}

\begin{minipage}[b]{0.32\linewidth}
  \centering
  \centerline{\includegraphics[width=\columnwidth]{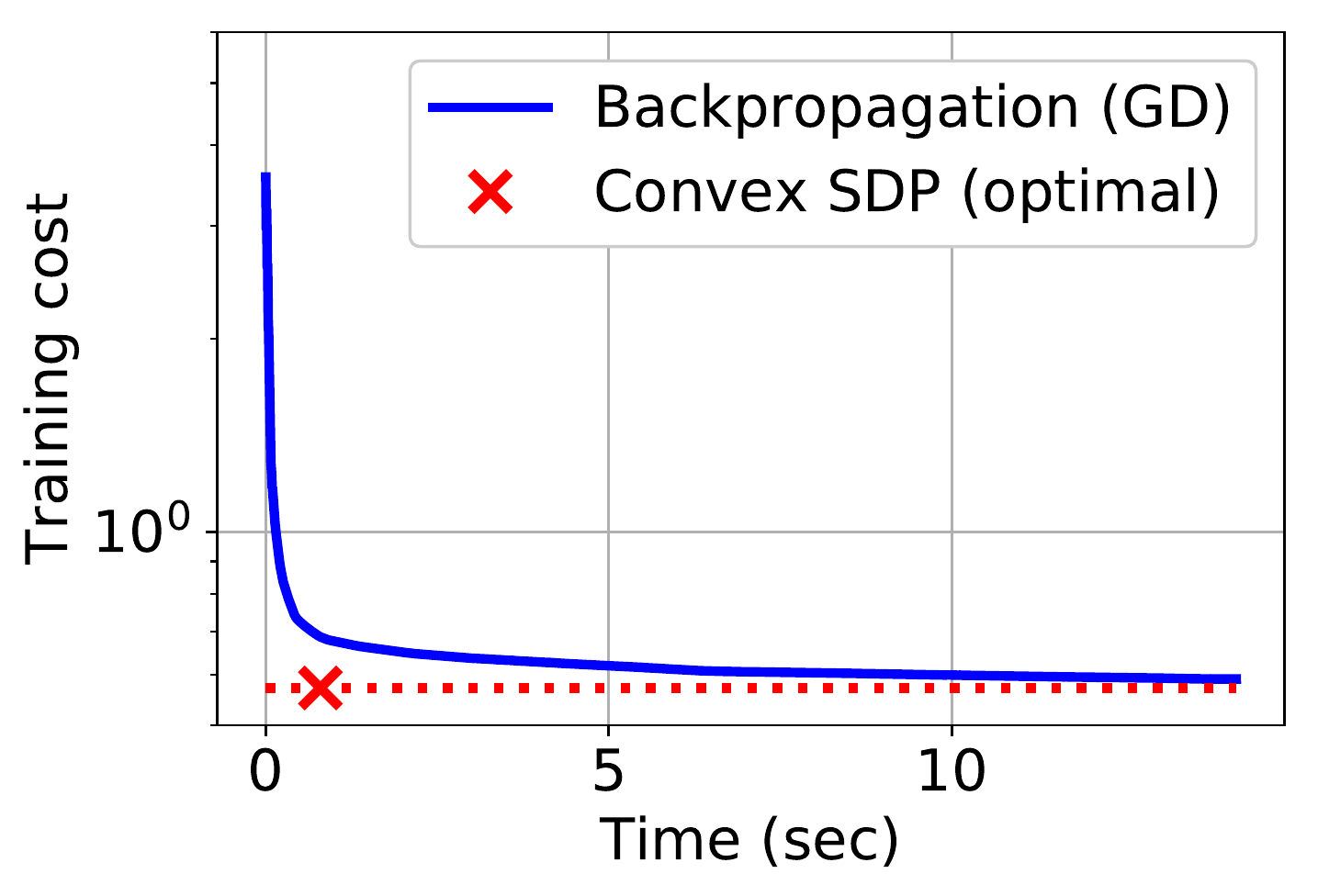}}
  \centerline{(j) Convol $f=3$, ($10,20,11$)}\medskip
\end{minipage}
\hfill
\begin{minipage}[b]{0.32\linewidth}
  \centering
  \centerline{\includegraphics[width=\columnwidth]{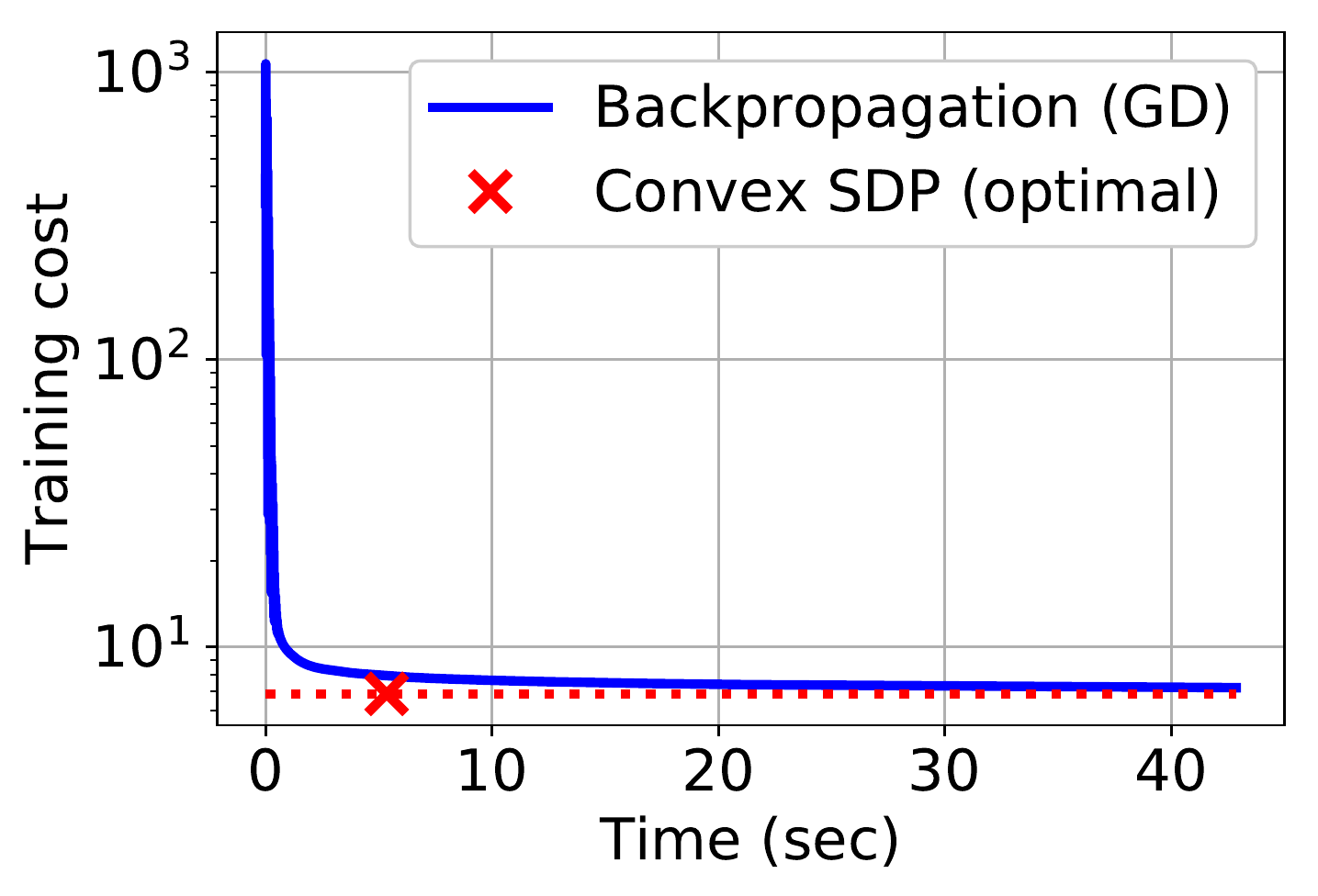}}
  \centerline{(k) Convol $f=3$, ($100,20,72$)}\medskip
\end{minipage}
\hfill
\begin{minipage}[b]{0.32\linewidth}
  \centering
  \centerline{\includegraphics[width=\columnwidth]{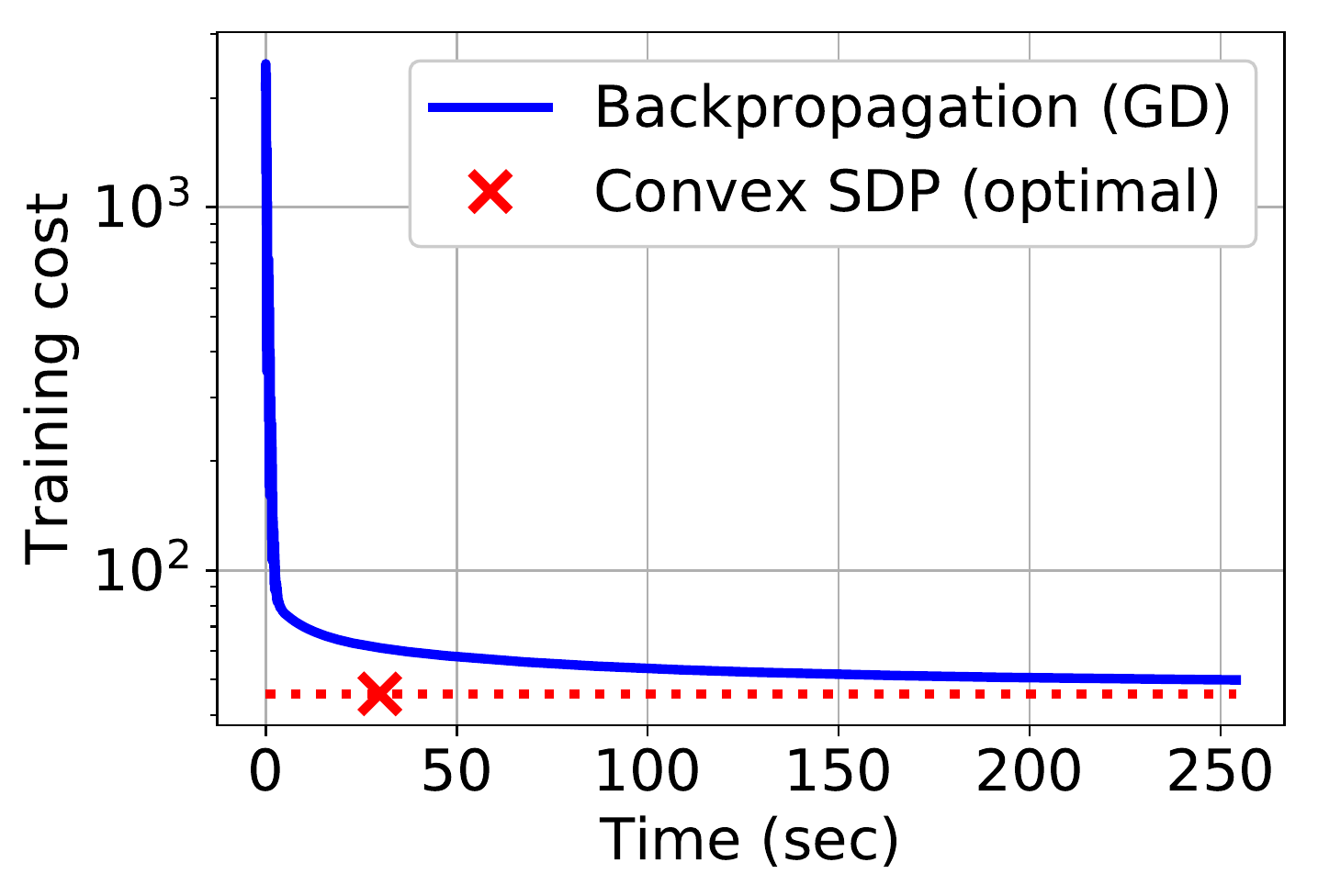}}
  \centerline{(l) Convol $f=3$, ($500,20,100$)}\medskip
\end{minipage}

\caption{The numbers in the sub-captions refer to the parameters ($n,d,m^*$). These figures show the training cost against time for backpropagation (blue solid curves) and the convex problem (red cross shows timing of the convex solver) for the following problems: a,b,c: Quadratic activation scalar output, d,e,f: Polynomial activation scalar output, g,h,i: Polynomial activation vector output, j,k,l: Polynomial activation convolutional. The data is artificially generated with $5$ planted neurons and the data matrix is the element-wise 4'th power of an i.i.d. Gaussian matrix. The regularization coefficient is $\beta=0.1$ in all of the experiments. The polynomial coefficients for the architectures with polynomial activation are $a=0.09$, $b=0.5$, $c=0.47$ (i.e. the ReLU approximation coefficients).}
\label{fig:verifying_theoretical_claims}
\end{figure}

\subsection{Experiments on UCI datasets}

We now show how the derived convex programs perform in the context of classification datasets. The datasets used in this subsection are from the UCI machine learning repository \cite{uci2019datasets}. The plots in Figure \ref{fig:uci_binary_classification} show the training and test set costs and classification accuracies for \textit{binary} classification datasets and the plots in Figure \ref{fig:uci_multiclass_classification} are for \textit{multiclass} classification datasets. The convex program used for solving the binary classification problem is the scalar output polynomial activation problem given in \eqref{eq:polyact_convex_program_final} and for the multiclass problem it is the vector output version given in \eqref{eq:vector_output_convex_program}. 

We note that the training cost plots of Figure \ref{fig:uci_binary_classification} and \ref{fig:uci_multiclass_classification} are consistent with the theoretical results. The accuracy plots show that the convex programs achieve the same final accuracy of the non-convex models or higher accuracies in shorter amounts of time.

\begin{figure} 
\begin{minipage}[b]{0.24\linewidth}
  \centering
  \centerline{\includegraphics[width=\columnwidth]{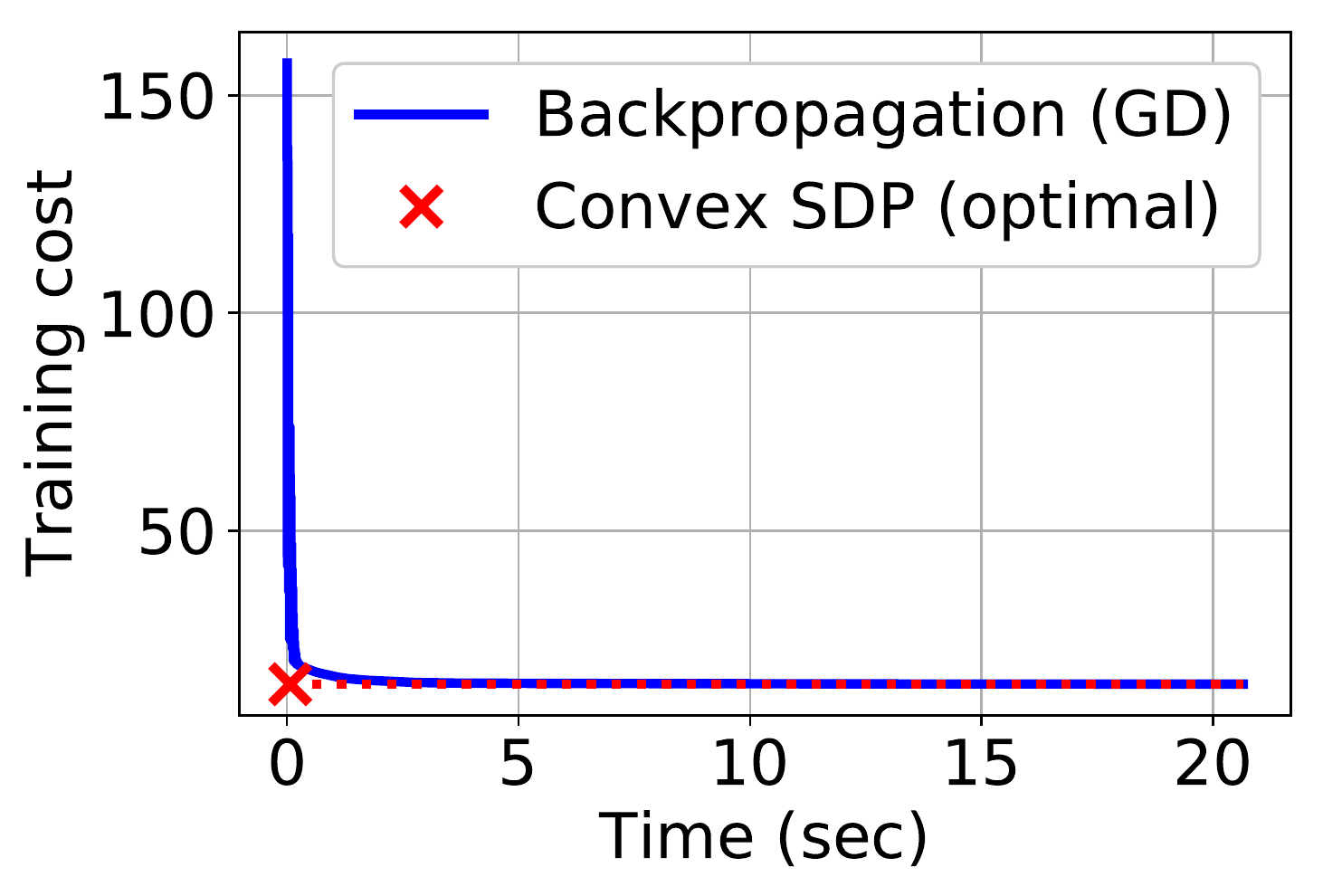}}
  \centerline{(a) DS1, training cost}\medskip
\end{minipage}
\hfill
\begin{minipage}[b]{0.24\linewidth}
  \centering
  \centerline{\includegraphics[width=\columnwidth]{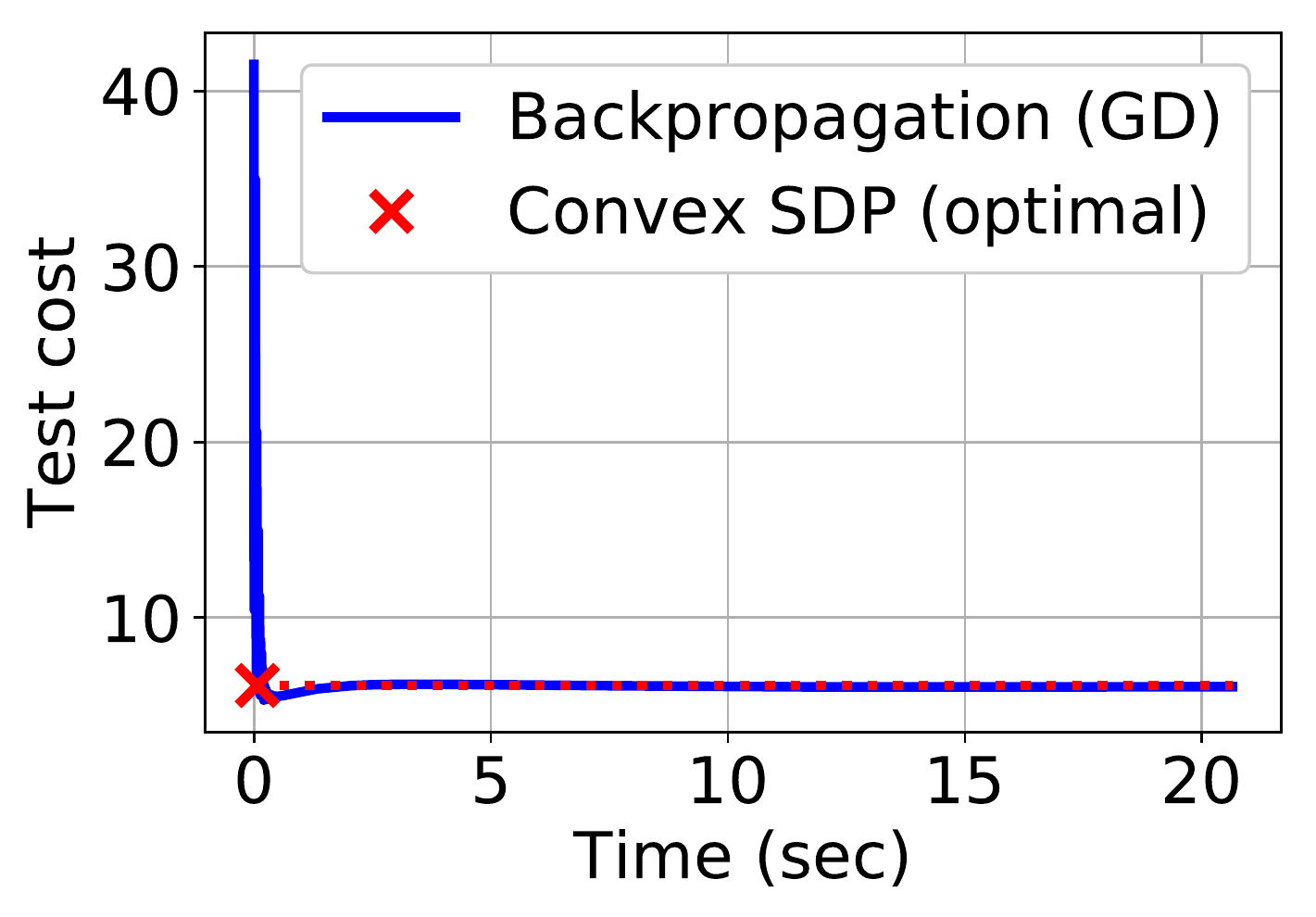}}
  \centerline{(b) DS1, test cost}\medskip
\end{minipage}
\hfill
\begin{minipage}[b]{0.24\linewidth}
  \centering
  \centerline{\includegraphics[width=\columnwidth]{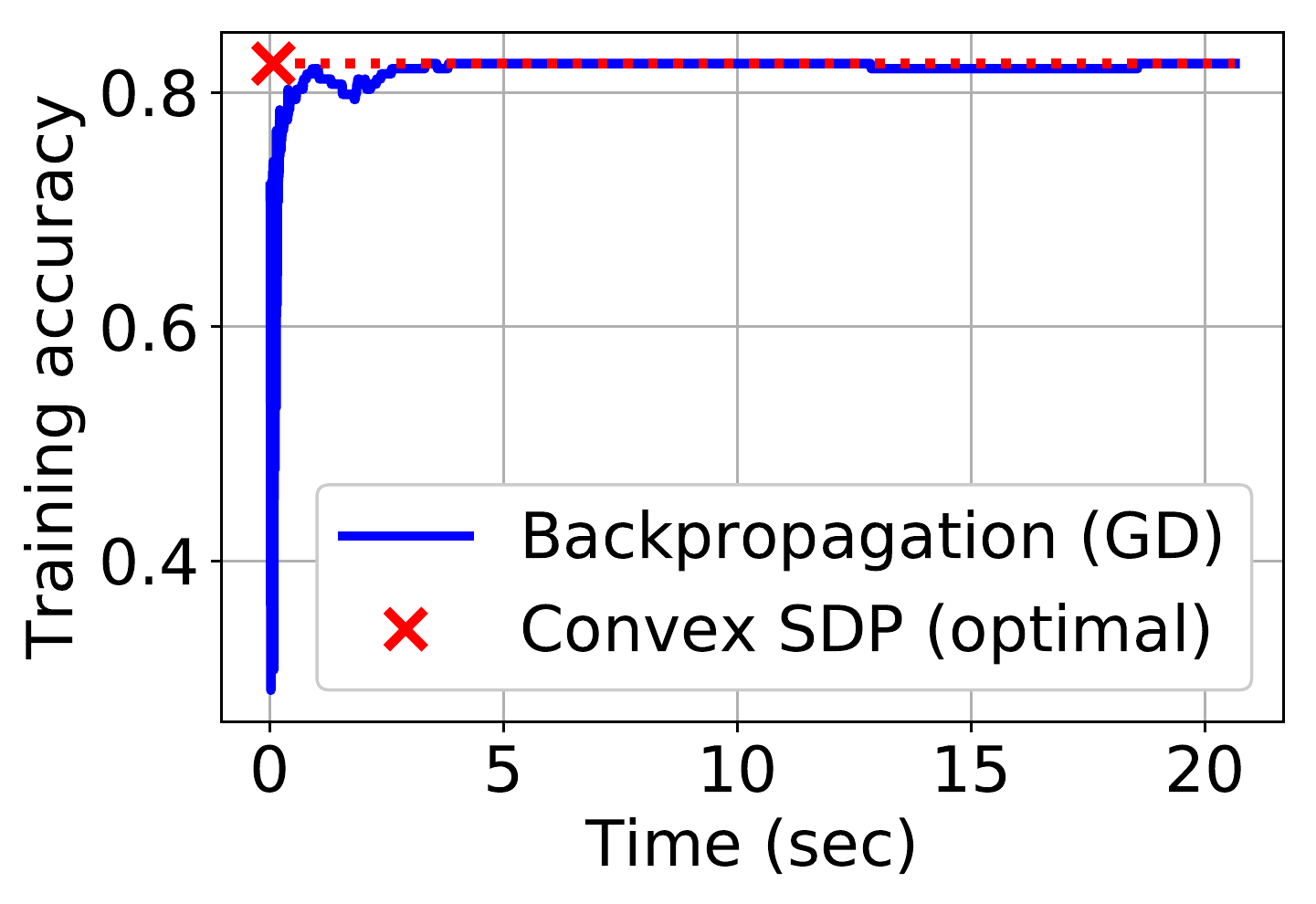}}
  \centerline{(c) DS1, training accuracy}\medskip
\end{minipage}
\hfill
\begin{minipage}[b]{0.24\linewidth}
  \centering
  \centerline{\includegraphics[width=\columnwidth]{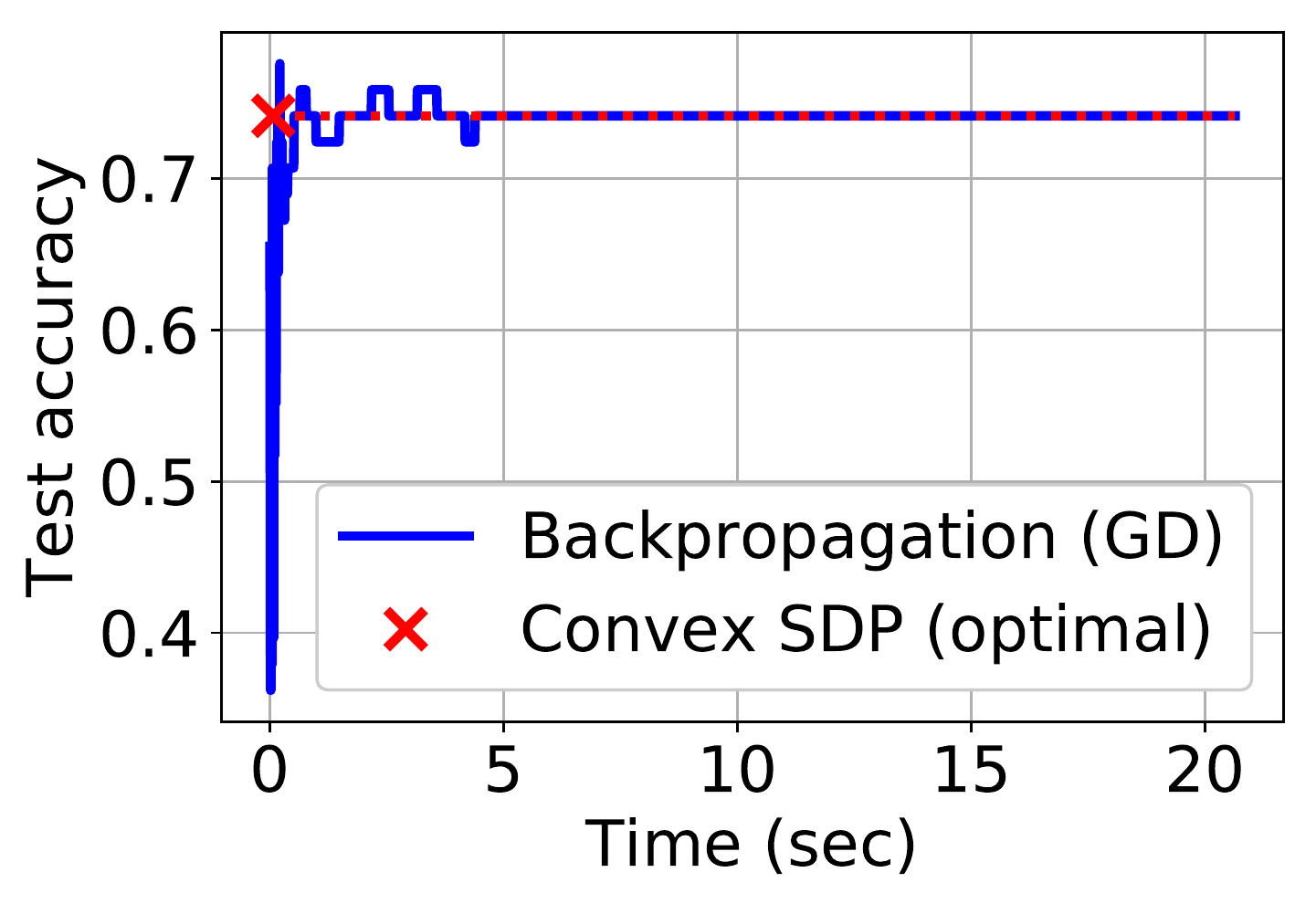}}
  \centerline{(d) DS1, test accuracy}\medskip
\end{minipage}

\begin{minipage}[b]{0.24\linewidth}
  \centering
  \centerline{\includegraphics[width=\columnwidth]{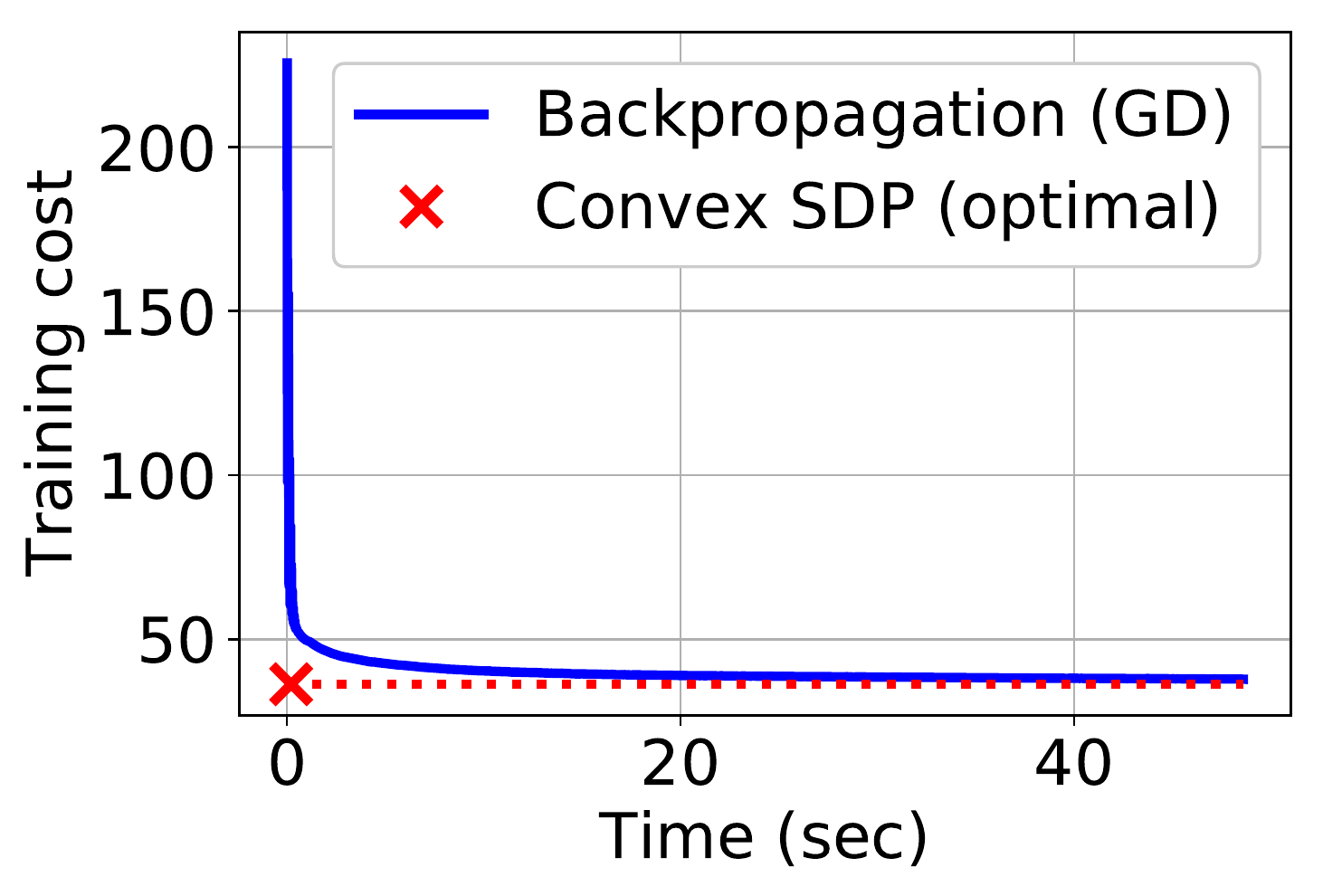}}
  \centerline{(e) DS2, training cost}\medskip
\end{minipage}
\hfill
\begin{minipage}[b]{0.24\linewidth}
  \centering
  \centerline{\includegraphics[width=\columnwidth]{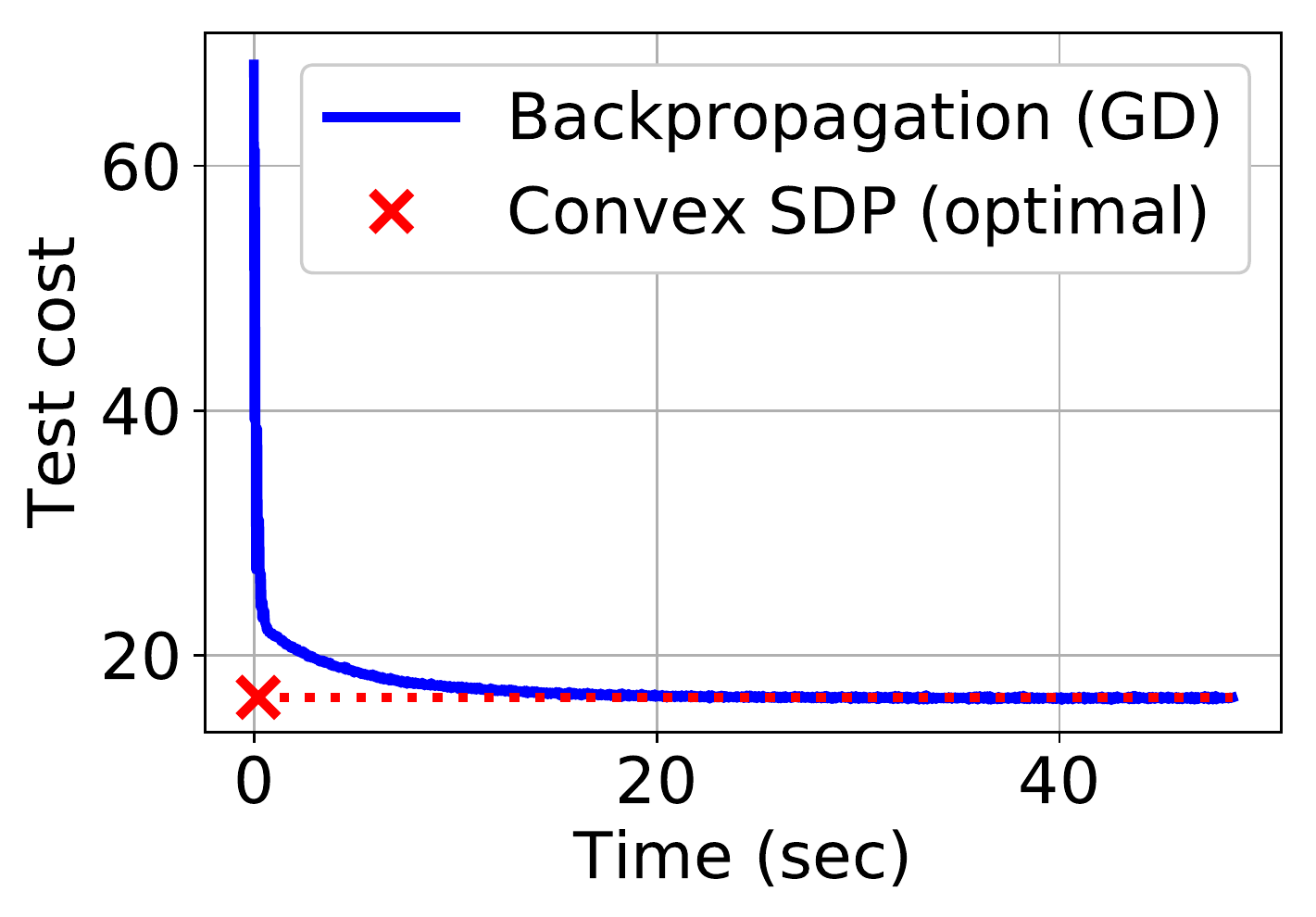}}
  \centerline{(f) DS2, test cost}\medskip
\end{minipage}
\hfill
\begin{minipage}[b]{0.24\linewidth}
  \centering
  \centerline{\includegraphics[width=\columnwidth]{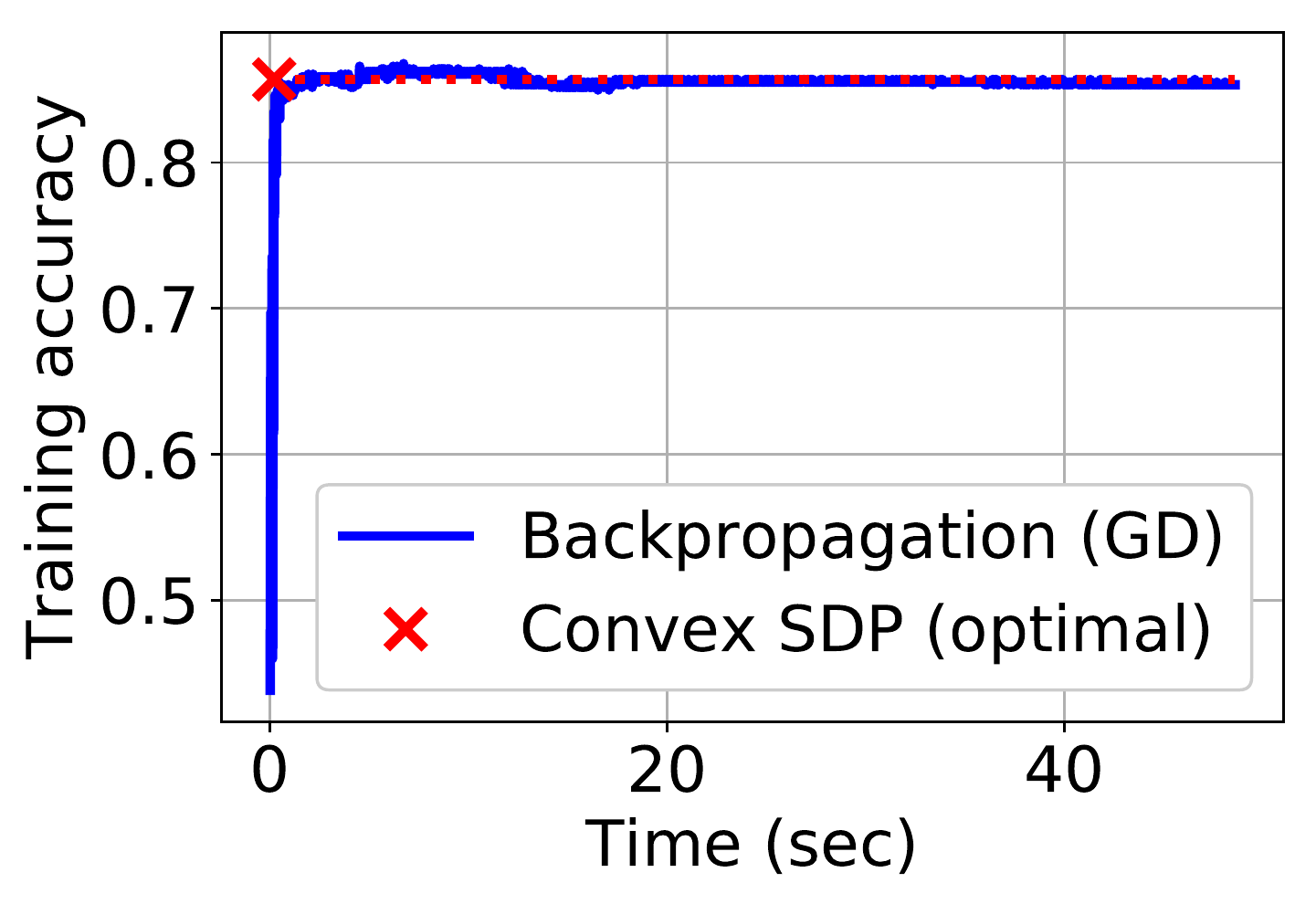}}
  \centerline{(g) DS2, training accuracy}\medskip
\end{minipage}
\hfill
\begin{minipage}[b]{0.24\linewidth}
  \centering
  \centerline{\includegraphics[width=\columnwidth]{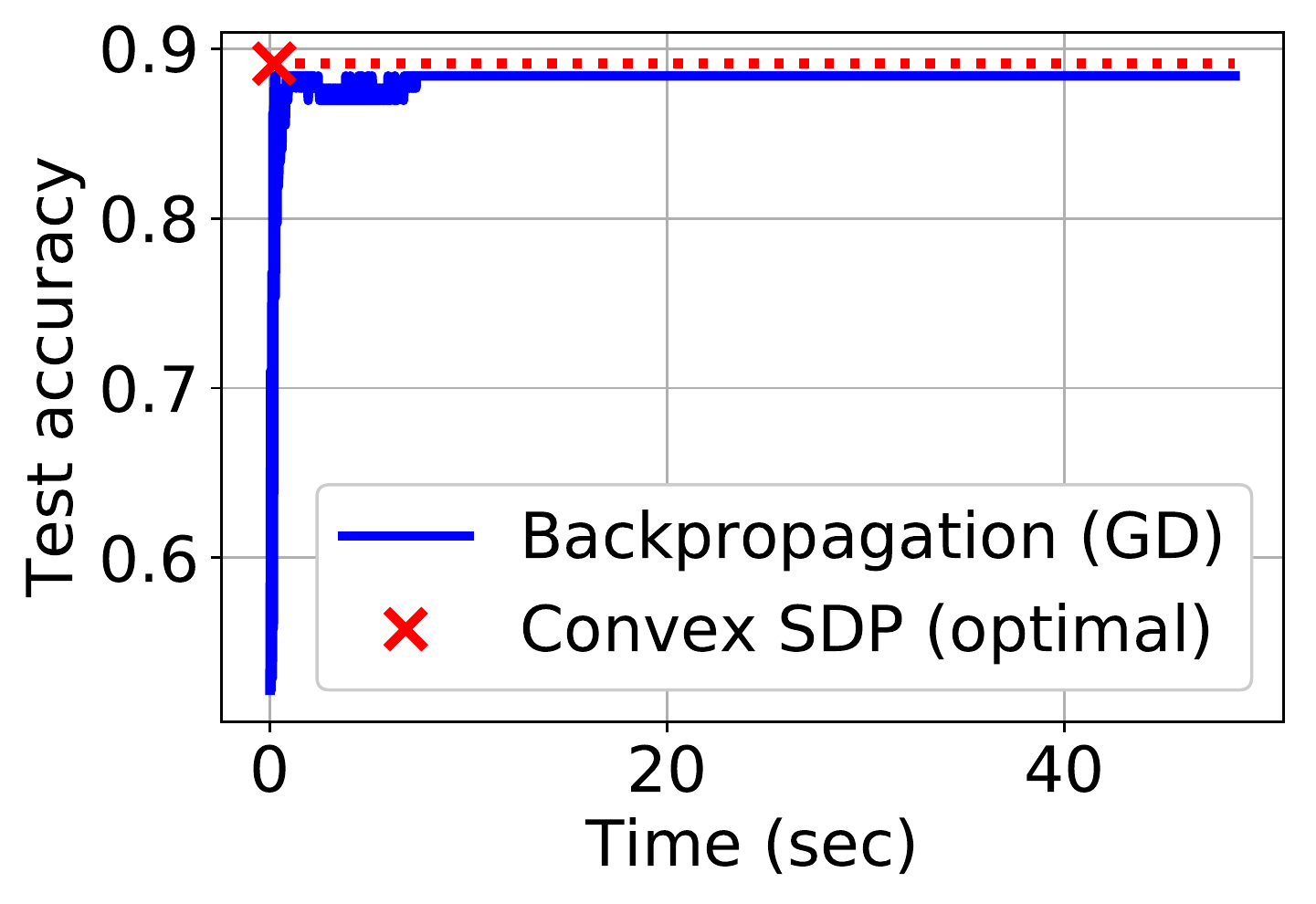}}
  \centerline{(h) DS2, test accuracy}\medskip
\end{minipage}
\caption{Results on UCI binary classification datasets. DS1: dataset 1 is the breast cancer dataset ($n=228,d=9$), DS2: dataset 2 is the credit approval dataset ($n=552,d=15$). Polynomial activation with $a=0.09$, $b=0.5$, $c=0.47$ is used. Number of neurons that the convex program found is $16$ and $18$ for DS1 and DS2, respectively. The regularization coefficient is $\beta=0.01$ and $\beta=10$ for DS1 and DS2, respectively.}
\label{fig:uci_binary_classification}
\end{figure}

\begin{figure} 
\begin{minipage}[b]{0.24\linewidth}
  \centering
  \centerline{\includegraphics[width=\columnwidth]{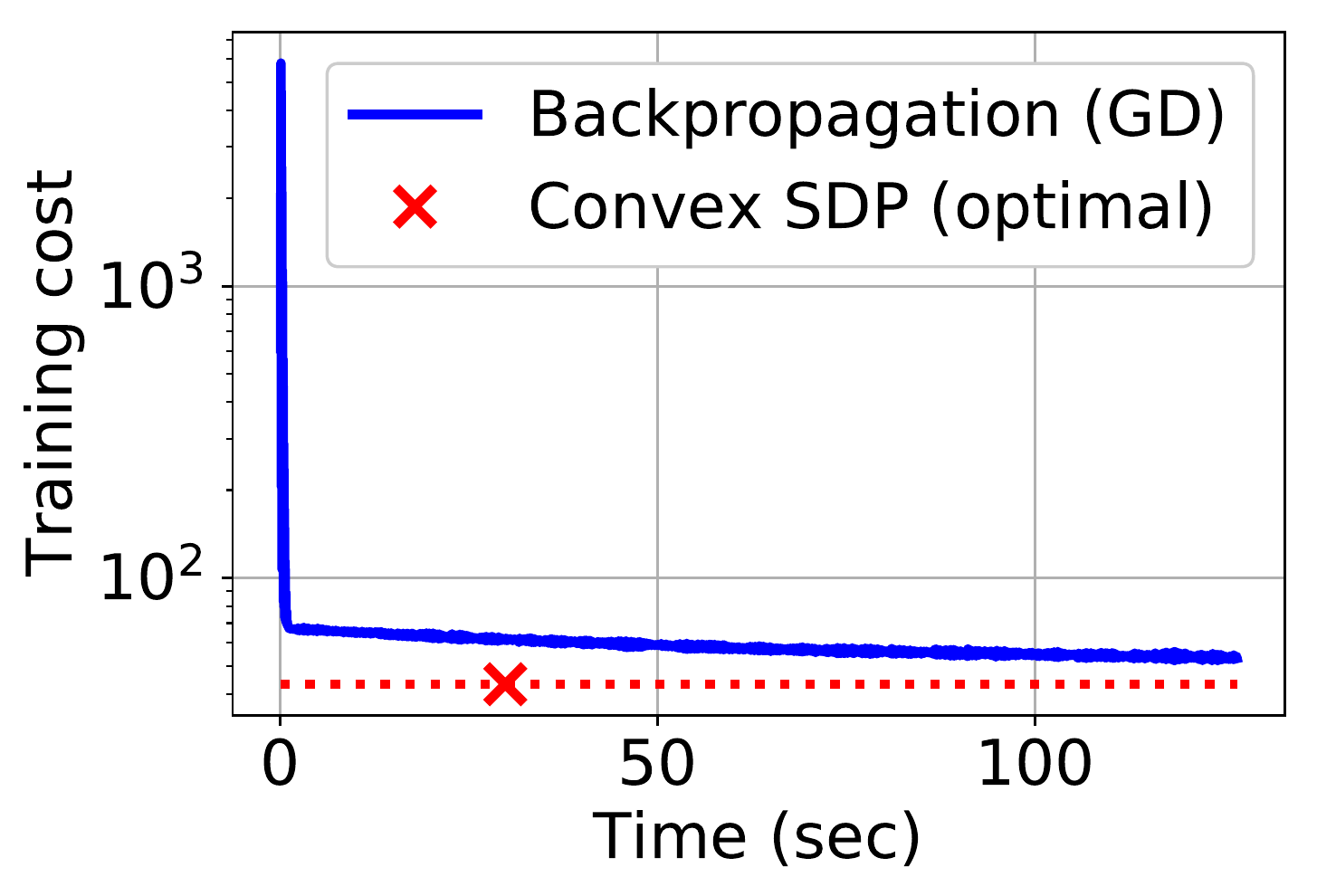}}
  \centerline{(a) DS3, training cost}\medskip
\end{minipage}
\hfill
\begin{minipage}[b]{0.24\linewidth}
  \centering
  \centerline{\includegraphics[width=\columnwidth]{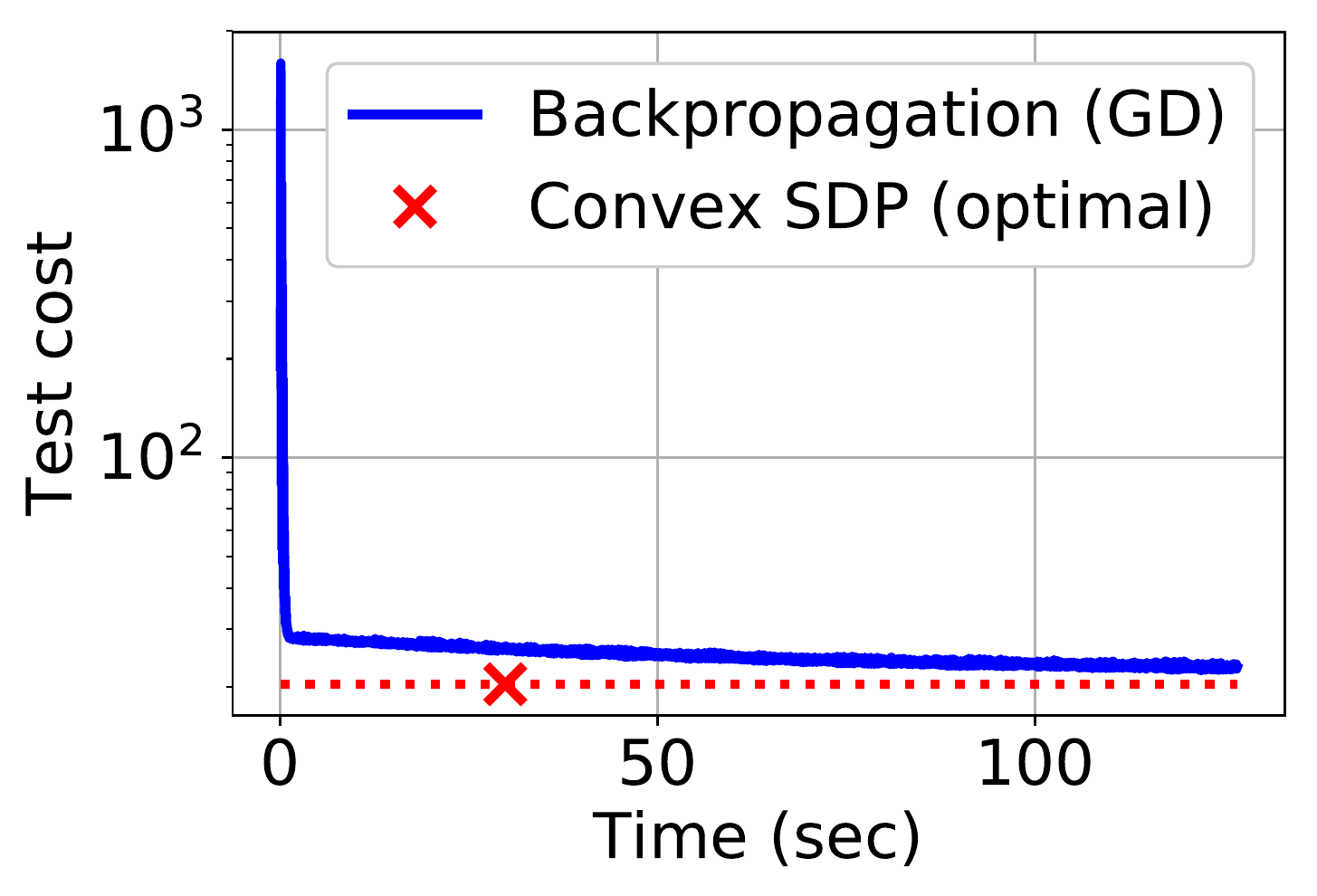}}
  \centerline{(b) DS3, test cost}\medskip
\end{minipage}
\hfill
\begin{minipage}[b]{0.24\linewidth}
  \centering
  \centerline{\includegraphics[width=\columnwidth]{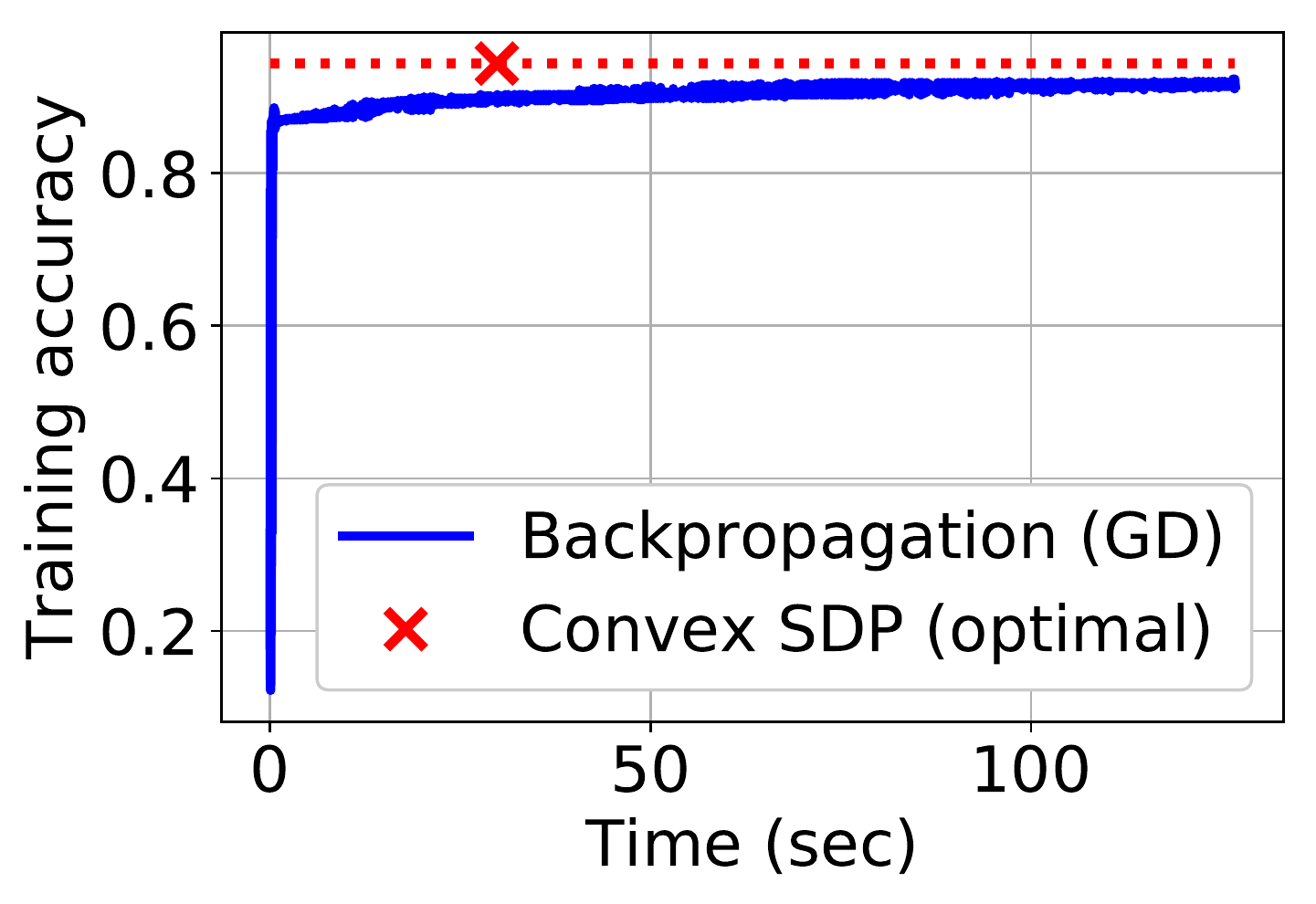}}
  \centerline{(c) DS3, training accuracy}\medskip
\end{minipage}
\hfill
\begin{minipage}[b]{0.24\linewidth}
  \centering
  \centerline{\includegraphics[width=\columnwidth]{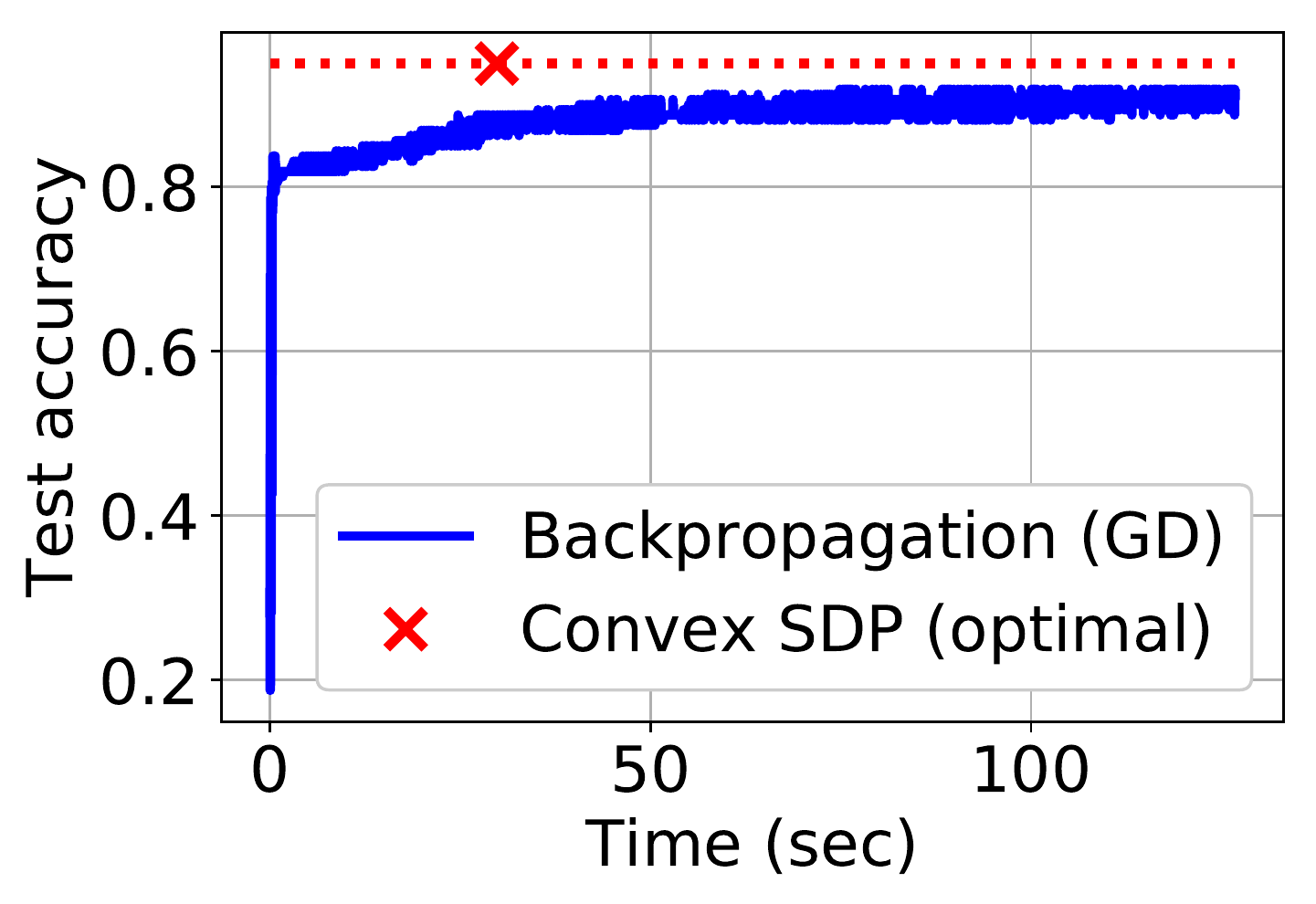}}
  \centerline{(d) DS3, test accuracy}\medskip
\end{minipage}

\begin{minipage}[b]{0.24\linewidth}
  \centering
  \centerline{\includegraphics[width=\columnwidth]{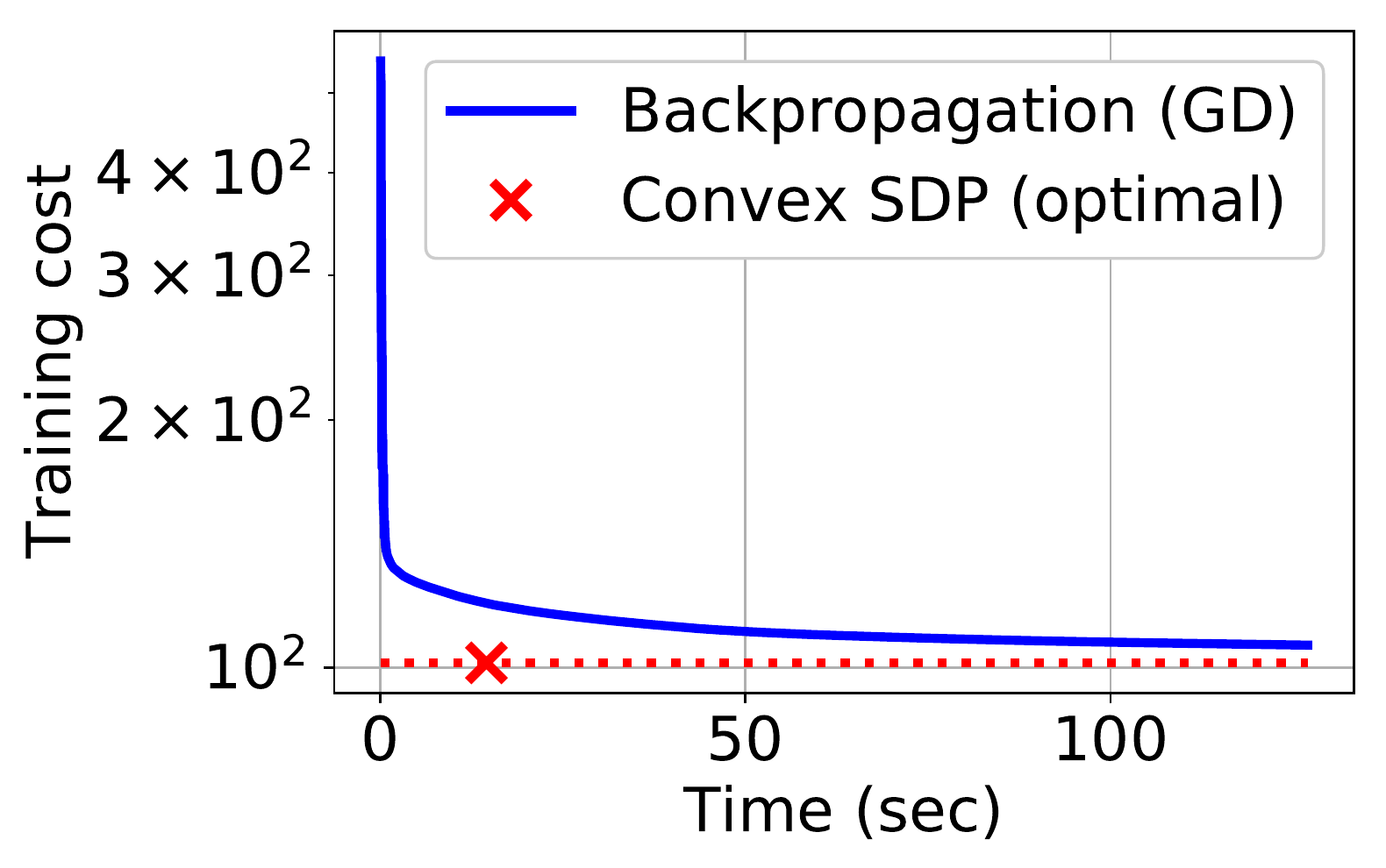}}
  \centerline{(e) DS4, training cost}\medskip
\end{minipage}
\hfill
\begin{minipage}[b]{0.24\linewidth}
  \centering
  \centerline{\includegraphics[width=\columnwidth]{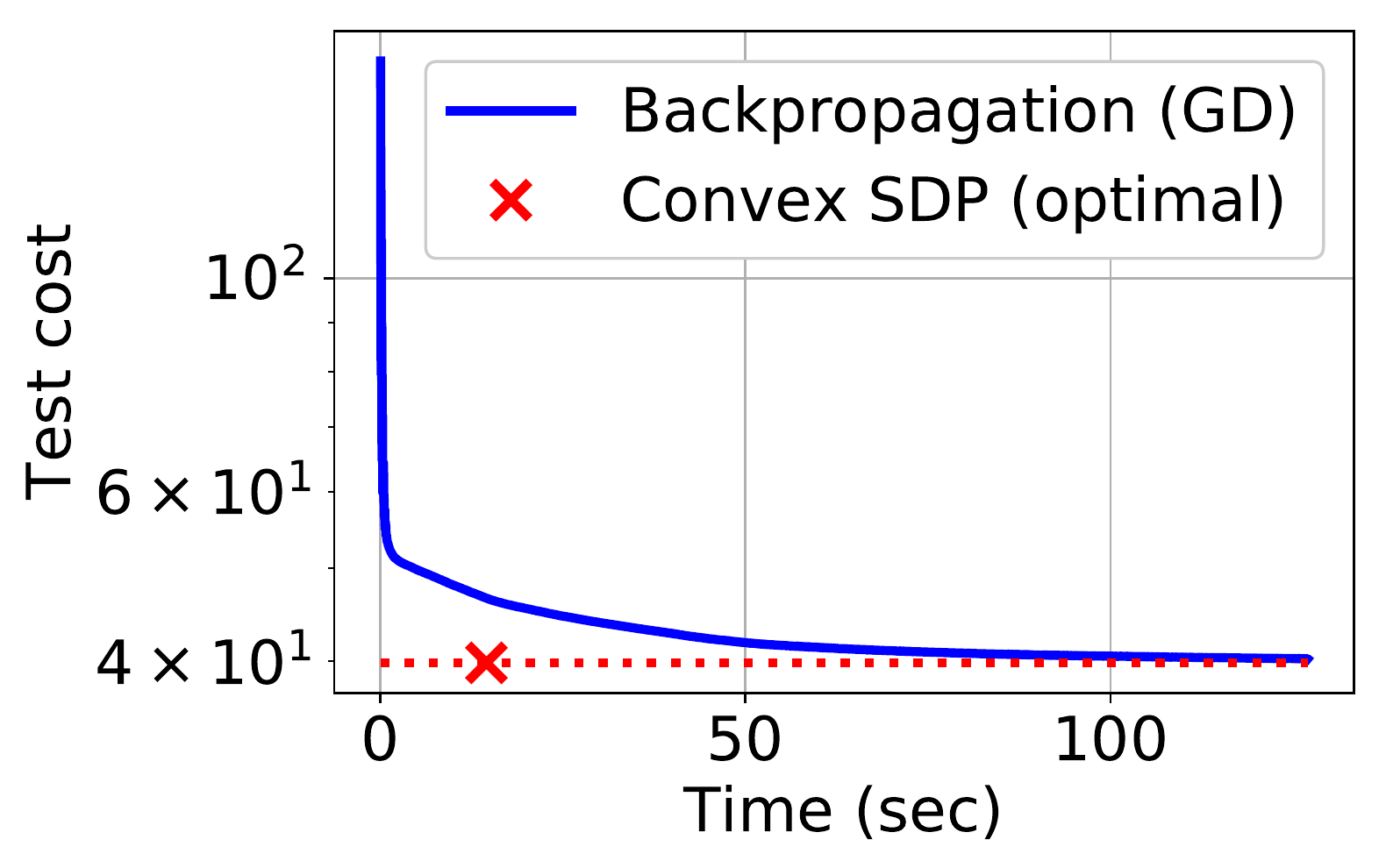}}
  \centerline{(f) DS4, test cost}\medskip
\end{minipage}
\hfill
\begin{minipage}[b]{0.24\linewidth}
  \centering
  \centerline{\includegraphics[width=\columnwidth]{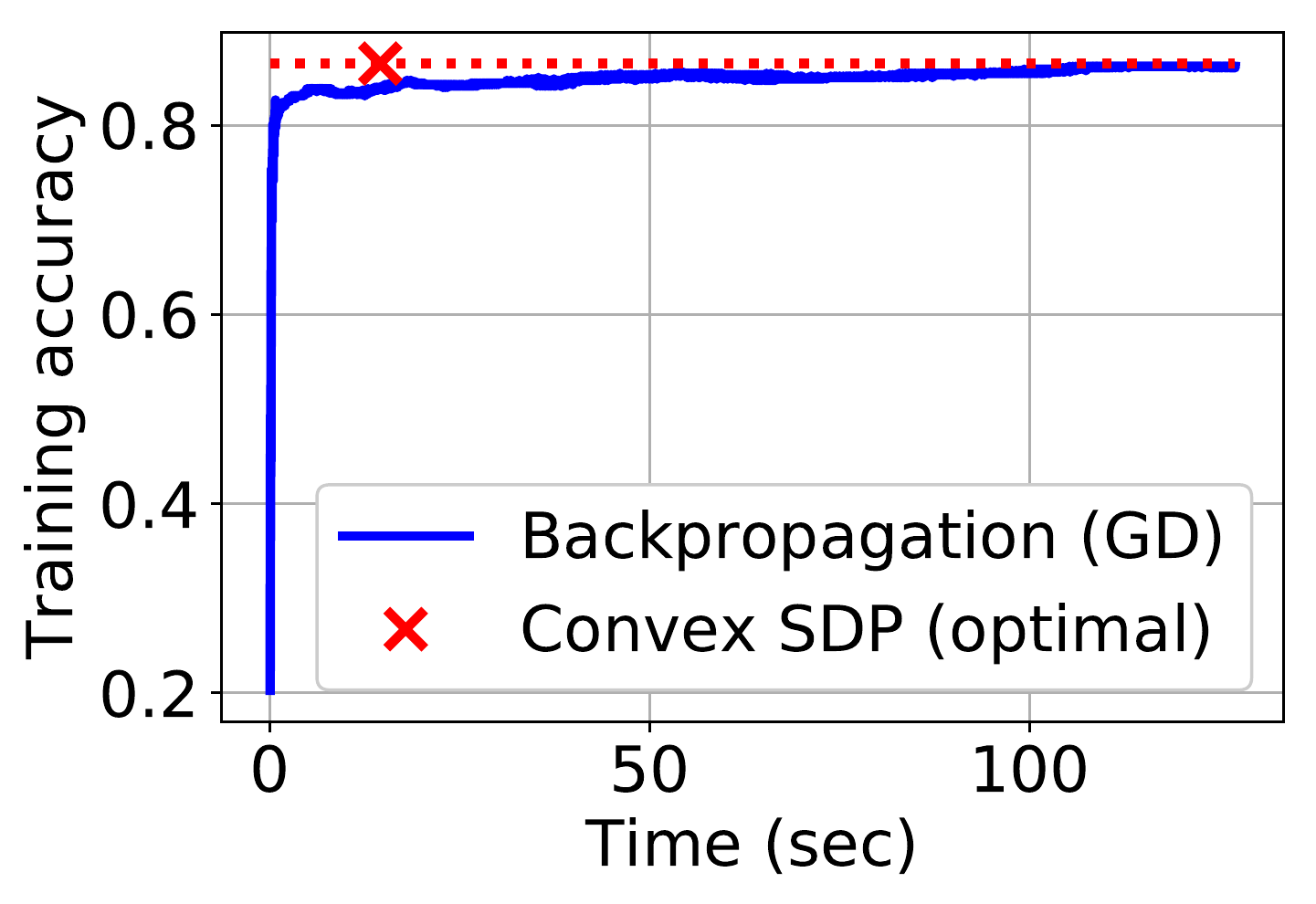}}
  \centerline{(g) DS4, training accuracy}\medskip
\end{minipage}
\hfill
\begin{minipage}[b]{0.24\linewidth}
  \centering
  \centerline{\includegraphics[width=\columnwidth]{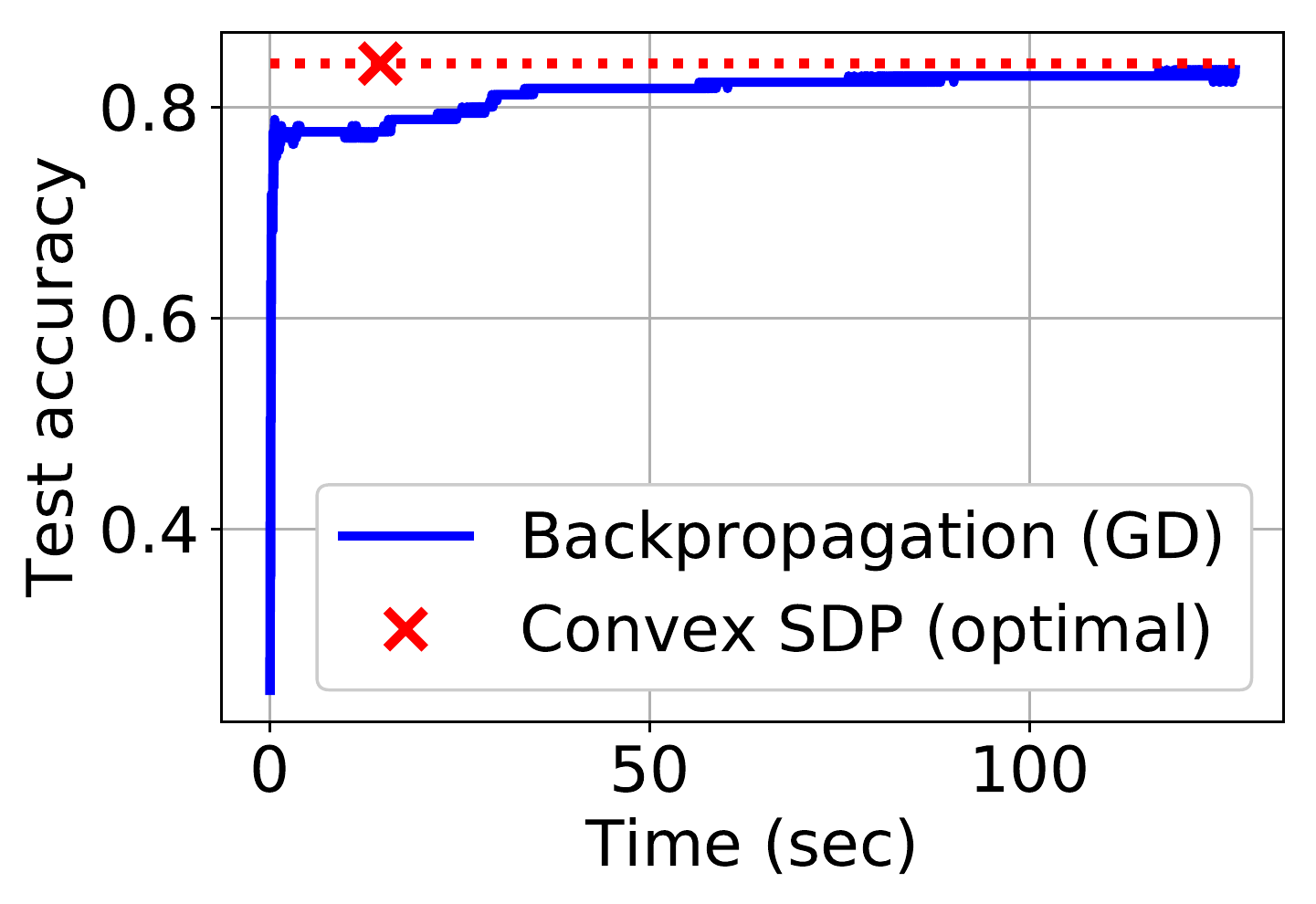}}
  \centerline{(h) DS4, test accuracy}\medskip
\end{minipage}
\caption{Results on UCI \textit{multiclass} classification datasets. DS3: dataset 3 is the annealing dataset ($n=638,d=31,C=5$), DS4: dataset 4 is the statlog vehicle dataset ($n=676,d=18,C=4$). Polynomial activation with $a=0.09$, $b=0.5$, $c=0.47$ is used. Number of neurons that the convex program found is $172$ and $107$ for DS3 and DS4, respectively. The regularization coefficient is $\beta=1$ both for DS3 and DS4.}
\label{fig:uci_multiclass_classification}
\end{figure}

Table \ref{tbl:uci_table_poly_act} shows the classification accuracies of various fully connected neural network architectures on binary classification UCI datasets. For each dataset, the training and validation partitions are as pre-processed in \cite{delgado2014ucipreprocessed}. The training and validation partitions are used to select the best hyperparameters. The hyperparameter search for the non-convex models includes searching for the best regularization coefficient $\beta$ and learning rate. Gradient descent has been used to optimize the non-convex models and the number of epochs is $1000$. After determining the best hyperparameters, we compute the $4$-fold cross validation accuracy and report it in this table. The partitions for the $4$-fold cross validation are also the same as those pre-processed by \cite{delgado2014ucipreprocessed}. Furthermore, for the results shown in Table \ref{tbl:uci_table_poly_act}, the number of neurons for all the non-convex models is set to $2(d+1)$, which is the maximum number of neurons that the polynomial activation convex SDP could output (see Theorem \ref{thm:poly_act_thm}). Table \ref{tbl:uci_table_poly_act} shows that the convex SDP achieves better or similar accuracy values compared to the non-convex models on most of the datasets.

\begin{table}
\scalebox{0.90}{
\begin{tabular}{lll|ll|lll|l}
\hline
 dataset                      & $n$   & $d$   & R-Q   & P-C   & Cvx 111 & Cvx r-app & Cvx s-app & max(Cvx) \\ \hline
 acute-inflammation           & 120 & 6   & 100.0 & 100.0 & 100.0   & 100.0     & 100.0     & 100.0    \\
 acute-nephritis              & 120 & 6   & 100.0 & 100.0 & 100.0   & 100.0     & 100.0     & 100.0    \\
 breast-cancer                & 286 & 9   & 69.37 & \textbf{73.59} & 73.59   & 72.89     & 72.89     & \textbf{73.59}    \\
 breast-cancer-wisc-diag      & 569 & 30  & 79.05 & 95.95 & 95.42   & 96.13     & 96.13     & \textbf{96.13}    \\
 breast-cancer-wisc-prog      & 198 & 33  & \textbf{80.1}  & 79.08 & 77.55   & 79.59     & 77.55     & 79.59    \\
 congressional-voting         & 435 & 16  & 61.47 & 61.47 & 61.47   & 61.7      & 61.47     & \textbf{61.7}     \\
 conn-bench-sonar-mines-rocks & 208 & 60  & 79.81 & 79.33 & 81.73   & 79.81     & 75.0      & \textbf{81.73}    \\
 cylinder-bands               & 512 & 35  & 75.59 & 75.2  & 75.59   & 76.95     & 76.37     & \textbf{76.95}    \\
 echocardiogram               & 131 & 10  & 84.09 & 83.33 & 85.61   & 85.61     & 84.09     & \textbf{85.61}    \\
 fertility                    & 100 & 9   & \textbf{89.0}  & 86.0  & 88.0    & 88.0      & 88.0      & 88.0     \\
 haberman-survival            & 306 & 3   & 73.03 & \textbf{73.68} & 71.38   & 73.36     & 72.04     & 73.36    \\
 heart-hungarian              & 294 & 12  & 83.56 & 83.9  & 83.22   & 84.25     & 84.25     & \textbf{84.25}    \\
 hepatitis                    & 155 & 19  & 80.13 & \textbf{89.1}  & 80.13   & 77.56     & 80.13     & 80.13    \\
 horse-colic                  & 368 & 25  & 81.67 & 81.0  & 81.67   & 80.33     & 84.0      & \textbf{84.0}     \\
 ilpd-indian-liver            & 583 & 9   & \textbf{73.63} & 72.95 & 71.92   & 73.12     & 72.95     & 73.12    \\
 molec-biol-promoter          & 106 & 57  & 77.88 & 78.85 & 72.12   & 82.69     & 78.85     & \textbf{82.69}    \\
 monks-1                      & 556 & 6   & \textbf{84.68} & 70.16 & 75.81   & 81.45     & 81.45     & 81.45    \\
 parkinsons                   & 195 & 22  & 90.82 & 87.24 & 88.27   & 86.73     & 91.33     & \textbf{91.33}    \\
 pittsburg-bridges-T-OR-D     & 102 & 7   & \textbf{88.0}  & 88.0  & 82.0    & 87.0      & 87.0      & 87.0     \\
 planning                     & 182 & 12  & 71.67 & 71.11 & 71.67   & 71.11     & 71.11     & 71.67    \\
 spect                        & 265 & 22  & 60.0  & \textbf{75.0}  & 71.25   & 60.0      & 58.75     & 71.25    \\
 spectf                       & 267 & 44  & 72.5  & 75.0  & 58.75   & 60.0      & 77.5      & \textbf{77.5}     \\
 statlog-heart                & 270 & 13  & 82.46 & \textbf{85.07} & 81.72   & 83.58     & 83.21     & 83.58    \\
 vertebral-column-2clases     & 310 & 6   & 87.01 & 85.71 & 82.79   & 87.01     & 84.42     & 87.01    \\
\hline
\end{tabular}
}
\caption{Classification accuracies on binary classification UCI datasets. The first 3 columns are the dataset name, the number of samples $n$ in the dataset, and the dimension $d$ of the samples. The remaining columns show the classification accuracies (percentage) for various models. The highest accuracies for each dataset are shown in bold font. Abbreviations used in the table are as follows: R-Q: Non-convex two-layer neural network model with ReLU activation and quadratic regularization (i.e. weight decay), P-C: Non-convex two-layer neural network model with polynomial activation with coefficients $a=0.09, b=0.5, c=0.47$ and normalized first layer weights and $\ell_1$ norm regularization on the second layer weights, Cvx 111: Convex SDP with polynomial coefficients $a=1, b=1, c=1$, Cvx r-app: Convex SDP with polynomial coefficients $a=0.09, b=0.5, c=0.47$ (approximating ReLU activation), Cvx s-app: Convex SDP with polynomial coefficients $a=0.1, b=0.5, c=0.24$ (approximating swish activation), max(Cvx): The highest accuracy among the convex SDPs.
}
\label{tbl:uci_table_poly_act}
\end{table}


\subsection{Comparison with ReLU Networks}

We compare the classification accuracies for polynomial activation and ReLU activation in Figure \ref{fig:relu_comparison} on three different binary classification UCI datasets. The regularization coefficient has been picked separately for polynomial activation and ReLU activation networks to maximize the accuracy. Figure \ref{fig:relu_comparison} demonstrates that the convex SDP shows competitive accuracy performance and faster run times compared to ReLU activation networks.

\begin{figure} 
\begin{minipage}[b]{0.42\linewidth}
  \centering
  \centerline{\includegraphics[width=\columnwidth]{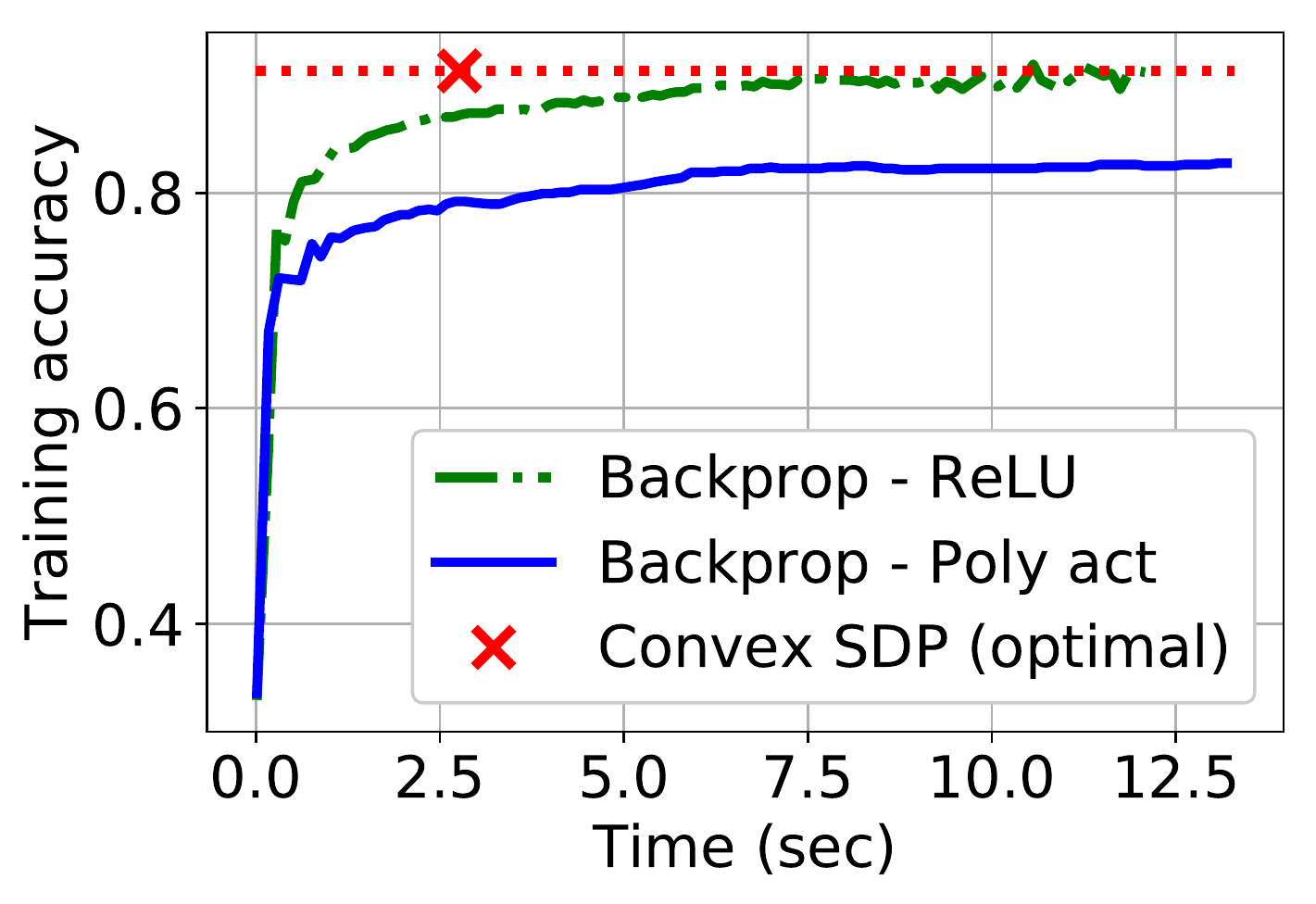}}
  \centerline{(a) DS1, training accuracy}\medskip
\end{minipage}
\hfill
\begin{minipage}[b]{0.42\linewidth}
  \centering
  \centerline{\includegraphics[width=\columnwidth]{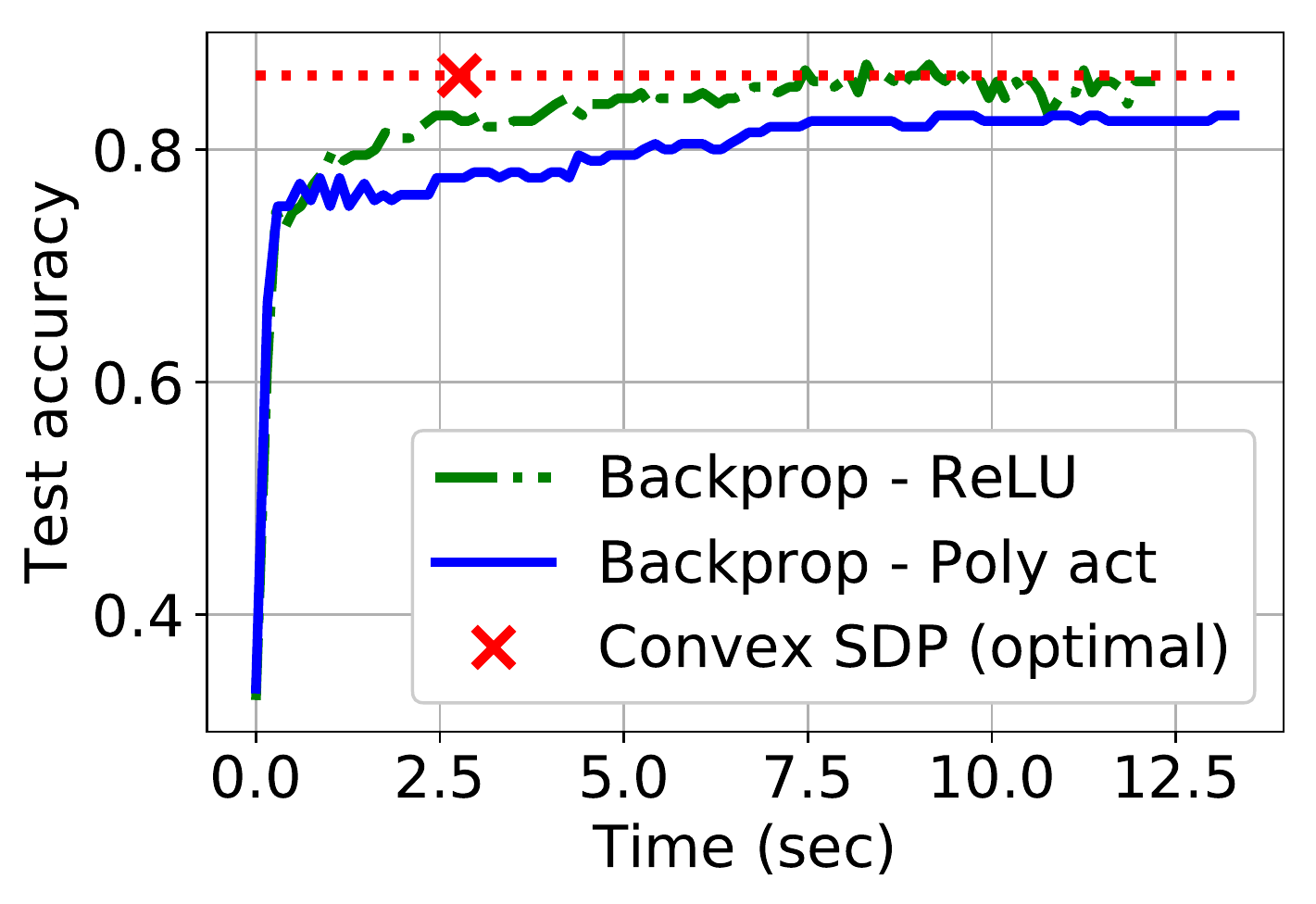}}
  \centerline{(b) DS1, test accuracy}\medskip
\end{minipage}
\begin{minipage}[b]{0.42\linewidth}
  \centering
  \centerline{\includegraphics[width=\columnwidth]{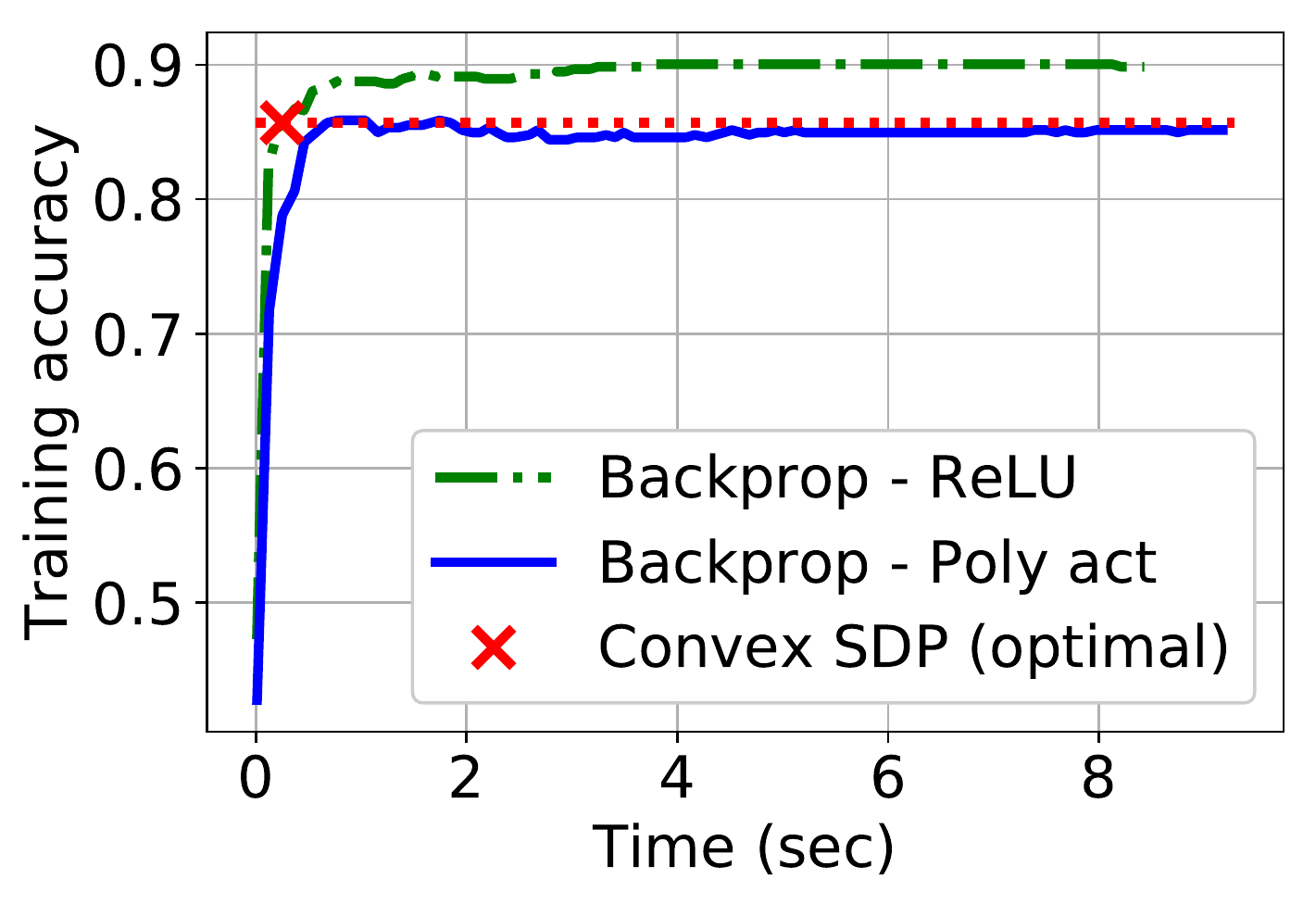}}
  \centerline{(c) DS2, training accuracy}\medskip
\end{minipage}
\hfill
\begin{minipage}[b]{0.42\linewidth}
  \centering
  \centerline{\includegraphics[width=\columnwidth]{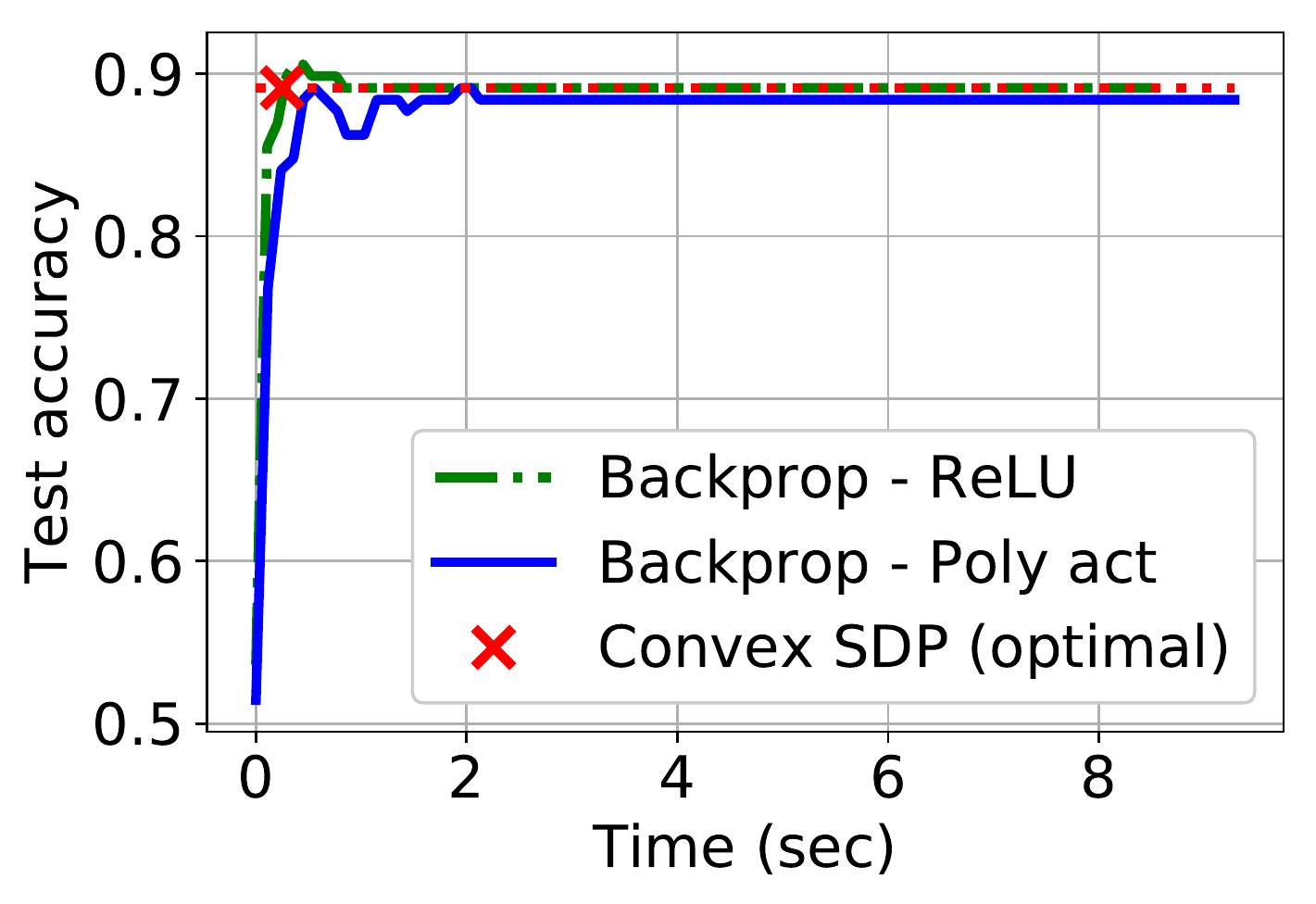}}
  \centerline{(d) DS2, test accuracy}\medskip
\end{minipage}
\begin{minipage}[b]{0.42\linewidth}
  \centering
  \centerline{\includegraphics[width=\columnwidth]{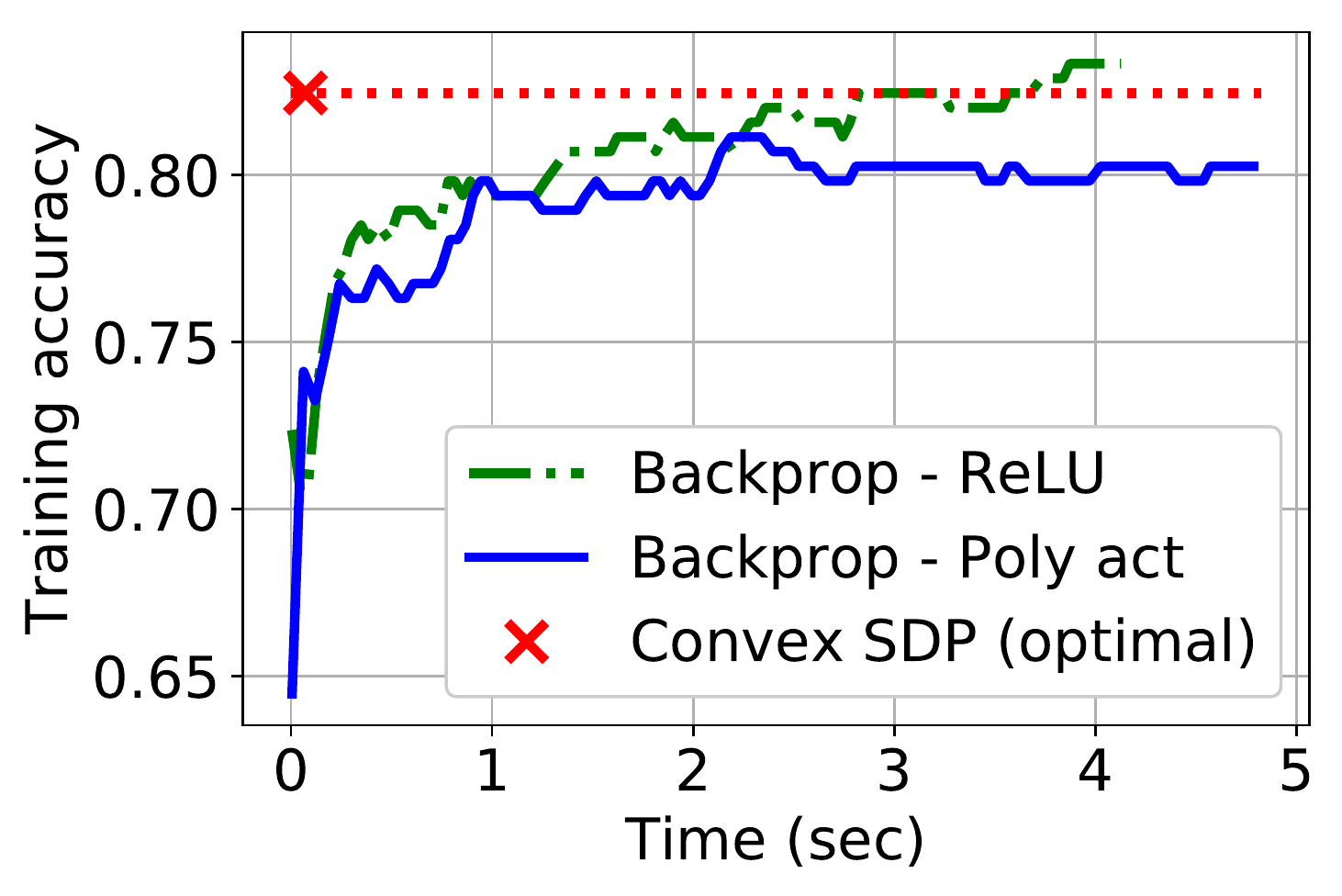}}
  \centerline{(e) DS3, training accuracy}\medskip
\end{minipage}
\hfill
\begin{minipage}[b]{0.42\linewidth}
  \centering
  \centerline{\includegraphics[width=\columnwidth]{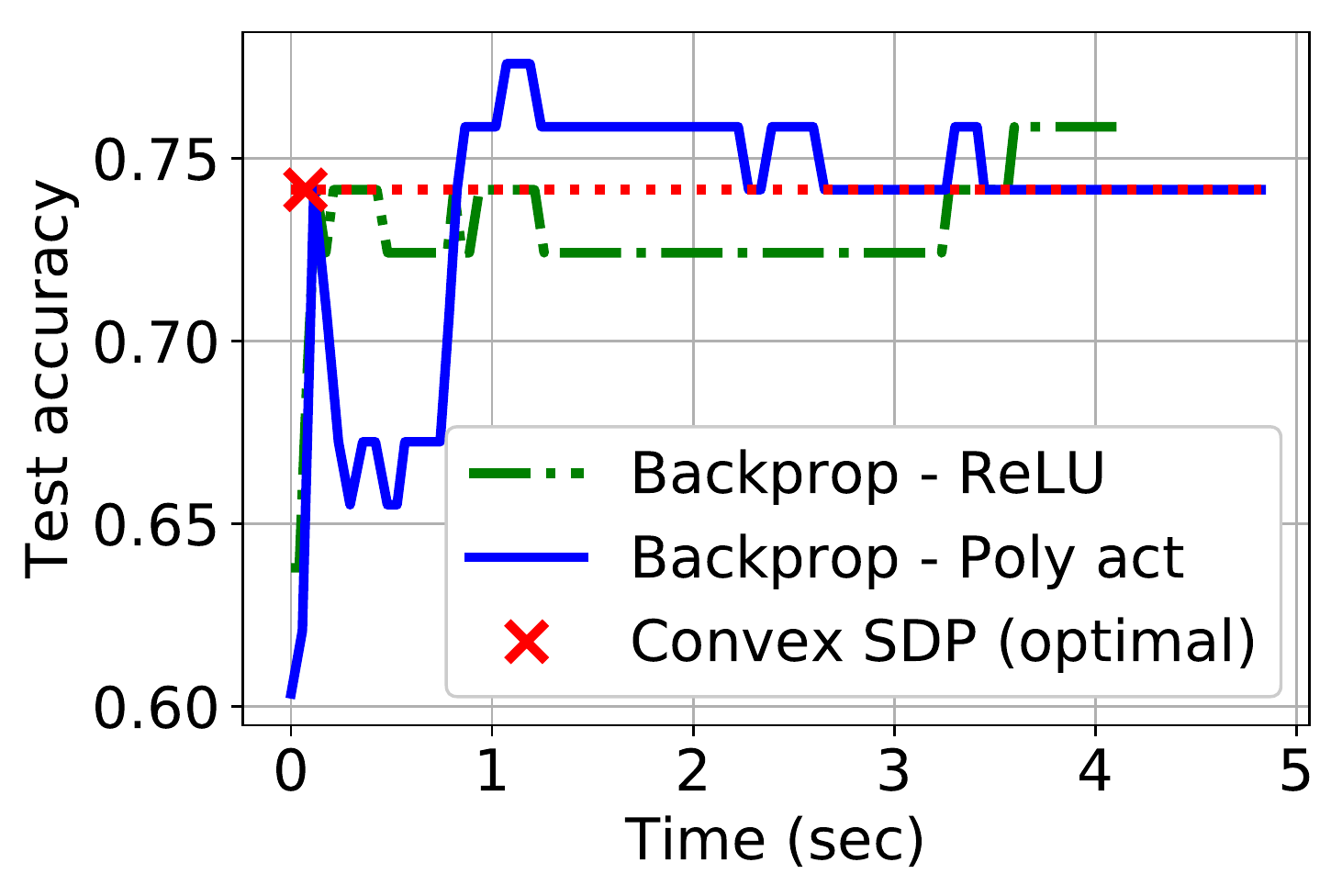}}
  \centerline{(f) DS3, test accuracy}\medskip
\end{minipage}

\caption{Comparison of classification accuracies for neural networks with ReLU activation, polynomial activation ($a=0.09,b=0.5,c=0.47$), and the convex SDP. DS1: dataset 1 is the oocytes-merluccius-nucleus-4d ($n=817,d=41$), DS2: dataset 2 is the credit approval dataset ($n=552,d=15$), DS3: dataset 3 is the breast cancer dataset ($n = 228, d = 9$). }
\label{fig:relu_comparison}
\end{figure}

\subsection{CNN Experiments}
Figure \ref{fig:mnist_cnn_global_avg_pooling} shows the binary classification accuracy performance of the CNN architecture with global average pooling on MNIST \cite{mnist_dataset}, Fashion MNIST \cite{fashionmnist_dataset}, and Cifar-10 \cite{cifar10_dataset} datasets. Figure \ref{fig:mnist_cnn_global_avg_pooling} compares the non-convex tractable problem, the corresponding convex formulation, and the non-convex weight decay formulation. By the weight decay formulation, we mean quadratic regularization on both the first layer filters and the second layer weights. 
We observe that the accuracy of the convex SDP is slightly better or the same as SGD while the run time for the convex SDP solution is consistently shorter than the time it takes for SGD to converge.

\begin{figure} 
\begin{minipage}[b]{0.42\linewidth}
  \centering
  \centerline{\includegraphics[width=\columnwidth]{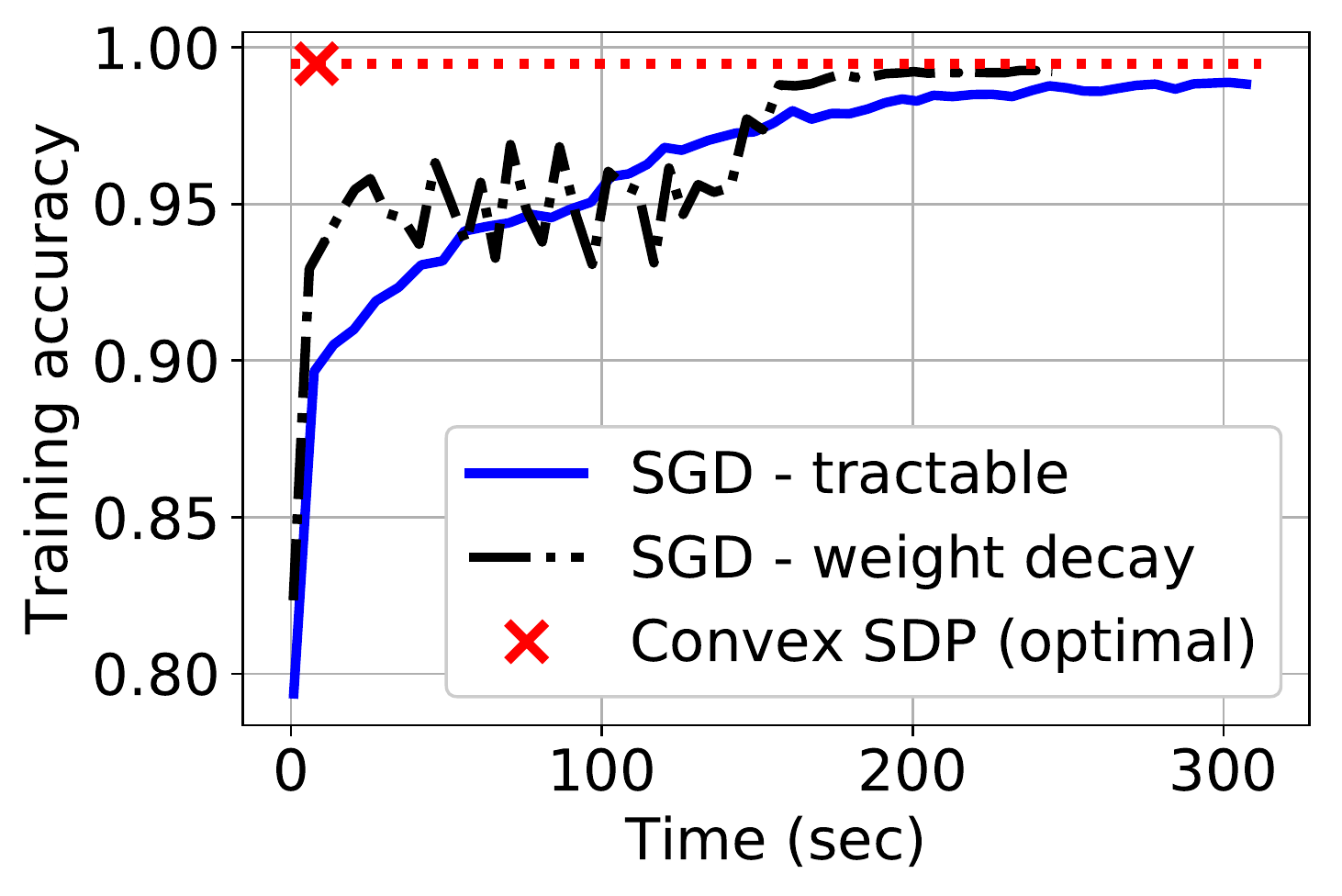}}
  \centerline{(a) MNIST, training accuracy}\medskip
\end{minipage}
\hfill
\begin{minipage}[b]{0.42\linewidth}
  \centering
  \centerline{\includegraphics[width=\columnwidth]{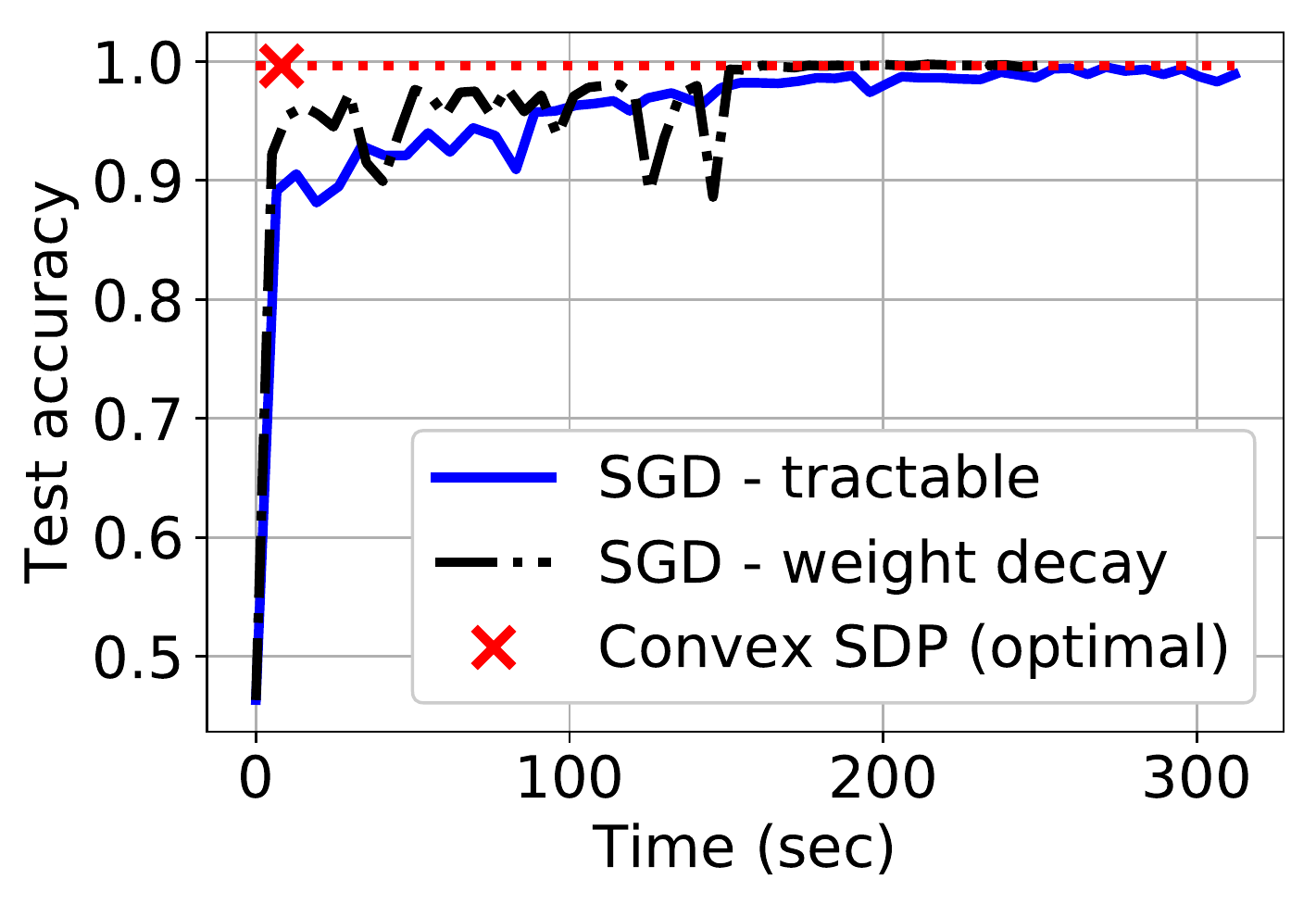}}
  \centerline{(b) MNIST, test accuracy}\medskip
\end{minipage}

\begin{minipage}[b]{0.42\linewidth}
  \centering
  \centerline{\includegraphics[width=\columnwidth]{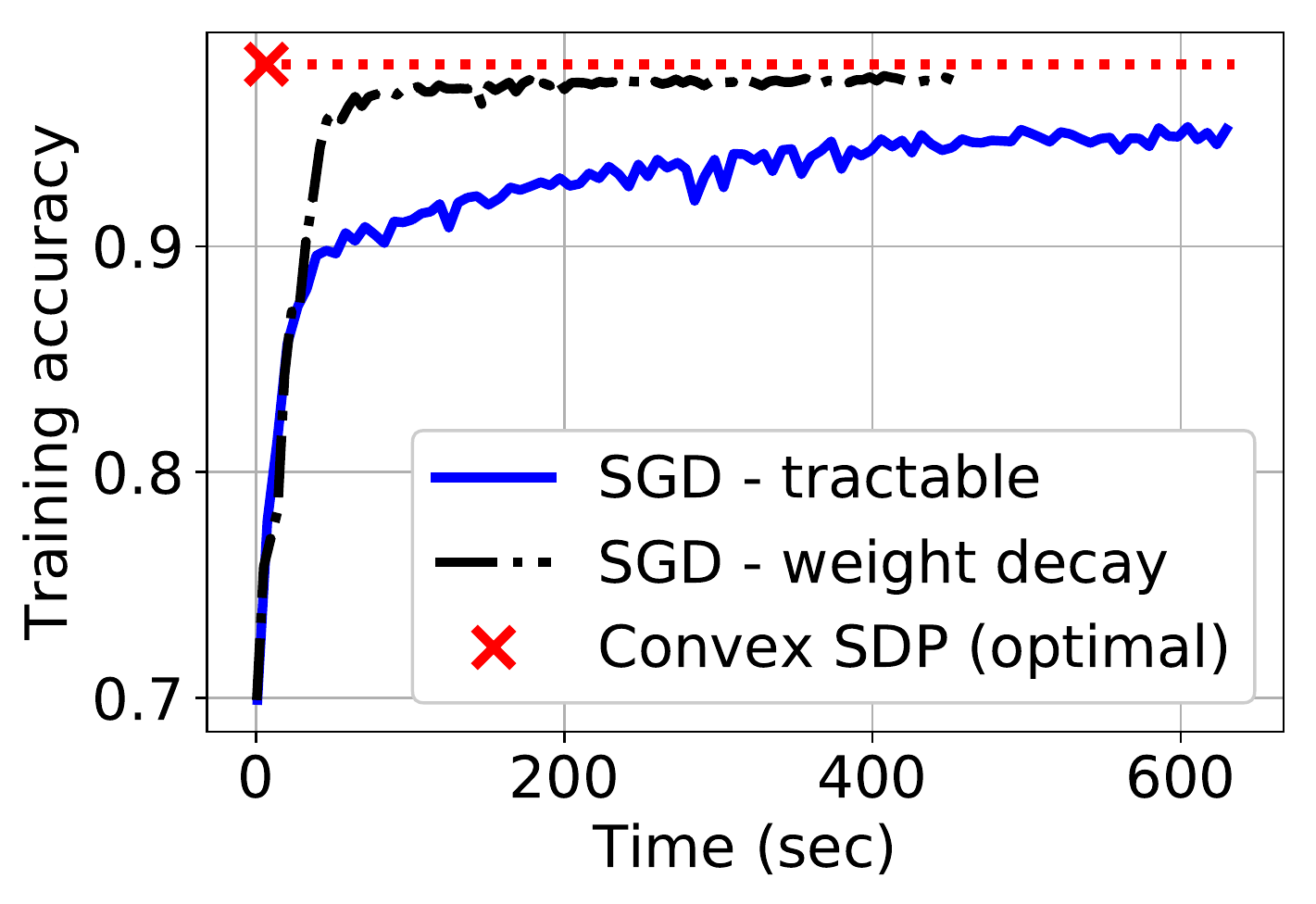}}
  \centerline{(c) Fashion-MNIST, training accuracy}\medskip
\end{minipage}
\hfill
\begin{minipage}[b]{0.42\linewidth}
  \centering
  \centerline{\includegraphics[width=\columnwidth]{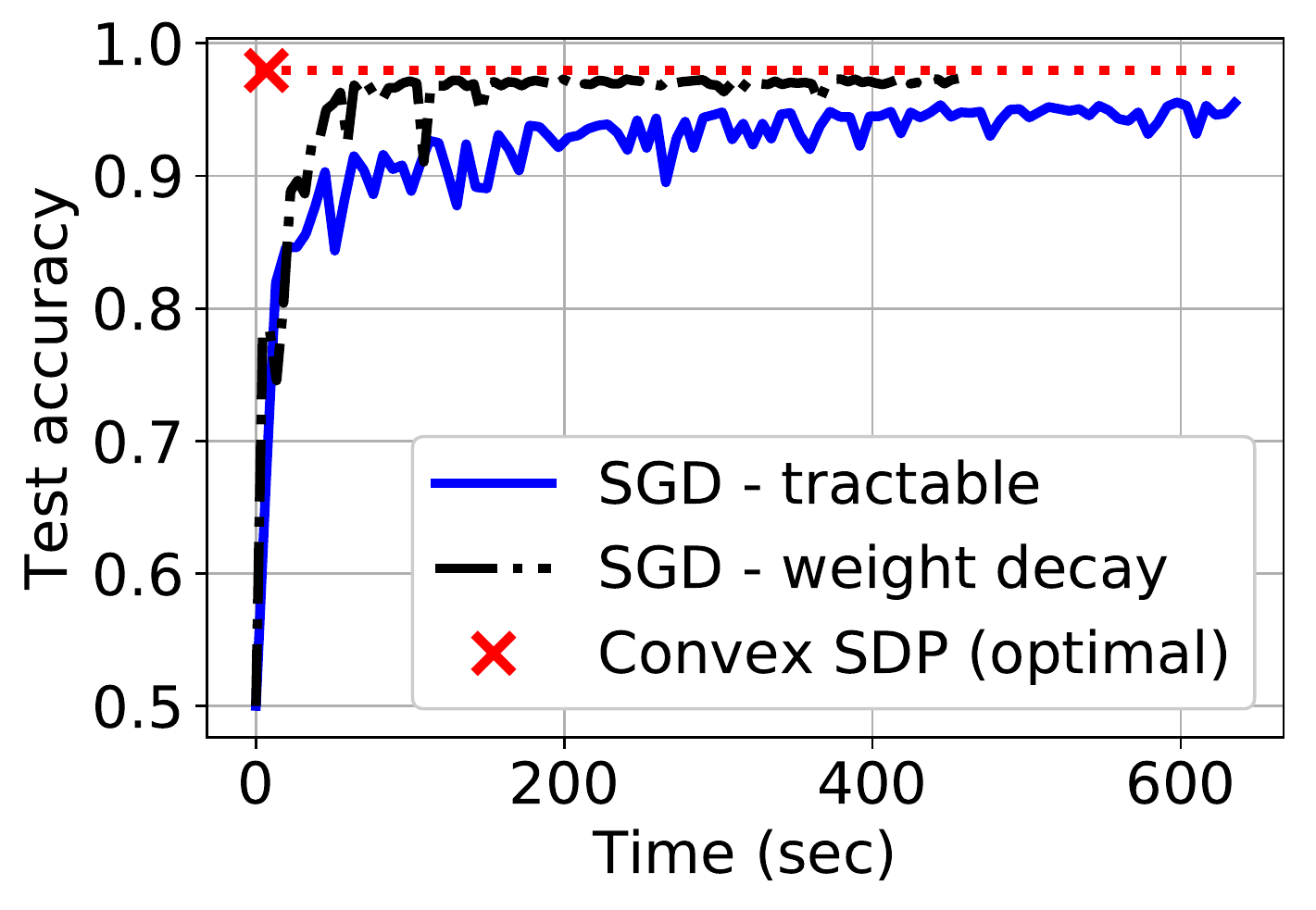}}
  \centerline{(d) Fashion-MNIST, test accuracy}\medskip
\end{minipage}

\begin{minipage}[b]{0.42\linewidth}
  \centering
  \centerline{\includegraphics[width=\columnwidth]{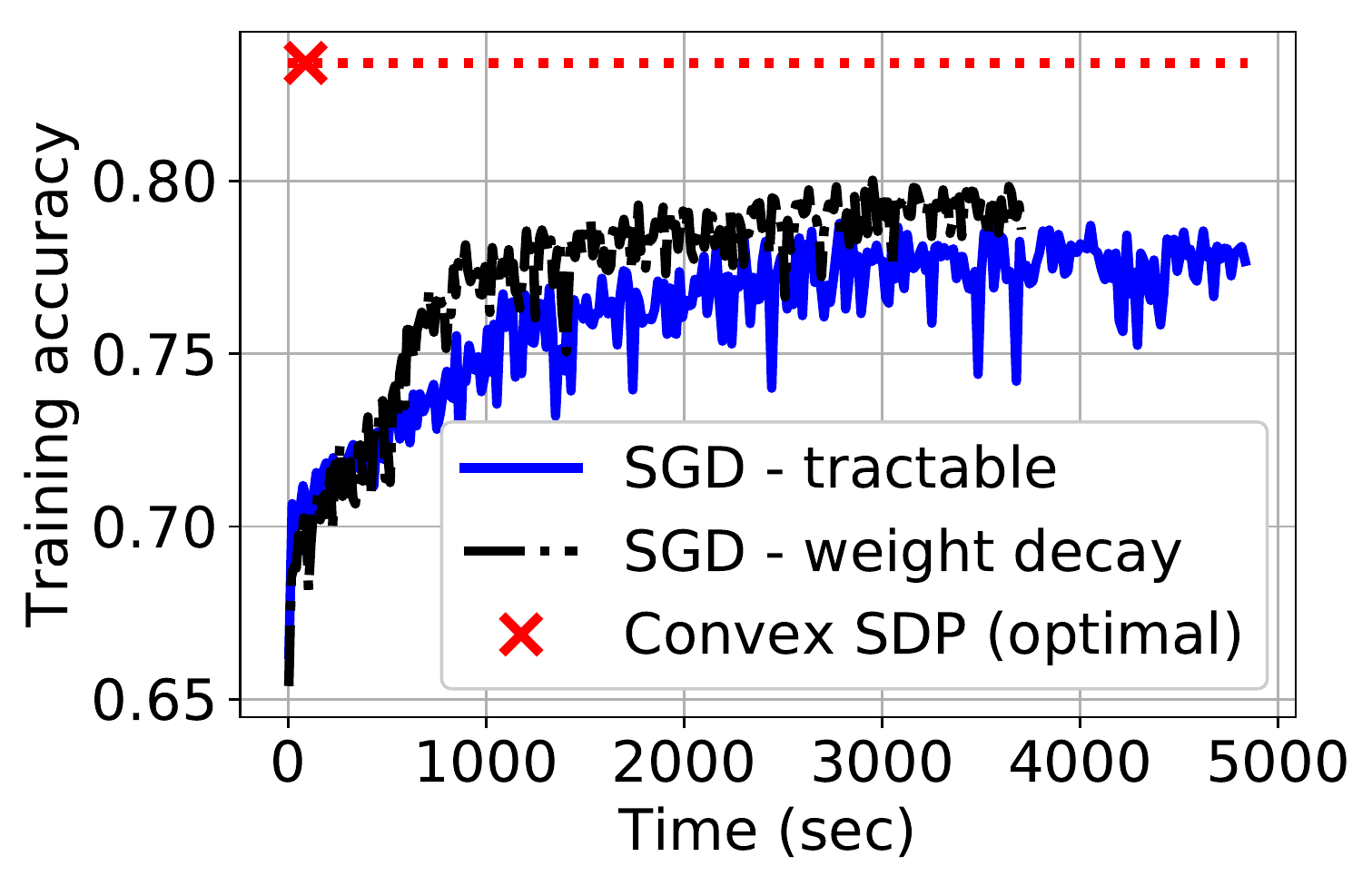}}
  \centerline{(e) Cifar, training accuracy}\medskip
\end{minipage}
\hfill
\begin{minipage}[b]{0.42\linewidth}
  \centering
  \centerline{\includegraphics[width=\columnwidth]{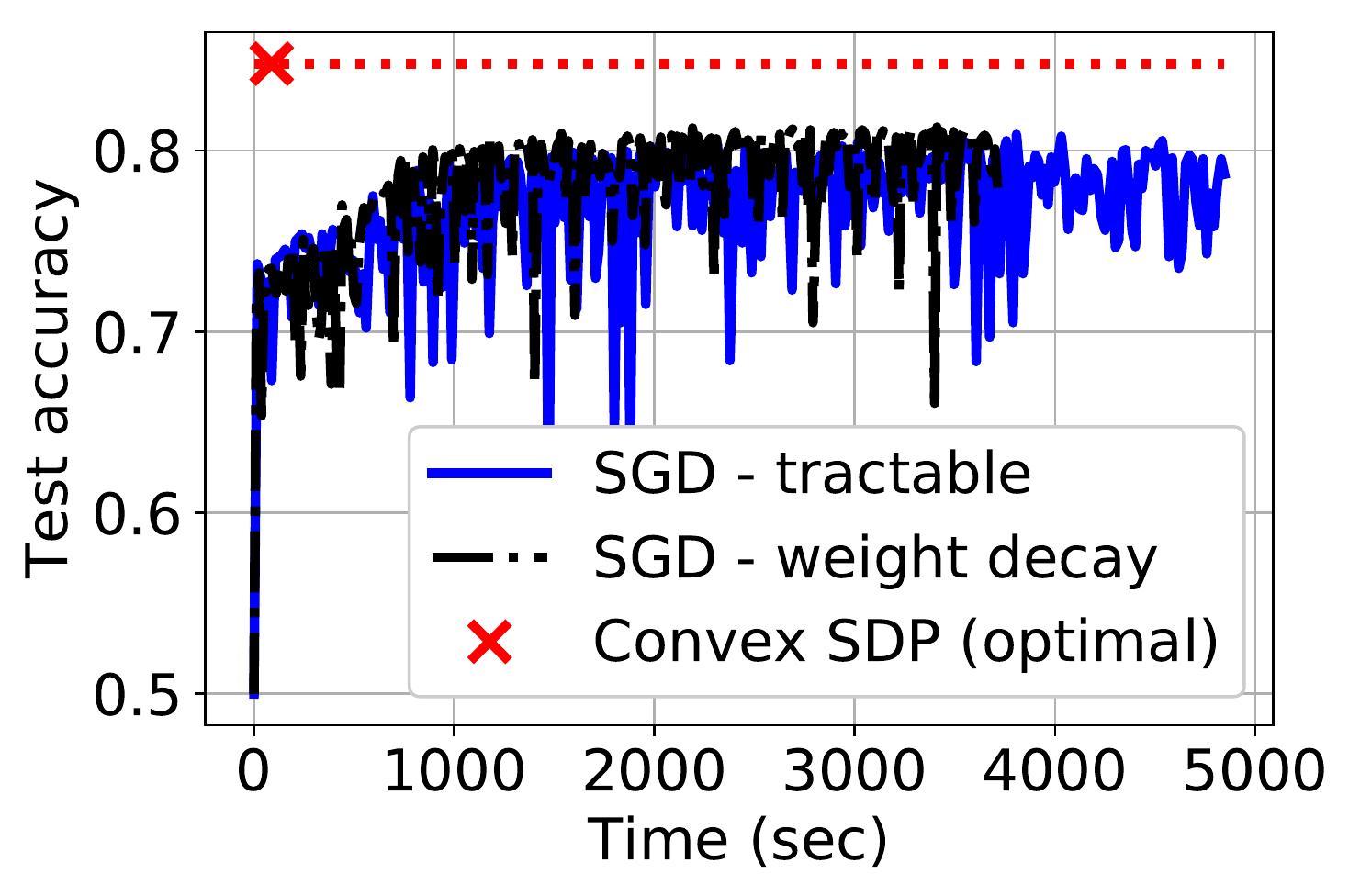}}
  \centerline{(f) Cifar, test accuracy}\medskip
\end{minipage}
\caption{Binary classification with polynomial activation convolutional neural network with pooling and the corresponding convex SDP. Legend labels are as follows. \textit{SGD - tractable}: The non-convex problem in \eqref{eq:convolutional_nonconvex_avgpool}, \textit{SGD - weight decay}: Non-convex problem with quadratic regularization on all weights, \textit{Convex SDP (optimal)}: The convex problem in \eqref{eq:avgpooling_convex_program}. Polynomial coefficients are $a=0.09,b=0.5,c=0.47$. Filter size is $f=3$, stride is $1$, and no padding is used. Batch size for SGD is $100$. The regularization coefficient is $\beta=10^{-6}$. 
Constrained least squares form of the convex program was used for speed (see section \ref{sec:constrained_ls_form}) and the pre-computation step, not shown in the plots, takes $2$ minutes.
Plots a, b show the binary classification accuracy on the first two classes of the MNIST dataset where the classes are the digits 0 and 1 and there are $12600$ gray-scale images of size $28\times 28$. Plots c, d show the binary classification accuracy on the first two classes of Fashion-MNIST dataset where the classes are 'T-shirt/top' and 'Trouser' and there are $12000$ gray-scale images of size $28\times 28$. For plots e, f, the dataset is Cifar-2 (the first two classes of the Cifar-10 dataset) and has $10000$ RGB images each of size $32\times 32\times 3$.
}
\label{fig:mnist_cnn_global_avg_pooling}
\end{figure}


\subsection{Regularization Parameter}

Figure \ref{fig:accuracy_vs_beta_plots} shows how the accuracy changes as a function of the regularization coefficient $\beta$ for the convex problem for two-layer polynomial activation networks. Figure \ref{fig:accuracy_vs_beta_plots} highlights that the choice of the regularization coefficient is critical in the accuracy performance. In plot a, we see that the value of $\beta$ that maximizes the test set accuracy is $\beta=10$ for which the optimal number of neurons $m^*$ is near $20$. We note that for the dataset in plot a, the optimal number of neurons is upper bounded by $m^*\leq 2(d+1)=32$. Similarly for plot b, the best choice for the regularization coefficient is $\beta=1$ and the optimal number of neurons for $\beta=1$ is near $40$. Furthermore, we observe that a higher value for $\beta$ tends to translate to a lower optimal number of neurons $m^*$ (plotted on the right vertical axis). Even though the convex optimization problem in \eqref{eq:polyact_convex_program_final} has a fixed number of variables (in this case, $2(d+1)^2$) for a given dataset, a low number of neurons is still preferable for many reasons such as inference speed. We observe that the number of neurons can be controlled via the regularization coefficient $\beta$.

\begin{figure} 
\begin{minipage}[b]{0.48\linewidth}
  \centering
  \centerline{\includegraphics[width=\columnwidth]{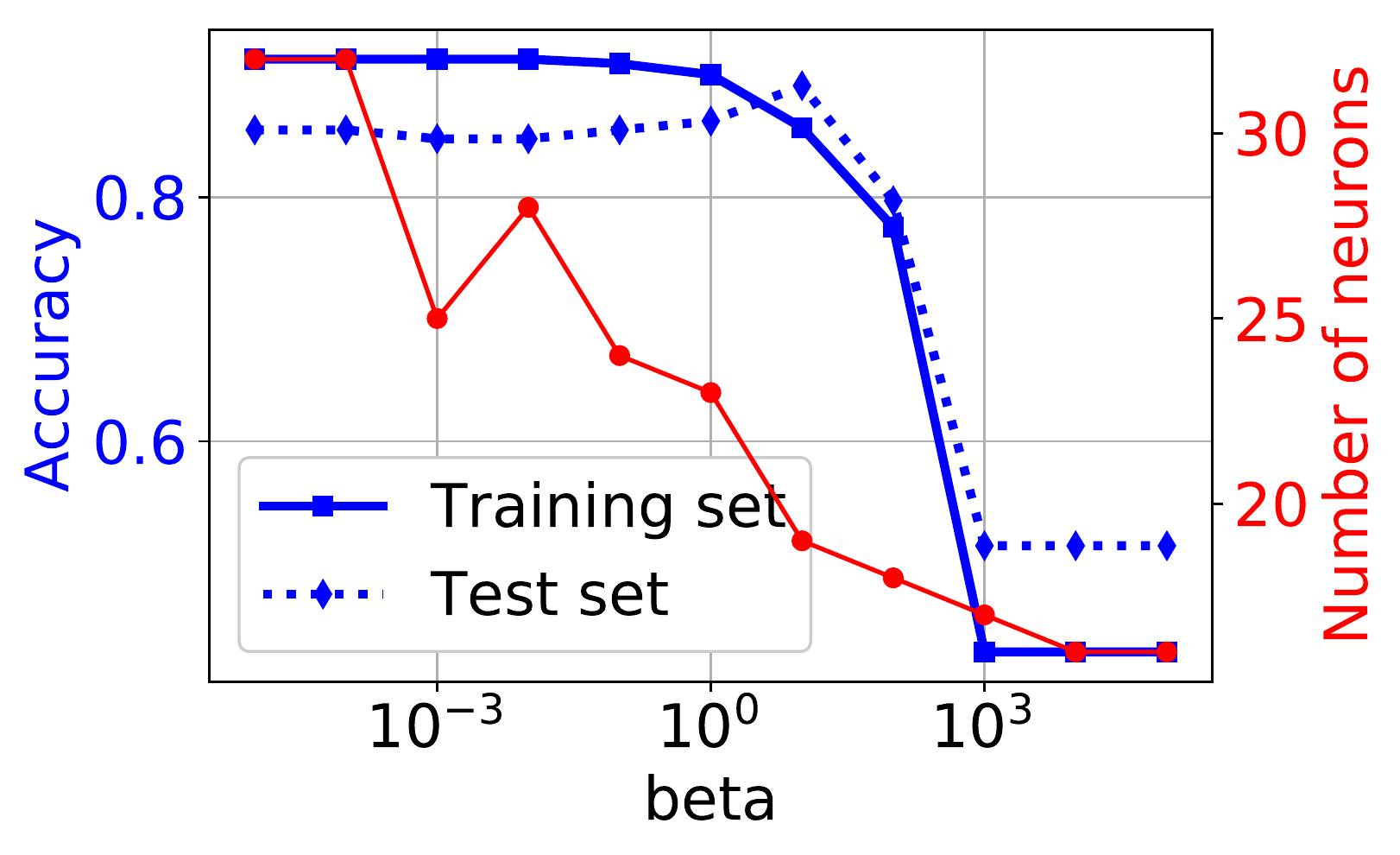}}
  \centerline{(a) credit approval ($n=552,d=15$)}\medskip
\end{minipage}
\hfill
\begin{minipage}[b]{0.48\linewidth}
  \centering
  \centerline{\includegraphics[width=\columnwidth]{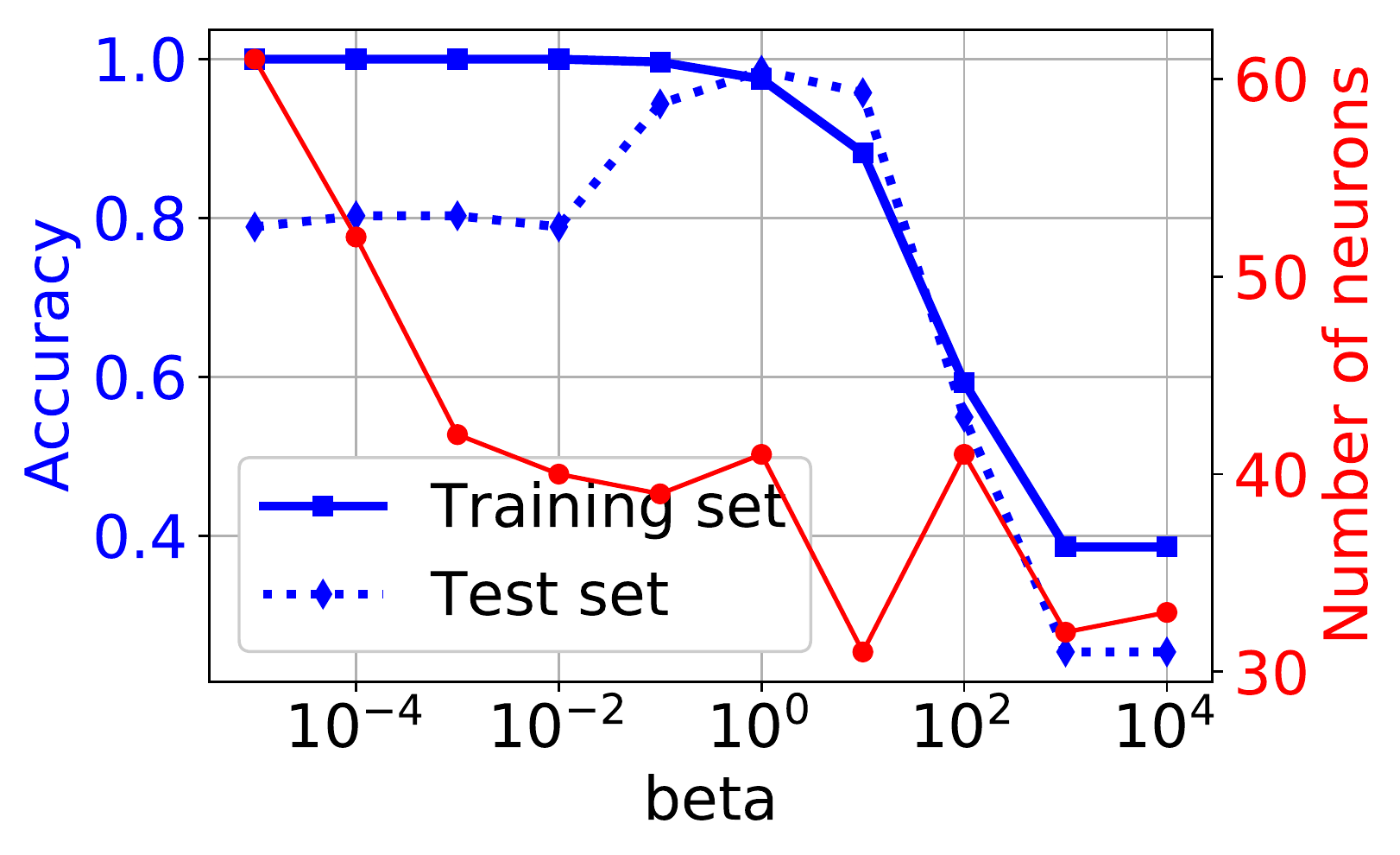}}
  \centerline{(b) ionosphere ($n=280,d=33$)}\medskip
\end{minipage}
\caption{Accuracy (left vertical axis) and optimal number of neurons (right vertical axis) against the regularization coefficient $\beta$ on binary classification datasets. These results have been obtained using the convex program in \eqref{eq:polyact_convex_program_final}.}
\label{fig:accuracy_vs_beta_plots}
\end{figure}

\subsection{Other Losses}
We have so far evaluated the performance of the derived convex programs for squared loss, i.e. $\ell(\hat{y}, y)=\|\hat{y}-y\|_2^2$. We reiterate that the derived convex programs are general in the sense the formulations hold for any convex loss function $\ell$. To verify this numerically, we now present results for additional loss functions such as Huber loss and $\ell_1$ norm loss in Figure \ref{fig:different_losses}. More concretely, Huber loss is defined as $\ell(\hat{y}, y) = \sum_{i=1}^n \text{Huber}(\hat{y}_i-y_i)$ where $\text{Huber}(x)=2|x|-1$ for $|x| > 1$ and $\text{Huber}(x)=x^2$ for $|x| \leq 1$. The $\ell_1$ norm loss is $\ell(\hat{y}, y)=\|\hat{y}-y\|_1$. We observe that in the case of $\ell_1$ norm loss, backpropagation takes longer to converge.

\begin{figure} 
\begin{minipage}[b]{0.48\linewidth}
  \centering
  \centerline{\includegraphics[width=\columnwidth]{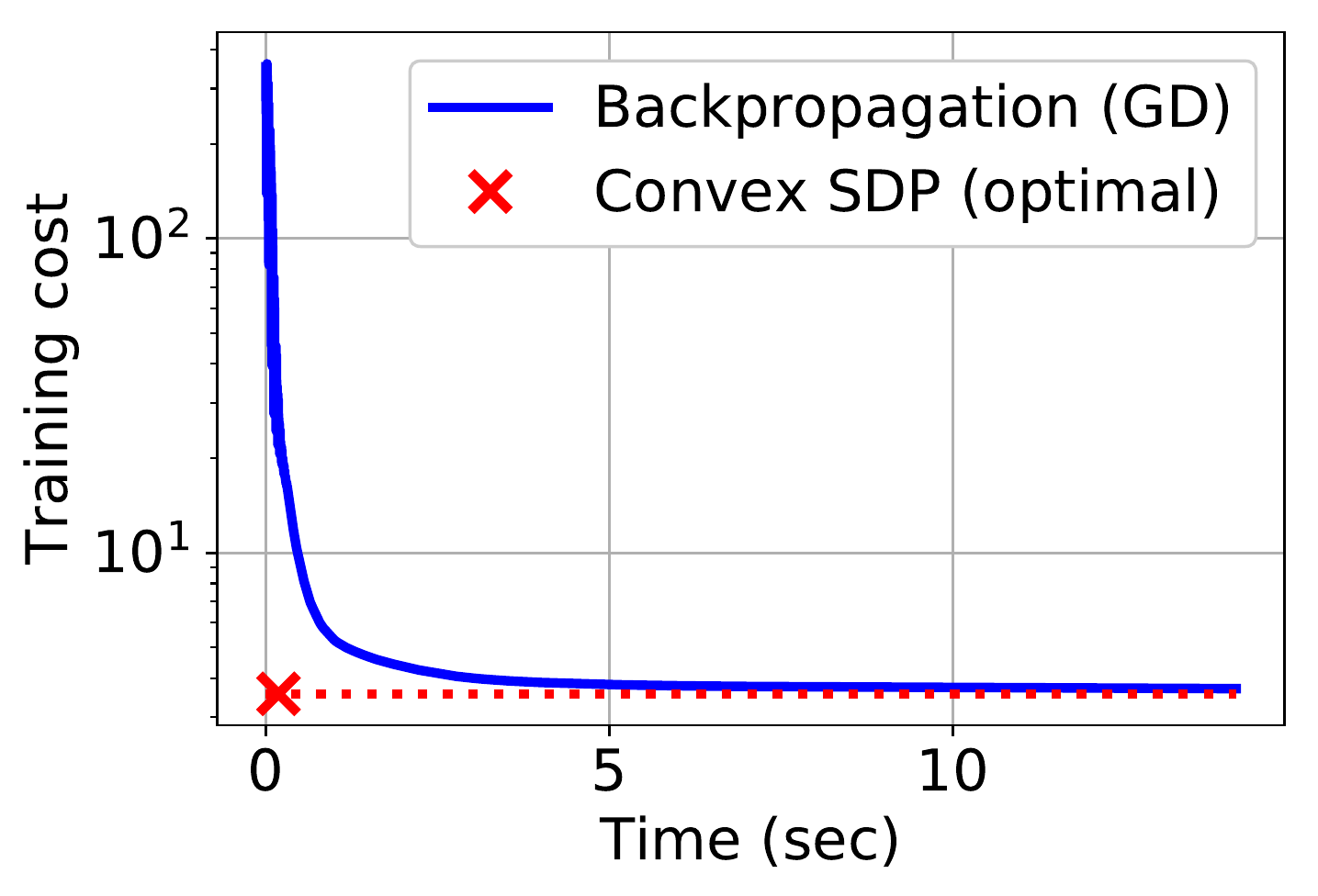}}
  \centerline{(a) Huber loss}\medskip
\end{minipage}
\hfill
\begin{minipage}[b]{0.48\linewidth}
  \centering
  \centerline{\includegraphics[width=\columnwidth]{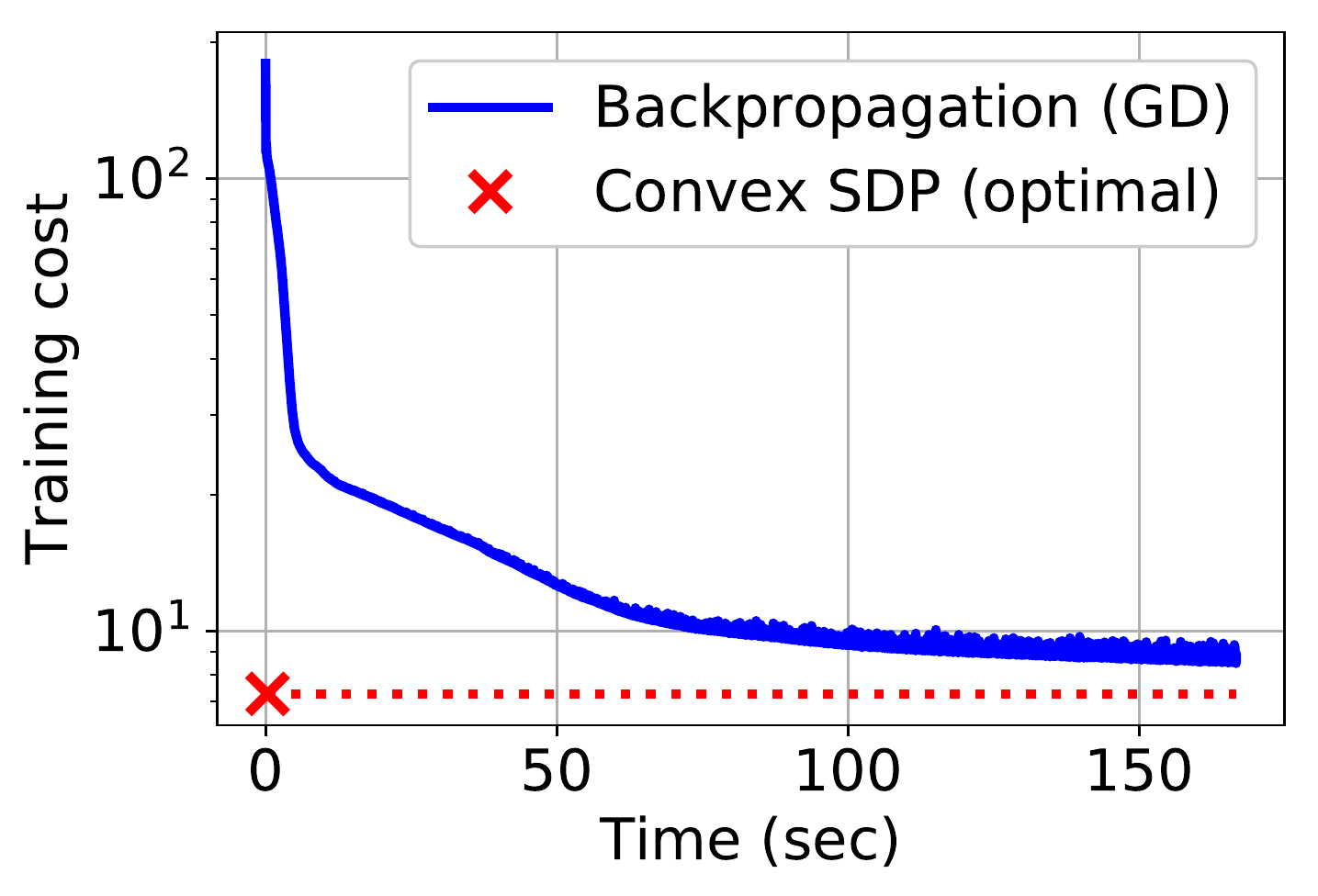}}
  \centerline{(b) $\ell_1$ norm loss}\medskip
\end{minipage}
\caption{Verifying the theoretical results for other convex loss functions: Huber and $\ell_1$ norm loss. An artificially generated dataset with dimensions $n=100,d=20$ is used. The regularization coefficient is $\beta=0.1$. The number of neurons $m^*$ is found to be $7$ and $9$ for plots a and b, respectively.}
\label{fig:different_losses}
\end{figure}

\subsection{The Effect of Polynomial Coefficients}
The plots in Figure \ref{fig:accuracy_vs_polycoeffs} show the classification accuracy against the polynomial coefficients $a,b,c$ for the polynomial activation convex problem. In each plot, we vary one of the coefficients and fix the other two coefficients as $1$. We observe that the coefficient of the quadratic term $a$ plays the most important role in the accuracy performance. The accuracy is not affected by the choice of the coefficient $c$.

\begin{figure} 
\begin{minipage}[b]{0.32\linewidth}
  \centering
  \centerline{\includegraphics[width=\columnwidth]{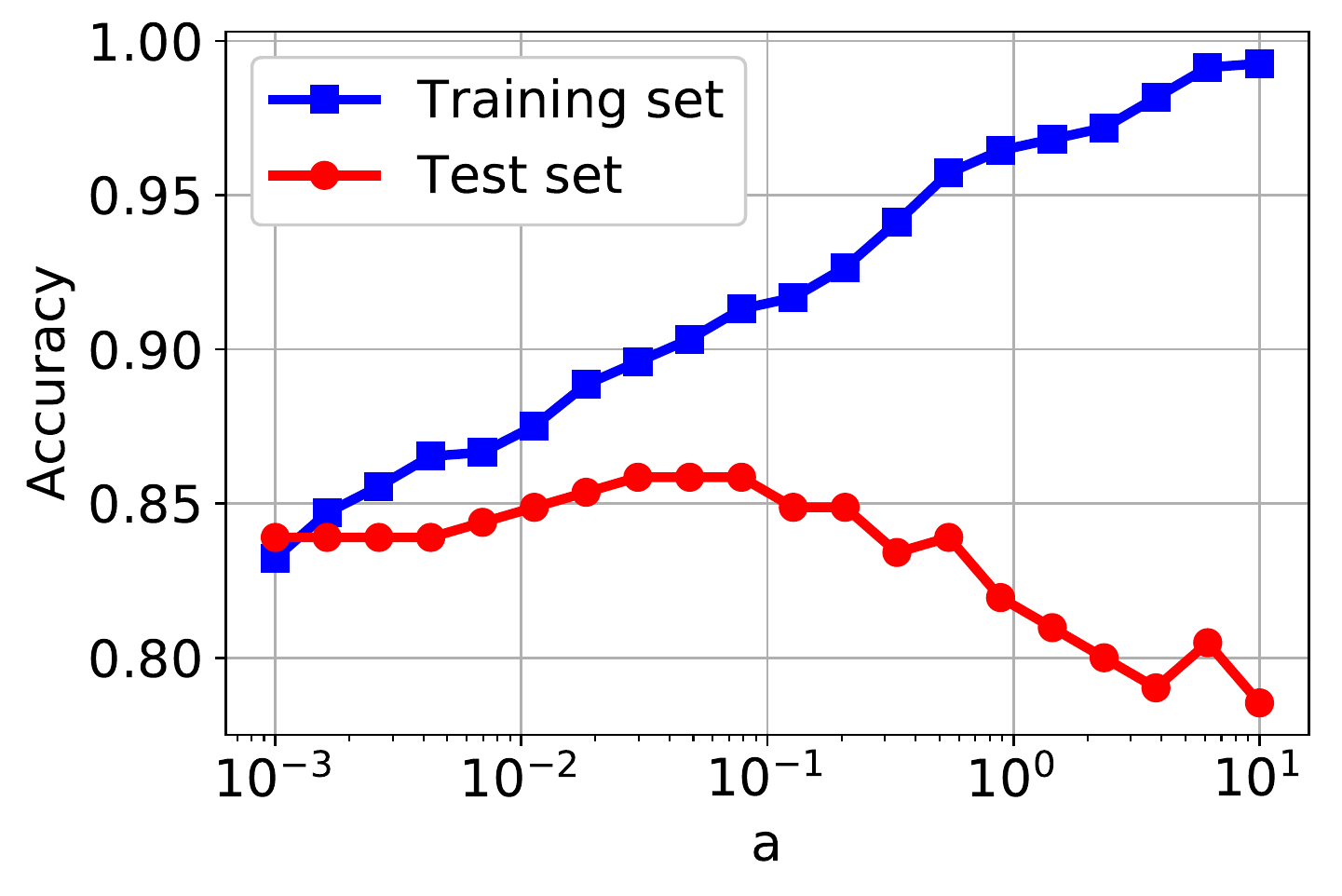}}
  \centerline{(a) $b=c=1$}\medskip
\end{minipage}
\hfill
\begin{minipage}[b]{0.32\linewidth}
  \centering
  \centerline{\includegraphics[width=\columnwidth]{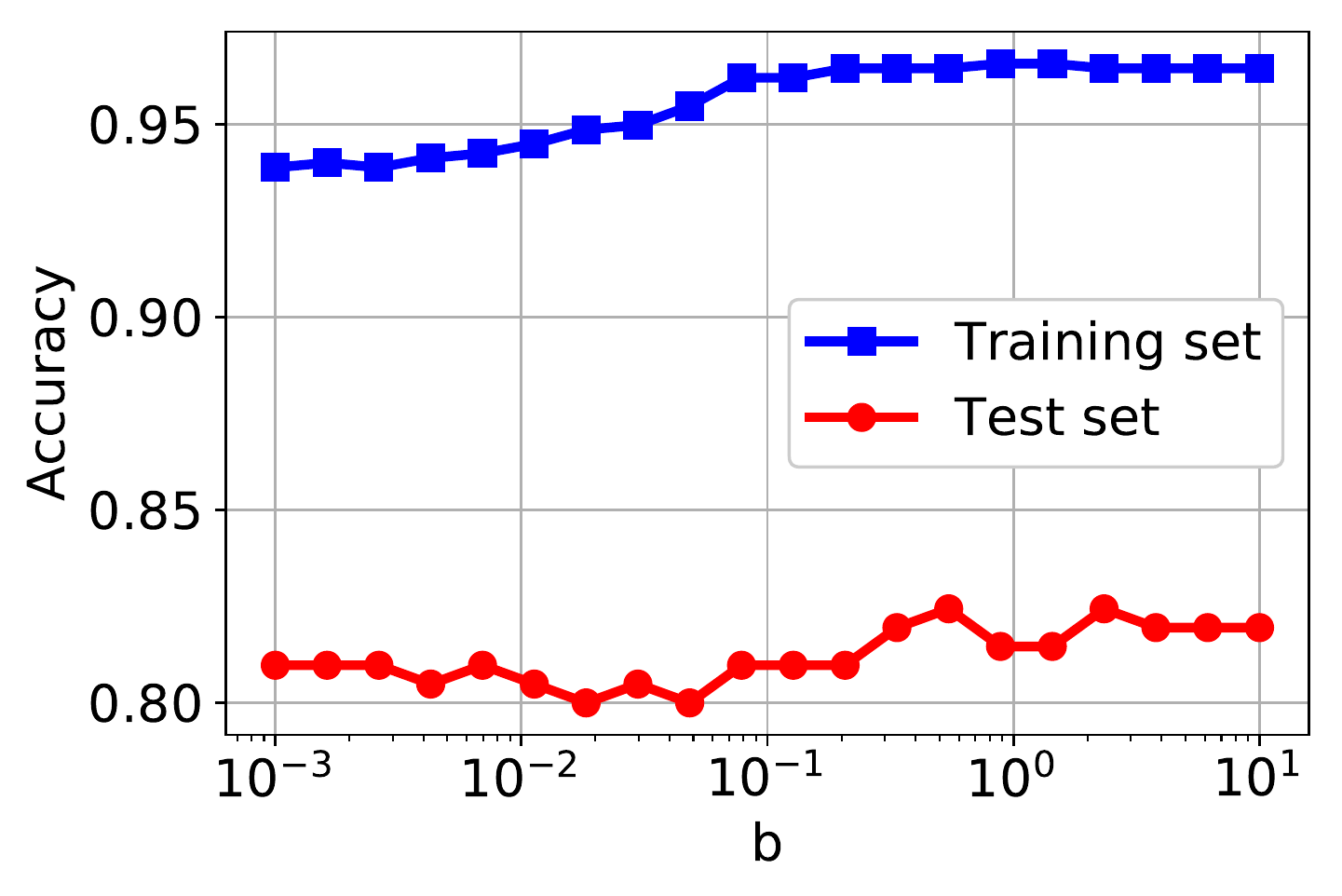}}
  \centerline{(b) $a=c=1$}\medskip
\end{minipage}
\hfill
\begin{minipage}[b]{0.32\linewidth}
  \centering
  \centerline{\includegraphics[width=\columnwidth]{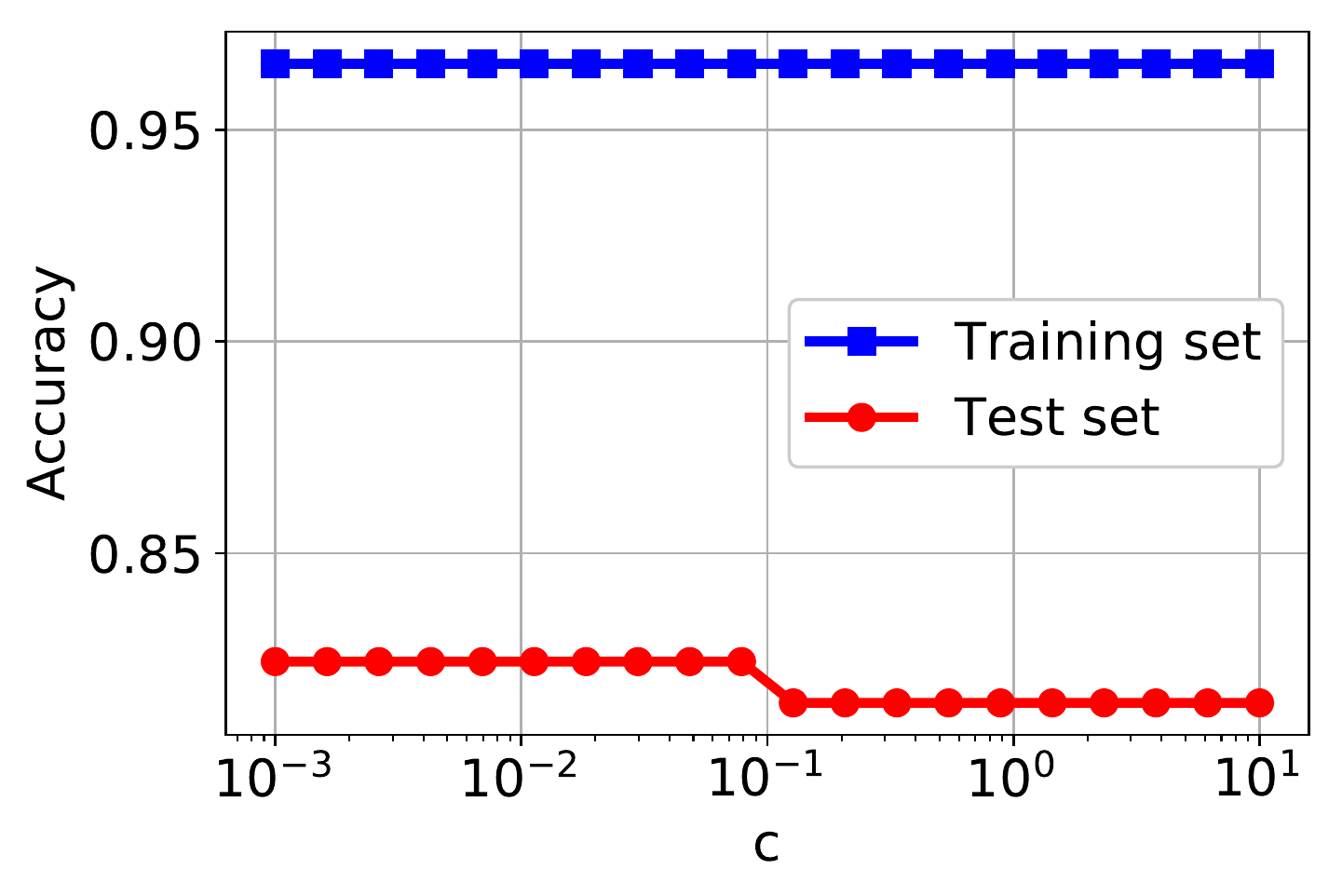}}
  \centerline{(c) $a=b=1$}\medskip
\end{minipage}
\caption{Training and test set classification accuracies against polynomial coefficients $a,b,c$. The regularization coefficient is $\beta=0.1$ and the dataset is oocytes-merluccius-nucleus-4d.}
\label{fig:accuracy_vs_polycoeffs}
\end{figure}

\section{Discussion} \label{sec:conclusion}

In this paper, we have studied the optimization of two-layer neural networks with degree two polynomial activations. We have shown that regularization plays an important role in the tractability of the problems associated with neural network training. We have developed convex programs for the cases where the regularization leads to tractable formulations. Convex formulations are useful since they have many well-known advantages over non-convex optimization such as having to optimize fewer hyperparameters and no risk of getting stuck at local minima.



The methods presented in this work optimize the neural network parameters in a higher dimensional space in which the problem becomes convex. For fully connected neural networks with quadratic activation, the standard non-convex problem requires optimizing $m$ neurons (i.e. a $d$-dimensional first layer weight and a $1$-dimensional second layer weight per neuron). The convex program for this neural network finds the optimal network parameters in the lifted space $\mathbb{S}^{d\times d}$. For polynomial activations, convex optimization takes place for $Z$ and $Z^\prime$ in $\mathbb{S}^{(d+1)\times (d+1)}$. We note that the dimensions of the convex programs are polynomial with respect to all problem dimensions. In contrast, the convex program of \cite{pilanci2020neural} has $2dP$ variables where $P$ grows exponentially with respect to the rank of the data matrix.

We have used the SCS solver with CVXPY for solving the convex problems in the numerical experiments. It is important to note that there is room for future work in terms of which solvers to use. Solvers specifically designed for the presented convex programs could enjoy faster run times.

The scope of this work is limited to two-layer neural networks. We note that it is a promising direction to consider the use of our convex programs for two-layer neural networks as building blocks in learning deep neural networks. Many recent works such as \cite{allenzhu2020backward} and \cite{belilovsky2018layerwise} investigate layerwise learning algorithms for deep neural networks. The training of individual layers in layerwise learning could be improved by the presented convex programs since the convex programs can be efficiently solved and eliminate much of the hyperparameter tuning involved in standard neural network training.

\section*{Acknowledgements}
This work was partially supported by the National Science Foundation under grants IIS-1838179, ECCS-2037304, Facebook Research, Adobe Research and Stanford SystemX Alliance.

\bibliography{refs}

\begin{thebibliography}{10}

\bibitem{agrawal2018rewriting}
Akshay Agrawal, Robin Verschueren, Steven Diamond, and Stephen Boyd.
\newblock A rewriting system for convex optimization problems.
\newblock {\em Journal of Control and Decision}, 5(1):42--60, 2018.

\bibitem{allenzhu2020backward}
Zeyuan Allen-Zhu and Yuanzhi Li.
\newblock Backward feature correction: How deep learning performs deep
  learning.
\newblock {\em arXiv preprint arXiv:2001.04413}, 2020.

\bibitem{arora2018understanding}
Raman Arora, Amitabh Basu, Poorya Mianjy, and Anirbit Mukherjee.
\newblock Understanding deep neural networks with rectified linear units.
\newblock In {\em 6th International Conference on Learning Representations,
  ICLR 2018}, 2018.

\bibitem{belilovsky2018layerwise}
Eugene Belilovsky, Michael Eickenberg, and Edouard Oyallon.
\newblock Greedy layerwise learning can scale to imagenet.
\newblock {\em CoRR}, abs/1812.11446, 2018.

\bibitem{bienstock2018principled}
Daniel Bienstock, Gonzalo Mu{\~n}oz, and Sebastian Pokutta.
\newblock Principled deep neural network training through linear programming,
  2018.

\bibitem{blondel2015convexFM}
Mathieu Blondel, Akinori Fujino, and Naonori Ueda.
\newblock Convex factorization machines.
\newblock {\em European Conference on Machine Learning and Principles and
  Practice of Knowledge Discovery in Databases (ECML PKDD)}, 2015.

\bibitem{blondel2017multi}
Mathieu Blondel, Vlad Niculae, Takuma Otsuka, and Naonori Ueda.
\newblock Multi-output polynomial networks and factorization machines.
\newblock In {\em Proceedings of the 31st International Conference on Neural
  Information Processing Systems}, NIPS'17, page 3351–3361, Red Hook, NY,
  USA, 2017. Curran Associates Inc.

\bibitem{boyd2004convex}
Stephen Boyd and Lieven Vandenberghe.
\newblock {\em Convex optimization}.
\newblock Cambridge university press, 2004.

\bibitem{burer2012copositive}
Samuel Burer.
\newblock Copositive programming.
\newblock In {\em Handbook on semidefinite, conic and polynomial optimization},
  pages 201--218. Springer, 2012.

\bibitem{chizat2019lazy}
Lenaic Chizat, Edouard Oyallon, and Francis Bach.
\newblock On lazy training in differentiable programming.
\newblock In {\em Advances in Neural Information Processing Systems}, pages
  2937--2947, 2019.

\bibitem{diamond2016cvxpy}
Steven Diamond and Stephen Boyd.
\newblock {CVXPY}: {A} {P}ython-embedded modeling language for convex
  optimization.
\newblock {\em Journal of Machine Learning Research}, 17(83):1--5, 2016.

\bibitem{du2018quadact}
Simon Du and Jason Lee.
\newblock On the power of over-parametrization in neural networks with
  quadratic activation.
\newblock In Jennifer Dy and Andreas Krause, editors, {\em Proceedings of the
  35th International Conference on Machine Learning}, volume~80 of {\em
  Proceedings of Machine Learning Research}, pages 1329--1338,
  Stockholmsmässan, Stockholm Sweden, 10--15 Jul 2018. PMLR.

\bibitem{uci2019datasets}
Dheeru Dua and Casey Graff.
\newblock {UCI} machine learning repository, 2017.

\bibitem{ergen2020aistats}
Tolga Ergen and Mert Pilanci.
\newblock Convex geometry of two-layer relu networks: Implicit autoencoding and
  interpretable models.
\newblock In {\em International Conference on Artificial Intelligence and
  Statistics}, pages 4024--4033. PMLR, 2020.

\bibitem{ergen2020cnn}
Tolga Ergen and Mert Pilanci.
\newblock Implicit convex regularizers of cnn architectures: Convex
  optimization of two- and three-layer networks in polynomial time.
\newblock {\em arXiv preprint arXiv:2006.14798}, 2020.

\bibitem{ergen2020convexdeep}
Tolga Ergen and Mert Pilanci.
\newblock Revealing the structure of deep neural networks via convex duality.
\newblock {\em arXiv preprint arXiv:2002.09773}, 2020.

\bibitem{delgado2014ucipreprocessed}
Manuel Fern{{\'a}}ndez-Delgado, Eva Cernadas, Sen{{\'e}}n Barro, and Dinani
  Amorim.
\newblock Do we need hundreds of classifiers to solve real world classification
  problems?
\newblock {\em Journal of Machine Learning Research}, 15(90):3133--3181, 2014.

\bibitem{fickus2013phase}
Matthew Fickus, Dustin~G. Mixon, Aaron~A. Nelson, and Yang Wang.
\newblock Phase retrieval from very few measurements.
\newblock {\em arXiv preprint arXiv:1307.7176}, 2013.

\bibitem{gamarnik2020stationary}
David Gamarnik, Eren~C. Kızıldağ, and Ilias Zadik.
\newblock Stationary points of shallow neural networks with quadratic
  activation function.
\newblock {\em arXiv preprint arXiv:1912.01599}, 2020.

\bibitem{goel17a}
Surbhi Goel, Varun Kanade, Adam Klivans, and Justin Thaler.
\newblock Reliably learning the relu in polynomial time.
\newblock In Satyen Kale and Ohad Shamir, editors, {\em Proceedings of the 2017
  Conference on Learning Theory}, volume~65 of {\em Proceedings of Machine
  Learning Research}, pages 1004--1042, Amsterdam, Netherlands, 07--10 Jul
  2017. PMLR.

\bibitem{Goemans95}
M.~X. Goemans and D.~P. Williamson.
\newblock Improved approximation algorithms for maximum cut and satisfiability
  problems using semidefinite programming.
\newblock {\em Journal of the ACM}, 42:1115--1145, 1995.

\bibitem{haastad2001some}
Johan H{\aa}stad.
\newblock Some optimal inapproximability results.
\newblock {\em Journal of the ACM (JACM)}, 48(4):798--859, 2001.

\bibitem{jacot2018neural}
Arthur Jacot, Franck Gabriel, and Cl{\'e}ment Hongler.
\newblock Neural tangent kernel: Convergence and generalization in neural
  networks.
\newblock In {\em Advances in neural information processing systems}, pages
  8571--8580, 2018.

\bibitem{khot2007optimal}
Subhash Khot, Guy Kindler, Elchanan Mossel, and Ryan O’Donnell.
\newblock Optimal inapproximability results for max-cut and other 2-variable
  csps?
\newblock {\em SIAM Journal on Computing}, 37(1):319--357, 2007.

\bibitem{cifar10_dataset}
Alex Krizhevsky.
\newblock Learning multiple layers of features from tiny images, 2009.

\bibitem{lacotte2020all}
Jonathan Lacotte and Mert Pilanci.
\newblock All local minima are global for two-layer relu neural networks: The
  hidden convex optimization landscape.
\newblock {\em arXiv preprint arXiv:2006.05900}, 2020.

\bibitem{laurent1995positive}
Monique Laurent and Svatopluk Poljak.
\newblock On a positive semidefinite relaxation of the cut polytope.
\newblock {\em Linear Algebra and its Applications}, 223(224):439--461, 1995.

\bibitem{mnist_dataset}
Yann LeCun, Corinna Cortes, and CJ~Burges.
\newblock Mnist handwritten digit database.
\newblock {\em ATT Labs [Online]. Available: http://yann.lecun.com/exdb/mnist},
  2, 2010.

\bibitem{lederer2020spurious}
Johannes Lederer.
\newblock No spurious local minima: on the optimization landscapes of wide and
  deep neural networks, 2020.

\bibitem{roi2014training}
Roi Livni, Shai Shalev-Shwartz, and Ohad Shamir.
\newblock On the computational efficiency of training neural networks.
\newblock {\em NIPS'14}, page 855–863, 2014.

\bibitem{mannelli2020optimization}
Stefano~Sarao Mannelli, Eric Vanden-Eijnden, and Lenka Zdeborová.
\newblock Optimization and generalization of shallow neural networks with
  quadratic activation functions.
\newblock {\em arXiv preprint arXiv:2006.15459}, 2020.

\bibitem{naber2012geometry}
Gregory~L Naber.
\newblock {\em The geometry of Minkowski spacetime: An introduction to the
  mathematics of the special theory of relativity}, volume~92.
\newblock Springer Science \& Business Media, 2012.

\bibitem{nesterov2000semidefinite}
Yuri Nesterov, Henry Wolkowicz, and Yinyu Ye.
\newblock Semidefinite programming relaxations of nonconvex quadratic
  optimization.
\newblock In {\em Handbook of semidefinite programming}, pages 361--419.
  Springer, 2000.

\bibitem{scs2016paper}
B.~O'Donoghue, E.~Chu, N.~Parikh, and S.~Boyd.
\newblock Conic optimization via operator splitting and homogeneous self-dual
  embedding.
\newblock {\em Journal of Optimization Theory and Applications},
  169(3):1042--1068, June 2016.

\bibitem{scs2016code}
B.~O'Donoghue, E.~Chu, N.~Parikh, and S.~Boyd.
\newblock {SCS}: Splitting conic solver, version 2.1.2.
\newblock \url{https://github.com/cvxgrp/scs}, November 2019.

\bibitem{pytorch}
Adam Paszke, Sam Gross, Francisco Massa, Adam Lerer, James Bradbury, Gregory
  Chanan, Trevor Killeen, Zeming Lin, Natalia Gimelshein, Luca Antiga, Alban
  Desmaison, Andreas Kopf, Edward Yang, Zachary DeVito, Martin Raison, Alykhan
  Tejani, Sasank Chilamkurthy, Benoit Steiner, Lu~Fang, Junjie Bai, and Soumith
  Chintala.
\newblock Pytorch: An imperative style, high-performance deep learning library.
\newblock {\em Advances in Neural Information Processing Systems 32}, pages
  8024--8035, 2019.

\bibitem{pilanci2020neural}
Mert Pilanci and Tolga Ergen.
\newblock Neural networks are convex regularizers: Exact polynomial-time convex
  optimization formulations for two-layer networks.
\newblock {\em Proceedings of the International Conference on Machine Learning
  (ICML 2020)}, 2020.

\bibitem{polik07slemma}
Imre Pólik and Tamás Terlaky.
\newblock A survey of the s-lemma.
\newblock {\em SIAM Review}, 49(3):371--418, 2007.

\bibitem{Ramachandran18swish}
Prajit Ramachandran, Barret Zoph, and Quoc Le.
\newblock Searching for activation functions.
\newblock {\em arXiv preprint arXiv:1710.05941}, 2018.

\bibitem{sahiner2020vector}
Arda Sahiner, Tolga Ergen, John Pauly, and Mert Pilanci.
\newblock Vector-output relu neural network problems are copositive programs:
  Convex analysis of two layer networks and polynomial-time algorithms.
\newblock {\em arXiv preprint arXiv:2012.13329}, 2020.

\bibitem{sahiner2020convex}
Arda Sahiner, Morteza Mardani, Batu Ozturkler, Mert Pilanci, and John Pauly.
\newblock Convex regularization behind neural reconstruction.
\newblock {\em arXiv preprint arXiv:2012.05169}, 2020.

\bibitem{soltani2019learning}
M.~{Soltani} and C.~{Hegde}.
\newblock Fast and provable algorithms for learning two-layer polynomial neural
  networks.
\newblock {\em IEEE Transactions on Signal Processing}, 67(13):3361--3371,
  2019.

\bibitem{soltan2019shallow_nn}
Mahdi Soltanolkotabi, Adel Javanmard, and Jason~D. Lee.
\newblock Theoretical insights into the optimization landscape of
  over-parameterized shallow neural networks.
\newblock {\em IEEE Trans. Inf. Theor.}, 65(2):742–769, February 2019.

\bibitem{trevisan2000gadgets}
Luca Trevisan, Gregory~B Sorkin, Madhu Sudan, and David~P Williamson.
\newblock Gadgets, approximation, and linear programming.
\newblock {\em SIAM Journal on Computing}, 29(6):2074--2097, 2000.

\bibitem{wolkowicz2012handbook}
Henry Wolkowicz, Romesh Saigal, and Lieven Vandenberghe.
\newblock {\em Handbook of semidefinite programming: theory, algorithms, and
  applications}, volume~27.
\newblock Springer Science \& Business Media, 2012.

\bibitem{fashionmnist_dataset}
Han Xiao, Kashif Rasul, and Roland Vollgraf.
\newblock Fashion-mnist: a novel image dataset for benchmarking machine
  learning algorithms, 2017.

\end{thebibliography}
\bibliographystyle{plain}

\newpage
\appendix
\section{Additional Discussion} \label{sec:discussion}


\subsection{Constrained Least Squares Form for the Squared Loss} \label{sec:constrained_ls_form}
Let us consider the polynomial activation scalar output case. In the case of squared loss $\ell(\hat{y}, y)=\|\hat{y} - y\|_2^2$, the convex program takes the following form:
\begin{align}
    \min_{Z=Z^T, Z^\prime={Z^\prime}^T} &\sum_{i=1}^n \left( ax_i^T(Z_1-Z_1^\prime)x_i + bx_i^T(Z_2-Z_2^\prime) + c(Z_4 - Z_4^\prime) - y_i \right)^2 + \beta( Z_4+Z_4^\prime) \nonumber \\
    \mbox{s.t.}\quad &\tr(Z_1) = Z_4, \, \tr(Z_1^\prime) = Z_4^\prime \nonumber\\
    &Z \succeq 0, \, Z^\prime \succeq 0 \,.
\end{align}
Noting that $ax_i^T(Z_1-Z_1^\prime)x_i = \vect(x_ix_i^T)^T \vect(Z_1-Z_1^\prime)$, we can write the squared loss term as
\begin{align*}
    &\sum_{i=1}^n \left( \begin{bmatrix} a\vect(x_ix_i^T)^T & bx_i^T & c \end{bmatrix} \begin{bmatrix}\vect(Z_1-Z_1^\prime) \\ Z_2-Z_2^\prime \\ Z_4 - Z_4^\prime \end{bmatrix} - y_i \right)^2 = \\
    &= \Big\| \begin{bmatrix} a\vect(x_1x_1^T)^T & bx_1^T & c \\ \vdots \\ a\vect(x_nx_n^T)^T & bx_n^T & c \end{bmatrix} \begin{bmatrix}\vect(Z_1-Z_1^\prime) \\ Z_2-Z_2^\prime \\ Z_4 - Z_4^\prime \end{bmatrix} - y \Big\|_2^2 = \| X_V z - y \|_2^2
\end{align*}
where we have defined $X_V \in \mathbb{R}^{n\times (d^2+d+1)}$ and $z \in \mathbb{R}^{(d^2+d+1)}$.
The squared loss term is equal to $z^TX_V^TX_Vz - 2y^TX_Vz + \|y\|_2^2$.

If we pre-compute $X_V^TX_V \in \mathbb{R}^{(d^2+d+1)\times (d^2+d+1)}$ and $X_V^Ty \in \mathbb{R}^{(d^2+d+1)}$, then the objective no longer has dependence on the number of samples $n$. We note that the pre-computation of $X_V^TX_V$ and $X_V^Ty$ is useful when one is performing hyperparameter tuning for the regularization coefficient $\beta$.

\section{Proofs} \label{sec:appendix_proofs}

\begin{proof}[Proof of Lemma \ref{lem:spectra_LMI}]
We will denote the set in \eqref{eq:neural_spectrahedron_1} as $\mathcal{S}_1$ and the set in \eqref{eq:neural_spectrahedron_2} as $\mathcal{S}_2$ to simplify the notation. We will prove $\mathcal{S}_1 = \mathcal{S}_2$ by showing $\mathcal{S}_1 \subseteq \mathcal{S}_2$ and $\mathcal{S}_2 \subseteq \mathcal{S}_1$.

\noindent We first show $\mathcal{S}_1 \subseteq \mathcal{S}_2$. Let us take a point $S \in \mathcal{S}_1$. This implies that $S$ is a matrix of the form
\begin{align}
    t\sum_{j=1}^m \append{u_j}\append{u_j}^T\alpha_j = t \sum_{j=1}^m \begin{bmatrix} u_ju_j^T \alpha_j & u_j \alpha_j \\ u_j^T\alpha_j &  \alpha_j \end{bmatrix} = \begin{bmatrix} t\sum_{j=1}^m u_ju_j^T \alpha_j & t\sum_{j=1}^m u_j \alpha_j \\ t\sum_{j=1}^m u_j^T\alpha_j & t\sum_{j=1}^m \alpha_j \end{bmatrix} 
\end{align}
with $\sum_j \alpha_j \leq 1$ and $\|u_j\|_2=1$ for all $j$. We note that $\tr(t\sum_{j=1}^m u_ju_j^T \alpha_j)=t\sum_{j=1}^m \tr(u_ju_j^T\alpha_j)=t\sum_{j=1}^m \tr(u_j^T u_j)\alpha_j = t\sum_{j=1}^m \alpha_j \leq t$. This shows that $S$ satisfies the equality condition in the definition \eqref{eq:neural_spectrahedron_2}. Now, we show that $S$ is a PSD matrix. Note that each of the rank-1 matrices $\append{u_j}\append{u_j}^T$ is a PSD matrix and since the coefficients $\alpha_j$'s and $t$ are nonnegative, it follows that $S$ is PSD. This proves that $S \in \mathcal{S}_2$.

\noindent We next show $\mathcal{S}_2 \subseteq \mathcal{S}_1$. Let us take a point $S \in \mathcal{S}_2$. This implies that $S$ is PSD and $\tr(S_1)=S_4 =t_0 \leq t$. We show in Section \ref{sec:neural_decomp} that it is possible to decompose $S$ via the neural decomposition procedure to obtain the expressions given in \eqref{eq:Z_star_decomp}. It follows that we can write $S$ in the following form
\begin{align}
    S = \frac{t_0}{\sum_{j=1}^m d_j^2} \begin{bmatrix} \sum_{j=1}^m u_ju_j^Td_j^2 & \sum_{j=1}^m u_jd_j^2 \\ \sum_{j=1}^m u_j^Td_j^2 & \sum_{j=1}^m d_j^2 \end{bmatrix} \,,
\end{align}
where the scaling factor $\frac{t_0}{\sum_{j=1}^m d_j^2}$ is to ensure that $\tr(S_1)=S_4 =t_0\leq t$. It is obvious to see that $S$ is in $\mathcal{S}_1$ when $t_0=t$ by the definition of $\mathcal{S}_1$ given in \eqref{eq:neural_spectrahedron_1}. When $t_0 < t$, we still have that $S$ is in $\mathcal{S}_1$ which can be seen by noting that $\mathcal{S}_1$ is defined as the convex hull of rank-1 matrices and the zero matrix. We can scale all the rank-1 matrices in the convex combination with $\frac{t_0}{t}$ and change the weight of the zero matrix accordingly.

\end{proof}


\begin{proof}[Proof of Lemma \ref{lem:lpmin}]
Let $\alpha_1,\dots,\alpha_m$ be any feasible point. 
First, note that for any $s\ge 0$ and $\alpha\in \mathbb{R},\,\alpha\neq 0$, we have 
\begin{align}
    \big( s\,  |\alpha| \big)^{p} \ge s\, |\alpha|^p\,,
\end{align}
where equality holds if and only if $s\in \{0,1\}$. The equality condition follows since $|\alpha|>0$ and $s^p=s$ implies $s\in\{0,1\}$ for $p\in(0,1)$. Then, define $s_i:=\frac{|\alpha_i|}{\sum_{j}|\alpha_j|}$, which satisfies $\sum_i s_i =1 $, and observe that
\begin{align*}
    \sum_i |\alpha_i|^p &= \sum_i \big|s_i \big(\sum_j |\alpha_j| \big)  \big|^p\\
    &\ge \Big(\sum_i s_i \Big) \Big(\sum_j |\alpha_j|\Big)^p\\
    &= \Big(\sum_i |\alpha_i|\Big)^p\\
    &\ge  \Big(\sum_i \alpha_i\Big)^p\\
    &=1\,,
\end{align*}
where the first inequality holds with equality if and only if $s_i\in\{0,1\},\,\forall i$. Hence, in order for the equality to hold, we necessarily have $\|\alpha\|_0\le1$. Since $\sum_i \alpha_i=1$, the all-zeros vector is infeasible. This implies that $\|\alpha\|_0=1$. Finally, note that all feasible vectors which are 1-sparse are of the form $(1,0,\dots,0)$, $(0,1,0,\dots,0),\dots,(0,\dots,1)$ and achieve an objective value $1$. We conclude that all feasible vectors with cardinality strictly greater than $1$ are suboptimal since they achieve objective value strictly larger than 1.
%
%

%
\end{proof}

\begin{proof}[Proof of Lemma \ref{lem:decision_pr_nphard}]
Let us define the set $\mathcal{A} = \{a_1,a_2,\dots,a_d\}$ where $a_i$ are integers. We need to show that the problem \eqref{eq:phase_retrieval_pr} finds a feasible solution $u_1$ if and only if there exists a subset $\mathcal{A}_S$ of the set $\mathcal{A}$ that satisfies $\sum_{a \in \mathcal{A}_S} a = z$. 

We assume $n=2d+1$ and hence $\tilde{X}$ is $(d+1)\times d$ and $\tilde{y}$ is $(d+1)$ dimensional. Let $\tilde{X}_D \in \mathbb{R}^{d \times d}$ denote the matrix with the first $d$ rows of $\tilde{X}$, and $\tilde{x}_{d+1}$ is the last sample in $\tilde{X}$.
Let us define $\tilde{y}_i$ as
\begin{align} \label{eq:setting_ytilde}
    \tilde{y}_i = \begin{cases} 
      (a_i / w_i)^2, & i = 1,\dots,d \\
      (2z - \sum_{j=1}^d a_j)^2, & i = d+1 \,,
   \end{cases}
\end{align}
where $w = \tilde{X}_D^{-T} \tilde{x}_{d+1} \in \mathbb{R}^d$.

\textbf{Direction 1:} Suppose there exists $u_1 \in \mathbb{R}^d$ such that $(\tilde{x}_i^Tu_1)^2=\tilde{y}_i$ for every $i=1,\dots,d+1$ and $u_{1k}^2=1/d$ for every $k=1,\dots,d$. Then there exists a subset $\mathcal{A}_S$ with $\sum_{a \in \mathcal{A}_S}a = z$.

\textbf{Proof of direction 1:} 
Assuming $\tilde{X}_D$ is invertible, it follows that $\tilde{X} \tilde{X}_D^{-1} = \begin{bmatrix}I_d & w \end{bmatrix}^T$ where $I_d$ is the $d\times d$ identity matrix.
Let us consider a feasible $u_1$. Then, $v=\tilde{X}_Du_1$ satisfies $((\tilde{X}_D^{-T} \tilde{x}_i)^T v)^2 = \tilde{y}_i$ for $i=1,\dots,d+1$. Consequently, we have $\tilde{X}_D^{-T} \tilde{x}_i = e_i$ for $i = 1,\dots, d$, and $\tilde{X}_D^{-T} \tilde{x}_{d+1} = w$. As a result, we obtain the following relation between $v$ and $\tilde{y}$:
\begin{align*}
    \tilde{y}_i = \begin{cases} v_i^2, & i=1,\dots,d \\ 
    (w^Tv)^2, & i=d+1 \,.
    \end{cases} 
\end{align*}
Next, because of \eqref{eq:setting_ytilde}, we have
\begin{align*}
    |v_i| = \left| \frac{a_i}{w_i}\right|, \, i=1,\dots,d, \quad \mbox{and} \quad |w^Tv| = \left|2z-\sum_{j=1}^d a_j\right| \,.
\end{align*}
Let us define $\varepsilon_i$ such that $v_i = \varepsilon_i a_i / w_i$ for $i=1,\dots,d$. Note that $\varepsilon_i \in \{-1,1\}$. Then,
\begin{align*}
    \left|2z-\sum_{j=1}^d a_j\right| &= |w^Tv| = \left| \sum_{i=1}^d \varepsilon_i a_i \right| = \left| \sum_{i:\varepsilon_i=1} a_i - \sum_{i:\varepsilon_i=-1} a_i\right| = \left| \sum_{i:\varepsilon_i=1} a_i - \sum_{i=1}^d a_i + \sum_{i: \varepsilon_i=1} a_i \right| \\
    &= \left| 2\sum_{i:\varepsilon_i=1} a_i - \sum_{i=1}^d a_i \right| \,.
\end{align*}
This means we either have $z = \sum_{i:\varepsilon_i=1} a_i$ or $z = -\sum_{i:\varepsilon_i=1} a_i + \sum_{i=1}^d a_i = \sum_{i:\varepsilon_i=-1} a_i$. This shows that the sum of the elements of $\mathcal{A}_S$ is equal to $z$ when $\mathcal{A}_S$ is either equal to $\{a_i | \varepsilon_i=1 \}$ or $\{a_i | \varepsilon_i=-1 \}$.

In proving direction 2, it is straightforward to show the existence of $u_1$ that satisfies the constraint $(\tilde{x}_i^T u_1)^2 = \tilde{y}_i$. To show that there is a $u_1$ that satisfies the constraint $u_{1k}^2 = \frac{1}{d}$, we pick $\tilde{X}$ in a certain way that we discuss now: To prove direction 2, we will need to make sure $|\tilde{X}_D^{-1}v| = \ones\frac{1}{\sqrt{d}}$ is satisfied, i.e.,  
\begin{align*}
    \left| \sum_{j=1}^d (\tilde{X}_{D}^{-1})_{i,j} \varepsilon_j \frac{a_j}{w_j} \right| = \frac{1}{\sqrt{d}} \quad \mbox{for} \quad i=1,\dots,d \,.
\end{align*}


We pick $\tilde X_D$ to be any diagonal matrix with arbitrary $-1$'s and $+1$'s on the diagonal and pick $\tilde{x}_{d+1}=\sqrt{d} \begin{bmatrix} a_1 & \dots & a_d \end{bmatrix}^T$. Since $w = \tilde{X}_D^{-T} \tilde{x}_{d+1}$, we will have $|w_i|=|\tilde{x}_{d+1,i}|=\sqrt{d}|a_i|$ for $i=1,\dots,d$. This choice for $\tilde X_D$ and $\tilde{x}_{d+1}$ ensures that $|\tilde{X}_D^{-1}v| = \ones\frac{1}{\sqrt{d}}$.

\textbf{Direction 2:} Suppose there is a subset $\mathcal{A}_S$ with $\sum_{a \in \mathcal{A}_S}a = z$. Then there exists a feasible $u_1 \in \mathbb{R}^d$.

\textbf{Proof of direction 2:} Define $\varepsilon_i$ such that for $a_i$ in $\mathcal{A}_S$, it is equal to $1$, and otherwise it is equal to $-1$. Next, 
\begin{align} \label{eq:direction_2_sum}
    \left| \sum_{i=1}^d \varepsilon_i a_i \right| &= \left| \sum_{i:\varepsilon_i=1} \varepsilon_ia_i + \sum_{i:\varepsilon_i = -1} \varepsilon_i a_i \right| = \left| \sum_{i:\varepsilon_i=1} a_i - \sum_{i:\varepsilon_i = -1} a_i \right| = \left| 2\sum_{i:\varepsilon_i=1} a_i - \sum_{i=1}^d a_i \right| \nonumber \\
    &= \left| 2z - \sum_{i=1}^d a_i \right| \,.
\end{align}
Let us take $v_i = \varepsilon_i a_i / w_i$ for $i=1,\dots,d$. Now we show that the point defined by $\tilde{X}_D^{-1}v$ is a feasible point. First, we check if $\tilde{X}_D^{-1}v$ satisfies the constraints $(\tilde{x}_i^T \tilde{X}_D^{-1}v)^2 = \tilde{y}_i$ for $i=1,\dots,i+1$. Note that
\begin{align*}
    (\tilde{x}_i^T \tilde{X}_D^{-1}v)^2 &= (e_i^Tv)^2 = v_i^2 = \frac{a_i^2}{w_i^2} = \tilde{y}_i \quad \mbox{for} \quad i=1,\dots,d, \quad \mbox{and} \\
    (\tilde{x}_{d+1}^T  \tilde{X}_D^{-1}v)^2 &= (w^Tv)^2 = (\sum_{j=1}^d \varepsilon_j a_j)^2 = (2z-\sum_{j=1}^d a_j)^2 = \tilde{y}_{i+1} \,,
\end{align*}
where the last two equalities follow from \eqref{eq:direction_2_sum} and the definition in \eqref{eq:setting_ytilde}. This shows that the constraints $(\tilde{x}_i^T \tilde{X}_D^{-1}v)^2 = \tilde{y}_i$ for $i=1,\dots,d+1$ are satisfied by $\tilde{X}_D^{-1}v$.

We now check for the other constraint; i.e. does $\tilde{X}_D^{-1}v$ satisfy $|\tilde{X}_D^{-1}v| = \frac{1}{\sqrt{d}} \ones$ where the absolute value is elementwise? This is true because $|(\tilde{X}_D^{-1}v)_i| = |\sum_{j=1}^d (\tilde{X}_{D}^{-1})_{i,j} \varepsilon_j \frac{a_j}{w_j}| = \frac{1}{\sqrt{d}}$ for $i=1,\dots,d$. The second equality follows from how we picked $\tilde{X}_D$ and $\tilde{x}_{d+1}$.
\end{proof}

\begin{proof} [Proof of Corollary \ref{l:S_lemma_w_equality}]
Let us define the quadratic functions $f_1(u) = -u^TQu - b^Tu + \beta$ and $f_2(u) = \|u\|_2^2 - 1$. 
We note that $f_2(u)$ is strictly convex and takes both negative and positive values.
Then by Lemma \ref{l:S_lemma_w_equality_orig}, we have that the system $-u^TQu - b^Tu < -\beta$ (or $u^TQu + b^Tu > \beta$) and $\|u\|_2 = 1$ is not solvable if and only if there exists $\lambda$ such that $-u^TQu - b^Tu + \beta + \lambda (\|u\|_2^2-1) \geq 0$, $\forall u$.

Equivalently, we have $\max_{\|u\|_2 = 1} u^TQu + b^Tu \leq \beta$ if and only if there exists $\lambda$ such that
\begin{align} \label{eq:quad_ineq}
    u^T(\lambda I - Q) u - b^Tu + \beta - \lambda \geq 0, \quad \forall u.
\end{align}
We note that if we make the change of variable $u \leftarrow \frac{u}{c}$ with $c \neq 0$, then \eqref{eq:quad_ineq} implies
\begin{align*}
    \frac{1}{c^2} u^T(\lambda I - Q) u - \frac{1}{c}b^Tu + \beta - \lambda \geq 0, \quad \forall u, \forall c\neq 0
\end{align*}
which is the same as
\begin{align*}
    u^T(\lambda I - Q) u - c b^Tu + c^2 (\beta - \lambda) \geq 0, \quad \forall u, \forall c\neq 0.
\end{align*}
We express this inequality in matrix form as follows
\begin{align} \label{eq:matrix_form_psd}
    \begin{bmatrix} u^T & c \end{bmatrix}\begin{bmatrix}
        \lambda I - Q & -\frac{1}{2}b \\ -\frac{1}{2}b^T & \beta - \lambda
    \end{bmatrix} \begin{bmatrix} u \\ c \end{bmatrix} \geq 0, \quad \forall u, \forall c\neq 0.
\end{align}
For the matrix in \eqref{eq:matrix_form_psd} to be PSD, we first need to show that \eqref{eq:quad_ineq} implies the inequality in \eqref{eq:matrix_form_psd} for $c=0$ as well. We note that \eqref{eq:quad_ineq} implies
\begin{align*}
     \frac{u^T}{\|u\|_2}(\lambda I - Q) \frac{u}{\|u\|_2} - b^T \frac{u}{\|u\|_2^2} + \frac{\beta - \lambda}{\|u\|_2^2} \geq 0, \quad \forall u \mbox{ s.t. } \|u\|_2 \neq 0.
\end{align*}
Next, taking the norm of $u$ to infinity, we have
\begin{align*}
    \lim_{\|u\|_2 \rightarrow \infty} \left( \frac{u^T}{\|u\|_2}(\lambda I - Q) \frac{u}{\|u\|_2} - b^T \frac{u}{\|u\|_2^2} + \frac{\beta - \lambda}{\|u\|_2^2} \right) = u_{n}^T (\lambda I - Q) u_{n},
\end{align*}
where $u_{n} = u / \|u\|_2$ is unit norm. We note that $u_{n}^T (\lambda I - Q) u_{n}$ is non-negative for all unit norm $u_n$, which is the same as the statement that it is non-negative for all $u_n$ (not necessarily unit norm). This shows that \eqref{eq:quad_ineq} implies $u^T (\lambda I - Q) u \geq 0$ for all $u$, which, we note, is the same as \eqref{eq:matrix_form_psd} with $c=0$.
Hence, because the inequality holds for all $\begin{bmatrix}u^T & c \end{bmatrix}^T$, we obtain the matrix inequality 
\begin{align} \label{eq:s_proc_psd}
    \begin{bmatrix}
        \lambda I - Q & -\frac{1}{2}b \\ -\frac{1}{2}b^T & \beta - \lambda
    \end{bmatrix} \succeq 0.
\end{align}

The proof for the other direction of the if and only if statement is straightforward. We note that, by the definition of a PSD matrix, \eqref{eq:s_proc_psd} implies that $u^T(\lambda I - Q) u - c b^Tu + c^2 (\beta - \lambda) \geq 0, \, \forall u, c$. Setting $c=0$, we obtain the inequality in \eqref{eq:quad_ineq}.
\end{proof}

\section{Additional Numerical Results} \label{sec:additional_num_results}
Figure \ref{fig:sgd_plots} compares the costs and accuracy performance of the convex formulation with minibatch SGD. 

\begin{figure} 
\begin{minipage}[b]{0.24\linewidth}
  \centering
  \centerline{\includegraphics[width=\columnwidth]{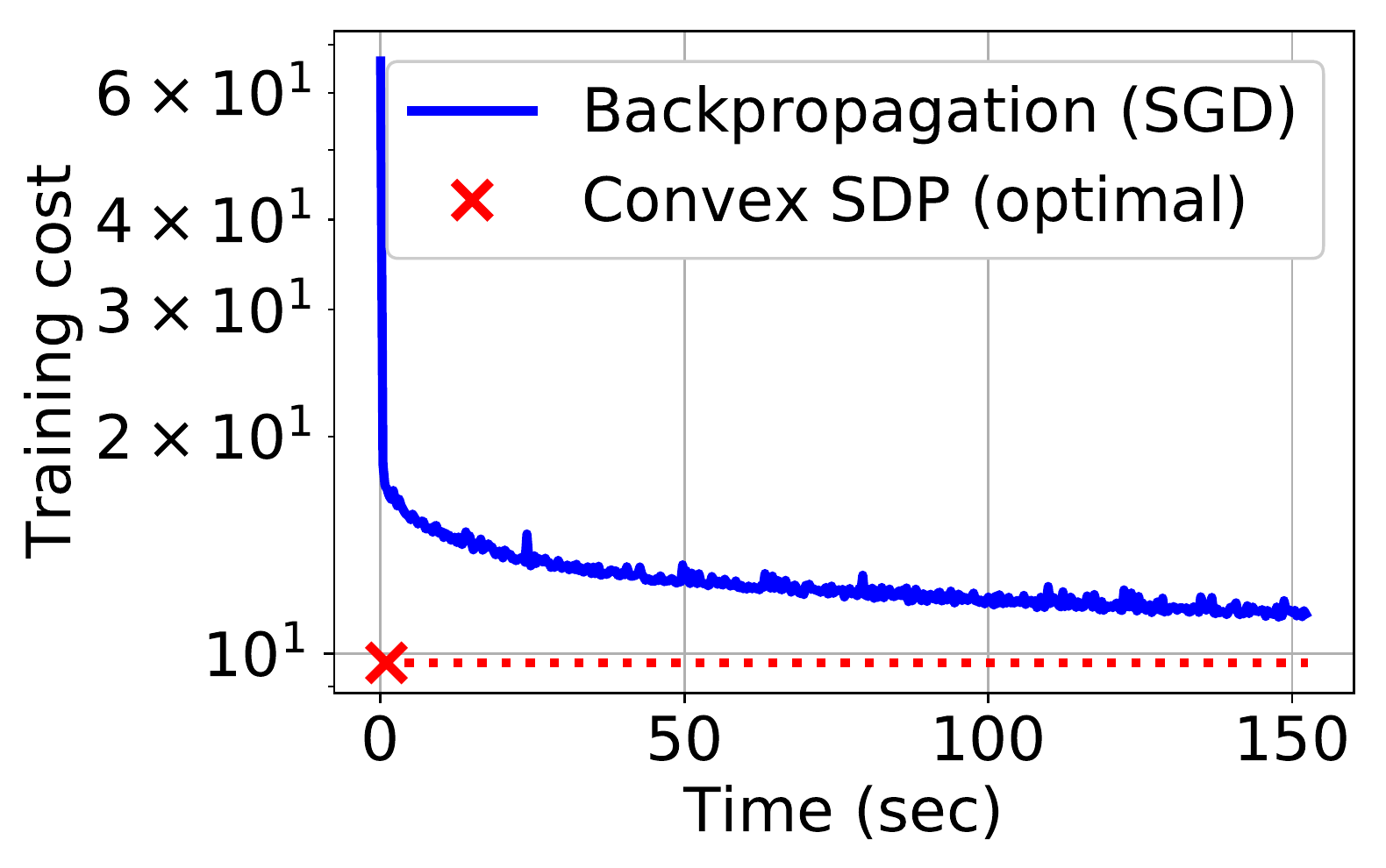}}
  \centerline{(a) DS1, training cost}
\end{minipage}
\hfill
\begin{minipage}[b]{0.24\linewidth}
  \centering
  \centerline{\includegraphics[width=\columnwidth]{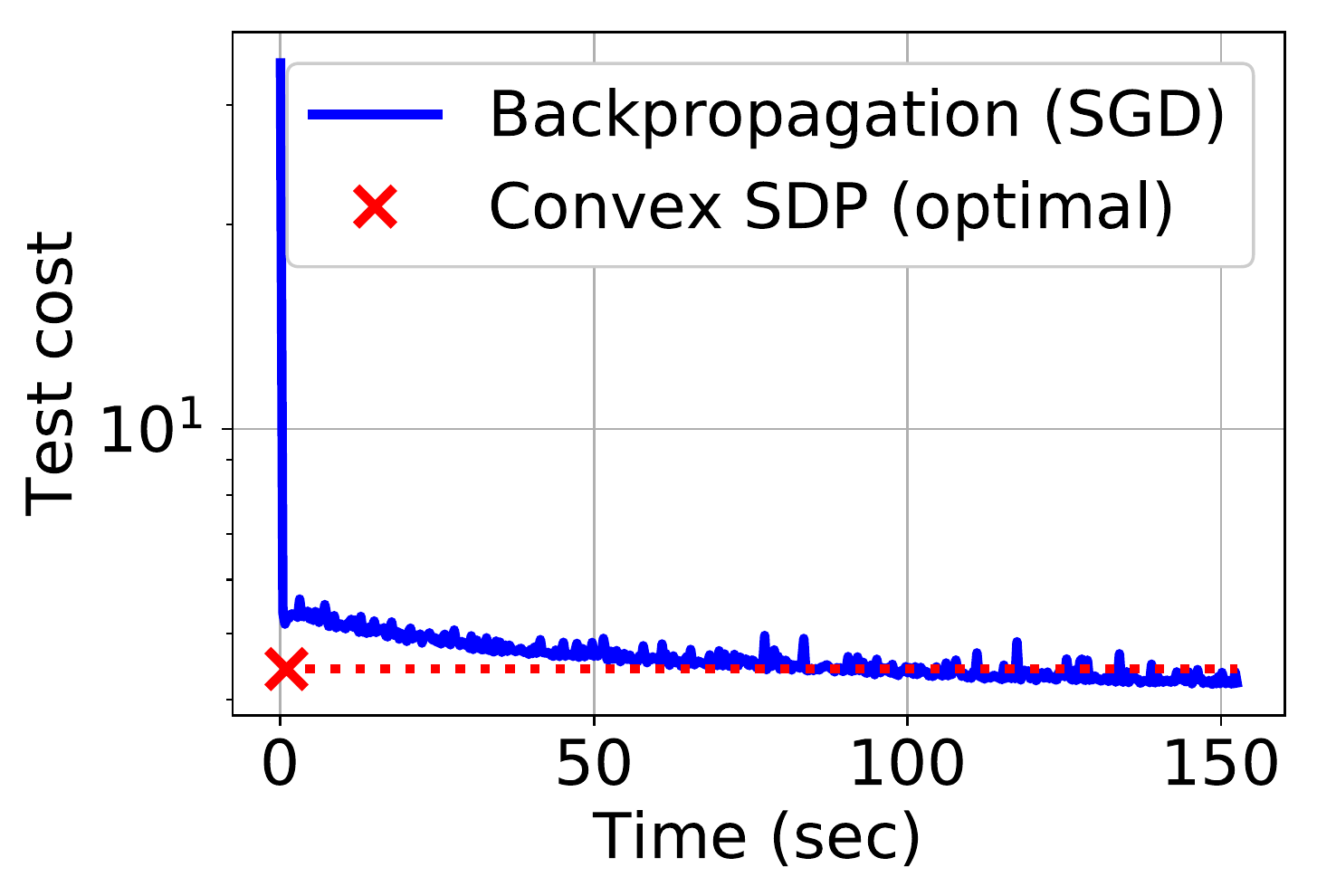}}
  \centerline{(b) DS1, test cost}
\end{minipage}
\hfill
\begin{minipage}[b]{0.24\linewidth}
  \centering
  \centerline{\includegraphics[width=\columnwidth]{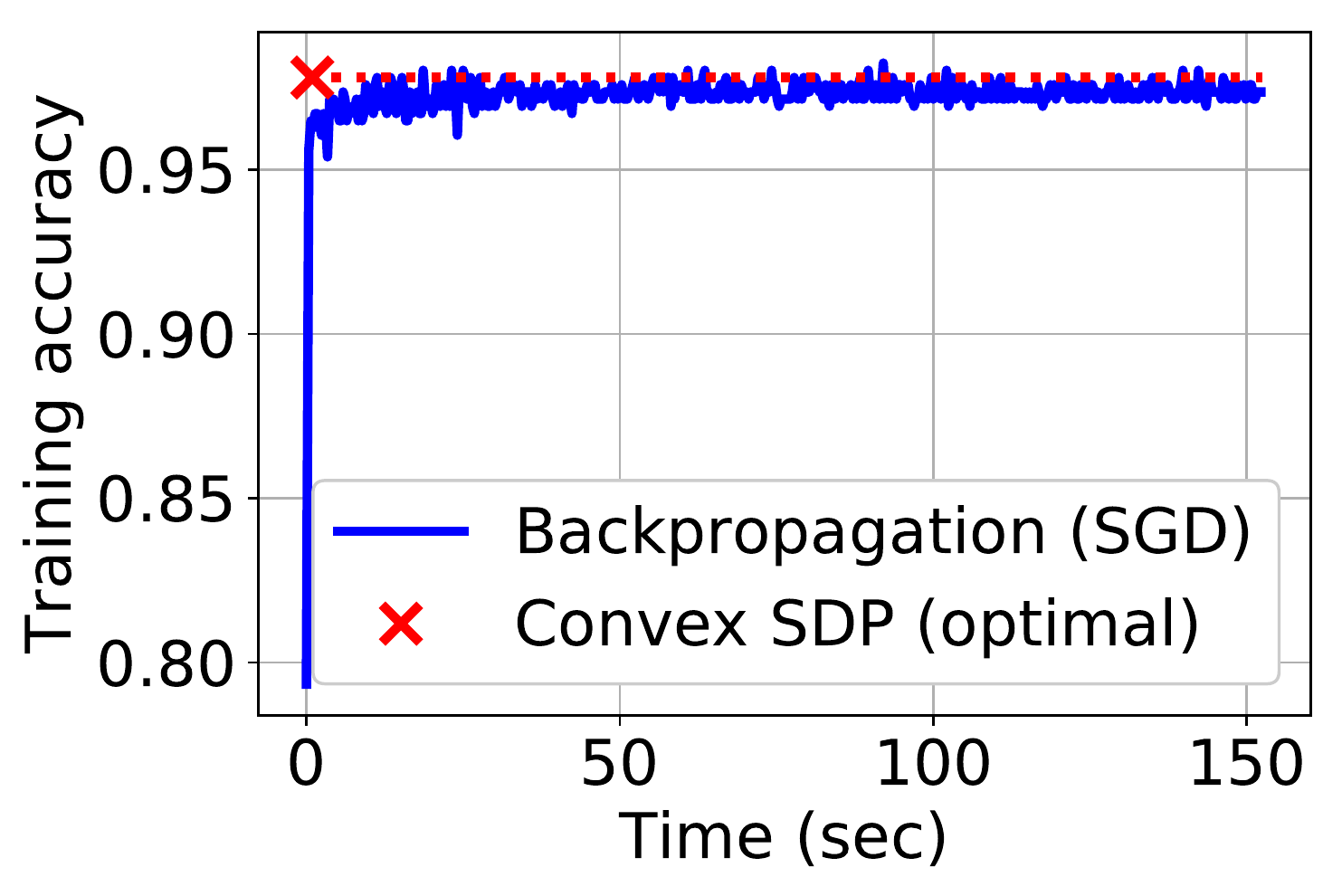}}
  \centerline{(c) DS1, training accuracy}
\end{minipage}
\hfill
\begin{minipage}[b]{0.24\linewidth}
  \centering
  \centerline{\includegraphics[width=\columnwidth]{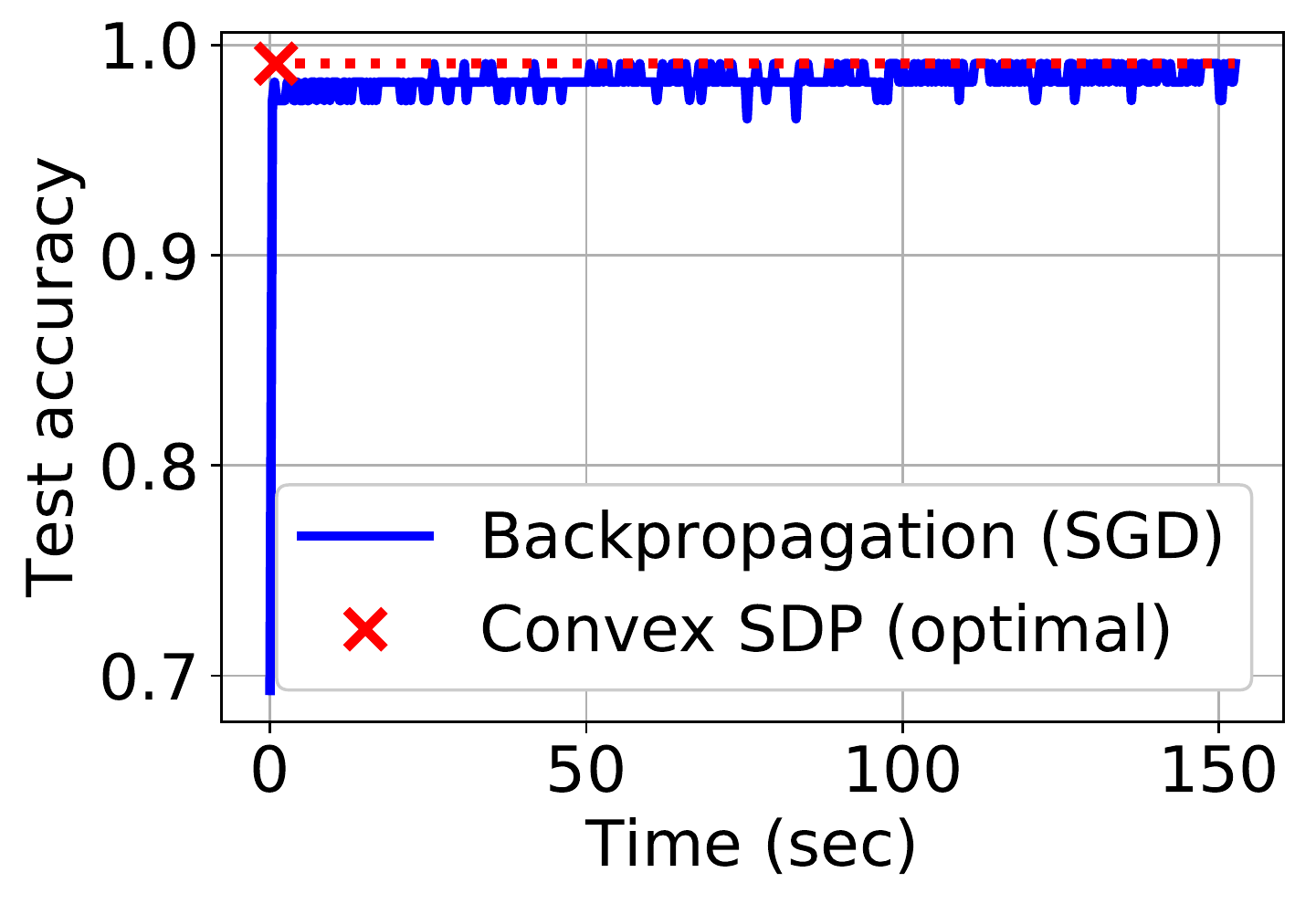}}
  \centerline{(d) DS1, test accuracy}
\end{minipage}

\begin{minipage}[b]{0.24\linewidth}
  \centering
  \centerline{\includegraphics[width=\columnwidth]{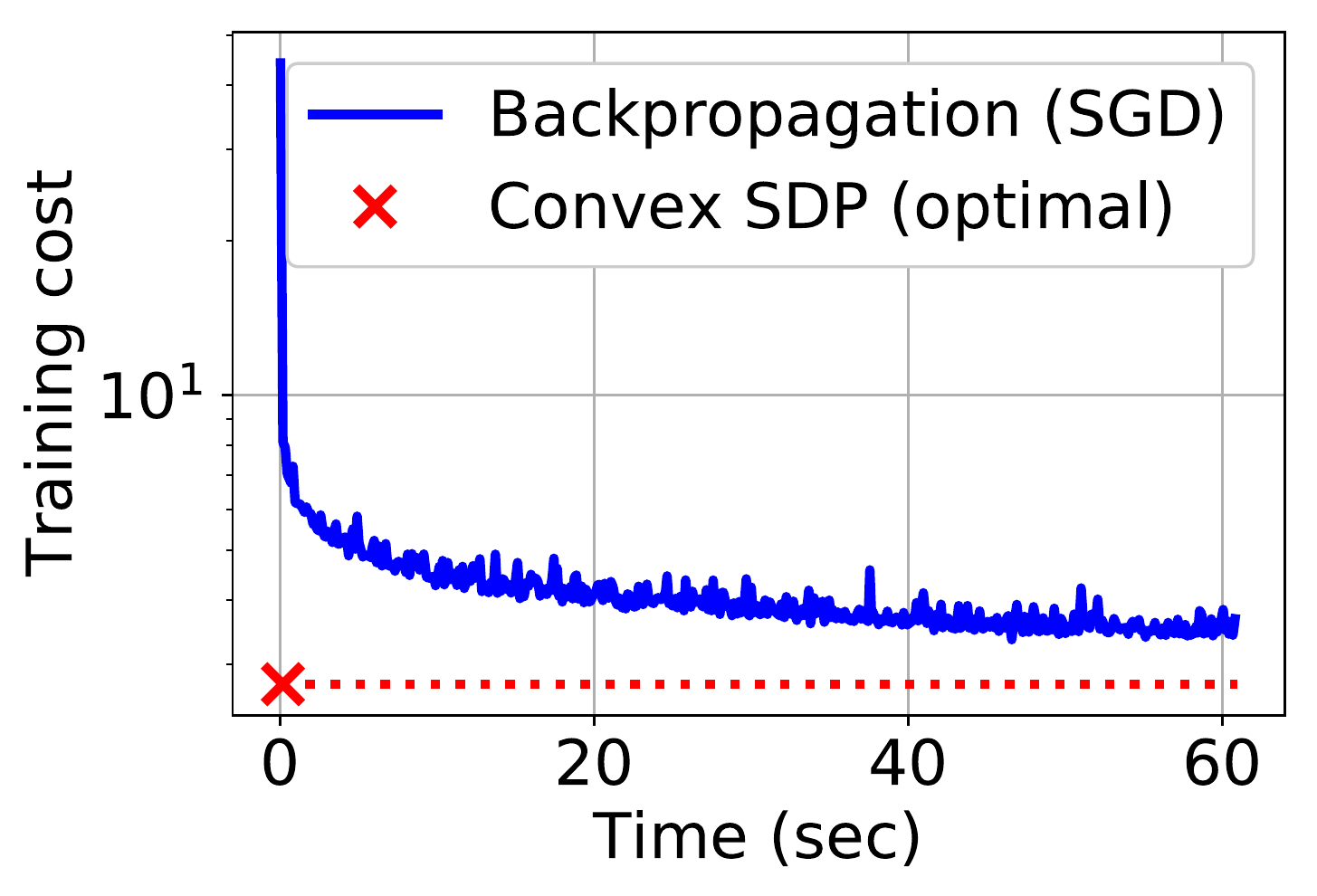}}
  \centerline{(e) DS2, training cost}
\end{minipage}
\hfill
\begin{minipage}[b]{0.24\linewidth}
  \centering
  \centerline{\includegraphics[width=\columnwidth]{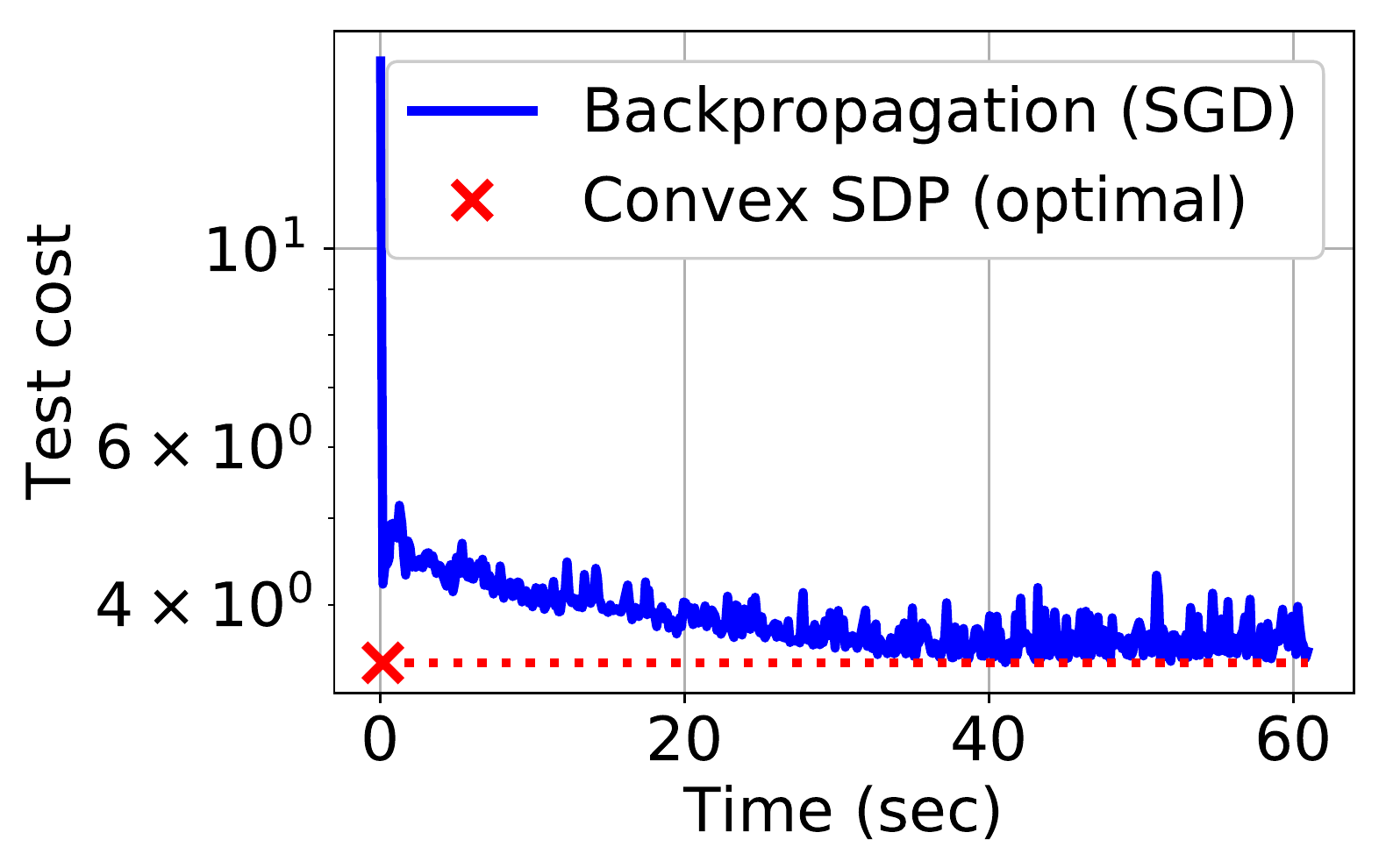}}
  \centerline{(f) DS2, test cost}
\end{minipage}
\hfill
\begin{minipage}[b]{0.24\linewidth}
  \centering
  \centerline{\includegraphics[width=\columnwidth]{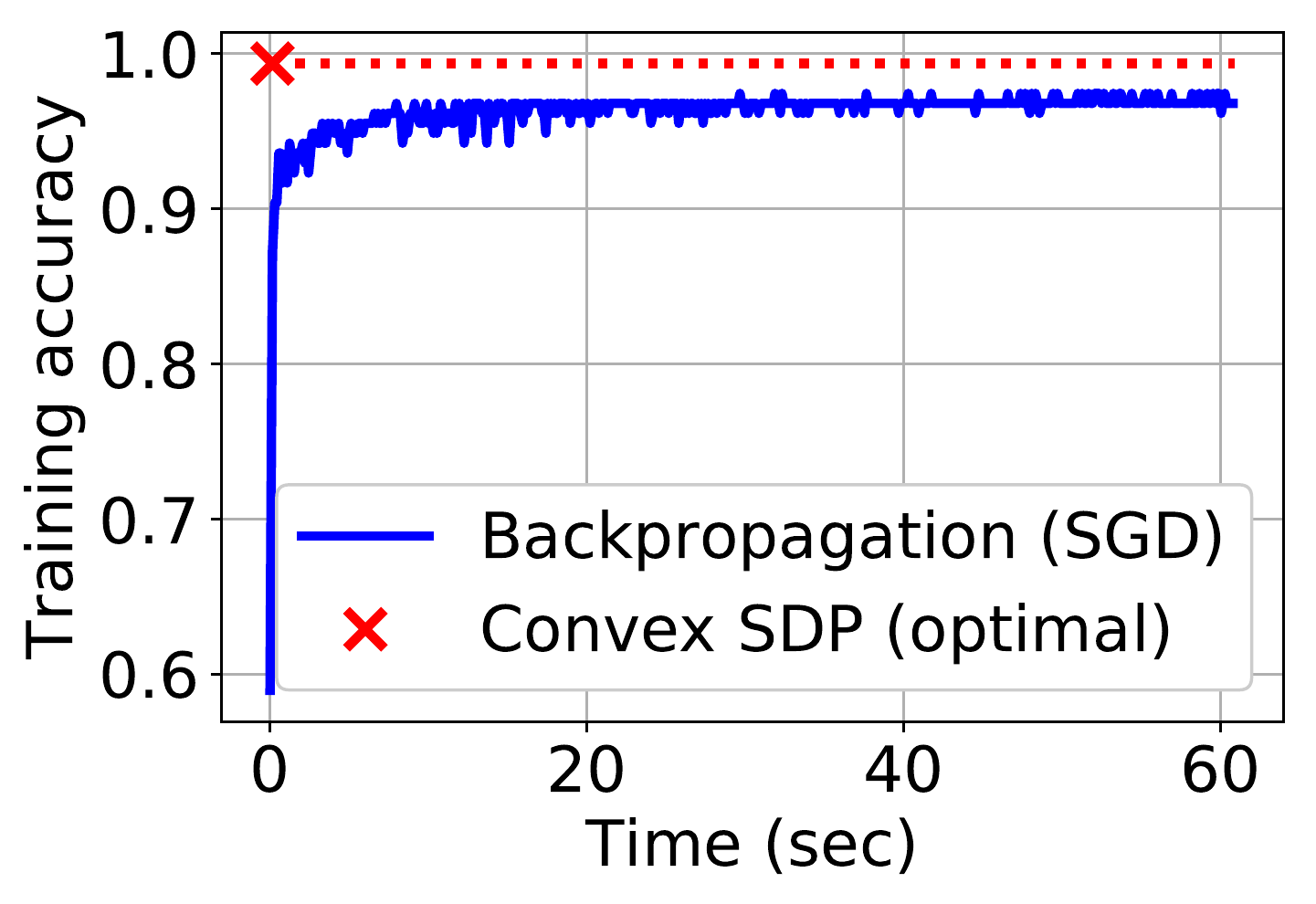}}
  \centerline{(g) DS2, training accuracy}
\end{minipage}
\hfill
\begin{minipage}[b]{0.24\linewidth}
  \centering
  \centerline{\includegraphics[width=\columnwidth]{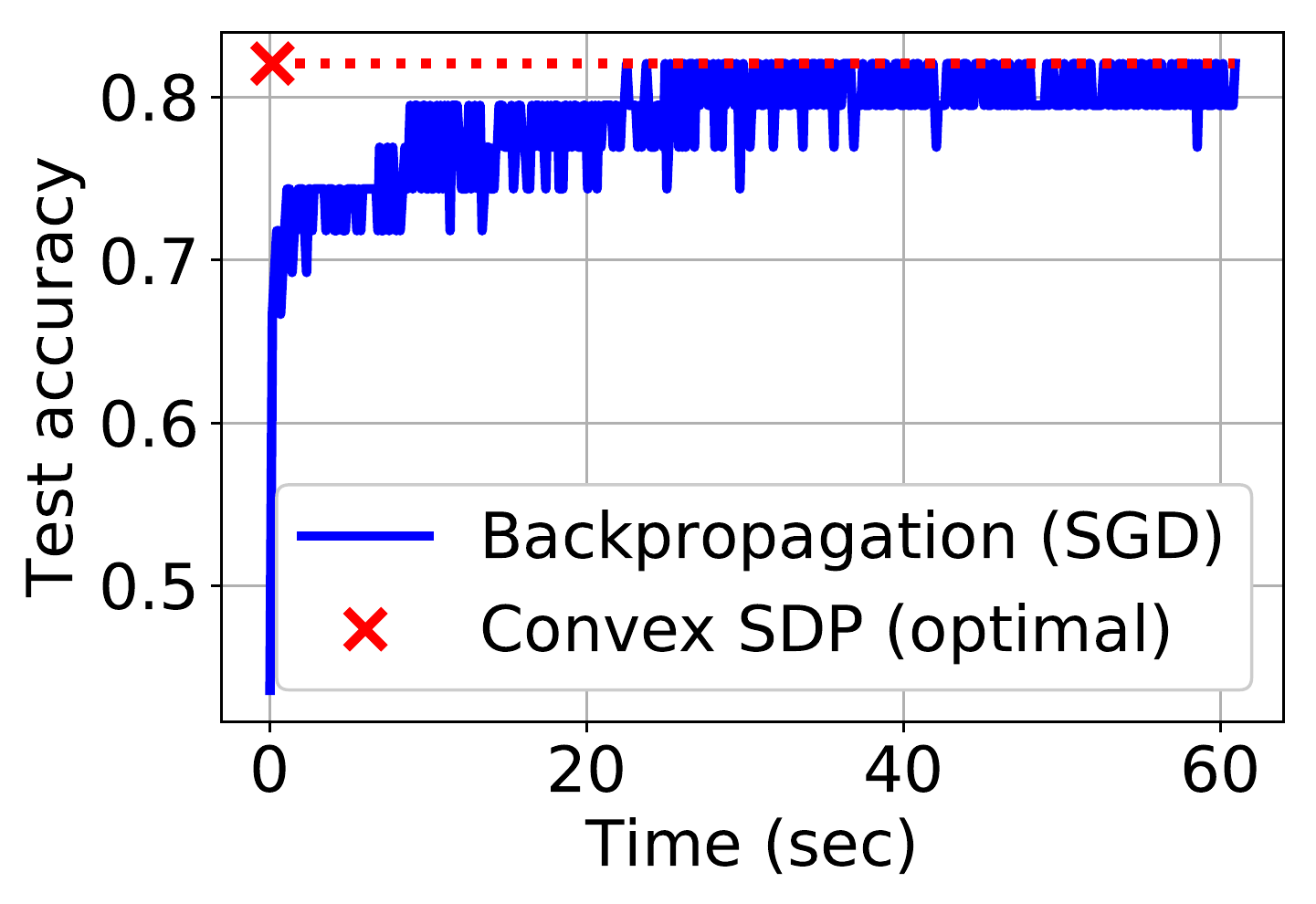}}
  \centerline{(h) DS2, test accuracy}
\end{minipage}
\caption{SGD with minibatch size 13 for two UCI datasets, DS1 is the breast-cancer-wisc-diag dataset with $n=455,d=30$ and DS2 is parkinsons dataset with $d=156,d=22$. The regularization coefficient is set to $\beta=1$ and $\beta=0.1$ and the number of neurons $m^*$ is found as $34$ and $27$ for DS1 and DS2, respectively.}
\label{fig:sgd_plots}
\end{figure}



\end{document}